%% file: main.tex
\documentclass[Royal,sageh,times]{sagej}

\usepackage{moreverb,url}

\usepackage[colorlinks,bookmarksopen,bookmarksnumbered,citecolor=red,urlcolor=red]{hyperref}

\newcommand\BibTeX{{\rmfamily B\kern-.05em \textsc{i\kern-.025em b}\kern-.08em
T\kern-.1667em\lower.7ex\hbox{E}\kern-.125emX}}

\def\volumeyear{2026}

\usepackage{algorithm}
\usepackage[noend]{algorithmic}
\usepackage{amsfonts}
\usepackage{amsmath}
\usepackage{amssymb}
\usepackage{booktabs}
\usepackage{calc}
\usepackage[capitalize]{cleveref}
\usepackage{enumitem}
\usepackage{graphicx}
\usepackage{microtype}
\usepackage{multirow}
\usepackage{natbib}
\usepackage{subcaption}
\usepackage{xcolor}

\usepackage{macros}

\begin{document}

\runninghead{A Mathematical Framework for NeSy}

\title{A Mathematical Framework \\ 
and a Suite of Learning Techniques \\ 
for Neural-Symbolic Systems}

\author{Charles Dickens\affilnum{1}, Connor Pryor\affilnum{1}, Changyu Gao\affilnum{2}, Alon Albalak\affilnum{3}, Eriq Augustine\affilnum{1}, William Wang\affilnum{3}, Stephen Wright\affilnum{2}, and Lise Getoor\affilnum{1}}

\affiliation{\affilnum{1}University of California Santa Cruz\\
\affilnum{2}University of Wisconsin Madison\\
\affilnum{3}University of California Santa Barbara}

\corrauth{Charles Dickens, Department of Computer Science and Engineering
University of California,
Santa Cruz, 
CA 95060, 
USA}

\email{cadicken@ucsc.edu}

\input{sections/abstract/abstract}


\maketitle

\input{sections/introduction/introduction}

\input{sections/related-work/introduction}
\input{sections/related-work/nesy-frameworks}
\input{sections/related-work/applications}

\input{sections/neural-symbolic-energy-based-models/introduction}
\input{sections/neural-symbolic-energy-based-models/definition}
\input{sections/neural-symbolic-energy-based-models/modeling-patterns}

\input{sections/neupsl/introduction}
\input{sections/neupsl/neural-probabilistic-soft-logic}
\input{sections/neupsl/deep-hinge-loss-markov-random-fields}
\input{sections/neupsl/smooth-formulation}

\input{sections/a-suite-of-learning-techniques-for-nesy/introduction}
\input{sections/a-suite-of-learning-techniques-for-nesy/nesy-ebm-learning}
\input{sections/a-suite-of-learning-techniques-for-nesy/learning-losses}
\input{sections/a-suite-of-learning-techniques-for-nesy/learning-algorithms}

\input{sections/experiments/introduction}
\input{sections/experiments/datasets}
\input{sections/experiments/constraint-satisfaction}
\input{sections/experiments/learning}
\input{sections/experiments/semi-supervision}

\input{sections/limitations/limitations}

\input{sections/conclusion/conclusion}

\input{sections/acknowledgements/acknowledgements}

\newpage

\appendix
\input{appendix/introduction}
\input{appendix/extended-related-work}

\input{appendix/extended-modeling-paradigms}
\input{appendix/nesy-approaches-as-nesy-ebms}
\input{appendix/extended-neupsl}
\input{appendix/experiments}

\vskip 0.2in
\bibliographystyle{SageH}
\bibliography{nesy25}

\end{document}

%% file: sections/abstract/abstract.tex
\begin{abstract}

The field of Neural-Symbolic (NeSy) systems is growing rapidly. 
Proposed approaches show great promise in achieving symbiotic unions of neural and symbolic methods.  
However, a unifying framework is needed to organize common NeSy modeling patterns and develop general learning approaches.
In this paper, we introduce Neural-Symbolic Energy-Based Models (NeSy-EBMs), a unifying mathematical framework for discriminative and generative NeSy modeling.
Importantly, NeSy-EBMs allow the derivation of general expressions for gradients of prominent learning losses, and we introduce a suite of four learning approaches that leverage methods from multiple domains, including bilevel and stochastic policy optimization.
Finally, we ground the NeSy-EBM framework with Neural Probabilistic Soft Logic (NeuPSL), an open-source NeSy-EBM library designed for scalability and expressivity, facilitating the real-world application of NeSy systems.
Through extensive empirical analysis across multiple datasets, we study NeSy-EBMs as instantiated through NeuPSL in tasks including image classification, graph node labeling, autonomous vehicle situation awareness, and question answering.
The results provide evidence for the practical utility of this instantiation in the studied settings.

\end{abstract}

%% file: sections/introduction/introduction.tex
\section{Introduction}
\label{sec:introduction}

The promise of mutually beneficial neural and symbolic integrations has motivated significant advancements in machine learning research.
Much of the recent progress has been achieved in the neural-symbolic (NeSy) computing literature \citep{garcez:book02, garcez:book09, garcez:jal19}.
NeSy is a large and rapidly growing community that has been hosting regular workshops since 2005 \citep{nesy05} and began holding conferences in 2024 \citep{nesy24}.
At a high level, NeSy research aims to build algorithms and architectures that combine neural and symbolic components \citep{xu:icml18, yang:ijcai20, cohen:jair20, manhaeve:ai21, wang:icml19, badreddine:ai22, ahmed:neurips22, pryor:ijcai23}.
With the continued growth of the field, NeSy requires a solid theoretical foundation built on a unifying framework.
Such a framework should support understanding and organizing the strengths and limitations of existing NeSy approaches, while guiding design decisions to better match the requirements of specific applications.
Furthermore, it should enable the development of general-purpose and widely applicable NeSy inference and learning algorithms.

In this paper, we introduce \emph{Neural-Symbolic Energy-Based Models} (NeSy-EBMs), a mathematical framework for NeSy.
NeSy-EBMs are a family of Energy-Based Models (EBMs) \citep{lecun:book06} defined by energy functions that are compositions of neural and symbolic components.
The neural component consists of a collection of deep models, and its output is provided to the symbolic component, which measures the compatibility of variables using domain knowledge and constraints.
This formulation serves as a foundation for characterizing NeSy modeling paradigms and for developing general-purpose inference and learning algorithms.
Moreover, by grounding NeSy in the well-established EBM perspective, NeSy-EBMs connects NeSy to the broader machine learning literature.

The insights derived from the NeSy-EBM framework motivate the development of a comprehensive system that supports key modeling paradigms.
To this end, we introduce Neural Probabilistic Soft Logic (NeuPSL), an open-source, expressive, and efficient library for constructing NeSy-EBMs.
NeuPSL uses the principled and comprehensive semantics of Probabilistic Soft Logic (PSL) \citep{bach:jmlr17} to create a NeSy-EBM symbolic component.
Then, the neural component can be built using any deep modeling library and seamlessly integrated with the PSL symbolic component.
Further, to ensure differentiability properties and provide principled forms of gradients for learning, we present a new formulation and regularization of PSL inference as a constrained quadratic program.

Further, we develop a suite of principled neural and symbolic parameter learning techniques for NeSy.
NeSy-EBM predictions are typically obtained by finding a state of variables with high compatibility (i.e., low energy).
The high compatibility state is found by minimizing the energy function via an optimization algorithm, for instance, an interior point method for continuous variables \citep{nocedal:wright:book06} or a branch-and-bound strategy for discrete problems \citep{papadimtriou:book98}. 
The complex prediction process makes finding a gradient or descent direction of a standard machine learning loss with respect to the parameters difficult.
To formalize these challenges and propose solutions, we introduce a categorization of learning losses based on the complexity of the relation to the NeSy-EBM energy function.
We derive general expressions, when they exist, for gradients of the categorized learning losses with respect to the neural and symbolic parameters. 
Additionally, we introduce four NeSy-EBM learning algorithms: one for learning the neural and symbolic weights separately and three for end-to-end learning.
Our end-to-end learning algorithms make use of ideas from the bilevel optimization and reinforcement learning literature.
Moreover, we discuss the strengths and limitations of each algorithm and describe its applicability to various modeling paradigms.

We empirically investigate the utility of NeSy-EBMs as instantiated through NeuPSL for four use cases:
1) constraint satisfaction and joint reasoning, 2) fine-tuning and adaptation, 3) few-shot and zero-shot reasoning, and 4) semi-supervised learning.
We simultaneously analyze multiple NeSy-EBM modeling paradigms and learning algorithms in an extensive empirical analysis across numerous variations of seven datasets.
We show compelling results for real-world applications, including graph node classification, computer vision object detection, and natural language question answering.
Notably, the NeuPSL instantiations show that NeSy-EBMs can enhance neural network prediction accuracy, enforce constraints, and improve label and data efficiency in semi-supervised and low-data settings, respectively. 

This paper integrates and expands on our prior work on NeSy systems and applications via the NeSy-EBM framework \citep{pryor:ijcai23, dickens:icml24, dickens:make24, dickens:phd24, pryor:thesis24}.
The strengths of NeSy-EBM models have been demonstrated on a variety of tasks, including dialog structure induction \citep{pryor:acl23}, natural language \citep{pan:emnlp23, dickens:make24} and visual question answering \citep{yi:neurips19}, autonomous vehicle situation awareness \citep{giunchiglia:ml23}, human activity recognition \citep{arrotta:acmimwut24}, recommendation \citep{carraro:aixia22}, and autonomous agent navigation and exploration \citep{zhou:icml23}.
Additionally, the NeSy-EBM framework has enabled a deeper understanding of the connections and capabilities of NeSy systems \citep{dickens:make24}.
Moreover, general NeSy inference and learning algorithms have been developed \citep{dickens:icml24, dickens:phd24} along with new open-source NeSy implementations \citep{pryor:ijcai23}.
The NeSy-EBM framework has become a powerful tool for formalizing connections and capabilities of NeSy models and for developing new NeSy architectures and learning algorithms.
In this work, we organize and advance our prior work to 
1) present NeSy-EBMs, a mathematical framework for NeSy, 
2) introduce NeuPSL, an open-source tool for building NeSy-EBMs, 
3) develop a suite of general NeSy learning techniques, and
4) analyze the utility of NeSy-EBMs as instantiated through NeuPSL and our learning techniques with an extensive empirical analysis.

This paper is organized as follows.
In \secref{sec:related-work} we discuss related work on NeSy frameworks and NeSy applications.
In \secref{sec:nesy-ebms}, we formally define NeSy-EBMs and introduce three practical NeSy modeling paradigms.
Next, in \secref{sec:neupsl-and-deep-hlmrfs}, we introduce NeuPSL, a scalable and expressive NeSy-EBM implementation.
In \secref{sec:nesy-ebm-learning}, we present a suite of NeSy learning techniques.
Then, in \secref{sec:experiments}, we use NeuPSL to build NeSy-EBMs for our empirical analysis of NeSy use cases, modeling paradigms, and learning algorithms.
Finally, we discuss limitations, takeaways, and future work in \secref{sec:limitations} and \secref{sec:conclusion}.

%% file: sections/related-work/introduction.tex
\section{Related Work}
\label{sec:related-work}

There is a long, rich history of research on the integration of symbolic knowledge and reasoning with neural networks, which has rapidly evolved in the past decade.
In this work, we establish a unifying framework for achieving such integration by connecting two foundational areas of machine learning research: Neural-Symbolic (NeSy) AI and energy-based modeling (EBMs).
The remainder of this section provides an overview of NeSy frameworks and applications.
Additionally, we provide an extended related
work in \appref{appendix:extended-related-work}, covering EBMs, and bilevel optimization.

%% file: sections/related-work/nesy-frameworks.tex
\subsection{Neural-Symbolic Frameworks}
\label{sec:related-work-nesy-frameworks}

NeSy  empowers neural models with domain knowledge and reasoning through integrations with symbolic systems \citep{garcez:book02, garcez:book09, garcez:jal19, besold:nesyai22, marra:ai24}.
Various taxonomies have been proposed to categorize NeSy literature.
\citenoun{bader:wwst05}, \citenoun{garcez:jal19}, and most recently \citenoun{besold:nesyai22} provide extensive surveys using characteristics such as knowledge representation, neural-symbolic connection, and applications to compare and describe methods.
Similarly, the works of \citenoun{marra:ai24} and \citenoun{lamb:ijcai20} propose taxonomies to connect NeSy to statistical relational learning and graph neural networks, respectively.
Focused taxonomies are described by \citenoun{giunchiglia:ijcai22} and \citenoun{krieken:ai22} for deep learning with constraints and symbolic knowledge representations and \citenoun{dash:sr22} for integrating domain knowledge into deep neural networks.
\citenoun{marconato:neurips23} characterizes the common reasoning mistakes made by NeSy models, and \citenoun{marconato:arxiv24} presents an ensembling technique that calibrates the model’s concept-level confidence to attempt to identify these mistakes.
Recently, \citenoun{wan:arxiv24} explored various NeSy AI approaches primarily focusing on workloads on hardware platforms, examining runtime characteristics and underlying compute operators.
Finally, \citenoun{krieken:arxiv24} propose a language for NeSy called ULLER that aims to unify the representation of major NeSy systems, with the long-term goal of developing a shared Python library.
Each of these surveys and taxonomies contributes to the comparison, understanding, and organization of the diverse collection of NeSy methodologies.
We contribute to these efforts by introducing a common mathematical framework (\secref{sec:nesy-ebms}) and describe a collection NeSy modeling paradigms (\secref{sec:nesy-ebms-a-taxonomy-of-modeling-paradigms}).

We organize our exposition of related NeSy AI frameworks into three research areas: learning from constraints, differentiable reasoning layers, and reasoner-agnostic systems.
The first subsection discusses NeSy learning losses, whereas the second and third subsections cover NeSy approaches to both learning and inference.
In the following subsections, we define each of the research areas and describe prominent examples of NeSy models belonging to the area. 

\subsubsection{Learning from Constraints}
\label{sec:related-work-learning-from-constraints}

Learning from constraints is using domain knowledge and common sense to construct a learning loss function \citep{giunchiglia:ijcai22, krieken:ai22}.
This approach encodes the knowledge captured by the loss into the weights of the network.
A key motivation is to ensure the compatibility of predictions with domain knowledge and common sense.
Moreover, learning with constraints avoids potentially expensive post-prediction interventions that would be necessary with a model that is not aligned with domain knowledge.  
However, consistency with domain knowledge and sound reasoning are not guaranteed during inference for NeSy models in this class.
This is because there is no symbolic reasoning performed to obtain predictions from the system. 

\citenoun{demeester:emnlp16}, \citenoun{rocktaschel:neurips17}, \citenoun{diligenti:icmla17}, \citenoun{bovsnjak:icml17}, and \citenoun{xu:icml18} are prominent examples of the learning-with-constraints NeSy paradigm.
\citenoun{demeester:emnlp16} incorporates domain knowledge and common sense into natural language and knowledge base representations by encouraging partial orderings over embeddings via a regularization of the learning loss. 
Similarly, \citenoun{rocktaschel:neurips17} leverage knowledge represented as a differentiable loss derived from logical rules to train a matrix factorization model for relation extraction.
\citenoun{diligenti:icmla17} use fuzzy logic to measure how much a model's output violates constraints, which is minimized during learning.
\citenoun{xu:icml18} introduces a loss function that represents domain knowledge and common sense by using probabilistic logic semantics.
More recently, \citenoun{giunchiglia:ml23} introduced an autonomous event detection dataset with logical requirements, and \citenoun{stoian:nesy23} shows that incorporating these logical requirements during the learning improves generalization.

\subsubsection{Differentiable Reasoning Layers}
\label{sec:related-work-differentiable-reasoning-layers}

Another successful area of NeSy is in differentiable reasoning layers.
The primary difference between this family of NeSy approaches and learning from constraints is that an explicit representation of knowledge and reasoning is maintained in the model architecture during both learning and inference.
Moreover, a defining aspect of differentiable reasoning layers is the instantiation of knowledge and reasoning components as differentiable computation graphs.
Differentiable reasoning layers support automatic differentiation during learning and symbolic reasoning during inference.

Pioneering works in differentiable reasoning include those of \citenoun{wang:icml19}, \citenoun{cohen:jair20}, \citenoun{yang:ijcai20}, \citenoun{manhaeve:ai21}, \citenoun{derkinderen:ijar24}, \citenoun{badreddine:ai22}, \citenoun{ahmed:neurips22} and \citenoun{ahmed:aistats23}.
\citenoun{wang:icml19} integrates logical reasoning and deep models by introducing a differentiable smoothed approximation to a maximum satisfiability (MAXSAT) solver as a layer. 
\citenoun{cohen:jair20} introduces a probabilistic first-order logic called TensorLog.
This framework compiles tractable probabilistic logic programs into differentiable layers.
A TensorLog system is end-to-end differentiable and supports efficient parallelizable inference.
Similarly, \citenoun{yang:ijcai20} and \citenoun{manhaeve:ai21} compile tractable probabilistic logic programs into differentiable functions with their frameworks NeurASP and DeepProbLog, respectively.
NeurASP and DeepProbLog use answer set programming \citep{brewka:acm11} and ProbLog \citep{deraedt:ijcai07} semantics, respectively.
\cite{winters:aaai22} proposes DeepStochLog, a NeSy framework based on stochastic definite clause grammars that define a probability distribution over possible derivations.
Recently, \cite{maene:neurips24} proposes DeepSoftLog, a superset of ProbLog, adding embedded terms that result in probabilistic rather than fuzzy semantics.
The logic tensor network (LTN) framework proposed by \citenoun{badreddine:ai22} uses neural network predictions to parameterize functions representing symbolic relations with real-valued or fuzzy logic semantics.
The fuzzy logic functions are aggregated to define a satisfaction level.
Predictions can be obtained by evaluating the truth value of all possible outputs and returning the highest-valued configuration.
\citenoun{badreddine:arxiv23} has expanded upon LTNs and presents a configuration of fuzzy operators for grounding formulas end-to-end in the logarithm space that is more effective than previous proposals.
Recently, \citenoun{ahmed:neurips22} introduced a method for compiling differentiable functions representing knowledge and logic using the semantics of probabilistic circuits (PCs) \citep{choi:unpub20}.
Their approach, called semantic probabilistic layers (SPLs), performs exact inference over tractable probabilistic models to enforce constraints over the predictions and uses the PC framework to ensure that the NeSy model is end-to-end trainable.

As pointed out by \citenoun{cohen:jair20}, answering queries in many (probabilistic) logics is equivalent to the weighted model counting problem, which is \#P-complete or worse.
Similarly, the MAXSAT problem studied by \citenoun{wang:icml19} is NP-hard.
Thus, since deep neural networks can be evaluated in time polynomial in their size, no polysize network can implement general logic queries unless \#P=P, or MAXSAT solving, unless NP=P.
For this reason, researchers have made progress towards building more efficient differentiable reasoning systems by, for example, restricting the probabilistic logic to tractable families \citep{cohen:jair20, ahmed:neurips22, maene:arxiv24}, or performing approximate inference \citep{wang:icml19, manhaeve:icpkrr21, krieken:neurips23}.

\subsubsection{Reasoner-Agnostic Systems}
\label{sec:related-work-reasoner-agnostic-systems}

More recently, researchers have sought to build NeSy frameworks with more general reasoning and knowledge representation capacities with expressive mathematical program blocks for reasoning.
Mathematical programs are capable of representing cyclic dependencies across variables and ensuring the satisfaction of prediction constraints during learning and inference.
Moreover, the system's high-level inference and training algorithms are agnostic to the solver used for the mathematical program.

Prominent reasoner-agnostic systems include the works of \citenoun{amos:icml17}, \citenoun{agrawal:neurips19}, \citenoun{vlastelica:iclr20}, and \citenoun{cornelio:iclr23}.
\citenoun{amos:icml17} integrate linearly constrained quadratic programming problems (LCQP) as layers in deep neural networks with their OptNet framework, and show that the solutions to the LCQP problems are differentiable with respect to the program parameters. 
The progress of OptNet was continued by the work of \citenoun{agrawal:neurips19} with the application of domain-specific languages (DSLs) for instantiating the LCQP program layers.
DSLs provide a syntax for specifying LCQPs representing knowledge and constraints, making optimization layers more accessible.
\citenoun{vlastelica:iclr20} propose a method for computing gradients of solutions to mixed integer linear programs based on a continuous interpolation of the program's objective.
In contrast to the works of \citenoun{amos:icml17} and \citenoun{agrawal:neurips19}, the approach introduced by \citenoun{vlastelica:iclr20} supports integer constraints and achieves this by approximating the true gradient of the program output.
\citenoun{cornelio:iclr23} takes a different approach from these three methods by employing reinforcement learning techniques to support more general mathematical programs.
Specifically, the neural model's predictions are interpreted as a state in a Markov decision process.
Actions from a policy are taken to identify components that violate constraints to obtain a new state.
The new state is provided to a solver, which corrects the violations, and a reward is computed.
The solver is not assumed to be differentiable, and the REINFORCE algorithm \citep{williams:ml92} with a standard policy loss is used to train the system end-to-end without the need to backpropagate through the solver.

%% file: sections/related-work/applications.tex
\subsection{Applications}
\label{sec:related-work-applications}

We highlight four proven applications of NeSy: 
1) constraint satisfaction and joint reasoning, 
2) post-training, 
3) few-shot and zero-shot reasoning, 
4) and semi-supervised learning
This list of use cases is not exhaustive. 
However, the efficacy of the NeSy approach in these applications is well established, and we will illustrate four of these use cases in our empirical evaluation.
The following subsections define the problem and the high-level motivation for utilizing NeSy techniques in such settings.
Additionally, we discuss collections of existing NeSy systems for each application. 

\subsubsection{Constraint Satisfaction and Joint Reasoning}
\label{sec:related-work-applications-constraint-satisfaction-and-joint-reasoning}

In real-world settings, a deployed model's predictions must meet well-defined requirements.
Additionally, leveraging known patterns or dependencies in the output can significantly improve a model's accuracy and trustworthiness.
{\em Constraint satisfaction} is finding a prediction that satisfies all requirements.
NeSy systems perform constraint satisfaction by reasoning across their output to provide a structured prediction, typically using some form of joint reasoning.
In other words, constraints are enforced during inference to produce predictions that satisfy domain requirements, rather than being used solely as training signals.

A commonly used example of constraint satisfaction and joint reasoning with NeSy techniques is puzzle-solving.
Many NeSy frameworks are introduced with an evaluation on visual Sudoku and its variants \citep{wang:icml19, augustine:nesy22}.
In the visual Sudoku problem, puzzles are constructed with handwritten digits, and a model must classify the digits and infer numbers to fill in the empty cells using the rules of Sudoku.
Empirical evaluations of NeSy systems that perform constraint satisfaction and joint reasoning on visual Sudoku problems can be found in \cite{wang:icml19}, \cite{augustine:nesy22}, \cite{pryor:ijcai23}, and \cite{morra:nesy23}.
Similarly, \citenoun{vlastelica:iclr20} introduces the shortest path finding problem as a NeSy task.
Images of terrain maps are partitioned into a grid, and the model must find a continuous, lowest-cost path between two points.
The works of \citenoun{vlastelica:iclr20} and \cite{ahmed:neurips22} perform constraint satisfaction and joint reasoning with NeSy models for shortest path finding.

Constraint satisfaction and joint reasoning with NeSy models are also effective for real-world natural language tasks.
For instance, \citenoun{sachan:neurips18} introduces the Nuts\&Bolts NeSy system to build a pipeline for parsing physics problems.
The NeSy system jointly infers a parsing from multiple components that incorporates domain knowledge and prevents the accumulation of errors that would occur from a naive composition.
In another work, \citenoun{zhang:icml23} propose GeLaTo (generating language with tractable constraints) for imposing constraints on text generated from language models.
GeLaTo generates text tokens by autoregressively sampling from a distribution constructed from a pre-trained language model and a tractable probabilistic model encoding the constraints.
More recently, \cite{pan:emnlp23} introduced the Logic-LM framework for integrating LLMs with symbolic solvers to improve complex problem-solving.
Logic-LM formulates a symbolic model using an LLM that uses prompts of the syntax and semantics of the symbolic language.
Finally, \citenoun{abraham:arxiv24} introduced CLEVR-POC, which requires leveraging logical constraints to generate plausible answers to questions about a hidden object in a given partial scene.
They then demonstrated remarkable performance improvements over neural methods by integrating an LLM with a visual perception network and a formal logical reasoner.

Computer vision systems also benefit from the constraint satisfaction and joint reasoning capabilities of NeSy models.
For instance, semantic image interpretation (SII) is the task of extracting structured descriptions from images.
\citenoun{donadello:ijcai17} implemented a NeSy model for SII using the Logic Tensor Network (LTN) \citep{badreddine:ai22} framework for reasoning about ``part-of'' relations between objects with logical formulas.
Similarly, \citenoun{yi:neurips19} propose a NeSy visual question-answering framework (NS-VQA). 
The authors employ deep representation learning for visual recognition to recover a structured representation of a scene and then language understanding to formulate a program from a question.
A symbolic solver executes the formulated program to obtain an answer.
\citenoun{sikka:techreport20} introduced Deep Adaptive Semantic Logic (DASL) for predicting relationships between pairs of objects in an image given the bounding boxes and object category labels, i.e., visual relationship detection.
The DASL system allows a modeler to express knowledge using first-order logic and to combine domain-specific neural components into a single deep network.
A DASL model is trained to maximize a measured truth value of the knowledge.

\subsubsection{Post-training}
\label{sec:related-work-applications-post-training}

We are in the era of foundation models in AI \citep{bommasani:arxiv22}.
It is now commonplace to adjust a model that is pre-trained on large amounts of general data (typically using self-supervision) for downstream tasks. 
Post-training is the process of updating the parameters of a pre-trained model to perform in a new domain \citep{devlin:arxiv19, hu:iclr22}.
Fine-tuning and alignment are two example post-training techniques that adjust pre-trained model parameters by minimizing a learning objective over a dataset, both of which are specialized for downstream tasks.
While related to the broader paradigm of learning from constraints (\secref{sec:related-work-learning-from-constraints}), post-training focuses specifically on applying such learning objectives when adapting pre-trained or foundation models to downstream tasks.

NeSy frameworks are used in post-training to design principled learning objectives that integrate knowledge and constraints relevant to the downstream task and the application domain.
NeSy learning losses are applied in the work of \citenoun{giunchiglia:ml23} to post-train a neural system for autonomous vehicle situation awareness \citep{singh:tpa2021}. 
In another post-training computer vision task, \citenoun{arrotta:acmimwut24} develop a NeSy loss for training a neural model to perform context-aware human activity recognition.
NeSy post-training has also been explored in the natural language processing literature.
Recently, \citenoun{ahmed:neurips23} proposed the pseudo-semantic loss for detoxifying large language models.
The authors disallow a list of toxic words and show this intuitive approach steers a language model's generation away from harmful language and achieves state-of-the-art detoxification scores.
\citenoun{feng:naacl24} has explored directly learning the reasoning process of logical solvers within the LLM to avoid parsing errors.
Finally, \citenoun{cunnington:arXiv24} introduced NeSyGPT, which post-trains a vision-language foundation model to extract symbolic features from raw data before learning some answer set program.

\subsubsection{Few-Shot and Zero-Shot Reasoning}
\label{sec:related-work-applications-few-shot-and-zero-shot-reasoning}

Training data for a downstream task may be limited or even nonexistent.
In {\em few-shot} settings, only a few examples are available, while in {\em zero-shot} settings, no explicit training data is provided for the task.
In these settings, few-shot and zero-shot reasoning techniques are used to enable a model to generalize beyond the limited available training data.
Leveraging pre-trained models and domain knowledge are key ideas for succeeding in few-shot and zero-shot contexts.

NeSy techniques have been successfully applied for various few-shot and zero-shot settings.
Integrating symbolic knowledge and reasoning enables better generalization from a small number of examples.
NeSy systems can utilize symbolic knowledge to make deductions about unseen classes or tasks.
For instance, providing recommendations for new items or users can be viewed as a few-shot or zero-shot problem. 
\citenoun{kouki:recsys15} introduce the HyPER (hybrid probabilistic extensible recommender) framework for incorporating and reasoning over a wide range of information sources.
By combining multiple information sources via logical relations, the authors outperformed the state-of-the-art approaches of the time.
More recently, \cite{carraro:aixia22} developed an LTN-based recommender system to overcome data sparsity.
This model uses background knowledge to generalize predictions for new items and users quickly.
Few-shot and zero-shot reasoning tasks are also prevalent in object navigation.
The ability to navigate to novel objects and unfamiliar environments is vital for the practical use of embodied agents in the real world.
In this context, \citenoun{zhou:icml23} presents a method for  ``exploration with soft commonsense constraints" (ESC).
ESC first employs a pre-trained vision and language model for semantic scene understanding, then a language model to reason from the spatial relations, and finally PSL to leverage symbolic knowledge and reasoning to guide exploration. 
In natural language processing, \citenoun{pryor:acl23} infers the latent dialog structure of a goal-oriented conversation using domain knowledge to overcome the challenges of limited data and out-of-domain generalization.
\citenoun{sikka:techreport20} (mentioned above) also finds that the few-shot and zero-shot capabilities of NeSy models help in visual relationship detection.
Specifically, the addition of commonsense reasoning and knowledge improves performance by over $10\%$ in data-scarce settings.

\subsubsection{Semi-Supervised Learning}
\label{sec:related-work-applications-semi-supervised-learning}

Semi-supervised approaches facilitate learning from labeled as well as unlabeled data by combining the goals of supervised and unsupervised machine learning.
We refer the reader to the excellent recent survey on semi-supervised approaches by \citenoun{vanengelen:ml20}.
In short, supervised methods fit a model to predict an output label given a corresponding input, while unsupervised methods infer the underlying structure in the data.
The ability to leverage both labeled and unlabeled data leads to performance improvements, better generalization, and reduced labeling costs.

NeSy is a functional approach to semi-supervised learning that leverages knowledge and domain constraints to train a model.
This is achieved with loss functions that encode domain knowledge and structure and depend only on the input and output; that is, they do not require a label. 
Early work on semi-supervision with knowledge was carried out by \citenoun{chang:acl07}, who unify and leverage task-specific constraints to encode structure in the input and output data and possible labels.
They evaluate their semi-supervised learning method on the task of named entity recognition in citations as well as advertisements. 
More recently, \citenoun{ahmed:uai22} introduced the neuro-symbolic entropy regularization loss to encourage model confidence in predictions satisfying a set of constraints on the output.
They demonstrate that the regularization improves model performances in the task of entity relation extraction in text.
Additionally, \citenoun{stoian:nesy23} studied the effect of various t-norms used to soften the logical constraints for the symbolic component and demonstrated on a challenging road event detection dataset with logical requirements \citep{giunchiglia:ml23} that the incorporation of a symbolic loss drastically improves performance.

%% file: sections/neural-symbolic-energy-based-models/introduction.tex
\section{A Mathematical Framework for NeSy}
\label{sec:nesy-ebms}
In this section, we introduce Neural-symbolic energy-based models (NeSy-EBMs): a unifying mathematical framework for NeSy.
Intuitively, NeSy-EBMs formalize the neural-symbolic interface as a composition of functions.
In other words, NeSy-EBMs organize modules by roles, specifically perception and reasoning.
The theory and notation introduced in this section are used throughout the rest of this paper.

%% file: sections/neural-symbolic-energy-based-models/definition.tex
\subsection{Neural Symbolic Energy-Based Models}
\label{sec:nesy-ebms-definition}

NeSy-EBMs are a family of EBMs \citep{lecun:book06} that integrate deep architectures with explicit encodings of symbolic relations via an energy function.
EBM energy functions measure the compatibility of variables, where low energy states correspond to high compatibility.
For NeSy-EBMs, high compatibility indicates that the variables are consistent with domain knowledge and common sense.
In the following section, the formal NeSy-EBM definition is grounded with intuitive examples of NeSy modeling paradigms.

As diagrammed in \figref{fig:nesy-ebm-energy-function}, a NeSy-EBM energy function composes a neural component with a symbolic component, represented by the functions $\mathbf{g}_{nn}$ and $g_{sy}$, respectively.
The neural component is a deep model (or collection of deep models) parameterized by weights from a domain $\mathcal{W}_{nn}$, that takes a neural input from a domain $\mathcal{X}_{nn}$ and outputs a real-valued vector of dimension $d_{nn}$.
The symbolic component encodes domain knowledge and is parameterized by weights from a domain $\mathcal{W}_{sy}$.
It maps the inputs of a domain $\mathcal{X}_{sy}$, target (or output) variables from $\mathcal{Y}$, and neural outputs from $\textrm{Range}(\mathbf{g}_{nn})$ to a scalar value.
In other words, the symbolic component measures the compatibility of targets, inputs, and neural outputs with domain knowledge.
Intuitively, the neural component has the capacity and responsibility to perform low-level perception or generation, while the symbolic component has the role of performing high-level symbolic reasoning.

\begin{definition}
    \label{def:nesy-ebm-energy-function}
    A \textbf{NeSy-EBM energy function} is a mapping parameterized by neural and symbolic weights from domains $\mathcal{W}_{nn}$ and $\mathcal{W}_{sy}$, respectively, and quantifies the compatibility of a target variable from a domain $\mathcal{Y}$ and neural and symbolic inputs from the domains $\mathcal{X}_{nn}$ and $\mathcal{X}_{sy}$, respectively, with a scalar value: 
    \begin{align}
        E: \mathcal{Y} \times \mathcal{X}_{sy} \times \mathcal{X}_{nn} \times \mathcal{W}_{sy} \times \mathcal{W}_{nn} \to \mathbb{R}.
        \label{eq:nesy-ebm-energy-function}
    \end{align}
    
    A NeSy-EBM energy function is a composition of a \textbf{neural} and \textbf{symbolic component}.
    Neural weights parameterize the neural component, which outputs a real-valued vector of dimension $d_{nn}$:
    \begin{align}
        \mathbf{g}_{nn} : \mathcal{X}_{nn} \times \mathcal{W}_{nn} \to \mathbb{R}^{d_{nn}}.
    \end{align}
    The symbolic component maps the symbolic variables, symbolic parameters, and a real-valued vector of dimension $d_{nn}$ to a scalar value:
    \begin{align}
        g_{sy}: \mathcal{Y} \times \mathcal{X}_{sy} \times \mathcal{W}_{sy} \times \mathbb{R}^{d_{nn}} \to \mathbb{R}.
    \end{align}
    The NeSy-EBM energy function is
    \begin{align}
        & E: (\mathbf{y}, \mathbf{x}_{sy}, \mathbf{x}_{nn}, \mathbf{w}_{sy}, \mathbf{w}_{nn}) \mapsto g_{sy}(\mathbf{y}, \mathbf{x}_{sy}, \mathbf{w}_{sy}, \mathbf{g}_{nn}(\mathbf{x}_{nn}, \mathbf{w}_{nn})). \nonumber \qed
    \end{align}
\end{definition}

\begin{figure}[t]
    \centering
    \includegraphics[width=0.6 \textwidth]{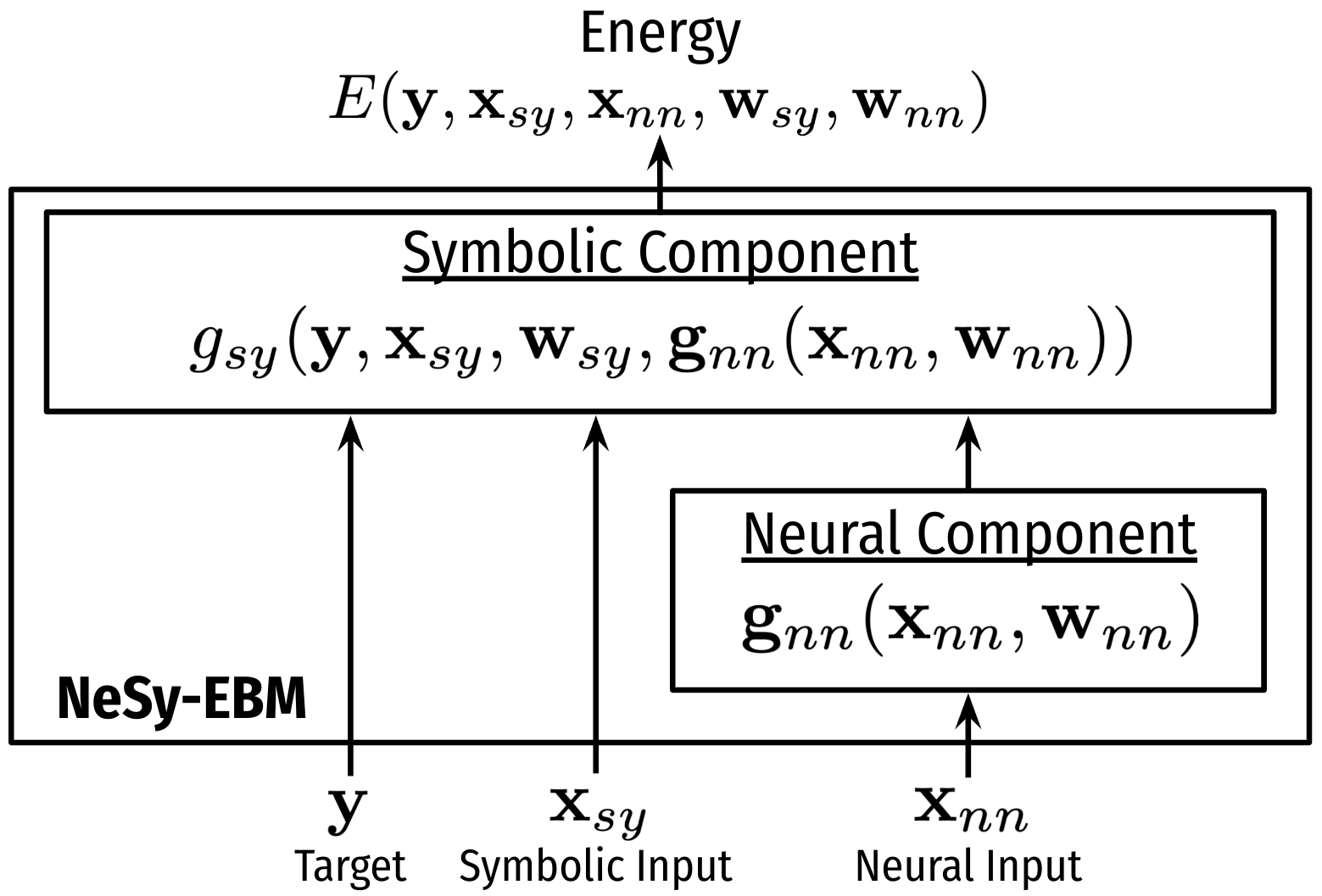}
    \caption{A neural-symbolic energy-based model.}
    \label{fig:nesy-ebm-energy-function}
\end{figure}
The energy, $E$, and the symbolic component $g_{sy}$ are distinct objects, and this distinction has both theoretical and practical value.
The symbolic component, $g_{sy}$, is defined at every vector $\mathbf{u} \in \mathbb{R}^{d_{nn}}$, independent of whether $\mathbf{u}$ is realized by any neural component composed with it.
The explicit $g_{sy}$, $\mathbf{g}_{nn}$ composition notation is necessary for the learning theory developed in \secref{sec:learning-losses} where gradients of value-functions with respect to the neural weights (\thmref{thm:value-function-gradient}) are obtained by perturbing the argument of $g_{sy}$ at the realized neural output, a perturbation that cannot be expressed in terms of $E$ alone.
Practically, this compositional structure reflects the modular design of many NeSy systems, as we demonstrate in \secref{sec:neupsl-and-deep-hlmrfs-neupsl} with NeuPSL.
Finally, the energy function, $E$,  abstracts away the neural output as an intermediate layer and succinctly expresses inference and learning objectives without relying on explicitly expressing the machinery of neural symbolic composition.

In this paper, we focus on the most common inference task: prediction, classification, and decision-making:
\begin{align}
    \label{eq:nesy-ebm-prediction}
    \argmin_{\hat{\mathbf{y}} \in \mathcal{Y}} E(\hat{\mathbf{y}}, \mathbf{x}_{sy}, \mathbf{x}_{nn}, \mathbf{w}_{sy}, \mathbf{w}_{nn}).
\end{align}
Prediction with NeSy-EBMs captures various forms of reasoning, including probabilistic, logical, arithmetic, and their combinations.
It can represent standard applications of prominent NeSy systems, including, DeepProbLog \citep{manhaeve:ai21}, LTNs \citep{badreddine:ai22}, Semantic Probabilistic Layers \citep{ahmed:neurips22}, and NeuPSL \citep{pryor:ijcai23}, to name a few.

%% file: sections/neural-symbolic-energy-based-models/modeling-patterns.tex
\subsection{Modeling Paradigms for NeSy}
\label{sec:nesy-ebms-a-taxonomy-of-modeling-paradigms}
NeSy-EBMs are highly expressive, and many NeSy approaches can in principle be represented by defining an appropriate energy function (e.g., through components such as $g_{\mathrm{sy}}$).
While this generality is powerful, overly abstract formulations can obscure important structural differences between neural-symbolic approaches.
To preserve these distinctions, this subsection introduces three typical NeSy modeling paradigms organized by how the neural component aligns with the symbolic component to define the prediction program in \eqnref{eq:nesy-ebm-prediction}.
Specifically, the modeling paradigms distinguish whether the neural component is integrated as symbolic variables, symbolic parameters, or symbolic potentials, yielding the deep symbolic variables (DSVar), deep symbolic parameters (DSPar), and deep symbolic potentials (DSPot) paradigms, respectively.
Formal definitions are provided in \appref{sec:appendix-modeling-paradigms-formal}.

Modeling paradigms, i.e., neural-symbolic interfaces, differ in the role of the neural component in prediction and learning.
This, in turn, leads to continuity properties and gradient forms of learning losses with respect to the neural parameters.
In DSVar, the neural component directly predicts a subset of the target variables, and these predictions are treated in practice as fixed evidence.
In DSPar, the neural component parameterizes the symbolic component while prediction still optimizes over all target variables.
In DSPot, the neural component defines or selects the symbolic potential itself.

As a practical guideline, DSVar is a natural choice when a simpler prediction program is preferred or when the goal is to enhance neural performance with a structured loss.
DSPar is a natural choice when symbolic inference should be able to correct weak neural predictions, such as in constraint satisfaction, joint reasoning, fine-tuning, few-shot, and semi-supervised settings.
DSPot is a natural choice when the relevant domain knowledge depends on context and the task is open-ended, so the neural component must define or select the symbolic potential itself.

In the following section, we present a novel NeSy framework that supports all of the outlined modeling paradigms.
Additionally, \appref{sec:appendix-nesy-approaches-as-nesy-ebms} formalizes three widely used NeSy approaches, DeepProbLog \citep{manhaeve:ai21}, Logic Tensor Networks \citep{badreddine:ai22}, and Semantic Loss \citep{xu:icml18}, within these modeling paradigms.
These formalizations capture common instantiations of the systems aligned with the modeling paradigms described in this section.
Furthermore, it leaves room for broader formulations of each NeSy system.

\begin{figure}[t]
    \begin{subfigure}{\textwidth}
        \centering
        \includegraphics[width=0.9 \textwidth]{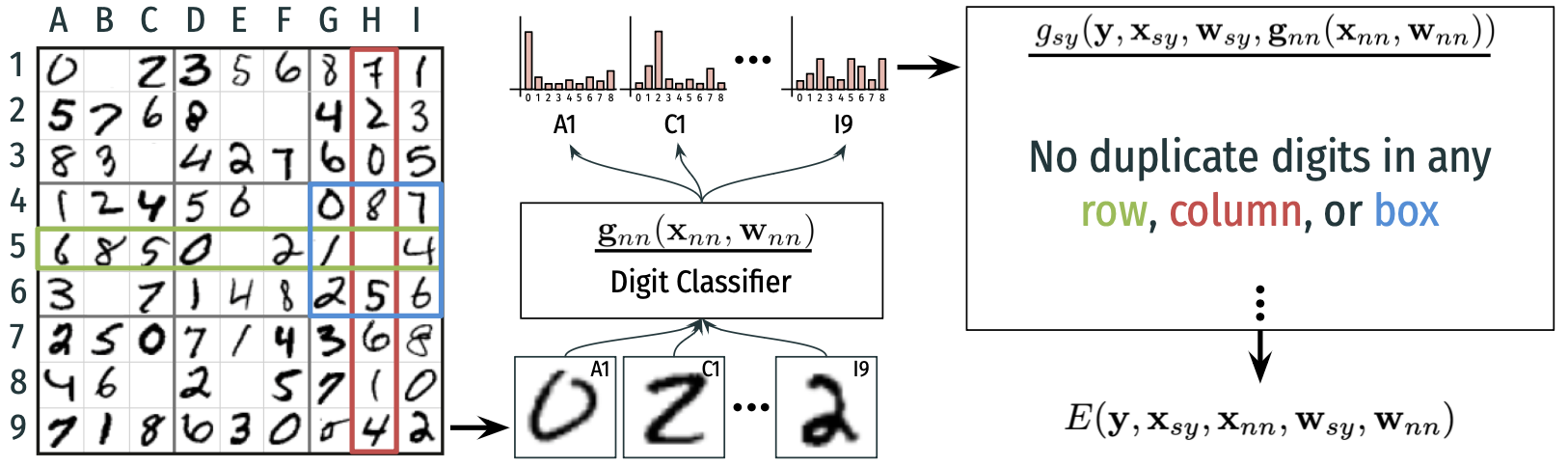}
        \caption{}
        \label{fig:visual-sudoku-nesy-ebm}
    \end{subfigure}
    \begin{subfigure}{\textwidth}
        \centering
        \includegraphics[width=0.9 \textwidth]{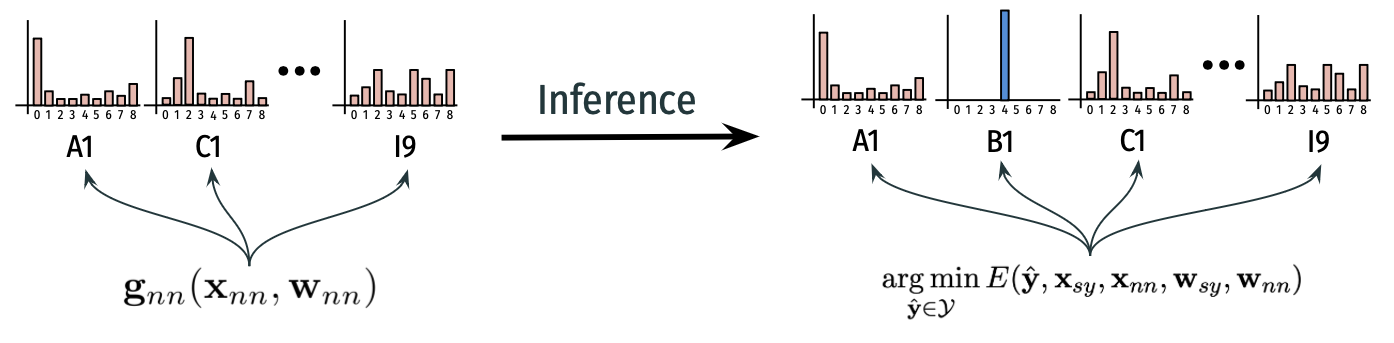}
        \caption{}
        \label{fig:visual-sudoku-dsvar}
    \end{subfigure}
    \begin{subfigure}{\textwidth}
        \centering
        \includegraphics[width=0.9 \textwidth]{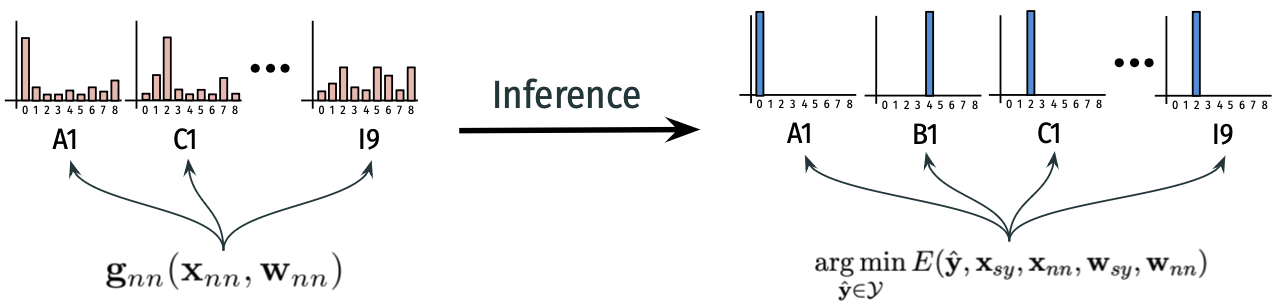}
        \caption{}
        \label{fig:visual-sudoku-dspar}
    \end{subfigure}
    \caption{
        (a) A NeSy-EBM for solving a Sudoku board constructed from handwritten digits.
        The neural component classifies handwritten digits.
        Then, the symbolic component uses the digit classifications and Sudoku rules to quantify the compatibility of the inputs, neural predictions, and targets.
        (b) In the DSVar modeling paradigm inference process, the neural component predicts the digit labels of squares containing MNIST images, while the symbolic component measures incompatibility and predicts the latent (blank) squares.
        (c) In the DSPar modeling paradigm inference process, the neural component predicts squares with digits, and the symbolic component can alter these predictions to adhere to symbolic constraints.
    }
    \label{fig:visual-sudoku-nesy-ebm-modeling-paradigms}
\end{figure}

\subsubsection{Deep Symbolic Variables}
\label{sec:deep-symbolic-variables}

The deep symbolic variables (DSVar) paradigm trains neural components efficiently with a loss that captures domain knowledge.
Concisely, the neural component directly predicts the values of target or latent variables in a symbolic potential.\footnote{This section focuses on deep symbolic variables in the context of target variables. Extending to latent variables is straightforward.}
In other words, there is a one-to-one mapping from the neural output to the targets.
However, note that the mapping is not necessarily {\em onto}, that is, there may be target or latent variables without a corresponding neural output.
Prominent NeSy approaches exemplifying this paradigm include logic tensor networks \citep{badreddine:ai22}, learning with logical constraints \citep{giunchiglia:ijcai22}, semantic-based regularization \citep{diligenti:jmlr17}, and deep logic models \citep{marra:ecmlkdd19}.

The DSVar modeling paradigm typically yields the most straightforward prediction program compared to the other modeling paradigms.
This is because the neural model fixes a subset of the inference decision variables, making the prediction program smaller.
However, for the same reason that this modeling paradigm typically has a simpler prediction program, the symbolic component cannot be used to resolve constraint violations made by the neural component.
Rather, DSVar models rely on learning to train a neural component to adhere to constraints.
The DSVar paradigm is demonstrated in the following example.
\begin{example}
    \label{example:deep-symbolic-variables}
    Visual Sudoku \citep{wang:icml19} puzzle solving is the problem of recognizing handwritten digits in non-empty puzzle cells and reasoning with the rules of Sudoku (no repeated digits in any row, column, or box) to fill in empty cells.
    \figref{fig:visual-sudoku-nesy-ebm-modeling-paradigms} shows a partially complete Sudoku puzzle created with MNIST images \citep{lecun:ieee98} and a NeSy-EBM designed for visual Sudoku solving.
    The neural component is a digit classifier predicting the label of MNIST images, and the symbolic component quantifies rule violations.
\end{example}

The DSVar modeling paradigm is specifically designed to allow the neural component to directly influence the target variables within the symbolic component.
Although this paradigm allows direct influence on predictions, its scope is strictly confined to target variables.
In scenarios where the neural model must exert indirect influence on variables or interact with other elements of the symbolic component, such as entire symbolic potentials or parameters associated with individual constraints, a different modeling paradigm becomes necessary.
The following subsection introduces a paradigm that extends the neural component's influence, enabling connections to other parameters or constants within the symbolic component, beyond just the target variables.

\subsubsection{Deep Symbolic Parameters}
\label{sec:deep-symbolic-parameters}

The deep symbolic parameters (DSPar) modeling paradigm allows targets and neural predictions to be unequal or represent different concepts.
Prominent NeSy frameworks supporting this technique include DeepProbLog \citep{manhaeve:ai21}, and semantic probabilistic layers \citep{ahmed:neurips22}.
Succinctly, the neural component is applied as a parameter in the symbolic potential.
This paradigm allows the symbolic component to correct constraints violated by the neural component during prediction.
This paradigm is demonstrated in the following example.

\begin{example}
    \label{example:deep-symbolic-parameters}
    Again, consider the Visual Sudoku puzzle-solving problem illustrated in \figref{fig:visual-sudoku-nesy-ebm-modeling-paradigms}.
    As in the DSVar model, the neural component of the DSPar model is a digit classifier predicting the label of MNIST images.
    However, the digit classifications of the neural component are used as initial predictions in the symbolic component, as a prior for a probabilistic model.
    Then, the symbolic component is used to quantify rule violations as well as the difference between neural outputs and target variables.
\end{example}

The DSPar modeling paradigm is widely applicable.
For instance, it is used for constraint satisfaction, fine-tuning, few-shot, and semi-supervised settings in our empirical analysis.
However, note that the DSVar and DSPar models have only a single fixed symbolic potential.
This property makes these paradigms well-suited for dedicated tasks but less applicable to open-ended settings, where the relevant domain knowledge depends on context.
To address this challenge, the following modeling paradigm leverages generative modeling to perform in open-ended tasks.

\subsubsection{Deep Symbolic Potentials}
\label{sec:deep-symbolic-potentials}

Deep-symbolic potentials (DSPot), the most advanced paradigm we propose, enhances deep models with symbolic reasoning tools.
The Logic-LM pipeline proposed by \citenoun{pan:emnlp23} is an excellent example of this modeling paradigm.
At a high level, the neural component is a generative model that samples symbolic potentials from a set to define the symbolic component.
Specifically, input data is used as context to retrieve relevant domain knowledge and formulate a program to perform inference in open-ended problems.
This paradigm is demonstrated in the following example.

\begin{figure}[t]
    \centering
    \includegraphics[width=0.95 \textwidth]{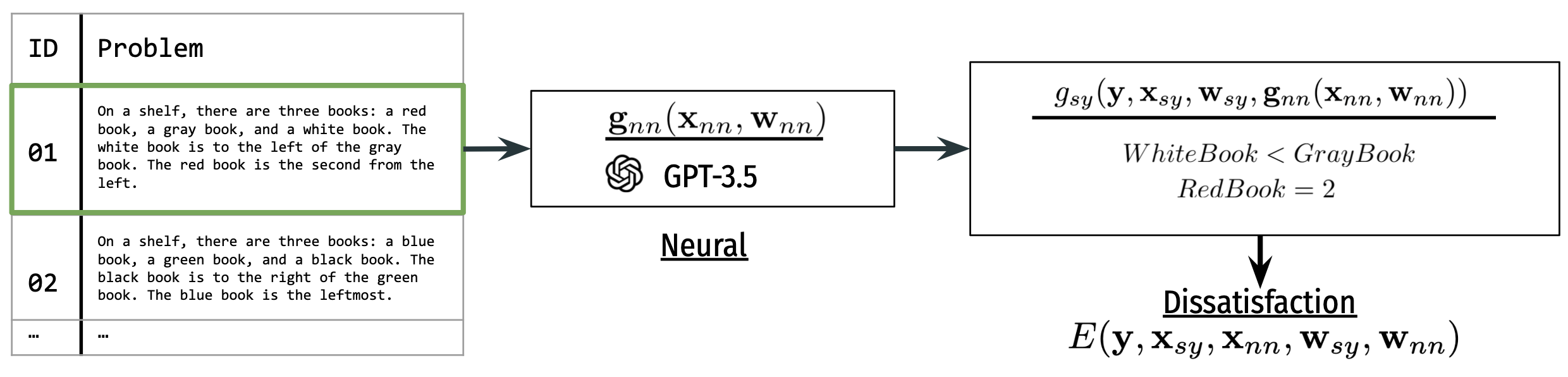}
    \caption{
    A deep symbolic potential model for answering questions about a set of objects' order described in natural language.
    The neural component is an LLM that generates syntax to create a symbolic potential.
    The symbolic potential is used to perform deductive reasoning and answer the question.
    See \exref{example:deep-symbolic-potentials} for details.}
    \label{fig:question-answering-dspot}
\end{figure}

\begin{example}
    \label{example:deep-symbolic-potentials}
    Question answering is the problem of giving a response to a question posed in natural language.
    \figref{fig:question-answering-dspot} shows a set of word problems asking for the order of a set of objects given information expressed in natural language and a NeSy-EBM designed for question answering.
    The neural component is a large language model (LLM) that is prompted with a word problem and tasked with generating a program within the syntax of a symbolic framework.
    The symbolic framework uses the generated program to instantiate a symbolic component used to perform deductive reasoning.
\end{example}

In our view, DSPot is the only applicable paradigm for truly open-ended tasks.
Moreover, DSPot enhances generative models, such as LLMs, with consistent symbolic reasoning capabilities.
This feature is demonstrated in constraint satisfaction and joint reasoning experiments in our empirical analysis.
DSPot's limitation is that the neural component must learn to sample from a large potential set.
For instance, in the example, an LLM must reliably generate syntax to define a symbolic potential for solving the word problem.
LLMs require a substantial amount of computational resources to train and then fine-tune for a specific NeSy framework.
Furthermore, the inference time is dependent on the sampled symbolic potential. 
If the neural component samples a complex symbolic potential, inference may be slow.

%% file: sections/neupsl/introduction.tex
\section{Neural Probabilistic Soft Logic and Deep Hinge-Loss Markov Random Fields}
\label{sec:neupsl-and-deep-hlmrfs}

We introduce \textit{Neural Probabilistic Soft Logic} (NeuPSL), an expressive framework for constructing a broad class of NeSy-EBMs by extending the probabilistic soft logic (PSL) probabilistic programming language \citep{bach:jmlr17}.
NeuPSL is designed to be expressive and efficient to support every modeling paradigm and easily be used for a range of applications. 
We begin by presenting the essential syntax and semantics of NeuPSL, encompassing Deep Hinge-Loss Markov Random Fields (deep HL-MRF), the underlying probabilistic graphical model.
We encourage readers to see \cite{bach:jmlr17} for an in-depth introduction to PSL syntax and semantics and \cite{pryor:thesis24} for a detailed description of NeuPSL model instantiation.
Then, we present a new formulation and regularization of (Neu)PSL inference as a constrained quadratic program.
Our formulation is utilized to guarantee differentiability properties and provide principled gradients to support end-to-end neural and symbolic parameter learning, instantiating the learning algorithms introduced in \secref{sec:nesy-ebm-learning}.

%% file: sections/neupsl/neural-probabilistic-soft-logic.tex
\subsection{Neural Probabilistic Soft Logic}
\label{sec:neupsl-and-deep-hlmrfs-neupsl}

NeuPSL is a declarative language for constructing NeSy-EBMs.
As discussed in \secref{sec:nesy-ebms-a-taxonomy-of-modeling-paradigms}, several modeling paradigms describe how neural and symbolic components can be integrated.
For clarity and to not be overly verbose in notation, this section presents NeuPSL as a NeSy-EBM instantiated using the DSVar and DSPar paradigms.\footnote{\citenoun{pryor:thesis24} further describes how NeuPSL can also be formulated within the DSPot modeling paradigm.}

Fundamentally, NeuPSL provides a syntax for encoding dependencies between relations and attributes of entities and for integrating neural components in a symbolic model.
Specifically, dependencies and neural component compositions are expressed as first-order logical or arithmetic statements referred to as \emph{rules}.
Each rule is a template for instantiating, i.e., \emph{grounding}, potentials or constraints to define the NeuPSL energy function.
Every rule is grounded over a set of domains, $\mathbf{D} = \{D_1, D_2, \cdots\}$, where each of the domains $D_{i}$ is a finite set of elements referred to as \textit{constants}.
For instance, referring to the visual Sudoku problem described in \exref{example:deep-symbolic-parameters}, the constant ``A1'' can denote the cell at position $A1$ in a Sudoku puzzle, and the constant ``1'' can denote the digit $1$.
Constants are grouped and aligned with a corresponding domain from $\mathbf{D}$ using placeholders or \emph{variables}.
Relations between constants are \textit{predicates}.
In NeuPSL, a predicate is referenced using its unique identifier.
For instance, $\pslpred{CellDigit}$ is a predicate that can represent whether a cell contains a specified digit.
Another example is the predicate $\pslpred{SudokuViolation}$ representing whether a Sudoku rule is violated given the digits in two specified cells.
Finally, the predicate $\pslpred{NeuralClassifier}$ is a predicate that represents the predicted digit in a cell made by a neural network classifier.
Predicates with specified constant domains are \textit{atoms}.
NeuPSL extends PSL with \textit{deep atoms}: atoms backed by a deep model.

\begin{definition}{Atom.} An \textbf{atom}, $A$, is a predicate associated with a list of $k > 0$ domains $D'_1, \cdots, D'_k$ from $\mathbf{D}$:
\begin{align*}
    A : (D'_1 \times \cdots \times D'_k) \rightarrow [0, 1]
\end{align*}
where $k$ is the corresponding predicate's arity.

A \textbf{deep atom}, $DA$, with domains $D'_1, \cdots, D'_k$ from $\mathbf{D}$ is an atom parameterized by a set of weights $\mathbf{w}_{nn}$ from a domain $\mathcal{W}_{nn}$
\begin{align*}
    DA : (D'_1 \times \cdots \times D'_k; \mathbf{w}_{nn}) \rightarrow [0, 1].
\end{align*}

A \textbf{ground atom} is an atom with constant arguments. \qed
\end{definition}

With the above definition, we can now formally define a NeuPSL rule.

\begin{definition}{Rule.}
A \textbf{rule}, $R$, is a function of $s \geq 1$ variables $v_{1}, \cdots, v_{s}$ from the domains $D'_1, \cdots, D'_s \in \mathbf{D}$, respectively:
\begin{align*}
    R: (D'_1 \times \cdots \times D'_s) & \rightarrow [0,1] \\
    v_{1}, \cdots, v_{s} & \mapsto R(v_{1}, \cdots, v_{s})
\end{align*}
Moreover, a rule is a composition of $l \geq 1$ atoms, $A_1, \cdots A_l$.

All rules are either associated with a non-negative weight and a value $q \in \{1,2\}$, or are unweighted. 
The weight and value $q$ (or absence of) of a rule determine the structure of the potentials the rule instantiates.
A weighted rule is known as a \textbf{soft rule}, and an unweighted rule is known as a \textbf{hard rule}.
The value $q$ is used as an exponent in potentials.

A \textbf{logical rule} is expressed as a logical implication of atoms. 

An \textbf{arithmetic rule} is expressed as a linear inequality of atoms.

A \textbf{ground rule} is a rule with constant arguments, i.e., a rule with only ground atoms. \qed
\end{definition}

For instance, the following is an example of two rules for solving visual Sudoku with NeuPSL.
{\small
\begin{align*}
    1.0: \quad & \pslpred{NeuralClassifier}(\pslarg{Pos}, \pslarg{Digit}) = \pslpred{CellDigit}(\pslarg{Pos}, \pslarg{Digit}) \\ 
    & \pslpred{CellDigit}(\pslarg{Pos1}, \pslarg{Digit1}) \psland
    \pslpred{SudokuViolation}(\pslarg{Pos1}, \pslarg{Pos2}, \pslarg{Digit1}, \pslarg{Digit2}) \\ 
    & \quad \pslthen \pslneg \pslpred{CellDigit}(\pslarg{Pos2}, \pslarg{Digit2}) \quad .
\end{align*}
}%
The first rule in the example above is soft as it is weighted with weight $1.0$.
Moreover, the first rule is arithmetic and encodes a dependency between the digit label predicted by a neural classifier and the atom $\pslpred{CellDigit}(\pslarg{Pos}, \pslarg{Digit})$, i.e., if the neural classifier predicts the digit $\pslarg{Digit}$ is in position $\pslarg{Pos}$, then the $\pslarg{Digit}$ is in position $\pslarg{Pos}$.
The second rule is a hard, logical rule that encodes the rules of Sudoku.
Moreover, the second rule can be read as follows: if the digit $\pslarg{Digit1}$ is in position $\pslarg{Pos1}$ and the $\pslarg{Digit2}$ in $\pslarg{Pos2}$ causes a Sudoku rule violation, then $\pslarg{Digit2}$ is not in $\pslarg{Pos2}$.

Rules are grounded by performing every distinct substitution of the variables in the atoms for constants in their respective domain.
For example, every substitution for the $\pslarg{Pos}$ and $\pslarg{Digit}$ variable arguments from the domains of non-empty Sudoku puzzle cell positions, $A1, \cdots, I9$, and digits $1, \cdots, 9$ is realized to ground the first rule:
\begin{align*}
    1.0: \quad & \pslpred{NeuralClassifier}(``A1", ``1") = \pslpred{CellDigit}(``A1", ``1")\\ 
    \vdots \quad & \\
    1.0: \quad & \pslpred{NeuralClassifier}(``19", ``9") = \pslpred{CellDigit}(``I9", ``9")\\ 
\end{align*}
Similarly, every substitution for the $\pslarg{Pos1}$, $\pslarg{Pos2}$, $\pslarg{Digit1}$, and $\pslarg{Digit2}$ variable arguments from the domains of all Sudoku puzzle cell positions, $A1, \cdots, I9$, and digits $1, \cdots, 9$ is realized to ground the second rule:
\begin{align*}
    & \pslpred{CellDigit}(``A1", ``1") \psland
    \pslpred{SudokuViolation}(``A1", ``A2", ``1", ``1") \\ 
    & \quad \pslthen \pslneg \pslpred{CellDigit}(``A2", ``1") \quad .\\
    & \quad \vdots \\
    & \pslpred{CellDigit}(``I9", ``9") \psland
    \pslpred{SudokuViolation}(``I9", ``I8", ``9", ``9") \\ 
    & \quad \pslthen \pslneg \pslpred{CellDigit}(``I8", ``9") \quad .
\end{align*}

%% file: sections/neupsl/deep-hinge-loss-markov-random-fields.tex
\subsection{Deep-Hinge Loss Markov Random Fields}
The rule instantiation process described in the previous subsection results in a set of ground atoms.
Each ground atom is mapped to either an observed variable, $x_{sy,i}$, target variable, $y_{i}$, or a neural function with inputs $\mathbf{x}_{nn}$ and parameters $\mathbf{w}_{nn, i}$: $g_{nn, i}(\mathbf{x}_{nn}, \mathbf{w}_{nn, i})$.
Specifically, all atoms instantiated from a deep atom are mapped to a neural function, and the observed and target atom partitions are pre-specified. 
Further, variables are aggregated into the vectors $\mathbf{x}_{sy}=[x_{sy_i}]_{i = 1}^{n_{x}}$ and $\mathbf{y}=[y_{i}]_{i = 1}^{n_y}$ and neural outputs are aggregated into the vector $\mathbf{g}_{nn} = [g_{nn, i}]_{i = 1}^{n_{g}}$.

The ground rules and variables are used to define linear inequalities in a standard form: $\ell(\mathbf{y}, \mathbf{x}_{sy}, \mathbf{g}_{nn}(\mathbf{x}_{nn}, \mathbf{w}_{nn})) \leq 0$, where $\ell$ is a linear function of its arguments.
To achieve this, logical rules are first converted into disjunctive normal form.
Then, the rules are translated into linear inequalities using an extended interpretation of the logical operators, namely Łukasiewicz logic \citep{klir:book95}.
Similarly, arithmetic rules define one or more standard form inequalities that preserve the rules' dependencies via algebraic operations.

Linear inequalities instantiated from hard ground rules are constraints in NeuPSL.
Further, linear inequalities instantiated from soft ground rules define \emph{potential functions} of the form:
\begin{align}
    \phi(\mathbf{y}, \mathbf{x}_{sy}, \mathbf{g}_{nn}(\mathbf{x}_{nn}, \mathbf{w}_{nn})) := (\max\{\ell(\mathbf{y}, \mathbf{x}_{sy}, \mathbf{g}_{nn}(\mathbf{x}_{nn}, \mathbf{w}_{nn})), 0\})^{q}.
    \label{eq:deep_hlmrf_potential_1}
\end{align}
Intuitively, the value of potential is the, possibly squared, level of dissatisfaction of the linear inequality created by the ground rule.
Further, each potential is associated with the weight of its instantiating rule.
Weight sharing among the potentials is formalized by defining a partitioning using the instantiating rules, i.e., every potential instantiated by the same rule belongs to the same partition and shares a weight.
The potentials and weights from the instantiation process are used to define a tractable class of graphical models, which we refer to as \emph{deep hinge-loss Markov random fields} (Deep HL-MRF):

\begin{definition}[Deep Hinge-Loss Markov Random Field]
    Let $\mathbf{g}_{nn} = [g_{nn, i}]_{i = 1}^{n_{g}}$ be functions with corresponding weights $\mathbf{w}_{nn}=[\mathbf{w}_{nn, i}]_{i = 1}^{n_{g}}$ and inputs $\mathbf{x}_{nn}$ such that $g_{nn, i}:(\mathbf{w}_{nn, i}, \mathbf{x}_{nn}) \mapsto [0, 1]$.
    Let $\mathbf{y} \in [0, 1]^{n_{y}}$ and $\mathbf{x}_{sy} \in [0, 1]^{n_{x}}$.
    A \textbf{deep hinge-loss potential} is a function of the form:
    {\small
    \begin{align}
        & \phi(\mathbf{y}, \mathbf{x}_{sy}, \mathbf{g}_{nn}(\mathbf{x}_{nn}, \mathbf{w}_{nn})) := (\max\{\mathbf{a}_{\phi, \mathbf{y}}^T \mathbf{y} + \mathbf{a}_{\phi, \mathbf{x}_{sy}}^T \mathbf{x}_{sy} + \mathbf{a}_{\phi, \mathbf{g}_{nn}}^T \mathbf{g}_{nn}(\mathbf{x}_{nn}, \mathbf{w}_{nn}) + b_{\phi}, 0\})^{q}
        \label{eq:deep_hlmrf_potential}
    \end{align}
    }%
    where $\mathbf{a}_{\phi, \mathbf{y}} \in \mathbb{R}^{n_{y}}$, $\mathbf{a}_{\phi, \mathbf{x}_{sy}} \in \mathbb{R}^{n_{x}}$, and $\mathbf{a}_{\phi, \mathbf{g}_{nn}} \in \mathbb{R}^{n_{g}}$ are variable coefficient vectors, $b_{\phi} \in \mathbb{R}$ is a vector of constants, and $q \in \{1, 2\}$.
    Let $\mathcal{T} = [\tau_i]_{i = 1}^{r}$ denote an ordered partition of a set of $m$ deep hinge-loss potentials. 
    Further, define 
    \begin{align}
        \mathbf{\Phi}(\mathbf{y} , \mathbf{x}_{sy}, \mathbf{g}_{nn}(\mathbf{x}_{nn}, \mathbf{w}_{nn})) := \left[ \sum_{k \in \tau_i} \phi_{k}(\mathbf{y}, \mathbf{x}_{sy}, \mathbf{g}_{nn}(\mathbf{x}_{nn}, \mathbf{w}_{nn})) \right]_{i = 1}^{r}.
    \end{align}
    Let $\mathbf{w}_{sy}$ be a vector of $r$ non-negative symbolic weights corresponding to the partition $\mathcal{T}$.
    Then, a \textbf{deep hinge-loss energy function} is:
    \begin{align}
        \label{eq:deep_hlmrf_energy_function}
        & E(\mathbf{y}, \mathbf{x}_{sy}, \mathbf{x}_{nn}, \mathbf{w}_{sy}, \mathbf{w}_{nn}) := \mathbf{w}_{sy}^T \mathbf{\Phi}(\mathbf{y} , \mathbf{x}_{sy}, \mathbf{g}_{nn}(\mathbf{x}_{nn}, \mathbf{w}_{nn})).
    \end{align}
    Let $\mathbf{a}_{c_k, \mathbf{y}} \in \mathbb{R}^{n_{y}}$, $\mathbf{a}_{c_k, \mathbf{x}_{sy}} \in \mathbb{R}^{n_{x}}$, $\mathbf{a}_{c_k, \mathbf{g}_{nn}} \in \mathbb{R}^{n_{g}}$, and $b_{c_k} \in \mathbb{R}$ for each $k \in {1, \dotsc, q}$ and $q \geq 0$ be vectors defining linear inequality constraints and a feasible set:
    {\small
    \begin{align*}
    & \mathbf{\Omega}(\mathbf{x}_{sy}, \mathbf{g}_{nn}(\mathbf{x}_{nn}, \mathbf{w}_{nn})) := \\
    & \quad \Big \{
    \mathbf{y} \in [0, 1]^{n_y} \, \vert \, 
            \mathbf{a}_{c_k, \mathbf{y}}^T \mathbf{y} + \mathbf{a}_{c_k, \mathbf{x}_{sy}}^T \mathbf{x}_{sy} + \mathbf{a}_{c_k, \mathbf{g}_{nn}}^T \mathbf{g}_{nn}(\mathbf{x}_{nn}, \mathbf{w}_{nn}) 
            + b_{c_k} \leq 0 
            \, , \forall \, k=1,\dotsc,q \,
    \Big \}.
    \end{align*}
    }%
    Then a \textbf{deep hinge-loss Markov random field} defines the conditional probability density:
    {
    \begin{align}
        P(\mathbf{y} \vert \mathbf{x}_{sy}, \mathbf{x}_{nn}) :=  
        \begin{cases}
            \frac{\exp (-E(\mathbf{y}, \mathbf{x}_{sy}, \mathbf{x}_{nn}, \mathbf{w}_{sy}, \mathbf{w}_{nn}))}{\int_{\hat{\mathbf{y}}} \exp(-E(\hat{\mathbf{y}}, \mathbf{x}_{sy}, \mathbf{x}_{nn}, \mathbf{w}_{sy}, \mathbf{w}_{nn})) d\hat{\mathbf{y}}} & \mathbf{y} \in \mathbf{\Omega}(\mathbf{x}_{sy}, \mathbf{g}_{nn}(\mathbf{x}_{nn}, \mathbf{w}_{nn})) \\
            0 & o.w.
        \end{cases}
    \end{align}
    \label{def:deep_hlmrf}
    }%
    \qed
\end{definition}
NeuPSL models are NeSy-EBMs with an extended-value deep HL-MRF energy function capturing the constraints that define the feasible set.
In other words, the symbolic component of NeuPSL is infinity if the targets are outside of the deep HL-MRF feasible set, else it is equal to the deep HL-MRF energy function:
{
\begin{align}
    & g_{sy}(\mathbf{y}, \mathbf{x}_{sy}, \mathbf{w}_{sy}, \mathbf{g}_{nn}(\mathbf{x}_{nn}, \mathbf{w}_{nn})) \\ 
    & \quad = \begin{cases}
       \mathbf{w}_{sy}^T \mathbf{\Phi}(\mathbf{y} , \mathbf{x}_{sy}, \mathbf{g}_{nn}(\mathbf{x}_{nn}, \mathbf{w}_{nn})) & \mathbf{y} \in \mathbf{\Omega}(\mathbf{x}_{sy}, \mathbf{g}_{nn}(\mathbf{x}_{nn}, \mathbf{w}_{nn})) \\
       \infty & o.w.
    \end{cases} \nonumber
    \label{eq:neupsl_symbolic_component}
\end{align}
}%
Further, NeuPSL prediction is finding the MAP state of the deep HL-MRF conditional distribution.
Note that in deep HL-MRFs, the partition function is constant over the target variables.
Moreover, as the exponential function is monotonically increasing, prediction is equivalent to finding the minimizer of the negative log probability of the deep HL-MRF joint distribution.
This reduces to minimizing the deep HL-MRF energy function constrained to the feasible set.
Therefore, deep HL-MRF MAP inference is equivalent to minimizing the NeuPSL symbolic component in \eqref{eq:neupsl_symbolic_component}:
\begin{align}
    \argmax_{\mathbf{y} \in \mathbb{R}^{n_{\mathbf{y}}}} P(\mathbf{y} \vert \mathbf{x}_{sy}, \mathbf{x}_{nn}) 
    & \equiv \argmin_{\mathbf{y} \in \mathbb{R}^{n_{\mathbf{y}}}} g_{sy}(\mathbf{y}, \mathbf{x}_{sy}, \mathbf{w}_{sy}, \mathbf{g}_{nn}(\mathbf{x}_{nn}, \mathbf{w}_{nn})) \\
    & \equiv \argmin_{\mathbf{y} \in \mathbb{R}^{n_{\mathbf{y}}}} \mathbf{w}_{sy}^T \mathbf{\Phi}(\mathbf{y} , \mathbf{x}_{sy}, \mathbf{g}_{nn}(\mathbf{x}_{nn}, \mathbf{w}_{nn})) \nonumber \\
    & \quad \quad \textrm{s.t.} \quad \mathbf{y} \in \mathbf{\Omega}(\mathbf{x}_{sy}, \mathbf{g}_{nn}(\mathbf{x}_{nn}, \mathbf{w}_{nn})) 
    \label{eq:inference_primal}
\end{align}

Deep HL-MRF potentials are non-smooth and convex.
Thus, as Deep HL-MRF energy functions are non-negative weighted sums of the potentials, they are also non-smooth and convex.
Moreover, Deep HL-MRFs feasible sets are, by definition, convex polyhedrons.
Therefore, Deep HL-MRF inference, as defined above in \eqref{eq:inference_primal}, is a non-smooth convex linearly constrained program.
A natural extension of the definition above that is often used in practice adds support for integer constraints on the target variables. 
This change is useful in discrete problems and for leveraging hard logic semantics.
However, adding integer constraints breaks the convexity property of MAP inference.
Nevertheless, for many problems of practical scale, global minimizers or high-quality approximations of the MAP inference problem with integer constraints can be quickly found with modern solvers.   

%% file: sections/neupsl/smooth-formulation.tex
\subsection{A Smooth Formulation of Deep HL-MRF Inference}
\label{sec:lcqp_inference}

This subsection introduces a primal and dual formulation of Deep HL-MRF MAP inference as a linearly constrained convex quadratic program (LCQP) (see \appref{appendix:extended_neupsl} for details).
The primal and dual LCQP formulation has theoretical and practical advantages.
Theoretically, the new formulation will be utilized to prove the continuity and curvature properties of Deep HL-MRFs that are valuable for learning.
Practically, LCQP solvers (e.g. Gurobi \citep{gurobi:misc24}) can be employed to achieve highly efficient MAP inference.
Moreover, features of modern solvers, including support for integer constraints, can be leveraged to improve prediction.

In summary, $m$ slack variables with lower bounds and $2 \cdot n_{\mathbf{y}} + m$ linear constraints are defined to represent the target variable bounds and deep hinge-loss potentials.
All $2 \cdot n_{\mathbf{y}} + m$ variable bounds, $m$ potentials, and $q \geq 0$ constraints are collected into a $(2 \cdot n_{\mathbf{y}} + q + 2 \cdot m) \times (n_{\mathbf{y}} + m)$ dimensional matrix $\mathbf{A}$ and a vector of $(2 \cdot n_{\mathbf{y}} + q + 2 \cdot m)$ elements that is an affine function of the neural predictions and symbolic inputs $\mathbf{b}(\mathbf{x}_{sy}, \mathbf{g}_{nn}(\mathbf{x}_{nn}, \mathbf{w}_{nn}))$.
Moreover, the slack variables and a $(n_{\mathbf{y}} + m) \times (n_{\mathbf{y}} + m)$ positive semi-definite diagonal matrix, $\mathbf{D}(\mathbf{w}_{sy})$, and a $(n_{\mathbf{y}} + m)$ dimensional vector, $\mathbf{c}(\mathbf{w}_{sy})$, are created using the symbolic weights to define a quadratic objective.
Further, we gather the original target variables and the slack variables into a vector $\mathbf{\nu} \in \mathbb{R}^{n_{\mathbf{y}} + m}$.
Altogether, the regularized convex LCQP reformulation of Deep HL-MRF MAP inference is:
{\small
\begin{align}
    V(\mathbf{w}_{sy}, \mathbf{b}(\mathbf{x}_{sy}, \mathbf{g}_{nn}(\mathbf{x}_{nn}, \mathbf{w}_{nn}))) & := \label{eq:regularized_lcqp_primal} \\ 
    \min_{\mathbf{\nu} \in \mathbb{R}^{n_{\mathbf{y}} + m}} \, 
        \mathbf{\nu}^T (\mathbf{D}(\mathbf{w}_{sy}) + \epsilon \mathbf{I}) & \mathbf{\nu} + \mathbf{c}(\mathbf{w}_{sy})^T \mathbf{\nu}  \quad
        \textrm{s.t. } \,       
     \mathbf{A} \mathbf{\nu} + \mathbf{b}(\mathbf{x}_{sy}, \mathbf{g}_{nn}(\mathbf{x}_{nn}, \mathbf{w}_{nn})) \leq 0 \nonumber,
\end{align}
}%
where $\epsilon \geq 0$ is a scalar regularization parameter added to the diagonal of $\mathbf{D}$ to ensure strong convexity and $V(\mathbf{w}_{sy}, \mathbf{b}(\mathbf{x}_{sy}, \mathbf{g}_{nn}(\mathbf{x}_{nn}, \mathbf{w}_{nn})))$ is the \emph{optimal value-function}.
While $\epsilon > 0$ introduces a regularization to the original energy function, the effect on prediction performance observed in practice is small in the range from $0.01$ to $10$, with less than $1\%$ variation on benchmark tasks.
Moreover, as $\epsilon$ increases, there is an observed decrease in inference runtime as the inference algorithm can find a solution meeting the stopping criterion in fewer iterations, realizing up to a $45 \times$ speedup on benchmark tasks.
The effect of the regularization on inference runtime and prediction performance is empirically studied in \appref{appendix:empirical_lcqp_regularization}.

By Slater's constraint qualification, we have strong duality when there is a feasible solution to \eqref{eq:regularized_lcqp_primal} \cite{boyd:book04}.
In this case, an optimal solution to the dual problem yields an optimal solution to the primal problem.
The Lagrange dual problem of \eqref{eq:regularized_lcqp_primal} is:
{
\begin{align}
    & \min_{\mathbf{\mu} \in \mathbb{R}_{\geq 0}^{2 \cdot (n_{\mathbf{y}} + m) + q}}
        \quad h(\mathbf{\mu}; \mathbf{w}_{sy}, \mathbf{b}(\mathbf{x}_{sy}, \mathbf{g}_{nn}(\mathbf{x}_{nn}, \mathbf{w}_{nn}))) \label{eq:dual_lcqp_inference} \\
        & \quad := \frac{1}{4} \mathbf{\mu}^T \mathbf{A} (\mathbf{D}(\mathbf{w}_{sy}) + \epsilon \mathbf{I})^{-1} \mathbf{A}^T \mathbf{\mu} \nonumber \\
        & \quad \quad + \frac{1}{2} (\mathbf{A} (\mathbf{D}(\mathbf{w}_{sy}) + \epsilon \mathbf{I})^{-1} \mathbf{c}(\mathbf{w}_{sy}) \nonumber \\
        & \quad \quad - 2 \mathbf{b}(\mathbf{x}_{sy}, \mathbf{g}_{nn}(\mathbf{x}_{nn}, \mathbf{w}_{nn})))^T \mathbf{\mu}, \nonumber
\end{align}
}%
where $\mathbf{\mu}$ is the vector of dual variables and $h(\mathbf{\mu}; \mathbf{w}_{sy}, \mathbf{b}(\mathbf{x}_{sy}, \mathbf{g}_{nn}(\mathbf{x}_{nn}, \mathbf{w}_{nn})))$ is the LCQP dual objective function.
As $(\mathbf{D}(\mathbf{w}_{sy}) + \epsilon \mathbf{I})$ is diagonal, it is easy to invert, and thus it is practical to work in the dual space and map dual to primal variables.
The dual-to-primal variable mapping is:
{
\begin{align}
    \mathbf{\nu} \gets - \frac{1}{2} (\mathbf{D}(\mathbf{w}_{sy}) + \epsilon \mathbf{I})^{-1} (\mathbf{A}^T \mathbf{\mu} + \mathbf{c}(\mathbf{w}_{sy})).
\end{align}
}%
On the other hand, the primal-to-dual mapping is more computationally expensive and requires calculating a pseudo-inverse of the constraint matrix $\mathbf{A}$.

We use the LCQP formulation in \eqref{eq:regularized_lcqp_primal} to establish continuity and curvature properties of the NeuPSL energy minimizer and the optimal value-function provided in the following theorem:
\begin{theorem}
    \label{thm:continuity_properties}
    Suppose for any setting of $\mathbf{w}_{nn} \in \mathbb{R}^{n_g}$ there is a feasible solution to NeuPSL inference \eqref{eq:regularized_lcqp_primal}.
    Further, suppose $\epsilon > 0$, $\mathbf{w}_{sy} \in \mathbb{R}_{+}^{r}$, and $\mathbf{w}_{nn} \in \mathbb{R}^{n_g}$.
    Then:
    \begin{itemize}[leftmargin=*,noitemsep,topsep=0pt]
        \item The minimizer of \eqref{eq:regularized_lcqp_primal}, $\mathbf{y}^*(\mathbf{w}_{sy}, \mathbf{w}_{nn})$, is a $O(1 / \epsilon)$ Lipschitz continuous function of $\mathbf{w}_{sy}$.
        \item $V(\mathbf{w}_{sy}, \mathbf{b}(\mathbf{x}_{sy}, \mathbf{g}_{nn} (\mathbf{x}_{nn}, \mathbf{w}_{nn})))$, is concave over $\mathbf{w}_{sy}$ and convex over $\mathbf{b}(\mathbf{x}_{sy}, \mathbf{g}_{nn} (\mathbf{x}_{nn}, \mathbf{w}_{nn}))$.
        \item $V(\mathbf{w}_{sy}, \mathbf{b}(\mathbf{x}_{sy}, \mathbf{g}_{nn} (\mathbf{x}_{nn}, \mathbf{w}_{nn})))$ is differentiable with respect to $\mathbf{w}_{sy}$. Moreover,
        {
        \begin{align*}
            \nabla_{\mathbf{w}_{sy}} V(\mathbf{w}_{sy}, & \mathbf{b}(\mathbf{x}_{sy}, \mathbf{g}_{nn} (\mathbf{x}_{nn}, \mathbf{w}_{nn}))) = \mathbf{\Phi}(\mathbf{y}^{*}(\mathbf{w}_{sy}, \mathbf{w}_{nn}), \mathbf{x}_{sy}, \mathbf{g}_{nn} (\mathbf{x}_{nn}, \mathbf{w}_{nn})).
        \end{align*}
        }%
        Furthermore, $\nabla_{\mathbf{w}_{sy}} V(\mathbf{w}_{sy}, \mathbf{b}(\mathbf{x}_{sy}, \mathbf{g}_{nn} (\mathbf{x}_{nn}, \mathbf{w}_{nn})))$ is Lipschitz continuous over $\mathbf{w}_{sy}$.
        \item If there is a feasible point $\nu$ strictly satisfying the $i'th$ inequality constraint of \eqref{eq:regularized_lcqp_primal}, i.e., $\mathbf{A}[i] \mathbf{\nu} + \mathbf{b}(\mathbf{x}_{sy}, \mathbf{g}_{nn} (\mathbf{x}_{nn}, \mathbf{w}_{nn}))[i] < 0$, then $V(\mathbf{w}_{sy}, \mathbf{b}(\mathbf{x}_{sy}, \mathbf{g}_{nn} (\mathbf{x}_{nn}, \mathbf{w}_{nn})))$ is subdifferentiable with respect to the $i'th$ constraint constant $\mathbf{b}(\mathbf{x}_{sy}, \mathbf{g}_{nn} (\mathbf{x}_{nn}, \mathbf{w}_{nn}))[i]$. Moreover, 
        {
        \begin{align*}
            \partial_{\mathbf{b}[i]} V(\mathbf{w}_{sy}, & \mathbf{b}(\mathbf{x}_{sy}, \mathbf{g}_{nn} (\mathbf{x}_{nn}, \mathbf{w}_{nn}))) 
            \\ &
            = \{\mathbf{\mu}^{*}[i] \, \vert \, \mathbf{\mu}^{*} \in \argmin_{\mathbf{\mu} \in \mathbb{R}_{\geq 0}^{2 \cdot (n_{\mathbf{y}} + m) + q}}
            h(\mathbf{\mu}; \mathbf{w}_{sy}, \mathbf{b}(\mathbf{x}_{sy}, \mathbf{g}_{nn} (\mathbf{x}_{nn}, \mathbf{w}_{nn}))) \}.
        \end{align*}
        }%
        Furthermore, if $\mathbf{g}_{nn} (\mathbf{x}_{nn}, \mathbf{w}_{nn})$ is a smooth function of $\mathbf{w}_{nn}$, then so is $\mathbf{b}(\mathbf{x}_{sy}, \mathbf{g}_{nn} (\mathbf{x}_{nn}, \mathbf{w}_{nn}))$, and the set of regular subgradients of $V(\mathbf{w}_{sy}, \mathbf{b}(\mathbf{x}_{sy}, \mathbf{g}_{nn} (\mathbf{x}_{nn}, \mathbf{w}_{nn})))$ is:
        {
        \begin{align}
            \label{eq:value_function_subgradient}
            \hat{\partial}_{\mathbf{w}_{nn}} & V(\mathbf{w}_{sy}, \mathbf{b}(\mathbf{x}_{sy}, \mathbf{g}_{nn} (\mathbf{x}_{nn}, \mathbf{w}_{nn}))) 
            \\ & 
            \supset \nabla_{\mathbf{w}_{nn}} \mathbf{b}(\mathbf{x}_{sy}, \mathbf{g}_{nn} (\mathbf{x}_{nn}, \mathbf{w}_{nn}))^T \partial_{\mathbf{b}} V(\mathbf{w}_{sy}, \mathbf{b}(\mathbf{x}_{sy}, \mathbf{g}_{nn} (\mathbf{x}_{nn}, \mathbf{w}_{nn}))) 
            \nonumber
            .
        \end{align}
        }%
    \end{itemize}
\end{theorem}
\begin{proof}
    See \appref{appendix:continuity_of_inference}.
\end{proof}

\thmref{thm:continuity_properties} establishes the continuity properties of the NeuPSL optimal value-function, complementing the results in the following section, specifically in \thmref{thm:value-function-gradient}.
Further, it provides a simple explicit form of the value-function gradient with respect to the symbolic weights and a regular subgradient with respect to the neural weights.
Thus, \thmref{thm:continuity_properties} supports the principled application of the end-to-end learning algorithms presented in the following sections for training both the symbolic and neural weights of a NeuPSL model.
Additionally, in \appref{appendix:dual_bcd}, we develop a block coordinate descent (BCD)~\citep{wright:mp15} algorithm for working directly with the dual LCQP formulation of inference that produces optimal dual variables directly yielding both optimal primal variables for inference and principled gradients for learning, all without the need to compute a pseudo-inverse of the constraint matrix.

%% file: sections/a-suite-of-learning-techniques-for-nesy/introduction.tex
\section{A Suite of Learning Techniques for NeSy}
\label{sec:nesy-ebm-learning}
Having identified a variety of modeling and inference paradigms, we turn to learning.
This section formalizes the NeSy-EBM learning problem, identifies challenges, and proposes effective solutions.
At a high level, NeSy-EBM learning is finding weights of an energy function that associates higher compatibility scores (lower energy) to targets and neural outputs near their true labels provided in training data.
Further, predictions with NeSy-EBMs are obtained by minimizing a complex mathematical program, raising several obstacles to learning.
For instance, NeSy-EBM predictions may not be differentiable with respect to the model parameters, and a direct application of automatic differentiation may not be possible or may fail to produce principled descent directions for the learning objective.
Moreover, we will show that even when predictions are differentiable, their gradients are functions of properties of the energy function at its minimizer that are prohibitively expensive to compute. 
We create general and principled learning frameworks for NeSy-EBMs that address these challenges.

This section is organized into three subsections.
We begin with preliminary notation and a general definition of NeSy-EBM learning.
Then, we present a classification of learning losses and establish theoretical differentiability results for NeSy-EBMs.
The learning losses motivate and organize the exposition of four NeSy-EBM learning frameworks, one for learning the neural and symbolic weights separately and three for end-to-end learning.

%% file: sections/a-suite-of-learning-techniques-for-nesy/nesy-ebm-learning.tex
\subsection{NeSy-EBM Learning}

We use the following notation and general definition of NeSy-EBM learning throughout this section.
The training dataset, denoted by $\mathcal{S}$, is comprised of $P$ samples and indexed by $\{1, \cdots, P\}$.
Each sample, $\mathcal{S}_{i}$ where $i \in \{1, \cdots, P\}$, is a tuple of \emph{inputs}, \emph{labels}, and \emph{latent variable domains}.
Sample \emph{inputs} consist of neural inputs, $\mathbf{x}^{i}_{nn}$ from $\mathcal{X}_{nn}$, and symbolic inputs, $\mathbf{x}^{i}_{sy}$ from $\mathcal{X}_{sy}$.
Similarly, sample \emph{labels} consist of neural and symbolic labels, which are truth values corresponding to a subset of the neural predictions and target variables, respectively.
Neural labels, denoted by $\mathbf{t}^{i}_{nn}$, are $d^{i}_{nn} \leq d_{nn}$ dimensional real vectors from a domain $\mathcal{T}^{i}_{nn}$, i.e., $\mathbf{t}^{i}_{nn} \in \mathcal{T}^{i}_{nn} \subseteq \mathbb{R}^{d^{i}_{nn}}$.
Target labels, denoted by $\mathbf{t}^{i}_{\mathcal{Y}}$, are from a domain $\mathcal{T}^{i}_{\mathcal{Y}}$ that is a $d_{\mathcal{T}^{i}_{\mathcal{Y}}} \leq d_{\mathcal{Y}}$ subspace of the target domain $\mathcal{Y}$, i.e., $\mathbf{t}^{i}_{\mathcal{Y}} \in \mathcal{T}^{i}_{\mathcal{Y}}$.
Lastly, the neural and symbolic \emph{latent variable domains} are subspaces of the range of the neural component and the target domain, respectively, corresponding to the set of unlabeled variables.
The range of the neural component, $\mathbb{R}^{d^{nn}}$, is a superset of the Cartesian product of the neural latent variable domain, denoted by $\mathcal{Z}^{i}_{nn}$, and $\mathcal{T}^{i}_{nn}$, i.e., $\mathbb{R}^{d^{nn}} \supseteq \mathcal{T}^{i}_{nn} \times \mathcal{Z}^{i}_{nn}$.
Similarly, the target domain $\mathcal{Y}$ is a superset of the Cartesian product of the latent variable domain, denoted by $\mathcal{Z}^{i}_{\mathcal{Y}}$, and $\mathcal{T}^{i}_{\mathcal{Y}}$, i.e., $\mathcal{Y} \supseteq \mathcal{T}^{i}_{\mathcal{Y}} \times \mathcal{Z}^{i}_{\mathcal{Y}}$.
With this notation, the training dataset is expressed as follows:
{
\begin{align}
    \mathcal{S} := \{ & (\mathbf{t}_{\mathcal{Y}}^{1}, \mathbf{t}_{nn}^{1}, \mathcal{Z}^{1}_{nn}, \mathcal{Z}^{1}_{\mathcal{Y}}, \mathbf{x}^{1}_{sy}, \mathbf{x}^{1}_{nn}), \cdots, (\mathbf{t}_{\mathcal{Y}}^{P}, \mathbf{t}_{nn}^{P}, \mathcal{Z}^{P}_{nn}, \mathcal{Z}^{P}_{\mathcal{Y}}, \mathbf{x}^{P}_{sy}, \mathbf{x}^{P}_{nn})\}.
\end{align}
}%

A learning objective, denoted by $\mathcal{L}$, is a functional that maps an energy function and a training dataset to a scalar value.
Formally, let $\mathcal{E}$ be a family of energy functions indexed by weights from $\mathcal{W}_{sy} \times \mathcal{W}_{nn}$:
{
\begin{align}
    \mathcal{E} := \{ E(\cdot, \cdot, \cdot, \mathbf{w}_{sy}, \mathbf{w}_{nn}) \, \vert \, (\mathbf{w}_{sy}, \mathbf{w}_{nn}) \in \mathcal{W}_{sy} \times \mathcal{W}_{nn} \}.
\end{align}
}%
Then, a learning objective is the function:
\begin{align}
    \mathcal{L} : \mathcal{E} \times \{\mathcal{S}\} \to \mathbb{R}.
\end{align}
We study a standard empirical risk minimization formulation in which the learning objective decomposes over elements of $\mathcal{S}$ as a sum of \emph{per-sample loss functionals} denoted by $L^{i}$ for each $i \in \{1, \cdots, P\}$.
This decomposition is at the dataset level and does not imply that target variables within a sample are independent.
Therefore, dependencies within a sample do not violate empirical risk minimization formulation.
However, learning objectives that do not decompose into a sum of per-sample losses are outside the framework studied in this section.
A loss functional for the sample $\mathcal{S}_{i} \in \mathcal{S}$ is the function: 
\begin{align}
    L^{i} : \mathcal{E} \times \{\mathcal{S}_{i}\} \to \mathbb{R}.
\end{align}
A regularizer, denoted by $\mathcal{R}: \mathcal{W}_{sy} \times \mathcal{W}_{nn} \to \mathbb{R}$, is added to the learning objective and NeSy-EBM learning is the following minimization problem:
{
\begin{align}
    & \argmin_{(\mathbf{w}_{sy}, \mathbf{w}_{nn}) \in \mathcal{W}_{sy} \times \mathcal{W}_{nn}} \mathcal{L}(E(\cdot, \cdot, \cdot, \mathbf{w}_{sy}, \mathbf{w}_{nn}), \mathcal{S}) + \mathcal{R}(\mathbf{w}_{sy}, \mathbf{w}_{nn}) \label{eq:nesy-ebm-learning} \\
    & = \argmin_{(\mathbf{w}_{sy}, \mathbf{w}_{nn}) \in \mathcal{W}_{sy} \times \mathcal{W}_{nn}} \frac{1}{P} \sum_{i = 1}^{P} L^{i}(E(\cdot, \cdot, \cdot, \mathbf{w}_{sy}, \mathbf{w}_{nn}), \mathcal{S}_{i}) + \mathcal{R}(\mathbf{w}_{sy}, \mathbf{w}_{nn}) .
    \nonumber
\end{align}
}%

%% file: sections/a-suite-of-learning-techniques-for-nesy/learning-losses.tex
\subsection{Learning Losses}
\label{sec:learning-losses}

A NeSy-EBM learning loss functional, $L^{i}$, is separable into three parts: \emph{neural}, \emph{value-based}, and \emph{minimizer-based} losses.
In this subsection, we formally define each of the three loss types.
At a high level, the neural loss measures the quality of the neural component independent from the symbolic component. 
Then, the value-based and minimizer-based losses measure the quality of the NeSy-EBM as a whole.
Moreover, value-based and minimizer-based losses are functionals mapping a parameterized energy function and a training sample to a real value and are denoted by $L_{Val}: \mathcal{E} \times \mathcal{S} \to \mathbb{R}$ and $L_{Min}: \mathcal{E} \times \mathcal{S} \to \mathbb{R}$, respectively.
The learning loss components are aggregated via a weighted sum with scalar coefficients $\lambda_{1}, \lambda_{2}, \lambda_{3} \in \mathbb{R}$:

\begin{align}
    L^{i}(E(\cdot, & \cdot, \cdot, \mathbf{w}_{sy}, \mathbf{w}_{nn}), \mathcal{S}_{i})  \\
    = & \quad \lambda_{1} \cdot L_{NN}(\mathbf{g}_{nn}(\mathbf{x}^{i}_{nn}, \mathbf{w}_{nn}), \mathbf{t}^{i}_{nn}) & \textrm{Neural} \nonumber \\
    & \quad + \lambda_{2} \cdot L_{Val}(E(\cdot, \cdot, \cdot, \mathbf{w}_{sy}, \mathbf{w}_{nn}), \mathcal{S}_{i}) & \textrm{Value-Based} \nonumber \\
    & \quad + \lambda_{3} \cdot L_{Min}(E(\cdot, \cdot, \cdot, \mathbf{w}_{sy}, \mathbf{w}_{nn}), \mathcal{S}_{i}) & \textrm{Minimizer-Based} \nonumber 
\end{align}

\subsubsection{Neural Learning Losses}
\textit{Neural} learning losses are scalar functions of the neural network output and the neural labels and are denoted by $L_{NN}: \textrm{Range}(\mathbf{g}_{nn}) \times \mathcal{T}^{i}_{nn} \to \mathbb{R}$.
For example, a neural learning loss may be the familiar binary cross-entropy loss applied in many categorical prediction settings.
Minimizing a neural learning loss with respect to neural component parameters is achievable via backpropagation and standard gradient-based algorithms.

\subsubsection{Value-Based Learning Losses}
\emph{Value-based} learning losses depend on the model weights strictly via minimizing values of an objective defined with the energy.
More formally, denote an objective function by $f$, which maps a compatibility score, target variables, and the training sample to a scalar value:
\begin{align}
    f: \mathbb{R} \times \mathcal{Y} \times \{\mathcal{S}_{i}\} \to \mathbb{R}.
\end{align}
An \emph{optimal value-function}, denoted by $V$, is the value of $f$ composed with the energy function and minimized over the target variables:
\begin{align}
    V(\mathbf{w}_{sy}, \mathbf{w}_{nn}, \mathcal{S}_{i}) 
    & := \min_{\hat{\mathbf{y}} \in \mathcal{Y}} f \left( E(\hat{\mathbf{y}}, \mathbf{x}^{i}_{sy}, \mathbf{x}^{i}_{nn}, \mathbf{w}_{sy}, \mathbf{w}_{nn}), \hat{\mathbf{y}}, \mathcal{S}_{i} \right) \nonumber \\
    & := \min_{\hat{\mathbf{y}} \in \mathcal{Y}} f(g_{sy}(\hat{\mathbf{y}}, \mathbf{x}_{sy}^{i}, \mathbf{w}_{sy}, \mathbf{g}_{nn}(\mathbf{x}_{nn}^{i}, \mathbf{w}_{nn})), \hat{\mathbf{y}}, S_{i}) \label{eq:optimal-value-function}
\end{align}
The minimization in \eqref{eq:optimal-value-function} is over the full target vector $\hat{\mathbf{y}}$, so value-based losses can preserve within-sample dependencies and joint reasoning even when the training objective is averaged over samples.
Value-based learning losses are functions of one or more optimal value functions.
In this work, we consider three instances of optimal value functions: 1) \emph{latent}, $V_{\mathcal{Z}}$, 2) \emph{full}, $V_{\mathcal{Y}}$, 3) and \emph{convolutional}, $V_{conv}$.
The latent optimal value function is the minimizing value of the energy over the latent targets. 
Further, the labeled targets are fixed to their true values using the following indicator function:
{
\begin{align}
    \label{eq:deep-symbolic-targets-indicator}
    & I_{\mathcal{T}^{i}_{\mathcal{Y}}}(\mathbf{y}, \mathbf{t}^{i}_{\mathcal{Y}}) := \begin{cases}
        0 & \mathbf{y} = \mathbf{t}^{i}_{\mathcal{Y}} \\
        \infty & \textrm{o.w.}
    \end{cases}.
\end{align}
}%
The full optimal value function is the minimizing value of the energy over all of the targets.
Lastly, the convolutional optimal value function is the infimal convolution of the energy function and a function $d: \mathcal{Y} \times \mathcal{Y} \to \mathcal{R}$ scaled by a positive real value $\lambda \in \mathcal{R}$.
Formally:
{\small
\begin{align}
    V_{\mathcal{Z}}(\mathbf{w}_{sy}, \mathbf{w}_{nn}, \mathcal{S}_{i}) 
    := & \min_{\hat{\mathbf{y}} \in \mathcal{Y}} E(\hat{\mathbf{y}}, \mathbf{x}^{i}_{sy}, \mathbf{x}^{i}_{nn}, \mathbf{w}_{sy}, \mathbf{w}_{nn}) + I_{\mathcal{T}^{i}_{\mathcal{Y}}}(\hat{\mathbf{y}}, \mathbf{t}^{i}_{\mathcal{Y}}), & \nonumber \\
    \label{eq:latent-optimal-value-function}
    = & \min_{\hat{\mathbf{z}} \in \mathcal{Z}^{i}_{\mathcal{Y}}} E((\mathbf{t}^{i}_{\mathcal{Y}}, \hat{\mathbf{z}}), \mathbf{x}^{i}_{sy}, \mathbf{x}^{i}_{nn}, \mathbf{w}_{sy}, \mathbf{w}_{nn}), & \textrm{latent} \\
    \label{eq:full-optimal-value-function}
    V_{\mathcal{Y}}(\mathbf{w}_{sy}, \mathbf{w}_{nn}, \mathcal{S}_{i}) 
    := & \min_{\hat{\mathbf{y}} \in \mathcal{Y}} E(\hat{\mathbf{y}}, \mathbf{x}^{i}_{sy}, \mathbf{x}^{i}_{nn}, \mathbf{w}_{sy}, \mathbf{w}_{nn}), & \textrm{full} \\
    \label{eq:conv-optimal-value-function}
    V_{conv}(\mathbf{w}_{sy}, \mathbf{w}_{nn}, \mathcal{S}_{i}; \mathbf{y}, \lambda) 
    := & \min_{\hat{\mathbf{y}} \in \mathcal{Y}} E(\hat{\mathbf{y}}, \mathbf{x}^{i}_{sy}, \mathbf{x}^{i}_{nn}, \mathbf{w}_{sy}, \mathbf{w}_{nn}) + \lambda \cdot d(\hat{\mathbf{y}}, \mathbf{y}). & \textrm{convolutional}
\end{align}
}%
An illustration of an example latent optimal value-function is provided in \figref{fig:latent-optimal-value-function}.
Intuitively, the latent optimal value-function is the greatest lower bound of the set of symbolic components defined for each latent variable.

\begin{figure}
    \centering
    \includegraphics[width=0.5 \textwidth]{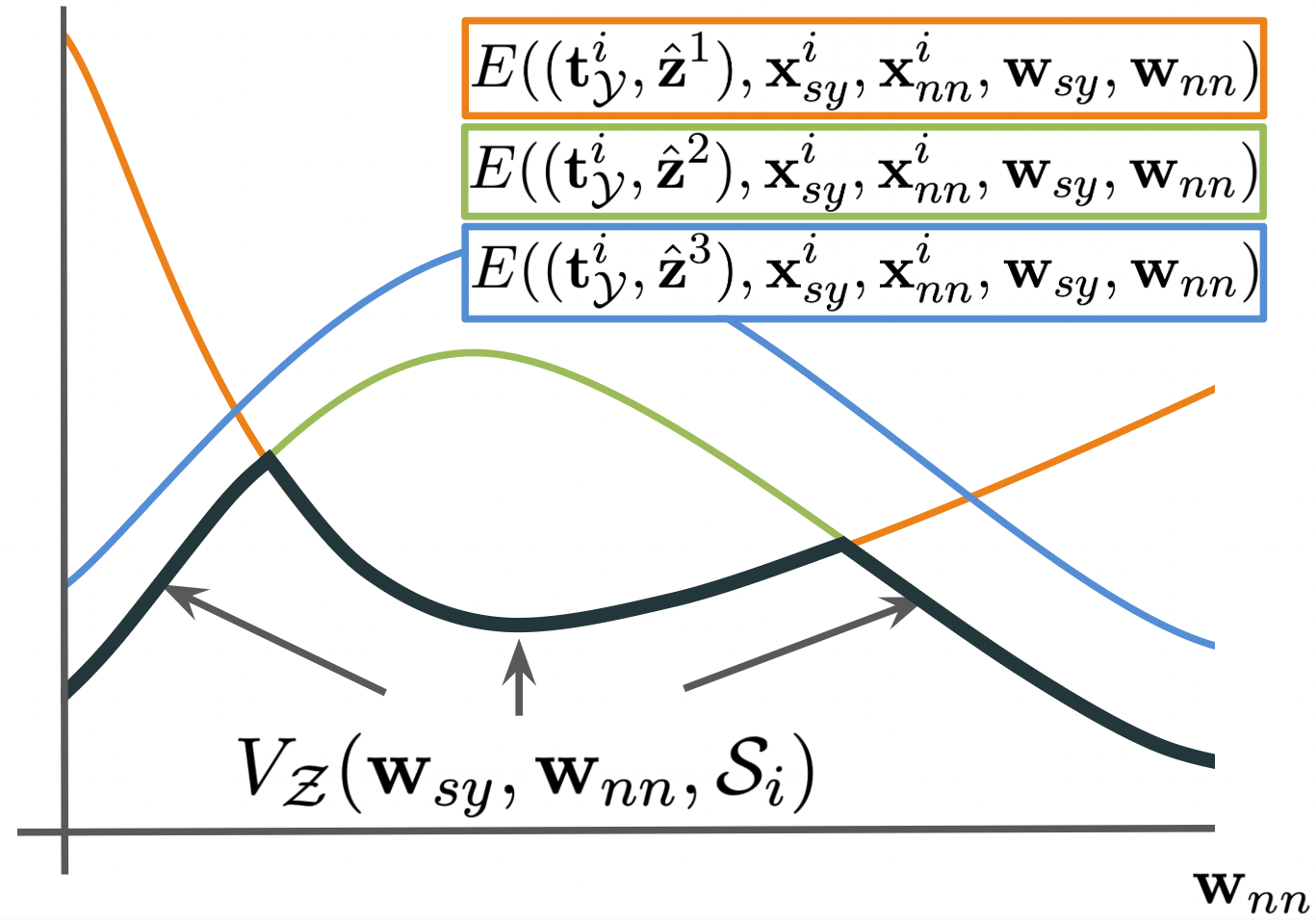}
    \caption{An illustrated example of a latent optimal value-function with a scalar neural component output and a discrete latent variable domain $\mathcal{Z} := \{\hat{\mathbf{z}}^{1}, \hat{\mathbf{z}}^{2}, \hat{\mathbf{z}}^{3}\}$.}
    \label{fig:latent-optimal-value-function}
\end{figure}

The simplest value-based learning loss is the \emph{energy loss}, denoted by $L_{Energy}$.
The energy loss is the latent optimal value function,
\begin{align}
    \label{eq:energy-loss}
    L_{Energy}(E(\cdot, \cdot, \cdot, \mathbf{w}_{sy}, \mathbf{w}_{nn}), \mathcal{S}_{i}) := V_{\mathcal{Z}}(\mathbf{w}_{sy}, \mathbf{w}_{nn}, \mathcal{S}_{i}).
\end{align}
Minimizing the energy loss encourages the parameters of the energy function to produce low energies given the observed true values of the input and target variables.
This loss is motivated by the intuition that the energy should be low for the desired values of the targets.
Notice, however, that the loss does not consider the energy of incorrect target variable values.
An extreme illustration of the issue this causes involves two energy functions.
In the first function, the minimizing point corresponds to the desired true values of the targets, while in the second function, the maximizing point corresponds to the desired true values of the targets.
Despite these differences, both functions could technically have the same energy loss; however, the first energy function is clearly preferred.
Thus, the energy loss does not universally lead to energy functions with better predictions.

The \emph{Structured Perceptron} loss, denoted by $L_{SP}$, pushes the energy of the current energy minimizer up and the energy of the true values of the targets down \citep{lecun:ieee98, collins:emnlp02}. 
Specifically, the structured perceptron loss is the difference between the latent and full optimal value functions, 
\begin{align}
    \label{eq:structured-perceptron}
    L_{SP}(E(\cdot, \cdot, \cdot, \mathbf{w}_{sy}, \mathbf{w}_{nn}), \mathcal{S}_{i}) := V_{\mathcal{Z}}(\mathbf{w}_{sy}, \mathbf{w}_{nn}, \mathcal{S}_{i}) - V_{\mathcal{Y}}(\mathbf{w}_{sy}, \mathbf{w}_{nn}, \mathcal{S}_{i}). 
\end{align}
Although the structured perceptron loss will technically encourage the target's desired values to be an energy minimizer, i.e., a valid prediction, it still has degenerate solutions for some energy function architectures.
For instance, one could minimize the energy for all target values, leading to a collapsed energy function (equal energy for all targets) with no predictive power.

The energy and structured perceptron losses require regularization and specific energy architectures to work well in practice.
For instance, such energy architectures naturally push up on other target values when pushing down on the desired targets
Energies with limited total energy mass are examples of functions with this property.

The gradient of a value-based loss with respect to neural and symbolic weights is non-trivial since both the energy function and the point the energy function is evaluated at are dependent on the neural output and symbolic weights, as exemplified by the definition of an optimal value function in \eqref{eq:optimal-value-function}.
Nonetheless, \citenoun{milgrom:econ02} delivers a general theorem providing the gradient of optimal value-functions with respect to problem parameters, if they exist.
We specialize their result in the following theorem for optimal value-functions of NeSy-EBMs.

\begin{theorem}[\citenoun{milgrom:econ02} Theorem 1 for NeSy-EBMs]
    \label{thm:value-function-gradient}
    Consider the weights $\mathbf{w}_{sy} \in \mathcal{W}_{sy}$ and $\mathbf{w}_{nn} \in \mathcal{W}_{nn}$ and the sample $\mathcal{S}_{i} = (\mathbf{t}_{\mathbf{y}}^{i}, \mathbf{t}_{nn}^{i}, \mathcal{Z}^{i}_{nn}, \mathcal{Z}^{i}_{\mathcal{Y}}, \mathbf{x}^{i}_{sy}, \mathbf{x}^{i}_{nn}) \in \mathcal{S}$.
    Suppose there exists a minimizer of the objective function $f$, 
    {
    \begin{align*}
        \mathbf{y}^{*} \in \argmin_{\hat{\mathbf{y}} \in \mathcal{Y}} f(E(\hat{\mathbf{y}}, \mathbf{x}_{sy}^{i}, \mathbf{x}_{nn}^{i}, \mathbf{w}_{sy}, \mathbf{w}_{nn}), \hat{\mathbf{y}}, \mathcal{S}_{i}),
    \end{align*}
    }%
    such that $f(E(\mathbf{y}^{*}, \mathbf{x}_{sy}^{i}, \mathbf{x}_{nn}^{i}, \mathbf{w}_{sy}, \mathbf{w}_{nn}), \mathbf{y}^{*}, \mathcal{S}_{i})$ is finite.
    
    If the optimal value-function:
    {
    \begin{align*}
        V(\mathbf{w}_{sy}, \mathbf{w}_{nn}, \mathcal{S}_{i}) & := \min_{\hat{\mathbf{y}} \in \mathcal{Y}} f(E(\hat{\mathbf{y}}, \mathbf{x}_{sy}^{i}, \mathbf{x}_{nn}^{i}, \mathbf{w}_{sy}, \mathbf{w}_{nn}), \hat{\mathbf{y}}, \mathcal{S}_{i}), \\
        & := \min_{\hat{\mathbf{y}} \in \mathcal{Y}} f(g_{sy}(\hat{\mathbf{y}}, \mathbf{x}_{sy}^{i}, \mathbf{w}_{sy}, \mathbf{g}_{nn}(\mathbf{x}_{nn}^{i}, \mathbf{w}_{nn})), \hat{\mathbf{y}}, S_{i}),
    \end{align*}
    }%
    is differentiable with respect to the neural weights, $\mathbf{w}_{nn}$, then the gradient of $V$ with respect to $\mathbf{w}_{nn}$ is:
    {
    \begin{align}
        \label{eq:value-function-gradient-wrt-neural}
        & \nabla_{\mathbf{w}_{nn}} V(\mathbf{w}_{sy}, \mathbf{w}_{nn}, \mathcal{S}_{i}) \\
        & \quad = \frac{\partial}{\partial 1} f(E(\mathbf{y}^{*}, \mathbf{x}^{i}_{sy}, \mathbf{x}^{i}_{nn}, \mathbf{w}_{sy}, \mathbf{w}_{nn}), \mathbf{y}^{*}, \mathcal{S}_{i}) \cdot \nabla_{5} E(\mathbf{y}^{*}, \mathbf{x}^{i}_{sy}, \mathbf{x}^{i}_{nn}, \mathbf{w}_{sy}, \mathbf{w}_{nn}), \nonumber
    \end{align}
    }%
    where $\frac{\partial}{\partial 1} f$ is the partial derivative of $f$ with respect to its \(1\)st argument, and $\nabla_{5} E$ is the gradient of the energy with respect to its \(5\)th argument with all other arguments fixed.

    Similarly, if $V$ is differentiable with respect to the symbolic weights, $\mathbf{w}_{sy}$, then the gradient of $V$ with respect to $\mathbf{w}_{sy}$ is:
    {
    \begin{align}
        \label{eq:value-function-gradient-wrt-symbolic}
        & \nabla_{\mathbf{w}_{sy}} V(\mathbf{w}_{sy}, \mathbf{w}_{nn}, \mathcal{S}_{i}) \\ 
        & \quad = \frac{\partial}{\partial 1} f(E(\mathbf{y}^{*}, \mathbf{x}^{i}_{sy}, \mathbf{x}^{i}_{nn}, \mathbf{w}_{sy}, \mathbf{w}_{nn}), \mathbf{y}^{*}, \mathcal{S}_{i}) \cdot \nabla_{4} E(\mathbf{y}^{*}, \mathbf{x}^{i}_{sy}, \mathbf{x}^{i}_{nn}, \mathbf{w}_{sy}, \mathbf{w}_{nn}). \nonumber
    \end{align}
    }%
\end{theorem}
\begin{proof}
    We first establish the partial derivative of the optimal value-function with respect to each component of the neural output, $\mathbf{g}_{nn}(\mathbf{x}_{nn}^{i}, \mathbf{w}_{nn})$.
    Then, we use the chain rule to derive the expression for the gradient of the optimal value-function with respect to the neural weights, $\mathbf{w}_{nn}$.
    
    For an arbitrary index $j \in \{ 1, \cdots, d_{nn} \}$, let $\mathbf{e}^{j}$ be the $j'th$ standard basis vector of $\mathbb{R}^{d_{nn}}$, i.e., $\mathbf{e}^{j} \in \mathbb{R}^{d_{nn}}$ such that $\mathbf{e}^{j}_{j} = 1$ and $\mathbf{e}^{j}_{k} = 0$ for $k \neq j$.
    Further, to clarify the relationship between the optimal value-function and the neural component output, define the following function:
    \begin{align*}
    \overline{V}: \, \mathcal{W}_{sy} \times \mathbb{R}^{d_{nn}} \times {\mathcal{S}_{i}} & \to \mathbb{R} \\
    (\mathbf{w}_{sy}, \mathbf{u}, S_{i}) & \mapsto \min_{\hat{\mathbf{y}} \in \mathcal{Y}} f(g_{sy}(\hat{\mathbf{y}}, \mathbf{x}_{sy}^{i}, \mathbf{w}_{sy}, \mathbf{u}), \hat{\mathbf{y}}, \mathcal{S}_{i})
    \end{align*}
    In other words, the optimal value-function, $V$, is equal to $\overline{V}$ evaluated at the neural output:
    \begin{align*} 
        V(\mathbf{w}_{sy}, \mathbf{w}_{nn}, \mathcal{S}_{i}) \triangleq \overline{V}(\mathbf{w}_{sy}, \mathbf{g}_{nn}(\mathbf{x}_{nn}^{i}, \mathbf{w}_{nn}), \mathcal{S}_{i}).
    \end{align*}
   
    For any $\delta \in \mathbb{R}$, by definition we have:
    {
    \begin{align*}
        & f(g_{sy}(\mathbf{y}^{*}, \mathbf{x}_{sy}^{i}, \mathbf{w}_{sy}, \mathbf{g}_{nn}(\mathbf{x}_{nn}^{i}, \mathbf{w}_{nn}) + \delta \mathbf{e}^{j}), \mathbf{y}^{*}, S_{i}) \\
        & \quad - f(g_{sy}(\mathbf{y}^{*}, \mathbf{x}_{sy}^{i}, \mathbf{w}_{sy}, \mathbf{g}_{nn}(\mathbf{x}_{nn}^{i}, \mathbf{w}_{nn})), \mathbf{y}^{*}, S_{i}) \\
         & \geq \overline{V}(\mathbf{w}_{sy}, \mathbf{g}_{nn}(\mathbf{x}_{nn}^{i}, \mathbf{w}_{nn}) + \delta \mathbf{e}^{j}, S_{i}) - \overline{V}(\mathbf{w}_{sy}, \mathbf{g}_{nn}(\mathbf{x}_{nn}^{i}, \mathbf{w}_{nn}), S_{i}).
    \end{align*}    
    }%
    For $\delta \neq 0$, dividing both sides by $\delta$ and taking the limit as $\delta \to 0+$ and as $\delta \to 0-$ yields upper and lower bounds relating partial derivatives of $f$ to $\overline{V}$ when $f$ and $\overline{V}$ are right and left hand differentiable, respectively.
    {
    \begin{align*}
        &
        \frac{\partial}{\partial \mathbf{g}_{nn}(\mathbf{x}_{nn}^{i}, \mathbf{w}_{nn})_{j}+} f(g_{sy}(\mathbf{y}^{*}, \mathbf{x}_{sy}^{i}, \mathbf{w}_{sy}, \mathbf{g}_{nn}(\mathbf{x}_{nn}^{i}, \mathbf{w}_{nn})), \mathbf{y}^{*}, S_{i}) \\
        & \quad \geq 
        \frac{\partial}{\partial \mathbf{g}_{nn}(\mathbf{x}_{nn}^{i}, \mathbf{w}_{nn})_{j}+} \overline{V}(\mathbf{w}_{sy}, \mathbf{g}_{nn}(\mathbf{x}_{nn}^{i}, \mathbf{w}_{nn}), S_{i}), \\
        &
        \frac{\partial}{\partial \mathbf{g}_{nn}(\mathbf{x}_{nn}^{i}, \mathbf{w}_{nn})_{j}-} f(g_{sy}(\mathbf{y}^{*}, \mathbf{x}_{sy}^{i}, \mathbf{w}_{sy}, \mathbf{g}_{nn}(\mathbf{x}_{nn}^{i}, \mathbf{w}_{nn})), \mathbf{y}^{*}, S_{i}) \\
        & \quad \leq  
        \frac{\partial}{\partial \mathbf{g}_{nn}(\mathbf{x}_{nn}^{i}, \mathbf{w}_{nn})_{j}-} \overline{V}(\mathbf{w}_{sy}, \mathbf{g}_{nn}(\mathbf{x}_{nn}^{i}, \mathbf{w}_{nn}), S_{i}),
    \end{align*}
    }%
    Then, by the squeeze theorem, we obtain the partial derivatives of $\overline{V}$ with respect to each component of the neural output when $\overline{V}$ is differentiable.
    {
    \begin{align*}
        & \frac{\partial}{\partial \mathbf{g}_{nn}(\mathbf{x}_{nn}^{i}, \mathbf{w}_{nn})_{j}} f(g_{sy}(\mathbf{y}^{*}, \mathbf{x}_{sy}^{i}, \mathbf{w}_{sy}, \mathbf{g}_{nn}(\mathbf{x}_{nn}^{i}, \mathbf{w}_{nn})), \mathbf{y}^{*}, S_{i}) \\
        & \quad = \frac{\partial}{\partial \mathbf{g}_{nn}(\mathbf{x}_{nn}^{i}, \mathbf{w}_{nn})_{j}} \overline{V}(\mathbf{w}_{sy}, \mathbf{g}_{nn}(\mathbf{x}_{nn}^{i}, \mathbf{w}_{nn}), S_{i}).
    \end{align*}
    }%
    Then, the chain rule of differentiation and the partial derivatives of $\overline{V}$ with respect to each component of the neural output derived above yields the gradient in \eqref{eq:value-function-gradient-wrt-neural}.
    
    A similar approach is used to obtain gradients with respect to symbolic weights in \eqref{eq:value-function-gradient-wrt-symbolic}.
\end{proof}

\thmref{thm:value-function-gradient} holds for arbitrary target variable domains and energy functions and is, therefore, widely applicable.
However, it is important to emphasize that \thmref{thm:value-function-gradient} states \emph{if} the value-function is differentiable, then the gradients have the form provided in \eqref{eq:value-function-gradient-wrt-neural} and \eqref{eq:value-function-gradient-wrt-symbolic}. 
\citenoun{milgrom:econ02} also provide sufficient conditions for guaranteeing the differentiability of optimal value-functions with arbitrary decision variable domains.
Beyond \citepossessive{milgrom:econ02} work, there is extensive literature on analyzing the sensitivity of optimal value-functions and guaranteeing their differentiability, including the seminal papers of \cite{danskin:siam66} on parameterized objective functions and \cite{rockafellar:rcsam74} for parameterized constraints.
We direct the reader to the cited articles for properties that guarantee differentiability of value-functions and, hence, NeSy-EBM value-based losses.

The conditions ensuring differentiability of the optimal value-functions as well as the tractability of computing the gradient of the symbolic component with respect to its arguments in \eqref{eq:value-function-gradient-wrt-neural} and \eqref{eq:value-function-gradient-wrt-symbolic} directly connect to the energy function architecture and modeling paradigms discussed in the previous section.
Specifically, if principled gradient-based learning is desired, then practitioners must design the symbolic potential such that it is 1) differentiable with respect to the neural output and symbolic potentials, 2) the gradient of the symbolic potential with respect to its arguments is tractable, and 3) it satisfies sufficient conditions for ensuring differentiability of its minimizing value over the targets.
These requirements can fail when the neural-symbolic interface or induced inference problem is discrete, integer, or otherwise non-smooth, since in such cases the symbolic potential or the value-function may not be differentiable with respect to the neural output or weights.
For these NeSy-EBMs, principled gradient-based learning cannot rely directly on the value-function and minimizer-gradient results above.

Performance metrics are not always aligned with value-based losses.
Moreover, they are known to have degenerate solutions \citep{lecun:book06, pryor:ijcai23}.
For example, without a carefully designed inductive bias, the energy loss in \eqref{eq:energy-loss} may only learn to reduce the energy of all target variables without improving the predictive performance of the NeSy-EBM. 
One fundamental cause of this issue is that value-based losses are not directly functions of the NeSy-EBM prediction as defined in \eqref{eq:nesy-ebm-prediction}, i.e., value-based losses are not functions of an energy minimizer, which is what we turn to next.

\subsubsection{Minimizer-Based Learning Losses}
A \emph{minimizer-based} loss is a composition of a differentiable loss, such as cross-entropy or mean squared error, with the energy minimizer.
Intuitively, minimizer-based losses penalize parameters yielding predictions distant from the labeled training data.
In the remainder of this subsection, we formally define minimizer-based learning losses.
Further, for completeness, we derive general expressions for gradients of minimizer-based losses with respect to symbolic and neural weights.
However, as will be shown, direct computation of minimizer-based loss gradients requires prohibitive assumptions on the energy function and can be impractical to compute.
Moreover, the derivation of the gradients motivates learning algorithms that do not perform direct gradient descent on minimizer-based losses.
For this reason, in the following subsection we propose algorithms that do not require minimizer gradients.

To ensure a minimizer-based loss is well-defined, we assume a unique energy minimizer exists, denoted by $\mathbf{y}^{*}$, for every training sample.
This assumption is formalized below.
\begin{assumption}
    \label{assumption:unique-minimizer}
    The energy function is minimized over the targets at a single point for every input and weight and is, therefore, a function:
    \begin{align}
        \mathbf{y}^{*}: \mathcal{X}_{sy} \times \mathcal{X}_{nn} \times \mathcal{W}_{sy} \times \mathcal{W}_{nn} & \to \mathcal{Y} \\
        (\mathbf{x}_{sy}, \mathbf{x}_{nn}, \mathbf{w}_{sy}, \mathbf{w}_{nn}) & \mapsto \argmin_{\hat{\mathbf{y}} \in \mathcal{Y}} E(\hat{\mathbf{y}}, \mathbf{x}_{sy}, \mathbf{x}_{nn}, \mathbf{w}_{sy}, \mathbf{w}_{nn}) \nonumber
    \end{align}
\end{assumption}
Under \assumptionref{assumption:unique-minimizer}, $d$ is a mapping of targets and labels to a scalar value:
\begin{align}
    d: \mathcal{Y} \times \mathcal{T}_{\mathcal{Y}}^{i} \to \mathbb{R},
\end{align}
and a minimizer-based loss is a composition of $d$ and $\mathbf{y}^{*}$:
{
\begin{align}
    \label{eq:differentiable-minimizer}
    L_{Min} (E(\cdot, \cdot, \cdot, \mathbf{w}_{sy}, \mathbf{w}_{nn}), \mathcal{S}_{i}) 
    & := d(\argmin_{\hat{\mathbf{y}} \in \mathcal{Y}} E(\hat{\mathbf{y}}, \mathbf{x}^{i}_{sy}, \mathbf{x}^{i}_{nn}, \mathbf{w}_{sy}, \mathbf{w}_{nn}), \mathbf{t}^{i}_{\mathcal{Y}}) \\
    & := d(\mathbf{y}^{*}(\mathbf{x}^{i}_{sy}, \mathbf{x}^{i}_{nn}, \mathbf{w}_{sy}, \mathbf{w}_{nn}), \mathbf{t}^{i}_{\mathcal{Y}}) \nonumber 
\end{align}
}%

To ensure principled gradient-based learning, we must further assume that the minimizer is differentiable. 
\begin{assumption}
    \label{assumption:differentiable-minimizer}
    The minimizer, $\mathbf{y}^{*}$, is differentiable with respect to the weights at every point in $\mathcal{X}_{sy} \times \mathcal{X}_{nn} \times \mathcal{W}_{sy} \times \mathcal{W}_{nn}$.
\end{assumption}
Under \assumptionref{assumption:differentiable-minimizer}, the chain rule of differentiation yields the gradient of a minimizer-based loss with respect to the neural and symbolic weights:
\begin{align}
    \label{eq:minimizer-based-loss-symbolic-gradient}
    & \nabla_{\mathbf{w}_{sy}} L_{Min}(\mathbf{y}^{*}(\mathbf{x}^{i}_{sy}, \mathbf{x}^{i}_{nn}, \mathbf{w}_{sy}, \mathbf{w}_{nn})), \mathbf{t}^{i}_{\mathcal{Y}}) \nonumber \\ 
    & \quad = \nabla_{3} \mathbf{y}^{*}(\mathbf{x}^{i}_{sy}, \mathbf{x}^{i}_{nn}, \mathbf{w}_{sy}, \mathbf{w}_{nn})^{T} \nabla_{1} d(\mathbf{y}^{*}(\mathbf{x}^{i}_{sy}, \mathbf{x}^{i}_{nn}, \mathbf{w}_{sy}, \mathbf{w}_{nn}), \mathbf{t}^{i}_{\mathcal{Y}}), \\
    \label{eq:minimizer-based-loss-neural-gradient}
    & \nabla_{\mathbf{w}_{nn}} L_{Min}(\mathbf{y}^{*}(\mathbf{x}^{i}_{sy}, \mathbf{x}^{i}_{nn}, \mathbf{w}_{sy}, \mathbf{w}_{nn})), \mathbf{t}^{i}_{\mathcal{Y}}) \nonumber \\ 
    & \quad = \nabla_{4} \mathbf{y}^{*}(\mathbf{x}^{i}_{sy}, \mathbf{x}^{i}_{nn}, \mathbf{w}_{sy}, \mathbf{w}_{nn})^{T} \nabla_{1} d(\mathbf{y}^{*}(\mathbf{x}^{i}_{sy}, \mathbf{x}^{i}_{nn}, \mathbf{w}_{sy}, \mathbf{w}_{nn}), \mathbf{t}^{i}_{\mathcal{Y}}), 
\end{align}
where $\nabla_{3} \mathbf{y}^{*}(\mathbf{x}^{i}_{sy}, \mathbf{x}^{i}_{nn}, \mathbf{w}_{sy}, \mathbf{w}_{nn})$ and $\nabla_{4} \mathbf{y}^{*}(\mathbf{x}^{i}_{sy}, \mathbf{x}^{i}_{nn}, \mathbf{w}_{sy}, \mathbf{w}_{nn})$ are the Jacobian matrices of the unique energy minimizer with respect to the third and fourth arguments of $\mathbf{y}^{*}$, the symbolic and neural weights, respectively, and $\nabla_{1} d(\mathbf{y}^{*}(\mathbf{x}^{i}_{sy}, \mathbf{x}^{i}_{nn}, \mathbf{w}_{sy}, \mathbf{w}_{nn}), \mathbf{t}^{i}_{\mathcal{Y}})$ is the gradient of the supervised loss with respect to its first argument.

A primary challenge of minimizer-based learning is computing the Jacobian matrices of partial derivatives, $\nabla_{3} \mathbf{y}^{*}(\mathbf{x}^{i}_{sy}, \mathbf{x}^{i}_{nn}, \mathbf{w}_{sy}, \mathbf{w}_{nn})$ and $\nabla_{4} \mathbf{y}^{*}(\mathbf{x}^{i}_{sy}, \mathbf{x}^{i}_{nn}, \mathbf{w}_{sy}, \mathbf{w}_{nn})$.
To derive explicit expressions for them typically demands the following additional assumption on the continuity properties of the energy function.
\begin{assumption}
    \label{assumption:twice-differentiable-energy}
    The energy, $E$, is twice differentiable with respect to the targets at the minimizer, $\mathbf{y}^{*}$, and the Hessian matrix of second-order partial derivatives with respect to the targets, $\nabla_{1,1} E(\mathbf{y}^{*}(\mathbf{x}^{i}_{sy}, \mathbf{x}^{i}_{nn}, \mathbf{w}_{sy}, \mathbf{w}_{nn}), \mathbf{x}^{i}_{sy}, \mathbf{x}^{i}_{nn}, \mathbf{w}_{sy}, \mathbf{w}_{nn})$, is invertible.
    Further, the minimizer is the unique target satisfying first-order conditions of optimality, i.e.,
    \begin{align}
        \forall \mathbf{y} \in \mathcal{Y}, \quad
        \nabla_{1} E(\mathbf{y}, \mathbf{x}^{i}_{sy}, \mathbf{x}^{i}_{nn}, \mathbf{w}_{sy}, \mathbf{w}_{nn}) = 0 
        \, \iff \,
        \mathbf{y} = \mathbf{y}^{*}(\mathbf{x}^{i}_{sy}, \mathbf{x}^{i}_{nn}, \mathbf{w}_{sy}, \mathbf{w}_{nn})
    \end{align}
\end{assumption}
\assumptionref{assumption:twice-differentiable-energy} is satisfied by energy functions that are, for instance, smooth and strongly convex in the targets.
Under \assumptionref{assumption:twice-differentiable-energy}, the first-order optimality condition establishes the minimizer as an implicit function of the weights, and implicit differentiation yields the following equalities:
\begin{align}
    & \nabla_{1,1} E(\mathbf{y}^{*}(\mathbf{x}^{i}_{sy}, \mathbf{x}^{i}_{nn}, \mathbf{w}_{sy}, \mathbf{w}_{nn}), \mathbf{x}^{i}_{sy}, \mathbf{x}^{i}_{nn}, \mathbf{w}_{sy}, \mathbf{w}_{nn}) \nabla_{3} \mathbf{y}^{*}(\mathbf{x}^{i}_{sy}, \mathbf{x}^{i}_{nn}, \mathbf{w}_{sy}, \mathbf{w}_{nn}) \\
    & \quad = -\nabla_{1,4} E(\mathbf{y}^{*}(\mathbf{x}^{i}_{sy}, \mathbf{x}^{i}_{nn}, \mathbf{w}_{sy}, \mathbf{w}_{nn}), \mathbf{x}^{i}_{sy}, \mathbf{x}^{i}_{nn}, \mathbf{w}_{sy}, \mathbf{w}_{nn}) \nonumber \\
    & \nabla_{1,1} E(\mathbf{y}^{*}(\mathbf{x}^{i}_{sy}, \mathbf{x}^{i}_{nn}, \mathbf{w}_{sy}, \mathbf{w}_{nn}), \mathbf{x}^{i}_{sy}, \mathbf{x}^{i}_{nn}, \mathbf{w}_{sy}, \mathbf{w}_{nn}) \nabla_{4} \mathbf{y}^{*}(\mathbf{x}^{i}_{sy}, \mathbf{x}^{i}_{nn}, \mathbf{w}_{sy}, \mathbf{w}_{nn}) \\
    & \quad = -\nabla_{1,5} E(\mathbf{y}^{*}(\mathbf{x}^{i}_{sy}, \mathbf{x}^{i}_{nn}, \mathbf{w}_{sy}, \mathbf{w}_{nn}), \mathbf{x}^{i}_{sy}, \mathbf{x}^{i}_{nn}, \mathbf{w}_{sy}, \mathbf{w}_{nn}) \nonumber
\end{align}
Solving for the Jacobians of the minimizer:
{\small
\begin{align}
    \label{eq:minimizer-symbolic-gradient}
    \nabla_{3} \mathbf{y}^{*}(\mathbf{x}^{i}_{sy}, \mathbf{x}^{i}_{nn}, \mathbf{w}_{sy}, \mathbf{w}_{nn}) = - \big ( & \nabla_{1,1} E(\mathbf{y}^{*}(\mathbf{x}^{i}_{sy}, \mathbf{x}^{i}_{nn}, \mathbf{w}_{sy}, \mathbf{w}_{nn}), \mathbf{x}^{i}_{sy}, \mathbf{x}^{i}_{nn}, \mathbf{w}_{sy}, \mathbf{w}_{nn})^{-1} \\
    & \quad \quad \nabla_{1,4} E(\mathbf{y}^{*}(\mathbf{x}^{i}_{sy}, \mathbf{x}^{i}_{nn}, \mathbf{w}_{sy}, \mathbf{w}_{nn}), \mathbf{x}^{i}_{sy}, \mathbf{x}^{i}_{nn}, \mathbf{w}_{sy}, \mathbf{w}_{nn}) \big ), \nonumber \\
    \label{eq:minimizer-neural-gradient}
    \nabla_{4} \mathbf{y}^{*}(\mathbf{x}^{i}_{sy}, \mathbf{x}^{i}_{nn}, \mathbf{w}_{sy}, \mathbf{w}_{nn}) = - \big ( & \nabla_{1,1} E(\mathbf{y}^{*}(\mathbf{x}^{i}_{sy}, \mathbf{x}^{i}_{nn}, \mathbf{w}_{sy}, \mathbf{w}_{nn}), \mathbf{x}^{i}_{sy}, \mathbf{x}^{i}_{nn}, \mathbf{w}_{sy}, \mathbf{w}_{nn})^{-1} \\
    & \quad \quad \nabla_{1,5} E(\mathbf{y}^{*}(\mathbf{x}^{i}_{sy}, \mathbf{x}^{i}_{nn}, \mathbf{w}_{sy}, \mathbf{w}_{nn}), \mathbf{x}^{i}_{sy}, \mathbf{x}^{i}_{nn}, \mathbf{w}_{sy}, \mathbf{w}_{nn}) \big ). \nonumber
\end{align}
}%
The Jacobians in \eqref{eq:minimizer-symbolic-gradient} and \eqref{eq:minimizer-neural-gradient} applied to \eqref{eq:minimizer-based-loss-symbolic-gradient} and \eqref{eq:minimizer-based-loss-neural-gradient}, respectively, are referred to as hypergradients in the machine learning literature and are utilized in hyperparameter optimization and meta-learning \citep{do:neurips07, pedregosa:icml16, rajeswaran:neurips19}.
Oftentimes, approximations of the (inverse) Hessian matrices are made to estimate the hypergradient.

%% file: sections/a-suite-of-learning-techniques-for-nesy/learning-algorithms.tex
\subsection{Learning Algorithms}
\label{sec:nesy-learning-algorithms}

Next, we present four principled techniques for learning the neural and symbolic weights of a NeSy-EBM to minimize the losses introduced in the previous subsection: 1) Modular, 2) Gradient Descent, 3) Bilevel Value-Function Optimization, and 4) Stochastic Policy Optimization.
The four techniques are defined, and we discuss their strengths and limitations in relation to the modeling paradigms in \secref{sec:nesy-ebms}. 

\subsubsection{Modular Learning}
\label{sec:modular-learning}

The first and most straightforward NeSy-EBM learning technique is to train and connect the neural and symbolic components as independent modules.
For instance, the neural component can be trained via backpropagation and Adam to optimize a neural loss given neural labels.
Then, the symbolic component can be trained using an appropriate method to optimize a value or minimizer-based loss. 
The neural component weights are frozen during the symbolic weight learning process.

By definition, modular learning algorithms are not trained end-to-end, i.e., the neural and symbolic parameters are not jointly optimized to minimize the learning loss.
For this reason, modular approaches may struggle to find a weight setting with a learning loss as low as end-to-end techniques.
Moreover, modular approaches are not suitable for fine-tuning and adaptation.
Additionally, they require labels to train the neural component.  
Thus, modular learning is not used to learn neural parameters in unsupervised or semi-supervised settings.

Nevertheless, modular learning approaches are appealing and widely used for their simplicity and general applicability.
Importantly, no assumptions are made about the neural-symbolic interface; hence, modular learning is effective for every modeling paradigm presented in \secref{sec:nesy-ebms}.
Notably, minimizers and value-functions of DSPot models are typically non-differentiable with respect to the neural weights due to the complex neural-symbolic interface.
However, because modular techniques are not end-to-end, this is not an issue.
Moreover, modular learning can be used to train a NeSy-EBM for constraint satisfaction and joint reasoning, zero-shot reasoning, and reasoning with noisy data.
There are many established and effective modular neural and symbolic learning algorithms (see \citenoun{srinivasan:mlj21} for a recent taxonomy of symbolic weight learning algorithms).

\subsubsection{Gradient Descent}
\label{sec:direct-gradient-descent}

A conceptually simple but oftentimes difficult in-practice technique for end-to-end NeSy-EBM training is direct gradient descent.
Specifically, the gradients derived in the previous subsection are directly used with a gradient-based algorithm to optimize a NeSy-EBM loss with respect to both the neural and symbolic weights.
Backpropagation and \thmref{thm:value-function-gradient} produce relatively inexpensive gradients for neural and value-based losses for a general class of NeSy-EBMs.
Moreover, for a smaller family of NeSy-EBMs, gradients of energy minimizers exist and may be cheap to compute.
For instance, this is the case if the energy minimizer is determined via a simple closed-form expression (e.g., if inference is an unconstrained strongly convex quadratic program or a finite computation graph)

As shown in \secref{sec:learning-losses}, learning loss gradients for fully expressive NeSy-EBMs only exist under certain conditions.
Further, computing the gradients generally requires expensive second-order information about the energy function at the minimizer.
For this reason, direct gradient descent only applies to a relatively small class of NeSy-EBMs with specialized architectures that ensure principled and efficient gradient computation.
Such specialized architectures are less likely to support more complex modeling paradigms such as DSPar and DSPot.
When these differentiability conditions do not hold, modular learning is the natural alternative, and stochastic policy optimization is the natural end-to-end alternative.

\subsubsection{Bilevel Value-Function Optimization}
\label{sec:bilevel-value-function-optimization}

As shown in the earlier subsection, minimizer gradients are relatively more computationally expensive to compute and require more assumptions than value-function gradients. 
In this subsection, we devise a technique for optimizing a minimizer-based loss with only first-order gradients.
This technique is built on the fact that the general definition of NeSy-EBM learning \eqref{eq:nesy-ebm-learning} is naturally formulated as bilevel optimization.
In other words, the NeSy learning objective is a function of variable values obtained by solving a lower-level inference problem that is symbolic reasoning:
{\small
\begin{align}
    & \argmin_{\substack{(\mathbf{w}_{sy}, \mathbf{w}_{nn}) \in \mathcal{W}_{sy} \times \mathcal{W}_{nn} \\ (\hat{\mathbf{y}}^{1}, \cdots , \hat{\mathbf{y}}^{P}) \in \mathcal{Y}_{1} \times \cdots \times \mathcal{Y}_{P}}} \frac{1}{P} \sum_{i = 1}^{P} \Bigg ( L_{NN}(\mathbf{g}_{nn}(\mathbf{x}_{nn}^{i}, \mathbf{w}_{nn}), \mathbf{t}^{i}_{nn}) + L_{Val}(E(\cdot, \cdot, \cdot, \mathbf{w}_{sy}, \mathbf{w}_{nn}), \mathcal{S}_{i}) 
    \nonumber \\ 
    & \quad \quad \quad \quad \quad \quad \quad \quad \quad \quad \quad \quad  
    + d(\hat{\mathbf{y}}^{i},\mathbf{t}^{i}_{\mathcal{Y}}) \Bigg ) + \mathcal{R}(\mathbf{w}_{sy}, \mathbf{w}_{nn}) \label{eq:bilevel-learning} \\
    & \quad \quad \quad \quad \textrm{s.t.} \quad \quad \hat{\mathbf{y}}^{i} \in \argmin_{\tilde{\mathbf{y}} \in \mathcal{Y}} E(\tilde{\mathbf{y}}, \mathbf{x}_{sy}^{i}, \mathbf{x}_{nn}^{i}, \mathbf{w}_{sy}, \mathbf{w}_{nn}), \quad \forall i \in \{1, \cdots, P\}. \nonumber
\end{align}
}%
Regardless of the continuity and curvature properties of the upper and lower level objectives, \eqref{eq:bilevel-learning} is equivalent to the following:
{\small
\begin{align}
    & \argmin_{\substack{(\mathbf{w}_{sy}, \mathbf{w}_{nn}) \in \mathcal{W}_{sy} \times \mathcal{W}_{nn} \\ (\hat{\mathbf{y}}^{1}, \cdots , \hat{\mathbf{y}}^{P}) \in \mathcal{Y}_{1} \times \cdots \times \mathcal{Y}_{P}}} \frac{1}{P} \sum_{i = 1}^{P} \Bigg ( L_{NN}(\mathbf{g}_{nn}(\mathbf{x}_{nn}^{i}, \mathbf{w}_{nn}), \mathbf{t}^{i}_{nn}) + L_{Val}(E(\cdot, \cdot, \cdot, \mathbf{w}_{sy}, \mathbf{w}_{nn}), \mathcal{S}_{i}) 
    \nonumber \\ 
    & \quad \quad \quad \quad \quad \quad \quad \quad \quad \quad \quad \quad 
    + d(\hat{\mathbf{y}}^{i},\mathbf{t}^{i}_{\mathcal{Y}}) \Bigg ) + \mathcal{R}(\mathbf{w}_{sy}, \mathbf{w}_{nn}) \label{eq:bilevel-learning-value-function-reformulation} \\
    & \quad \quad \quad \quad \textrm{s.t.} \quad \quad E(\hat{\mathbf{y}}^{i}, \mathbf{x}_{sy}^{i}, \mathbf{x}_{nn}^{i}, \mathbf{w}_{sy}, \mathbf{w}_{nn}) - V_{\mathcal{Y}}(\mathbf{w}_{sy}, \mathbf{w}_{nn}, \mathcal{S}_{i}) \leq 0, \quad \forall i \in \{1, \cdots, P\}. \nonumber
\end{align}
}%
The formulation in \eqref{eq:bilevel-learning-value-function-reformulation} is referred to as a \emph{value-function} approach in bilevel optimization literature \citep{outrata:zor90, liu:icml21, liu:neurips22, sow:arxiv22, kwon:icml23}.
Value-function approaches view the bilevel program as a single-level constrained optimization problem by leveraging the value-function as a tight lower bound on the lower-level objective.

The inequality constraints in \eqref{eq:bilevel-learning-value-function-reformulation} do not satisfy any of the standard \emph{constraint qualifications} that ensure the feasible set near the optimal point is similar to its linearized approximation \citep{nocedal:wright:book06}.
This raises a challenge for providing theoretical convergence guarantees for constrained optimization techniques.
Following a recent line of value-function approaches to bilevel programming \citep{liu:icml21, sow:arxiv22, liu:arxiv23}, we overcome this challenge by allowing at most an $\iota > 0$ violation in each constraint in \eqref{eq:bilevel-learning-value-function-reformulation}.
With this relaxation, strictly feasible points exist and, for instance, the linear independence constraint qualification (LICQ) can hold.

Another challenge that arises from \eqref{eq:bilevel-learning-value-function-reformulation} is that the energy function of NeSy-EBMs is typically non-smooth with respect to the targets and even infinite-valued to represent constraints implicitly.
As a result, penalty or augmented Lagrangian functions derived from \eqref{eq:bilevel-learning-value-function-reformulation} are intractable.
Therefore, we substitute each instance of the energy function evaluated at the training sample $\mathcal{S}_{i}$, where $i \in \{1, \cdots, P\}$, and with weights $\mathbf{w}_{sy}$ and $\mathbf{w}_{nn}$ in the constraints of \eqref{eq:bilevel-learning-value-function-reformulation} with the following function:
{
\begin{align}
    \label{eq:moreau-envelope-of-energy}
    M(\hat{\mathbf{y}}^{i}, \mathcal{S}_{i}, \mathbf{w}_{sy}, \mathbf{w}_{nn}; \rho) 
    & := \inf_{\tilde{\mathbf{y}} \in \mathcal{Y}} \left( E(\tilde{\mathbf{y}}, \mathbf{x}^{i}_{sy}, \mathbf{x}^{i}_{nn}, \mathbf{w}_{sy}, \mathbf{w}_{nn}) + \frac{1}{2 \rho} \Vert \tilde{\mathbf{y}} - \hat{\mathbf{y}}^{i} \Vert_{2}^{2} \right), \\
    & = V_{conv}(\mathbf{w}_{sy}, \mathbf{w}_{nn}, \mathcal{S}_{i}; \hat{\mathbf{y}}^{i}, \frac{1}{2 \rho}) \nonumber
\end{align}
}%
where $\rho$ is a positive scalar.
For convex $E$, \eqref{eq:moreau-envelope-of-energy} is the Moreau envelope of the energy function \citep{rockafellar:book70, parikh:ftml13}.
In general, even for non-convex energy functions, $M$ is finite for all $\mathbf{y} \in \mathcal{Y}$ and it preserves global minimizers and minimum values, i.e., 
\begin{align}
    \mathbf{y}^{*}(\mathbf{x}^{i}_{sy}, \mathbf{x}^{i}_{nn}, \mathbf{w}_{sy}, \mathbf{w}_{nn}) & = \argmin_{\hat{\mathbf{y}}^{i} \in \mathcal{Y}} M(\hat{\mathbf{y}}^{i}, \mathcal{S}_{i}, \mathbf{w}_{sy}, \mathbf{w}_{nn}; \rho), \\
    V_{\mathcal{Y}}(\mathbf{w}_{sy}, \mathbf{w}_{nn}, \mathcal{S}_{i}) & = \min_{\hat{\mathbf{y}}^{i} \in \mathcal{Y}} M(\hat{\mathbf{y}}^{i}, \mathcal{S}_{i}, \mathbf{w}_{sy}, \mathbf{w}_{nn}; \rho).
\end{align}
When the energy function is a lower semi-continuous convex function, its Moreau envelope is convex, finite, and continuously differentiable, and its gradient with respect to $\hat{\mathbf{y}}^{i}$ is: 
{
\begin{align}
    & \nabla_{\hat{\mathbf{y}}^{i}} M(\hat{\mathbf{y}}^{i}, \mathcal{S}_{i}, \mathbf{w}_{sy}, \mathbf{w}_{nn}; \rho) \\ 
    & \quad = \frac{1}{\rho} \left( \hat{\mathbf{y}}^{i} - \argmin_{\tilde{\mathbf{y}} \in \mathcal{Y}} \left ( \rho E(\tilde{\mathbf{y}}, \mathbf{x}^{i}_{sy}, \mathbf{x}^{i}_{nn}, \mathbf{w}_{sy}, \mathbf{w}_{nn}) + \frac{1}{2} \Vert \tilde{\mathbf{y}} - \hat{\mathbf{y}}^{i} \Vert_{2}^{2} \right ) \right). \nonumber
\end{align}
}%
Convexity is a sufficient but not necessary condition to ensure $M$ is differentiable with respect to $\hat{\mathbf{y}}^{i}$.
See \cite{bonnans:book00} for results regarding the sensitivity of optimal value-functions to perturbations.
Further, as $M$ is a value-function, gradients of $M$ with respect to weights are derived using \thmref{thm:value-function-gradient}.

We propose the following relaxed and smoothed value-function approach to finding an approximate solution of \eqref{eq:bilevel-learning}:
{\small
\begin{align}
    & \argmin_{\substack{(\mathbf{w}_{sy}, \mathbf{w}_{nn}) \in \mathcal{W}_{sy} \times \mathcal{W}_{nn} \\ (\hat{\mathbf{y}}^{1}, \cdots , \hat{\mathbf{y}}^{P}) \in \mathcal{Y}_{1} \times \cdots \times \mathcal{Y}_{P}}} \frac{1}{P} \sum_{i = 1}^{P} \Bigg ( L_{NN}(\mathbf{g}_{nn}(\mathbf{x}_{nn}^{i}, \mathbf{w}_{nn}), \mathbf{t}^{i}_{nn}) + L_{Val}(E(\cdot, \cdot, \cdot, \mathbf{w}_{sy}, \mathbf{w}_{nn}), \mathcal{S}_{i}) 
    \nonumber \\ 
    & \quad \quad \quad \quad \quad \quad \quad \quad \quad \quad \quad \quad 
    + d(\hat{\mathbf{y}}^{i},\mathbf{t}^{i}_{\mathcal{Y}}) \Bigg ) + \mathcal{R}(\mathbf{w}_{sy}, \mathbf{w}_{nn})
    \label{eq:relaxed-smoothed-value-function-bound-constrained-approach} \\
    & \quad \quad \quad \quad \textrm{s.t.} \quad \quad M(\hat{\mathbf{y}}^{i}, \mathcal{S}_{i}, \mathbf{w}_{sy}, \mathbf{w}_{nn}; \rho) - V_{\mathcal{Y}}(\mathbf{w}_{sy}, \mathbf{w}_{nn}, \mathcal{S}_{i}) \leq \iota, \quad \forall i \in \{1, \cdots, P\}, \nonumber
\end{align}
}%

The formulation \eqref{eq:relaxed-smoothed-value-function-bound-constrained-approach} is the core of our proposed NeSy-EBM learning framework outlined in \algoref{alg:bilevel-nesy-ebm-learning} below. 
The algorithm proceeds by approximately solving instances of \eqref{eq:relaxed-smoothed-value-function-bound-constrained-approach} in a sequence defined by a decreasing $\iota$.
This is a graduated approach to solving \eqref{eq:bilevel-learning-value-function-reformulation} with instances of \eqref{eq:relaxed-smoothed-value-function-bound-constrained-approach} that are increasingly tighter approximations.

\begin{algorithm}[H]
\caption{Bilevel Value-Function Optimization for NeSy-EBM Learning}
\label{alg:bilevel-nesy-ebm-learning}
\begin{algorithmic}[1]
    \REQUIRE{Moreau Param.: $\rho$, Starting weights: $(\mathbf{w}_{sy}, \mathbf{w}_{nn}) \in \mathcal{W}_{sy} \times \mathcal{W}_{nn}$}
    \STATE{$\hat{\mathbf{y}}^{i} \gets (\mathbf{t}^{i}_{\mathcal{Y}}, \argmin_{\hat{\mathbf{z}} \in \mathcal{Z}^{i}_{\mathcal{Y}}} E((\mathbf{t}^{i}_{\mathcal{Y}}, \hat{\mathbf{z}}), \mathbf{x}^{i}_{sy}, \mathbf{x}^{i}_{nn}, \mathbf{w}_{sy}, \mathbf{w}_{nn})), \quad \forall{i = 1, \cdots, P}$}
    \STATE{$\iota \gets \max_{i \in \{1, \cdots, P\}} M(\hat{\mathbf{y}}^{i}, \mathcal{S}_{i}, \mathbf{w}_{sy}, \mathbf{w}_{nn}; \rho) - V_{\mathcal{Y}}(\mathbf{w}_{sy}, \mathbf{w}_{nn}, \mathcal{S}_{i})$}
    \FOR{$t = 0, 1, 2, \cdots$}
        \STATE{Find $\mathbf{w}_{sy}, \mathbf{w}_{nn}, \mathbf{y}^{1}, \cdots, \mathbf{y}^{P}$ minimizing \eqref{eq:relaxed-smoothed-value-function-bound-constrained-approach} with $\iota$}
        \IF{Stopping criterion satisfied}
            \STATE{Stop with: $\mathbf{w}_{sy}, \mathbf{w}_{nn}, \mathbf{y}^{1}, \cdots, \mathbf{y}^{P}$}
        \ENDIF
        \STATE{$\iota \gets \frac{1}{2} \cdot \iota$}
    \ENDFOR
\end{algorithmic}
\end{algorithm} 
We suggest starting points for each $\hat{\mathbf{y}}^{i}$ to be the corresponding latent inference minimizer and $\iota$ to be the maximum difference in the value-function and the smooth energy function. 
At this suggested starting point, the supervised loss is initially $0$, and the subproblem reduces to minimizing the learning objective without increasing the most violated constraint.
Then, the value for $\iota$ is halved every time an approximate solution to the subproblem, \eqref{eq:relaxed-smoothed-value-function-bound-constrained-approach}, is reached.
The outer loop of the NeSy-EBM learning framework may be stopped by either watching the progress of a training or validation evaluation metric or by specifying a final value for $\iota$.

Each instance of \eqref{eq:relaxed-smoothed-value-function-bound-constrained-approach} in \algoref{alg:bilevel-nesy-ebm-learning} can be optimized using only first-order gradient-based methods.
Specifically, we employ the bound-constrained augmented Lagrangian algorithm, Algorithm 17.4 from \citenoun{nocedal:wright:book06}, which finds approximate minimizers of the problem's augmented Lagrangian for a fixed setting of the penalty parameters using gradient descent.
To simplify notation, let the constraints in \eqref{eq:relaxed-smoothed-value-function-bound-constrained-approach} be denoted by:
{
\begin{align}
    c(\hat{\mathbf{y}}^{i}, \mathcal{S}_{i}, \mathbf{w}_{sy}, \mathbf{w}_{nn}; \iota) := M(\hat{\mathbf{y}}^{i}, \mathcal{S}_{i}, \mathbf{w}_{sy}, \mathbf{w}_{nn}; \rho) - V_{\mathcal{Y}}(\mathbf{w}_{sy}, \mathbf{w}_{nn}, \mathcal{S}_{i}) - \iota,
\end{align}
}%
for each constraint indexed $i \in \{1, \cdots, P\}$.
Moreover, let 
\begin{align}
    \mathbf{c}(\mathbf{y}^{1}, \cdots, \mathbf{y}^{P}, \mathcal{S}, \mathbf{w}_{sy}, \mathbf{w}_{nn}; \iota) := [ c(\hat{\mathbf{y}}^{i}, \mathcal{S}_{i}, \mathbf{w}_{sy}, \mathbf{w}_{nn}; \iota) ]_{i = 1}^{P}.
\end{align}
The augmented Lagrangian function corresponding to \eqref{eq:relaxed-smoothed-value-function-bound-constrained-approach} introduces a quadratic penalty parameter $\mu$ and $P$ linear penalty parameters $\mathbf{\lambda} := [ \lambda_{i} ]_{i = 1}^{P}$, as follows:
{
\begin{align}
    \label{eq:bilevel-value-function-augmented-lagrangian}
    & \mathcal{L}_{A}(\hat{\mathbf{y}}^{1}, \cdots, \hat{\mathbf{y}}^{P}, \mathbf{w}_{sy}, \mathbf{w}_{nn}, \mathcal{S}, \mathbf{s}; \mathbf{\lambda}, \mu, \iota) \\ 
    & :=  \frac{1}{P} \sum_{i = 1}^{P} \left ( L_{NN}(\mathbf{g}_{nn}(\mathbf{x}_{nn}^{i}, \mathbf{w}_{nn}), \mathbf{t}^{i}_{nn}) + L_{Val}(E(\cdot, \cdot, \cdot, \mathbf{w}_{sy}, \mathbf{w}_{nn}), \mathcal{S}_{i}) + d(\hat{\mathbf{y}}^{i},\mathbf{t}^{i}_{\mathcal{Y}}) \right ) \nonumber \\ 
    & \quad + \frac{\mu}{2} \sum_{i = 1}^{P} \left (c(\hat{\mathbf{y}}^{i}, \mathcal{S}_{i}, \mathbf{w}_{sy}, \mathbf{w}_{nn}; \iota) + s_{i} \right )^{2}  \nonumber \\ 
    & \quad + \sum_{i = 1}^{P} \lambda_{i} \left (c(\hat{\mathbf{y}}^{i}, \mathcal{S}_{i}, \mathbf{w}_{sy}, \mathbf{w}_{nn}; \iota ) + s_{i}\right) + \mathcal{R}(\mathbf{w}_{sy}, \mathbf{w}_{nn}), \nonumber
\end{align}
}%
where we introduced $P$ slack variables, $\mathbf{s} = \left[s_{i} \right]_{i = 1}^{P}$, for each inequality constraint.
The bound-constrained augmented Lagrangian algorithm provides a principled method for updating the penalty parameters and ensures fundamental convergence properties of our learning framework.
Notably, we have that limit points of the iterate sequence are stationary points of $\Vert \mathbf{c}(\mathbf{y}^{1}, \cdots, \mathbf{y}^{P}, \mathcal{S}, \mathbf{w}_{sy}, \mathbf{w}_{nn}; \iota) + \mathbf{s} \Vert^{2}$ when the problem has no feasible points.
When the problem is feasible, and LICQ holds at the limits, they are KKT points of \eqref{eq:relaxed-smoothed-value-function-bound-constrained-approach} (Theorem 17.2 in \cite{nocedal:wright:book06}).
Convergence rates and stronger guarantees are possible by analyzing the structure of the energy function for specific NeSy-EBMs.

The bilevel value-function optimization technique in \algoref{alg:bilevel-nesy-ebm-learning} is an end-to-end algorithm for minimizing a general NeSy-EBM learning loss with only first-order value-function gradients.
Thus, \algoref{alg:bilevel-nesy-ebm-learning} is a more practical and widely applicable technique for NeSy-EBM learning than modular and direct gradient descent methods.
The bilevel approach can be employed for a broader class of NeSy-EBMs than direct gradient descent methods and for every motivating application.
Moreover, we demonstrate that it can be used to train DSVar and DSPar NeSy-EBMs in our empirical evaluation.

\subsubsection{Stochastic Policy Optimization}
\label{sec:stochastic-policy-optimization}

\begin{figure}
    \centering
    \includegraphics[width=0.8 \textwidth]{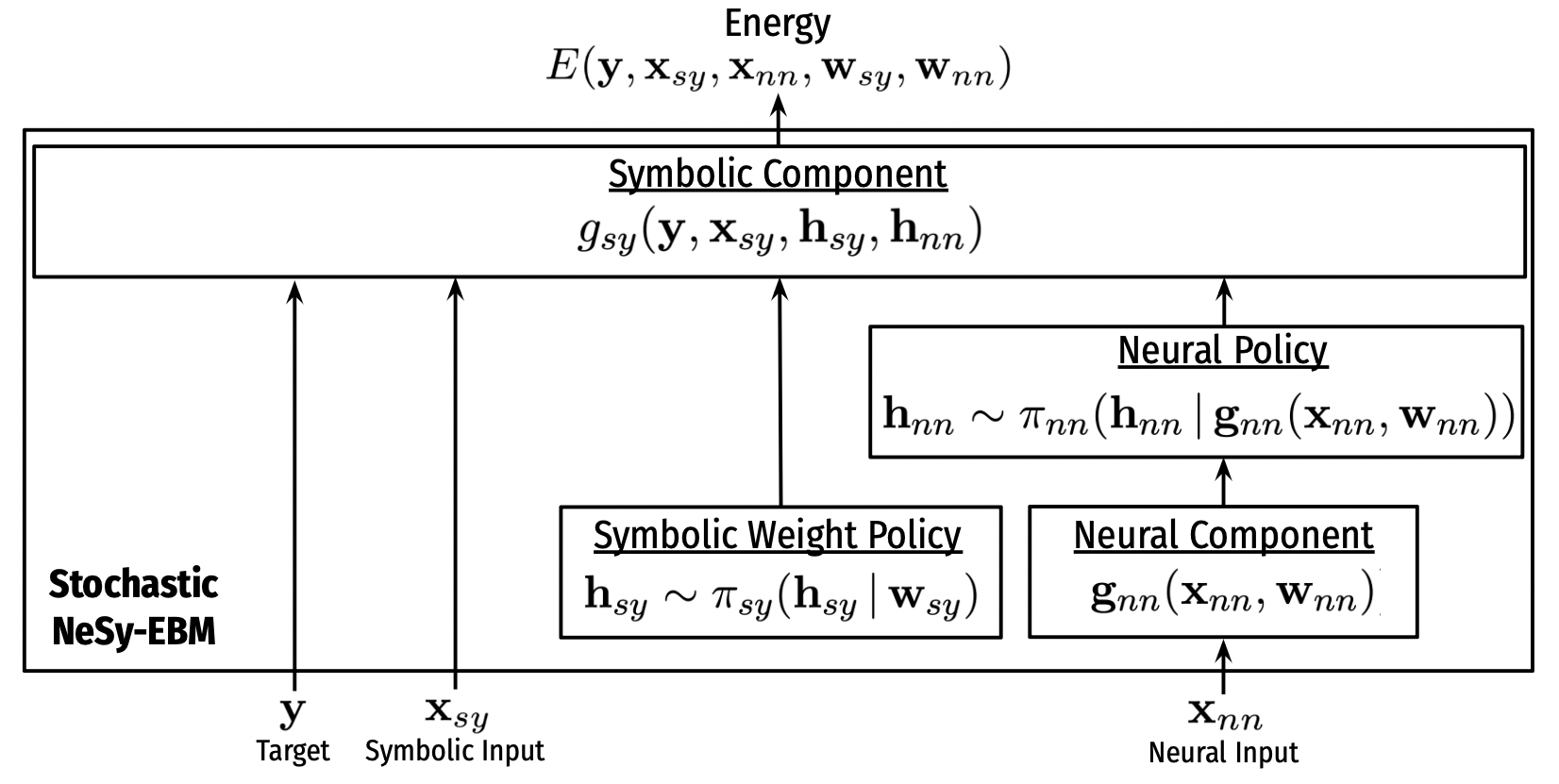}
    \caption{A stochastic NeSy-EBM. The symbolic weights and the neural component parameterize stochastic policies. A sample from the policies is drawn to produce arguments of the symbolic component.}
    \label{fig:stochastic-nesy-ebm}
\end{figure}

Finally, another approach to NeSy-EBM learning that avoids directly computing the energy minimizer's gradients with respect to the weights is to re-formulate NeSy learning as stochastic policy optimization.
This approach is especially useful when end-to-end learning is desired but the neural-symbolic interface is discrete, integer, or otherwise non-smooth.
\figref{fig:stochastic-nesy-ebm} shows the modifications to the standard NeSy-EBM framework to create a stochastic NeSy-EBM.
The symbolic and neural weights are used to condition a symbolic weight and neural policy, denoted by $\pi_{sy}$ and $\pi_{nn}$, respectively.
Samples from the policies replace the symbolic weights and neural output as arguments of the symbolic component.
Specifically, given symbolic and neural weights $\mathbf{w}_{sy}$ and $\mathbf{w}_{nn}$ and input features $\mathbf{x}^{i}_{nn}$ from a training sample $\mathcal{S}_{i} \in \mathcal{S}$,  $\mathbf{h}_{sy}$ and $\mathbf{h}^{i}_{nn}$ are random variables with the following conditional distributions:
\begin{align}
    \mathbf{h}_{sy} & \sim \pi_{sy} (\mathbf{h}_{sy} \, \vert \, \mathbf{w}_{sy}), \\
    \mathbf{h}^{i}_{nn} & \sim \pi_{nn} (\mathbf{h}^{i}_{nn} \, \vert \, \mathbf{g}_{nn}(\mathbf{x}^{i}_{nn}, \mathbf{w}_{nn})).
\end{align}
Moreover, the random variables $\mathbf{h}_{sy}$ and $\mathbf{h}^{i}_{nn}$ are modeled independently, thus the conditional joint distribution, denoted by $\pi$, is:
\begin{align}
    \pi(\mathbf{h}_{sy}, \mathbf{h}^{i}_{nn} \, \vert \, \mathbf{w}_{sy}, \mathbf{g}_{nn}(\mathbf{x}^{i}_{nn}, \mathbf{w}_{nn})) := \pi_{sy} (\mathbf{h}_{sy} \, \vert \, \mathbf{w}_{sy}) \cdot \pi_{nn} (\mathbf{h}^{i}_{nn} \, \vert \, \mathbf{g}_{nn}(\mathbf{x}^{i}_{nn}, \mathbf{w}_{nn}))
\end{align}
The stochastic NeSy-EBM energy is the symbolic component evaluated at a sample from the joint distribution above:
\begin{align}
    E(\mathbf{y}, \mathbf{x}^{i}_{sy}, \mathbf{x}^{i}_{nn}, \mathbf{w}_{sy}, \mathbf{w}_{nn}) := g_{sy} (\mathbf{y}, \mathbf{x}^{i}_{sy}, \mathbf{h}_{sy}, \mathbf{h}^{i}_{nn})
    \label{eq:stochastic-energy-function}
\end{align}
The NeSy-EBM energy and all of the NeSy-EBM per-sample loss functionals discussed in \secref{sec:learning-losses} are, therefore, random variables with distributions that are defined by $\pi$.
Under the stochastic policy optimization framework, loss functionals are generally denoted by the function $J^{i}$ for each $i \in \{1, \cdots, P\}$ such that: 
\begin{align}
    J^{i}(g_{sy} (\cdot, \mathbf{x}^{i}_{sy}, \mathbf{h}_{sy}, \mathbf{h}^{i}_{nn}), \mathcal{S}_{i}) := L^{i}(E(\cdot, & \cdot, \cdot, \mathbf{w}_{sy}, \mathbf{w}_{nn}), \mathcal{S}_{i})
\end{align}

Learning is minimizing the expected value of the stochastic loss functional and is formulated as:
\begin{align}
    \label{eq:stochastic-policy-learning}
    \argmin_{(\mathbf{w}_{sy}, \mathbf{w}_{nn}) \in \mathcal{W}_{sy} \times \mathcal{W}_{nn}} \frac{1}{P} \sum_{i = 1}^{P} \mathbb{E}_{\pi} \left [ J^{i} ( g_{sy} (\cdot, \mathbf{x}^{i}_{sy}, \mathbf{h}_{sy}, \mathbf{h}^{i}_{nn}), \mathcal{S}_{i} ) \right ] + \mathcal{R}(\mathbf{w}_{sy}, \mathbf{w}_{nn}), 
\end{align}
where $\mathbb{E}_{\pi}$ is the expectation over the joint distribution $\pi$.

We apply gradient-based learning algorithms to find an approximate solution to \eqref{eq:stochastic-policy-learning}. 
The policy gradient theorem \citep{williams:ml92, sutton:neurips99, sutton:book18} yields the following expression for the gradients of the expected value of a loss functional: 
\begin{align} 
    \label{eq:loss-functional-policy-gradient-nn}
    & \nabla_{\mathbf{w}_{nn}} \mathbb{E}_{\pi} \left [ J^{i} ( g_{sy} (\cdot, \mathbf{x}^{i}_{sy}, \mathbf{h}_{sy}, \mathbf{h}^{i}_{nn}), \mathcal{S}_{i} ) \right ] \\ 
    & \quad = \mathbb{E}_{\pi} \left [ J^{i} ( g_{sy} (\cdot, \mathbf{x}^{i}_{sy}, \mathbf{h}_{sy}, \mathbf{h}^{i}_{nn}), \mathcal{S}_{i} ) \cdot \nabla_{\mathbf{w}_{nn}} \log \pi(\mathbf{h}_{sy}, \mathbf{h}^{i}_{nn} \, \vert \, \mathbf{w}_{sy}, \mathbf{g}_{nn}(\mathbf{x}^{i}_{nn}, \mathbf{w}_{nn})) \right ]. \nonumber \\
    \label{eq:loss-functional-policy-gradient-sy}
    & \nabla_{\mathbf{w}_{sy}} \mathbb{E}_{\pi} \left [ J^{i} ( g_{sy} (\cdot, \mathbf{x}^{i}_{sy}, \mathbf{h}_{sy}, \mathbf{h}^{i}_{nn}), \mathcal{S}_{i} ) \right ] \\ 
    & \quad = \mathbb{E}_{\pi} \left [ J^{i} ( g_{sy} (\cdot, \mathbf{x}^{i}_{sy}, \mathbf{h}_{sy}, \mathbf{h}^{i}_{nn}), \mathcal{S}_{i} ) \cdot \nabla_{\mathbf{w}_{sy}} \log \pi(\mathbf{h}_{sy}, \mathbf{h}^{i}_{nn} \, \vert \, \mathbf{w}_{sy}, \mathbf{g}_{nn}(\mathbf{x}^{i}_{nn}, \mathbf{w}_{nn})) \right ]. \nonumber
\end{align}
The expression for the gradient of the expected loss functional above motivates a family of gradient estimators.
Notably, the REINFORCE gradient estimator for NeSy-EBM learning is:
{\small
\begin{align}
    \label{eq:loss-functional-reinforce-estimator-nn}
    & \nabla_{\mathbf{w}_{nn}} \mathbb{E}_{\pi} \left [ J^{i} ( g_{sy} (\cdot, \mathbf{x}^{i}_{sy}, \mathbf{h}_{sy}, \mathbf{h}^{i}_{nn}), \mathcal{S}_{i} ) \right ] \\ 
    & \quad \approx \frac{1}{N} \sum_{k = 1}^{N} \left ( J^{i} ( g_{sy} (\cdot, \mathbf{x}^{i}_{sy}, \mathbf{h}^{(k)}_{sy}, \mathbf{h}^{i (k)}_{nn}), \mathcal{S}_{i} ) \nabla_{\mathbf{w}_{nn}} \log \pi(\mathbf{h}^{(k)}_{sy}, \mathbf{h}^{i (k)}_{nn} \, \vert \, \mathbf{w}_{sy}, \mathbf{g}_{nn}(\mathbf{x}^{i}_{nn}, \mathbf{w}_{nn})) \right ), \nonumber \\
    \label{eq:loss-functional-reinforce-estimator-sy}
    & \nabla_{\mathbf{w}_{sy}} \mathbb{E}_{\pi} \left [ J^{i} ( g_{sy} (\cdot, \mathbf{x}^{i}_{sy}, \mathbf{h}_{sy}, \mathbf{h}^{i}_{nn}), \mathcal{S}_{i} ) \right ] \\ 
    & \quad \approx \frac{1}{N} \sum_{k = 1}^{N} \left ( J^{i} ( g_{sy} (\cdot, \mathbf{x}^{i}_{sy}, \mathbf{h}^{(k)}_{sy}, \mathbf{h}^{i (k)}_{nn}), \mathcal{S}_{i} ) \nabla_{\mathbf{w}_{sy}} \log \pi(\mathbf{h}^{(k)}_{sy}, \mathbf{h}^{i (k)}_{nn} \, \vert \, \mathbf{w}_{sy}, \mathbf{g}_{nn}(\mathbf{x}^{i}_{nn}, \mathbf{w}_{nn})) \right ), \nonumber
\end{align}
}%
where each $\mathbf{h}^{(k)}_{sy}$ and $\mathbf{h}^{i (k)}_{nn}$ for $k \in \{1, \cdots, N\}$ is an independent sample of the random variables.

Stochastic policy optimization techniques are broadly applicable for end-to-end training of NeSy-EBMs because they do not require differentiation through the neural-symbolic interface and the symbolic inference process.
Moreover, they can be used for every modeling paradigm.
The tradeoff with the stochastic policy approach, however, is the high variance in the sample estimates for the policy gradient.
This is a common challenge in policy optimization that becomes more prominent with increasing dimensionality of the policy output space \citep{sutton:book18}.
Thus, learning with a stochastic policy optimization approach may take significantly more iterations to converge compared to the other presented techniques.

%% file: sections/experiments/introduction.tex
\section{Empirical Analysis}
\label{sec:experiments}

In this section, we perform an empirical analysis of the NeSy-EBM modeling paradigms and learning algorithms presented in this work using the NeuPSL system introduced in \secref{sec:neupsl-and-deep-hlmrfs}.
The experiments are designed to study NeSy-EBMs as instantiated through NeuPSL in the investigated settings rather than to provide an exhaustive empirical comparison against all NeSy frameworks.
We encourage readers to explore the broader literature discussing neural-symbolic systems \citep{marra:ai24, besold:nesyai22, pryor:ijcai23} and different logical semantics \citep{krieken:ai22}, which provides valuable insight into the strengths of different approaches.
Our experiments are designed to investigate the following four research questions:

\begin{itemize}
    \item RQ1: Can NeSy-EBMs as instantiated through NeuPSL enhance the accuracy and reasoning capabilities of deep learning models?
    \item RQ2: Can the value-function gradients provided in \thmref{thm:value-function-gradient} be used as a reliable descent direction for value-based learning losses?
    \item RQ3: Can symbolic constraints be used to train a deep learning model with partially labeled data?
    \item RQ4: What are the prediction performance and runtime tradeoffs among the presented modular, value-based, and minimizer-based learning approaches?
\end{itemize}

Our empirical analysis is organized into four subsections.
First, in \secref{sec:datasets-models}, we introduce the neural-symbolic datasets and models used in the experiments.
In \secref{sec:experiments-joint-reasoning}, we study the application of NeSy-EBMs for constraint satisfaction and joint reasoning.
In \secref{sec:experiments-learning}, we evaluate the performance of modular learning and the performance and empirical convergence properties of the value-based, bilevel, and stochastic policy optimization learning algorithms presented in \secref{sec:nesy-learning-algorithms} for fine-tuning and few-shot learning.
Finally, in \secref{sec:experiments-semi-supervision}, we analyze the effectiveness of the NeSy-EBM framework for training a neural component in a semi-supervised setting.
All code and data for reproducing our empirical analysis are available at \url{https://github.com/linqs/dickens-arxiv24}.

%% file: sections/experiments/datasets.tex
\subsection{Datasets and Models}
\label{sec:datasets-models}

This subsection introduces the NeSy datasets and models, which will be utilized throughout the empirical analysis.
Moreover, any modifications made to answer specific research questions will be described in the following subsections.
Additional details on the architectures of both the neural and symbolic components are available at \url{https://github.com/linqs/dickens-arxiv24}.

\begin{itemize}
    \item \textbf{MNIST-Add-k Dataset:} MNIST-Add-$k$ is a canonical NeSy dataset introduced by \citenoun{manhaeve:ai21} where models must determine the sum of each pair of digits from two lists of MNIST images.
    An MNIST-Add$k$ equation consists of two lists of $k > 0$ MNIST images.
    For instance, $\big[\inlinegraphics{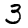}\big] + \big[\inlinegraphics{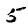}\big] = \mathbf{8}$
    is an MNIST-Add$1$ equation, and $\big[\inlinegraphics{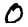}, \inlinegraphics{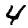}\big] + \big[\inlinegraphics{sections/experiments/figures/MNIST-3.png}, \inlinegraphics{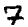}\big] = \mathbf{41}$
    is an MNIST-Add$2$ equation.

    \textbf{Evaluation:} For all experiments, we evaluate models over $5$ splits of the low-data setting proposed by \citenoun{manhaeve:ai21} with $600$ total images for training and $1,000$ images each for validation and test.
    Prediction performance in this setting is measured by the accuracy of the image classifications and the inferred sums.
    Constraint satisfaction consistency in this setting is the proportion of predictions that satisfy the semantics of addition.

    \textbf{Baseline Architecture:} The baseline neural architecture for all MNIST-Add$k$ datasets is a ResNet18 convolutional neural network backbone \citep{he:cvpr16} with a 2-layer multi-layer perceptron (MLP) prediction head.
    The baseline is trained and applied as a digit classifier.
    Further, to allow the baseline to leverage the unlabeled training data in the semi-supervised settings, the digit classifier backbone is pre-trained using the SimCLR self-supervised learning framework \citep{chen:icml20}.
    Augmentations are used to obtain positive pairs for the contrastive pre-training process.

    \textbf{NeSy-EBM Architecture:} The NeSy-EBM architecture is a composition of the baseline digit classifier and a symbolic component created with NeuPSL that encodes the semantics of addition.
    The target variables of the symbolic component are the labels of the MNIST digits and their sum.
    The neural classification is used as a prior for the digit labels.

    \item \textbf{Visual-Sudoku Dataset:} Visual-Sudoku, first introduced by \citenoun{wang:icml19}, is a dataset containing a collection of $9 \times 9$ Sudoku puzzles constructed from MNIST images. 
    In each puzzle, $30$ cells are filled with MNIST images and are referred to as \emph{clues}.
    The remaining cells are empty.
    The task is to correctly classify all clues and fill in the empty cells with digits that satisfy the rules of Sudoku: no repeated digits in any row, column, or box.
    
    \textbf{Evaluation:} For all experiments, results are reported across $5$ splits with $20$ puzzles for training and $100$ puzzles each for validation and test.
    There is an equal number of MNIST images ($600$) in the training datasets for Visual-Sudoku and MNIST-Add-k.
    Prediction performance in this setting is measured by the accuracy of the image classifications.
    Constraint satisfaction consistency in this setting is the proportion of predictions that satisfy the rules of Sudoku.

    \textbf{Baseline Architecture:} The baseline neural architecture for Visual-Sudoku is the same as that of the MNIST-Add$k$.
    
    \textbf{NeSy-EBM Architecture:} The NeSy-EBM architecture is a composition of the baseline digit classifier and a symbolic component created with NeuPSL that encodes the rules of Sudoku.
    The target variables of the symbolic component are the labels of the clues and the empty cells.
    The neural classification is used as a prior for the clues.

    \item \textbf{Pathfinding Dataset:} Pathfinding is a NeSy dataset introduced by \citenoun{vlastelica:iclr20} consisting of $12000$ randomly generated images of terrain maps from the Warcraft II tileset.
    The images are partitioned into $12 \times 12$ grids where each vertex represents a terrain with a cost.
    The task is to find the lowest cost path from the top left to the bottom right corner of each image.
    
    \textbf{Evaluation:} For all experiments, results are reported over $5$ splits generated by partitioning the images into sets of $10,000$ for training, $1,000$ for validation, and $1,000$ for testing.
    Prediction performance in this setting is measured by the proportion of valid predicted paths, i.e., continuous, and that have a minimum cost.
    Constraint satisfaction continuity in this setting is measured by the proportion of predictions with a continuous predicted path.

    \textbf{Baseline Architecture:} The baseline neural architecture for the Pathfinding dataset is a ResNet18 convolutional neural network.
    The input of the ResNet18 path-finder baseline is the full Warcraft II map, and the output is the predicted shortest path.
    The model is trained using the labeled paths from the training data set.

    \textbf{NeSy-EBM Architecture:} The NeSy-EBM architecture is a composition of the baseline path-finder and a symbolic component created with NeuPSL that encodes end-points and continuity constraints, i.e., the path from the top left corner of the map to the bottom right corner must be continuous.
    The target variables of the symbolic component are variables indicating whether a vertex of the map grid is on the path.
    The neural classification is used as a prior for the path, and the symbolic component finds a valid path near the neural prediction.
    
    \item \textbf{Citeseer and Cora Datasets:} Citeseer and Cora are two widely studied citation network node classification datasets first introduced by \citenoun{sen:aim08}.
    Citeseer consists of $3,327$ scientific publications classified into one of $6$ topics, while Cora contains $2,708$ scientific publications classified into one of $7$ topics.
    
    \textbf{Evaluation:} For all experiments, we evaluate models over $5$ randomly sampled splits using $20$ examples of each topic for training, $200$ of the nodes for validation, and $1000$ nodes for testing.
    Prediction performance in this setting is measured by the categorical accuracy of a paper label.

    \textbf{Baseline Architecture:} The baseline neural architecture for the Citation network settings is a Simple Graph Convolutional Network (SGC) \citep{wu:icml19}.
    SGCs are graph convolutional networks with linear activations in the hidden layers to reduce computational complexity.
    The SGC neural baseline uses bag-of-words feature vectors associated with each paper as node features and citations as bi-directional edges.
    Then, a MLP is trained to predict the topic label given the SGC-transformed features.

    \textbf{NeSy-EBM Architecture:} The NeSy-EBM architecture is a composition of the baseline SGC and a symbolic component created with NeuPSL that encodes the homophilic structure of the citation network, i.e., two papers connected in the network are more likely to have the same label. 
    Target variables indicate the degree to which a paper has a particular topic.
    The neural classification is used as a prior for the labels of the nodes, and the symbolic component propagates this knowledge to its neighbors.

    \item \textbf{RoadR Dataset:} RoadR is an extension of the ROAD (Road event Awareness Dataset) dataset, initially introduced by \citenoun{singh:tpa2021}.
    The ROAD dataset was developed to evaluate the situational awareness of autonomous vehicles in various road environments, weather conditions, and times of day.
    It contains $22$ videos, $122k$ labeled frames, $560k$ bounding boxes, and a total of $1.7M$ labels, which include $560k$ agents, $640k$ actions, and $499k$ locations.
    RoadR builds upon this by adding $243$ logical requirements that must be satisfied, further enhancing its utility for testing autonomous vehicles.
    For instance, a traffic light should never be simultaneously predicted as red and green.
    
    \textbf{Evaluation:} For all experiments, we evaluate models with $15$ videos for training and $3$ videos for testing.
    Prediction performance in this setting is measured by the matching boxes using Intersection over Union (IoU) and then multi-class f1.
    Constraint satisfaction consistency in this setting is the proportion of frame predictions with no constraint violations.

    \textbf{Baseline Architecture:} The baseline neural architecture for the RoadR dataset is a DEtection TRansformer (DETR) model with a ResNet50 backbone \citep{carion:eccv20}.
    The baseline is trained and applied to detect objects in a frame, along with a multi-label classification for its class labels (e.g., car, red, traffic light, etc.).

    \textbf{NeSy-EBM Architecture:} The NeSy-EBM architecture is a composition of the baseline object detector and classifier and a symbolic component created with NeuPSL that encodes the logical requirements.
    The target variables are the classification labels of a bounding box.
    The neural classification is used as both the bounding box creation and a prior on the labels that the symbolic component uses as a starting point to find a valid solution to the constraints.

    \item \textbf{Logical-Deduction} is a multiple-choice question-answering dataset introduced by \cite{srivastava:arxiv22}.
    These questions require deducing the order of a sequence of objects given a natural language description and then answering a multiple-choice question about that ordering.
    
    \textbf{Evaluation:} We report results for a single test set of $300$ deduction problems, with a prompt containing two examples.
    Prediction performance in this setting is measured by the accuracy of the predicted multiple-choice answer.

    \textbf{Baseline Architecture:} The baseline neural architectures for the Logical-Deduction dataset are the models presented in \citenoun{pan:emnlp23} on GPT-3.5-turbo and GPT-4 \cite{openai:techreport24}.
    Each model is run using \textit{Standard} and \textit{Chain-of-Thought (CoT)} \citep{wei:neurips22} prompting.

    \textbf{NeSy-EBM Architecture:} The NeSy-EBM architecture is a composition of the baseline LLM that is being prompted to create the constraints within the symbolic program.
    Symbolic inference is then performed, and the output is returned to the LLM for final evaluation.
    In this sense, the NeSy-EBM writes a program to perform reasoning rather than depending on the language model to reason independently.
\end{itemize}

%% file: sections/experiments/constraint-satisfaction.tex
\subsection{Constraint Satisfaction and Joint Reasoning}
\label{sec:experiments-joint-reasoning}

We begin our experimental evaluation by exploring the advantages of employing NeSy-EBMs as instantiated through NeuPSL for performing constraint satisfaction and joint reasoning, which is relevant to answering research question RQ1.
We employ a modular training approach to set up these experiments to obtain weights for our models' neural and symbolic components.
Specifically, neural components undergo training using the complete training dataset for supervision, and symbolic weights are trained using a simple random grid search.
After this modular training phase, NeSy-EBM inference is carried out to predict binary, $0, 1$, valued target variables that align with established domain knowledge and logical reasoning.
For this reason, \textit{DSPar} NeSy-EBMs are used for MNIST-Add-k, Visual-Sudoku, Pathfinding, RoadR, Citeseer, and Cora, and a \textit{DSPot} is used for Logical Deduction.

To investigate constraint satisfaction and joint reasoning, we use the dataset settings outlined in \secref{sec:datasets-models} for Visual-Sudoku, Pathfinding, RoadR, Citeseer, Cora, and Logic Deduction.
Additionally, we introduce the following variant of the MNIST-Add-$k$ dataset.
\begin{itemize}
    \item \textbf{MNIST-Add}$k$: The $k=1,2,4$ MNIST-Add$k$ datasets with the sums of the MNIST-Add-$k$ equations available as observations during inference.
    Prediction performance is measured by the accuracy of the image classifications.
\end{itemize}
The MNIST-Add-$k$ modification allows the NeSy-EBM to use the semantics of addition and the sum observation to form constraints to correct the neural component predictions.
For instance, consider the MNIST-Add-$1$ equation
$\big[\inlinegraphics{sections/experiments/figures/MNIST-3.png}\big] + \big[\inlinegraphics{sections/experiments/figures/MNIST-5.png}\big] = \mathbf{8}$.
If the neural component incorrectly classifies the first MNIST image, $\inlinegraphics{sections/experiments/figures/MNIST-3.png}$, as an $8$ with low confidence but correctly classifies the second MNIST image, $\inlinegraphics{sections/experiments/figures/MNIST-5.png}$, as a $5$ with high confidence, then it can use the sum label, $8$, to correct the first digit label.

\begin{table}[ht]
    \centering
    \caption{
        Digit accuracy and constraint satisfaction consistency of the ResNet18 and NeuPSL models on the MNIST-Add-$k$ and Visual-Sudoku datasets.
    }
    \label{tab:joint-reasoning-mnist-datasets}
    \begin{tabular}{l||cc|cc}
         \toprule 
         & \multicolumn{2}{c}{\textbf{ResNet18}} & \multicolumn{2}{c}{\textbf{NeuPSL}} \\
         & \textbf{Digit Acc.} & \textbf{Consistency} & \textbf{Digit Acc.} & \textbf{Consistency} \\
         \midrule
         \midrule
         \emph{MNIST-Add1} & $97.60 \pm 0.55$ & $93.04 \pm 1.33$ & $\mathbf{99.80 \pm 0.14}$ & $\mathbf{100.0 \pm 0.00}$ \\
         \emph{MNIST-Add2} & $97.60 \pm 0.55$ & $86.56 \pm 2.72$ & $\mathbf{99.68\pm 0.22}$ & $\mathbf{100.0 \pm 0.00}$ \\
         \emph{MNIST-Add4} & $97.60 \pm 0.55$ & $75.04 \pm 4.81$ & $\mathbf{99.72 \pm 0.29}$ & $\mathbf{100.0 \pm 0.00}$ \\
         \emph{Visual-Sudoku} & $97.60 \pm 0.55$ & $70.20 \pm 2.17$ & $\mathbf{99.37 \pm 0.11}$ & $\mathbf{100.0 \pm 0.00}$ \\
         \bottomrule
    \end{tabular}
\end{table}

\begin{table}[ht]
    \centering
    \caption{Accuracy of finding a minimum cost path (Min. Cost Acc.) and consistency in satisfying continuity constraints (Continuity) of the ResNet18 and NeuPSL models on the Pathfinding dataset.}
    \label{tab:joint-reasoning-pathfinding}
    \scalebox{0.9}{
    \begin{tabular}{l||cc|cc}
         \toprule 
         & \multicolumn{2}{c}{\textbf{ResNet18}} & \multicolumn{2}{c}{\textbf{NeuPSL}} \\
         & \textbf{Min. Cost Acc.} & \textbf{Continuity} & \textbf{Min. Cost Acc.} & \textbf{Continuity} \\
         \midrule
         \midrule
         \emph{Pathfinding} & $ 80.12 \pm 22.44 $ & $ 84.80 \pm 17.11 $ & $ \mathbf{90.02 \pm 11.70} $ & $\mathbf{100.0 \pm 0.00}$ \\
         \bottomrule
    \end{tabular}
    }
\end{table}

\begin{table}[ht]
    \centering
    \caption{Object detection F1 and constraint satisfaction consistency of the DETR and NeuPSL models on the RoadR dataset.}
    \label{tab:joint-reasoning-roadr}
    \begin{tabular}{l||cc|cc}
        \toprule 
         & \multicolumn{2}{c}{\textbf{DETR}} & \multicolumn{2}{c}{\textbf{NeuPSL}} \\
         & \textbf{F1} & \textbf{Consistency} & \textbf{F1} & \textbf{Consistency} \\
         \midrule
         \midrule
         \emph{RoadR} & $0.457$ & $27.5$ & $\textbf{0.461}$ & $\textbf{100.0}$ \\
         \bottomrule
    \end{tabular}
\end{table}

\cref{tab:joint-reasoning-mnist-datasets,tab:joint-reasoning-pathfinding,tab:joint-reasoning-roadr} report the prediction performance and constraint satisfaction consistency of a neural baseline and NeuPSL model on the MNIST-Add$k$, Visual-Sudoku, Pathfinding, and RoadR datasets, respectively.
Across all settings, the baseline neural models frequently violate constraints within the test dataset.
Further, the frequency of these violations increases with the complexity of the constraints.
This behavior is best illustrated in the MNIST-Add$k$ datasets, where consistency decreases as the number of digits, $k$, increases.
This decline can be attributed to the baseline ResNet18 model treating each digit prediction independently and thus failing to account for the dependencies from the sum relation.
For Visual-Sudoku, the baseline ResNet18 attains $97.60 \pm 0.55$ digit accuracy but only $70.20 \pm 2.17$ consistency, whereas NeuPSL reaches $99.37 \pm 0.11$ digit accuracy and $100.0 \pm 0.00$ consistency.
Moreover, in the RoadR experiment, the DETR baseline adheres to road event constraints only $27.5\%$ of the time.
On the other hand, NeuPSL always satisfies the problem constraints in the MNIST-Add$k$, Visual-Sudoku, Pathfinding, and RoadR datasets, achieving $100 \%$ consistency.
For RoadR, this large gain in constraint satisfaction consistency is achieved while maintaining comparable object detection F1 to the DETR baseline on the reported split.\footnote{This experiment highlights the ability to perform consistent reasoning, and therefore the result is reported from a single run to reduce computational cost, as additional runs would not affect the consistency. While we do not claim that this necessarily improves F1 performance, prior work \cite{giunchiglia:ml23} and subsequent studies exploring different forms of symbolic knowledge representation \cite{stoian:nesy23} demonstrate substantial improvements over purely neural baselines.}
This is because the DSPar NeSy-EBM models used in these experiments can enforce constraints on all target variables.
This allows the NeSy-EBM models to leverage the structural relations inherent in the constraints to infer target variables and jointly improve or preserve prediction performance.
Prediction performance gains from constraint satisfaction and joint reasoning are possible when the neural component accurately quantifies its confidence.
The symbolic component uses the confidence of the neural component and the constraints together to correct the neural model's erroneous predictions.
This observation motivates an exciting avenue of future research: exploring whether calibrating the confidence of the neural component can further improve the structured prediction and joint reasoning capabilities of NeSy-EBMs.

\begin{table}[ht]
    \centering
    \caption{Node classification accuracy of the SGC and NeuPSL models on the Citeseer and Cora datasets.}
    \label{tab:joint-reasoning-node-classification}
    \begin{tabular}{l||c|c}
        \toprule 
         & \textbf{SGC} & \textbf{NeuPSL} \\
         \midrule
         \midrule
         \emph{Citeseer} & $65.14 \pm 2.96$ & $\mathbf{66.52 \pm 3.26}$ \\
         \emph{Cora} & $80.90 \pm 1.54$ & $\mathbf{81.82 \pm 1.73}$ \\
         \bottomrule
    \end{tabular}
\end{table}

\begin{table}[ht]
    \centering
    \caption{Comparison of accuracy in answering logical deduction questions using two large language models, GPT-3.5-turbo and GPT-4 \cite{openai:techreport24}, across three methods: Standard, Chain of Thought (CoT), and NeuPSL.}
    \label{tab:constraint-satisfaction-logical-deduction}
    \begin{tabular}{cc||cc|c}
        \toprule
        & \textbf{LLM} & \textbf{Standard} & \textbf{CoT} & \textbf{NeuPSL} \\
        \midrule
        \midrule
        \multirow{2}{*}{\textit{Logical Deduction}} & \emph{GPT-3.5-turbo} & 40.00 & 42.33 & \textbf{70.33} \\
        & \emph{GPT-4} & 71.33 & 75.25 & \textbf{90.67} \\
        \bottomrule
    \end{tabular}
\end{table}

Unlike MNIST-Add$k$, Visual-Sudoku, Pathfinding, and RoadR, which have hard constraints on the target variables, the citation network datasets showcase the capacity of NeSy-EBMs to perform joint reasoning with constraints and dependencies that are not strictly adhered to.
For Citeseer and Cora, NeuPSL enhances prediction accuracy by leveraging the homophilic structure of the citation networks, i.e., papers that are linked tend to share topic labels.
Similarly, in the question-answering logical deduction problem, NeuPSL uses an LLM to generate rules representing the dependencies described in natural language.
Although the LLM may sometimes fail to generate accurate rules, NeuPSL will consistently use the rules for logical reasoning.

\cref{tab:joint-reasoning-node-classification,tab:constraint-satisfaction-logical-deduction} report the baseline and NeuPSL NeSy-EBM prediction performance on the citation network node classification and logical deduction datasets, respectively.
Across all studied settings, NeuPSL outperforms the corresponding baseline.
The performance gain from NeuPSL in the citation network experiments is verified to be statistically significant with a paired t-test and p-value less than $0.05$.
Further, in the Logical Deduction setting, NeuPSL obtains a $15\%$ improvement over the LLM.
This performance gain is achieved despite the fact that the LLM neural component in NeuPSL could produce invalid syntax or an infeasible set of logical constraints.
The LLM was able to produce valid programs $89.0\%$ and $98.7\%$ of the time with gpt-3.5-turbo and gpt-4, respectively.
This observation motivates a promising avenue of future research in employing self-refinement approaches similar to that of \cite{pan:emnlp23} to attempt to correct the infeasible programs and further improve LLM reasoning capabilities.

The results in this section are evidence for the affirmative answer to the research question RQ1.
Across diverse datasets—MNIST-Add$k$, Visual-Sudoku, Pathfinding, and RoadR—the baseline neural models frequently violate constraints, with violations increasing in complexity.
For instance, the baseline model's performance on MNIST-Add$k$ deteriorates as the number of digits increases, and in RoadR, it adheres to constraints only 27.5\% of the time.
In stark contrast, the NeSy-EBM framework, incorporating symbolic components to enforce constraints, consistently achieves 100\% constraint satisfaction and improves or preserves prediction performance.
Additionally, in citation network and logical deduction scenarios, NeSy-EBM facilitates joint reasoning, leveraging the inherent structure and logical coherence, leading to significant performance gains.

%% file: sections/experiments/learning.tex
\subsection{NeSy-EBM Learning}
\label{sec:experiments-learning}

Next, we investigate NeSy-EBM learning, focusing on answering research questions RQ2, RQ3, and RQ4.
Our analysis is divided into two parts.
First, we study the performance of modular learning NeSy-EBM methods in \secref{sec:experiments-modular-learning}.
Second, we examine the performance and empirical convergence properties of end-to-end gradient-based NeSy-EBM learning algorithms in \secref{sec:experiments-end-to-end-learning}.
All models within this subsection use the \textit{DSPar} modeling paradigm.

\subsubsection{Modular NeSy-EBM Learning}
\label{sec:experiments-modular-learning}

In our modular learning experiments, the neural components are first trained using supervised neural losses and are then frozen.
The symbolic component is trained using either a minimizer-based or value-based loss.
Specifically, we compare the prediction performance of two value-based losses, Energy and Structure Perceptron (SP), and two minimizer-based losses, Mean Square Error (MSE), and Binary Cross Entropy (BCE).
The modular learning experiments are conducted on the seven datasets listed in \tabref{tab:modular-learning-datasets} below.
The table overviews each dataset's inference task and the corresponding prediction performance metric.
Additional details on these datasets are provided in \appref{appendix:modular-datasets}.

\begin{table}[ht]
    \centering
    \caption{Datasets used for modular experimental evaluations.}
    \begin{tabular}{l||c|c}
        \toprule
        \textbf{Dataset} & \textbf{Task} & \textbf{Perf. Metric} \\
        \midrule
        \midrule
        Debate~\citep{hasan:ijcnlp13} & Stance Class. & AUROC \\
        4Forums~\citep{walker:lrec12} & Stance Class. & AUROC \\
        Epinions~\citep{richardson:iswc03} & Link Pred. & AUROC \\
        DDI~\citep{wishart:nar06} & Link Pred. & AUROC \\
        Yelp~\citep{kouki:recsys15} & Regression & MAE \\
        Citeseer~\citep{sen:aim08} & Node Class. & Accuracy \\
        Cora~\citep{sen:aim08} & Node Class. & Accuracy \\
        \bottomrule
    \end{tabular}
    \label{tab:modular-learning-datasets}
\end{table}

\begin{table}[ht]
    \centering
    \caption{Prediction performance of HL-MRF models trained on value and minimizer-based losses.}
    \label{tab:modular-learning-performance}
    \begin{tabular}{l||c|c||c|c}
        \toprule  
        & \multicolumn{2}{c}{\textbf{Value-Based}} & \multicolumn{2}{c}{\textbf{Bilevel}} \\
        & \textbf{Energy} & \textbf{SP} & \textbf{MSE} & \textbf{BCE} \\
        \midrule
        \midrule
        \textbf{Debate} & $64.76 \pm 9.54$ & $64.68 \pm 11.05$ & $\mathbf{65.33 \pm 11.98}$ & $64.83 \pm 9.70$ \\
        \textbf{4Forums} & $62.96 \pm 6.11$ & $63.15 \pm 6.40$ & $64.22 \pm 6.41$ & $\mathbf{64.85 \pm 6.01}$ \\
        \hline
        \textbf{Epinions} & $78.96 \pm 2.29$ & $79.85 \pm 1.62$ & $\mathbf{81.18 \pm 2.21}$ & $80.89 \pm 2.32$ \\
        \textbf{Citeseer} & $70.29 \pm 1.54$ & $70.92 \pm 1.33$ & $71.22 \pm 1.56$ & $\mathbf{71.94 \pm 1.17}$ \\
        \hline
        \textbf{Cora} & $54.30 \pm 1.74$ & $74.16 \pm 2.32$ & $81.05 \pm 1.41$ & $\mathbf{81.07 \pm 1.31}$ \\
        \textbf{DDI} & $94.54 \pm 0.00$ & $94.61 \pm 0.00$ & $94.70 \pm 0.00$ & $\mathbf{95.08 \pm 0.00}$ \\
        \hline
        \textbf{Yelp} & $18.11 \pm 0.34$ & $18.57 \pm 0.66$ & $18.14 \pm 0.36$ & $\mathbf{17.93 \pm 0.50}$\\
        \bottomrule
    \end{tabular}
\end{table}

\tabref{tab:modular-learning-performance} reports the prediction performance achieved by each of the four learning techniques across the seven modular datasets.
Models trained with bilevel-based losses consistently achieve better average predictive performance than those trained with value-based losses.
Notably, on the Cora dataset, the NeuPSL model trained with the BCE loss achieved a remarkable improvement of over six percentage points compared to the SP loss, which was the better-performing value-based loss.
The models trained with the Energy and SP loss suffered from a collapsed solution, i.e., symbolic parameters giving nearly equal energy to all settings of the target variables.

\subsubsection{End-to-End NeSy-EBM Learning}
\label{sec:experiments-end-to-end-learning}
This subsection analyzes the performance and empirical convergence properties of the three following end-to-end gradient-based NeSy-EBM learning algorithms.

\begin{itemize}
    \item \textbf{Energy}: Gradient descent on the value-based energy loss.

    \item \textbf{Bilevel}: The bilevel value-function optimization for NeSy-EBM learning algorithm.
    For all datasets, binary cross-entropy is the minimizer-based loss, and the energy loss is the value-based loss.

    \item \textbf{IndeCateR}: The stochastic policy optimization algorithmic framework with the Independent Categorical REINFORCE gradient estimator~\citep{desmet:neurips23}.
    The evaluation metric of the dataset is directly applied as the learning loss.
\end{itemize}

\thmref{thm:value-function-gradient} in \secref{sec:learning-losses} is used to compute the learning gradients with respect to the neural output and symbolic weights for the Energy and Bilevel algorithms.
Similarly, the IndeCateR estimate is used to compute the learning gradients with respect to the neural output and symbolic weights for stochastic policy optimization.
Then, gradients with respect to the neural parameters are found via backpropagation for all methods.
The neural parameters are updated via AdamW \citep{loshchilov:iclr19}, and the symbolic parameters are updated using gradient descent with a fixed step size.
Additional details on the hardware and hyperparameters settings of the learning algorithms are provided in \appref{appendix:experiments-hyperparameters}.

To investigate the performance of our NeSy-EBM learning algorithms, we use the dataset settings outlined in \secref{sec:datasets-models} for Citeseer and Cora.
Additionally, we introduce the following variant of the MNIST-Add$k$, Visual-Sudoku, and Pathfinding datasets:
\begin{itemize}
    \item \textbf{MNIST-Add}$k$: The $k=1,2$ MNIST-Add$k$ datasets with no digit supervision, i.e., parameters are learned only from the addition relations.

    \item \textbf{Visual-Sudoku}: A few-shot setting with $5$ labeled examples of each of the $9$ possible classes available for training.
    The remaining images in the training data are unlabeled, and the model must primarily rely on the Sudoku rules for learning.

    \item \textbf{Pathfinding}: A limited supervision setting where only $10\%$ of the training data is labeled, and the remaining training data is unlabeled.
    Specifically, only $5\%$ of the map vertices are observed to be on or off the labeled minimum cost path.
    In other words, supervision is distributed across maps, and the minimum cost paths for a map are only partially observed.
\end{itemize}

\begin{table}[ht]
    \centering
    \caption{
        The average and standard deviation of the prediction performance of NeuPSL NeSy-EBMs trained using gradient-based learning algorithms on $7$ datasets.
    }
    \label{tab:nesy-ebm-learning-performance}
    \begin{tabular}{l||ccc}
         \toprule 
         & \multicolumn{3}{c}{\textbf{NeuPSL}} \\
         & \textbf{Energy} & \textbf{IndeCateR} & \textbf{Bilevel} \\
         \midrule
         \midrule
         \emph{MNIST-Add1} & $93.80 \pm 1.12$ & $94.52 \pm 0.99$ & $\mathbf{94.92 \pm 1.40}$ \\
         \emph{MNIST-Add2} & $87.92 \pm 1.63$ & $86.88 \pm 1.82$ & $\mathbf{89.36 \pm 1.54}$ \\
         \emph{Visual-Sudoku} & $\mathbf{98.12 \pm 0.37}$ & TIMEOUT & $98.10 \pm 0.19$ \\
         \emph{Path-Finding} & $22.53 \pm 0.75$ & TIMEOUT & $\mathbf{22.85 \pm 1.33}$ \\
         \emph{Citeseer} & $67.04 \pm 1.82 $ & TIMEOUT & $\mathbf{67.96 \pm 1.11}$ \\
         \emph{Cora} & $80.40 \pm 0.74$ & TIMEOUT & $\mathbf{81.88 \pm 0.65}$ \\
         \bottomrule
    \end{tabular}
\end{table}

\tabref{tab:nesy-ebm-learning-performance} presents the average and standard deviation of the prediction performance for the symbolic component of the NeuPSL NeSy-EBM model across the six datasets examined in this subsection.
In five of the six datasets, the Bilevel learning algorithm achieves the best results.
Notably, in MNIST-Add$1$, IndeCateR's performance was comparable to Bilevel's.
However, as the complexity of the target variable constraints increased, IndeCateR's performance deteriorated, exemplified by poor results in MNIST-Add$2$ and failures to find viable solutions within the allotted time in the other datasets.
This pattern is consistent with the known difficulty of stochastic policy optimization on larger structured prediction problems.
IndeCateR sums over values of individual categorical decisions while sampling the remaining ones, which is more efficient than full joint enumeration but still leads to high-variance updates and requires many rounds of inference before a parameter update.
Slower wall-clock convergence and timeouts are not surprising in these structured prediction settings, and variance reduction and more efficient policy-gradient estimators are promising directions for future work.

\begin{figure}[ht]
    \caption{Validation image classification accuracy versus training (a) epoch and (b) time in minutes for NeuPSL models trained with the Energy, IndeCateR, and Bilevel NeSy-EBM learning algorithms.}
    \begin{subfigure}{\textwidth}
        \centering
        \includegraphics[width=0.99 \textwidth]{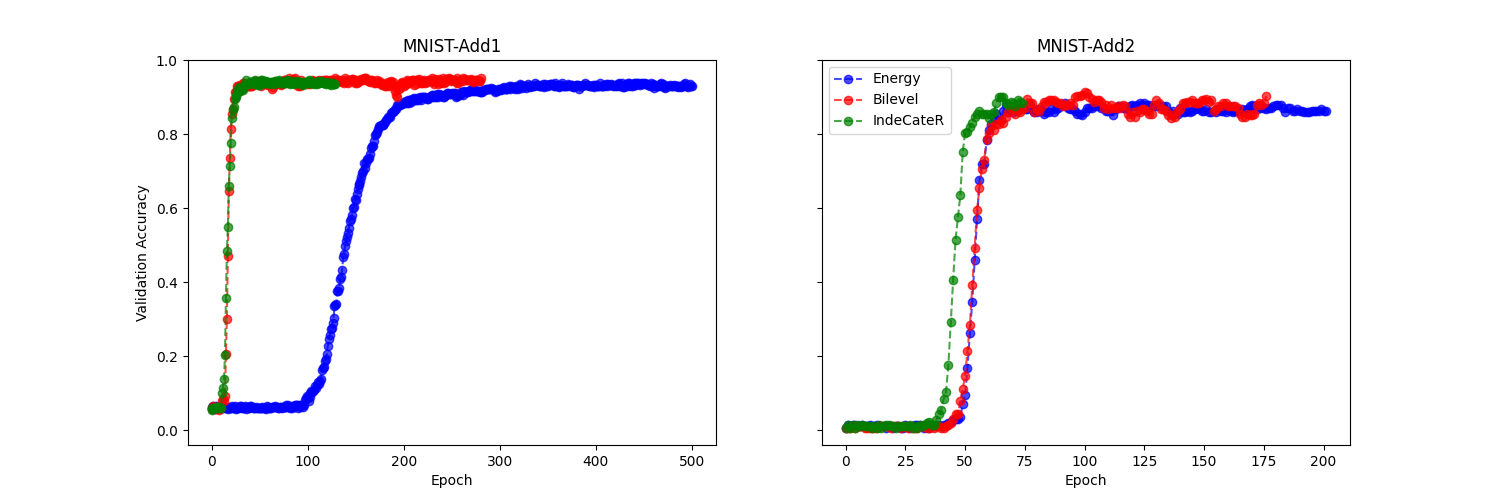}
        \caption{}
        \label{fig:mnist-learning-iteration-convergence}
    \end{subfigure}
    \begin{subfigure}{\textwidth}
        \centering
        \includegraphics[width=0.99 \textwidth]{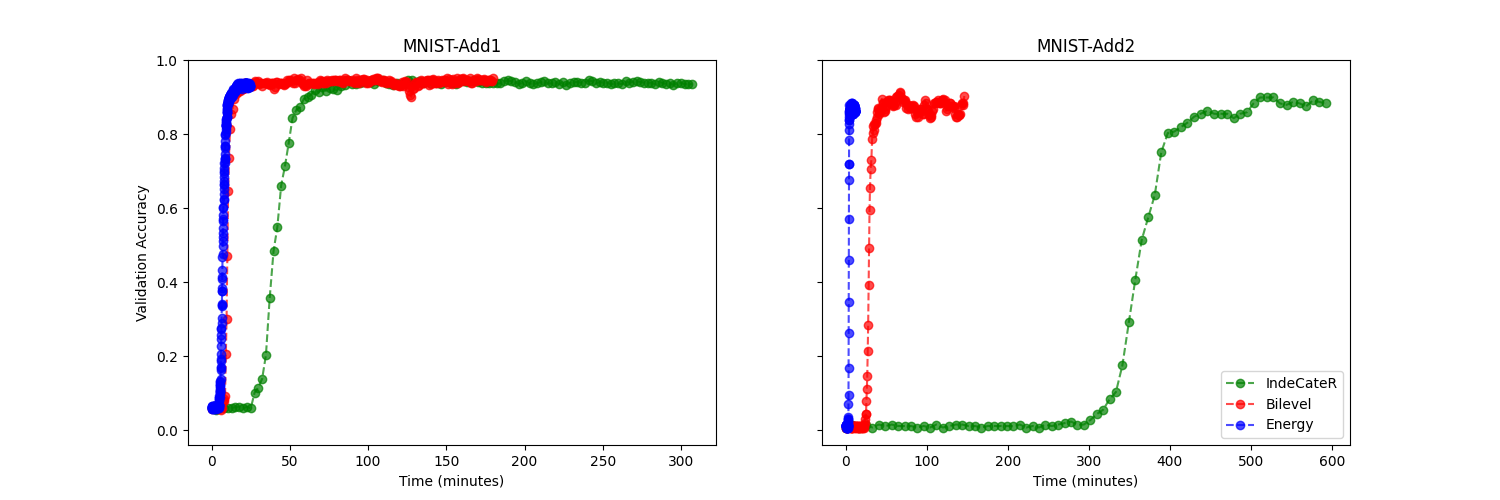}
        \caption{}
        \label{fig:mnist-learning-timing-convergence}
    \end{subfigure}
    \label{fig:mnist-learning-convergence}
\end{figure}

While Energy generally underperformed compared to Bilevel across most settings, it was the fastest in execution time.
For instance, \figref{fig:mnist-learning-convergence} plots the validation image classification accuracy of the MNIST-Add$1$ and MNIST-Add$2$ NeuPSL NeSy-EBMs trained with the Energy, IndeCateR, and Bilevel learning algorithms versus the training epoch and wall-clock time for a single fold.
The Bilevel and IndeCateR algorithms reach higher validation performance levels than the Energy algorithm on both MNIST-Add$k$ datasets for the reported fold.
This pattern is consistent with the average prediction performance results reported in \tabref{tab:nesy-ebm-learning-performance} for MNIST-Add$1$.
For the MNIST-Add$2$ dataset, on the other hand, the IndeCateR algorithm was timed out after $10$ hours of training rather than allowing it to fully converge, which explains the drop in the relatively lower average test performance results in \tabref{tab:nesy-ebm-learning-performance}.
In the reported fold, the plotted runs illustrate that IndeCateR improves most rapidly with respect to training epochs on both datasets, followed by Bilevel and then Energy.
However, the IndeCateR algorithm's per-iteration cost counteracts this advantage, and the plotted runs show a slower rate of improvement with respect to wall-clock time. 
On the other end of the spectrum, Energy shows the slowest improvement in prediction performance with respect to training epochs, but its per iteration cost is low enough that it improves the fastest with respect to wall-clock time.
The Bilevel algorithm illustrates a balance between these two behaviors in the plotted runs.
It has a lower per-iteration cost because it only uses value-function gradients and optimizes a minimizer-based loss.
These convergence results in \figref{fig:mnist-learning-convergence} illustrate a tradeoff that motivates future work on training pipelines that pre-train with a value-based loss and fine-tune with a more expensive minimizer-based loss to achieve the fastest training time and best final prediction performance.

%% file: sections/experiments/semi-supervision.tex
\subsection{Semi-Supervision}
\label{sec:experiments-semi-supervision}

In this set of experiments, we investigate the effectiveness of the NeSy-EBM framework as instantiated through NeuPSL in training a deep learning model in a semi-supervised setting.
This experiment aims to further investigate research questions RQ3 and RQ4.
Specifically, we compare the prediction performance of a neural baseline trained solely with a supervised neural loss to that of a NeuPSL model's neural component (with the same architecture) trained using an end-to-end NeSy-EBM loss.
In both cases, only a subset of the training set labels is available to the neural component.
To enhance neural performance with a structured loss, the MNIST-Addk and Visual-Sudoku models in this subsection employ the \textit{DSVar} modeling paradigm due to its simplicity and speed, while Pathfinding, Citeseer, and Cora models use the \textit{DSPar} modeling paradigm.
We use the following dataset variants for our experiments.
\begin{itemize}
    \item \textbf{MNIST-Add}$k$: The $k=1,2$ MNIST-Add$k$ datasets with the proportion of image class labels available in the training data varying over $\{1.0, 0.5, 0.1, 0.05\}$.
    Prediction performance in this subsection is measured by the accuracy of the image classifications.

    \item \textbf{Visual-Sudoku}: The proportion of image class labels available in the training data varies over $\{1.0, 0.5, 0.1, 0.05\}$.

    \item \textbf{Pathfinding}: Supervision is distributed across all training maps, so the shortest paths in the training data are only partially observed. 
    The proportion of vertex labels available in the training data varies over $\{1.0, 0.5, 0.1\}$.

    \item \textbf{Citeseer and Cora}: The proportion of paper topic labels available in the training data varies over $\{1.0, 0.5, 0.1, 0.05\}$.
\end{itemize}
The Bilevel learning algorithm is applied to train the NeSy-EBM neural components for the MNIST-Add$k$, Citeseer, and Cora datasets.
The Energy learning algorithm is applied to train the NeSy-EBM neural components for the Visual-Sudoku and Pathfinding datasets.
Additional details on the hardware and hyperparameter settings of the learning algorithms are provided in \appref{appendix:experiments-hyperparameters}.

\begin{table}[htb]
    \centering
    \caption{Digit accuracy of the ResNet18 models trained with varying levels of supervision.}
    \label{tab:semi-supervision-mnist-datasets}
    \begin{tabular}{lc||c|c}
         \toprule 
         & & \multicolumn{2}{c}{\textbf{ResNet18}} \\
         & & \multirow{2}{*}{\textbf{Supervised}} & \textbf{NeuPSL} \\
         & \textbf{Labeled} & & \textbf{Semi-Supervised} \\
         \midrule
         \midrule
         \multirow{4}{*}{\emph{MNIST-Add1}} 
         & $1.00$ & $ \mathbf{97.84 \pm 0.23} $ & $ 97.40 \pm 0.51 $ \\
         & $0.50$ & $ \mathbf{97.42 \pm 0.30} $ & $ 97.02 \pm 0.65 $ \\
         & $0.10$ & $ 93.05 \pm 0.69 $ & $ \mathbf{96.78 \pm 0.80} $ \\
         & $0.05$ & $ 75.35 \pm 0.33 $ & $ \mathbf{96.82 \pm 0.72} $ \\
         \hline
         \multirow{4}{*}{\emph{MNIST-Add2}}
         & $1.00$ & $ \mathbf{97.84 \pm 0.23} $ & $ 97.22 \pm 0.19 $ \\
         & $0.50$ & $ \mathbf{97.42 \pm 0.30} $ & $ 96.84 \pm 0.42 $ \\
         & $0.10$ & $ 93.05 \pm 0.69 $ & $ \mathbf{95.14\pm 1.21} $ \\
         & $0.05$ & $ 75.35 \pm 0.33 $ & $ \mathbf{95.90 \pm 0.43} $ \\
         \hline
         \multirow{4}{*}{\emph{Visual-Sudoku}}
         & $1.00$ & $ 97.84 \pm 0.23 $ & $ \mathbf{97.89 \pm 0.15} $ \\
         & $0.50$ & $  \mathbf{97.42 \pm 0.30} $ & $ 97.26 \pm 0.70 $ \\
         & $0.10$ & $ 93.05 \pm 0.69 $ & $ \mathbf{96.82 \pm 0.32} $ \\
         & $0.05$ & $ 75.35 \pm 0.33 $ & $ \mathbf{96.49 \pm 0.67} $ \\
         \bottomrule
    \end{tabular}
\end{table}

\begin{table}[htb]
    \centering
    \caption{Topic accuracy of the trained SGC models with varying levels of supervision.}
    \label{tab:semi-supervision-node-classification}
    \begin{tabular}{lc||c|c}
         \toprule 
         & & \multicolumn{2}{c}{\textbf{SGC}} \\
         & & \multirow{2}{*}{\textbf{Supervised}} & \textbf{NeuPSL} \\
         & \textbf{Labeled} & & \textbf{Semi-Supervised} \\
         \midrule
         \midrule
         \multirow{4}{*}{\emph{Citeseer}}
         & $1.00$ & $\mathbf{76.12 \pm 1.71}$ & $75.92 \pm 2.23$ \\
         & $0.50$ & $\mathbf{74.70 \pm 1.68}$ & $74.38 \pm 1.82$ \\
         & $0.10$ & $68.64 \pm 1.06$ & $\mathbf{69.66 \pm 0.16}$ \\
         & $0.05$ & $64.56 \pm 1.68$ & $\mathbf{66.12 \pm 1.22}$ \\
         \hline
         \multirow{4}{*}{\emph{Cora}}
         & $1.00$ & $\mathbf{87.62 \pm 0.97}$ & $87.18 \pm 1.08$ \\
         & $0.50$ & $85.82 \pm 0.50$ & $\mathbf{86.74 \pm 0.54}$ \\
         & $0.10$ & $80.88 \pm 2.00$ & $\mathbf{81.96 \pm 2.62}$ \\
         & $0.05$ & $74.98 \pm 3.32$ & $\mathbf{78.88 \pm 2.85}$ \\
         \bottomrule
    \end{tabular}
\end{table}

\begin{table}[htb]
    \centering
    \caption{Accuracy of finding a minimum cost path (Min. Cost Acc.) and consistency in satisfying continuity constraints (Continuity) of the ResNet18 models with varying levels of supervision.}
    \label{tab:semi-supervision-pathfinding}
    \scalebox{0.8}{
    \begin{tabular}{lc||cc|cc}
         \toprule 
         & & \multicolumn{4}{c}{\textbf{ResNet18}} \\
         & & \multicolumn{2}{c}{\multirow{2}{*}{\textbf{Supervised}}} & \multicolumn{2}{c}{\textbf{NeuPSL}} \\
         & & & & \multicolumn{2}{c}{\textbf{Semi-Supervised}} \\
         & \textbf{Labeled} & \textbf{Min. Cost Acc.} & \textbf{Continuity} & \textbf{Min. Cost Acc.} & \textbf{Continuity} \\
         \midrule
         \midrule
         \multirow{3}{*}{\emph{Pathfinding}}
         & $1.00$ & $ 80.12 \pm 22.44 $ & $ \mathbf{84.80 \pm 17.11} $ & $ \mathbf{80.90 \pm 21.93} $ & $ 83.02 \pm 20.09 $ \\
         & $0.50$ & $ 52.06 \pm 14.77 $ & $ 61.86 \pm 14.28 $ & $ \mathbf{59.84 \pm 16.51} $ & $ \mathbf{67.94 \pm 14.25} $ \\
         & $0.10$ & $2.60 \pm 1.04$ & $9.02 \pm 1.90$ & $ \mathbf{4.26 \pm 1.40} $ & $ \mathbf{35.18 \pm 3.40} $ \\
         \bottomrule
    \end{tabular}
    }
\end{table}

\cref{tab:semi-supervision-mnist-datasets,tab:semi-supervision-node-classification,tab:semi-supervision-pathfinding} report the average and standard deviation of the prediction performance of the supervised neural baseline and the semi-supervised neural component on the MNIST-Add$k$, Visual-Sudoku, Citeseer, Cora, and Pathfinding datasets.
Across all datasets, as the proportion of unlabeled data increases, the semi-supervised neural component begins to outperform the supervised baseline.
This behavior indicates that NeSy-EBMs are able to leverage the unlabeled training data by using the knowledge encoded in the NeuPSL rules.
The benefit of utilizing symbolic knowledge is most evident in the lowest supervision settings, with the NeuPSL semi-supervised ResNet18 model achieving over $20$ percentage points of improvement when there is only $5\%$ percent of the training labels in the MNIST-Add$k$ and Visual-Sudoku datasets.
Surprisingly, this outcome is repeated in the Citeseer and Cora datasets, where the NeuPSL rules are not always adhered to.
In other words, leveraging domain knowledge becomes more valuable for improving prediction performance as the amount of supervision decreases, even if the domain knowledge is not strictly accurate.

The Pathfinding results in \tabref{tab:semi-supervision-pathfinding} show there is not only a prediction performance gain achievable by making use of the symbolic component but also a reliability improvement.
The reported Continuity metric measuring the consistency of the ResNet18 model in satisfying path continuity constraints is significantly improved when there is limited supervision and the model is trained with a NeSy-EBM loss. 
The NeuPSL semi-supervised ResNet18 model attains an over $25$ percentage point improvement in path continuity consistency when only $10\%$ of training labels are available.
These results show NeSy-EBMs are valuable for aligning neural networks with desirable properties beyond accuracy.

%% file: sections/limitations/limitations.tex
\section{Limitations}
\label{sec:limitations}

In this section, we discuss the limitations of the NeSy-EBM framework, NeuPSL, and our empirical analysis.
The NeSy-EBM framework is a high-level and general paradigm for NeSy.
The value of the framework is that it provides a  unifying theory for NeSy and a foundation for creating widely applicable modeling paradigms and learning algorithms. 
Progress on developing highly efficient NeSy inference algorithms, on the other hand, may benefit from a perspective that considers the specific structure of the energy function and inference task.
For instance, the inference task of density estimation for NeSy systems such as semantic probabilistic layers \citep{ahmed:neurips22} is made highly efficient by levering constraints on the design of the energy function.
Similarly, we show prediction in NeuPSL is a quadratic program, a property that is leveraged by \cite{dickens:icml24} to create an inference algorithm tailored for leveraging warm starts to realize learning runtime improvements.
In this work, we make limited assumptions on the form of the energy function to develop modeling paradigms and learning algorithms and do not focus on building or analyzing specific inference techniques.

The collection of NeSy modeling paradigms introduced in \secref{sec:nesy-ebms-a-taxonomy-of-modeling-paradigms} is not exhaustive.
For instance, it omits NeSy systems that integrate symbolic knowledge extraction from deep neural networks \citep{tran:ieee18}.
Moreover, we do not discuss DSVar, DSPar, and DSPot model combinations.
We leave the exploration of utilizing multiple NeSy modeling paradigms to fuse neural components operating over multiple modalities for future work.

The four learning techniques proposed in this manuscript are presented with necessary assumptions on the energy function.
For instance, direct gradient descent can only be principally applied to minimize a NeSy-EBM loss at points where the energy function is twice differentiable with respect to the neural output and symbolic weights.
Similarly, the bilevel technique is principled at points where the optimal value-function is differentiable with respect to the neural output and symbolic weights.
We do not explore methods for extending the gradient descent and bilevel learning techniques to support NeSy-EBMs that do not satisfy all assumptions.
One approach is to substitute the inference program with an approximation.
The modular and stochastic policy optimization learning techniques require significantly fewer assumptions on the form of the energy function.
However, these two techniques have their own limitations, which we discuss in their respective subsections.

The NeuPSL system, while expressive, does not support every NeSy-EBM energy function and inference task.
Specifically, NeuPSL can create energy functions defined as a weighted sum of potentials derived via arithmetic, logic, and Lukasiewicz real-logic semantics, as described in \secref{sec:neupsl-and-deep-hlmrfs}.
NeuPSL does not support potentials constructed from other real-logic semantics.
Further, NeuPSL is currently only designed to perform non-probabilistic inference tasks (e.g., prediction, ranking, and detection). 
This is due to the complexities of computing marginal distributions with the Gibbs partition function defined from the energy.

The primary contributions of this paper are the NeSy-EBM mathematical framework, the NeuPSL instantiation, and the suite of learning techniques.
Accordingly, the empirical analysis is designed to study NeSy-EBMs as instantiated through NeuPSL in the selected settings rather than to provide an exhaustive comparison against external NeSy frameworks.
Readers interested in broader cross-framework comparisons should see the surveys and taxonomies discussed in \secref{sec:related-work-nesy-frameworks}, prior NeuPSL papers \cite{pryor:ijcai23, dickens:icml24, dickens:make24, dickens:phd24, pryor:thesis24}, as well as the application-specific empirical evaluations summarized in \secref{sec:related-work-applications}.
Further, our empirical evaluations do not encompass every NeSy application, for instance, reasoning with noisy data.
Finally, although our research advances the incorporation of commonsense reasoning and domain knowledge into LLMs for question answering, we have not extended our investigation to more complex reasoning tasks like summarization or explanation.

%% file: sections/conclusion/conclusion.tex
\section{Conclusion and Future Work}
\label{sec:conclusion}
This paper establishes a mathematical framework for neural-symbolic (NeSy) reasoning with Neural-Symbolic Energy-Based Models (NeSy-EBMs).
The NeSy-EBM framework is a unifying foundation and a bridge for adapting techniques from the broader machine learning literature to solve challenges in NeSy.
Additionally, we introduce Neural Probabilistic Soft Logic (NeuPSL), an open-source and highly expressive implementation of NeSy-EBMs.
NeuPSL supports the primary modeling paradigms and continuity properties required for efficient end-to-end neural and symbolic parameter learning.

We show that NeSy-EBMs provide a unifying view of NeSy by categorizing fundamental NeSy modeling paradigms.
Our modeling paradigms organize the strengths and limitations of NeSy systems and clarify architecture requirements for applications.
NeSy-EBMs and the paradigms are valuable mechanisms for practitioners and researchers to understand the growing NeSy literature and design effective systems.
Further, NeSy-EBMs illuminate connections between NeSy and the broader machine learning community.
Specifically, we formalize a general NeSy learning loss and the necessary assumptions for supporting direct gradient descent on the loss.
Moreover, we leverage methods from reinforcement learning and bilevel optimization literature to work around the assumptions and design more practical and general algorithms.

The insights we gained from creating the mathematical framework, the collection of modeling paradigms, and the suite of learning techniques shaped the development of the NeuPSL NeSy modeling library. 
NeuPSL is built to support every modeling paradigm and learning technique we cover.
We demonstrate the effectiveness of NeuPSL in our empirical analysis.
Specifically, we explore four practical use cases of NeSy.
We show compelling results in real-world applications and see NeSy-EBMs enhance neural network predictions, enforce constraints, improve label and data efficiency, and empower LLMs with consistent reasoning.

Looking ahead, several promising avenues for future research have emerged.
For instance, a more extensive exploration into techniques for leveraging symbolic knowledge to fine-tune and adapt foundation models is a promising direction.
The NeSy-EBM framework and our proposed learning techniques are a solid basis for building pipelines to fine-tune foundation models.
Moreover, stochastic policy optimization for end-to-end NeSy learning has great potential due to its general applicability to every modeling paradigm and most NeSy-EBMs.
Finally, contributing to the active area of research on overcoming the challenge of high-variance gradient estimates would be highly beneficial for improving NeSy learning.

%% file: sections/acknowledgements/acknowledgements.tex
\begin{acks}
This work was partially supported by the National Science Foundation grants CCF-2023495 and a Google Faculty Research Award.
\end{acks}

%% file: appendix/introduction.tex
\section{Introduction}
\label{appendix:introduction}

The appendix includes the following sections: Extended Related work (\appref{appendix:extended-related-work}), Formal Modeling Paradigm Definitions (\appref{sec:appendix-modeling-paradigms-formal}), NeSy Approaches as NeSy-EBMs (\appref{sec:appendix-nesy-approaches-as-nesy-ebms}), Extended Neural Probabilistic Soft Logic (\appref{appendix:extended_neupsl}), and Extended Empirical Analysis (\appref{appendix:experiments}).

%% file: appendix/extended-related-work.tex
\section{Extended Related Work}
\label{appendix:extended-related-work}

\subsection{Bilevel Optimization}
\label{appendix:extended-related-work-bilevel-optimization}

In this work, we use bilevel optimization as a natural formulation of learning for a general class of NeSy systems \citep{bracken:or71, colson:aor07, bard:book13, dempe:book20}.
The NeSy learning objective is a function of predictions obtained by solving a lower-level program that encapsulates symbolic reasoning.
In the broader deep learning community, bilevel optimization also arises in hyperparameter optimization \citep{pedregosa:icml16}, meta-learning \citep{franceschi:icml18, rajeswaran:neurips19}, generative adversarial networks \citep{goodfellow:neurips14}, and reinforcement learning \citep{sutton:book18}. 
Researchers typically take one of the following three approaches to bilevel optimization.

\textbf{Implicit Differentiation.} 
There is a long history of research on analyzing the stability of solutions to optimization problems using implicit differentiation \citep{fiacco:book68, robinson:mor80, bonnans:book00}.
These methods compute or approximate the Hessian matrix at the lower-level problem solution to derive an analytic expression for the gradient of the upper-level objective, sometimes called a hypergradient.
Bilevel algorithms of this type make varying assumptions about the problem structure \citep{do:neurips07, pedregosa:icml16, ghadimi:arxiv18, rajeswaran:neurips19, giovannelli:arxiv22, khanduri:icml23}.
Building on these foundational techniques, the deep learning community has proposed architectures that contain layers that are functions of convex programs with analytic expressions for gradients derived from implicit differentiation \citep{amos:icml17, agrawal:neurips19, agrawal:jano19, wang:icml19}. 

\textbf{Automatic Differentiation.}
This approach unrolls inference into a differentiable computational graph \citep{stoyanov:aistats11, domke:aistats12, belanger:icml17, ji:icml21}, and then leverages automatic differentiation techniques \citep{griewank:book08}.
However, unrolling the inference computation creates a large, complex computational graph that can accumulate numerical errors dependent on the solver.

\textbf{Bilevel Value-Function Approach.}
An increasingly popular approach is to reformulate the bilevel problem as a single-level constrained program using the optimal value of the lower-level objective (the value-function) to develop principled gradient-based algorithms that do not require the calculation of Hessian matrices for the lower-level problem \citep{outrata:zor90, ye:opt95, liu:icml21, sow:arxiv22, liu:neurips22, liu:arxiv23, kwon:icml23}.
Existing bilevel value-function approaches are not directly applicable to NeSy systems as they typically assume the lower-level problem to be unconstrained and the objective to be smooth.
Bilevel optimization with constraints in the lower level problem is an open area of research. 
Until now, implicit differentiation methods are applied with strong assumptions about the structure of the lower-level problem \citep{giovannelli:arxiv22, khanduri:icml23}.
Our framework is, to the best of our knowledge, the first value-function approach to work with lower-level problem constraints.

\subsection{Energy-Based Models}
\label{appendix:extended-related-work-energy-based-models}

Our NeSy framework makes use of Energy-Based Models (EBMs) \citep{lecun:book06}.
EBMs measure the compatibility of a collection of observed (input) variables $\mathbf{x} \in \mathcal{X}$ and target (output) variables $\mathbf{y} \in \mathcal{Y}$ via a scalar-valued energy function: $E: \mathcal{Y} \times \mathcal{X} \to \mathbb{R}$.
Low energy states represent high compatibility.
EBMs are appealing due to their generality in both modeling and application. 
For instance, EBMs can be used to perform density estimation by defining conditional, joint, and marginal Gibbs distributions with the energy function:
\begin{align}
    & P(\mathbf{y} \vert \mathbf{x}) := \frac{e^{-\beta E(\mathbf{y}, \mathbf{x})}}{\int_{\hat{\mathbf{y}} \in \mathcal{Y}} e^{-\beta E(\hat{\mathbf{y}}, \mathbf{x})}}, \label{eq:gibbs_conditional}\\
    & P(\mathbf{y}, \mathbf{x}) := \frac{e^{-\beta E(\mathbf{y}, \mathbf{x})}}{\int_{\hat{\mathbf{y}} \in \mathcal{Y}, \hat{\mathbf{x}} \in \mathcal{X}} e^{-\beta E(\hat{\mathbf{y}}, \hat{\mathbf{x}})}}, \label{eq:gibbs_joint}\\
    & P(\mathbf{x}) := \frac{\int_{\hat{\mathbf{y}} \in \mathcal{Y}}e^{-\beta E(\mathbf{y}, \mathbf{x})}}{\int_{\hat{\mathbf{y}} \in \mathcal{Y}, \hat{\mathbf{x}} \in \mathcal{X}} e^{-\beta E(\hat{\mathbf{y}}, \hat{\mathbf{x}})}} \label{eq:gibbs_marginal}.
\end{align}
A fundamental motivation for the use of the Gibbs distribution is that any density function can be represented by the distribution shown above with a (potentially un-normalized) energy function $E$.
For this reason, EBMs are a unified framework for probabilistic and non-probabilistic approaches and are applicable for generative and discriminative modeling.

EBMs are applied throughout machine learning to model data and provide predictions.
The Boltzmann machine \citep{ackley:cs85, salakhutdinov:aistats10} and Helmholtz machine \citep{dayan:nc95} are some of the earliest EBMs to appear in the machine learning literature.
The paper by \citenoun{hinton:nc02} is another seminal work that shows EBMs to be useful for building mixture-of-expert models.
Specifically, a single complex distribution is produced by multiplying many simple distributions together and then renormalizing.

More recently, the EBM framework has been utilized for generative modeling \citep{zhao:iclr17, du:neurips19, du:icml23}.
\citenoun{zhao:iclr17} introduce energy-based generative adversarial networks (EBGANs), which view the GAN discriminator as an energy function that attributes low energies (high compatibility) to points near the data manifold.
The EBGAN approach is a principled framework for using GAN discriminators with a variety of architectures and learning loss functionals to achieve more stable training than traditional GANs.
\cite{du:neurips19} advocate for using EBMs \emph{directly} for generative modeling, citing as motivation their simplicity, stability, parameter efficiency, flexibility of generation, and compositionality.
They show generative results that achieve performance close to modern GANs, achieving state-of-the-art results in out-of-distribution classification, adversarially robust classification, and other tasks.
In more recent work, \citenoun{du:icml23} propose an energy-based parameterization of diffusion models to support compositional generation.

The EBM framework was shown recently to improve discriminative modeling \citep{grathwohl:iclr20, liu:neurips20}.
\citenoun{grathwohl:iclr20} reinterpret discriminative classifiers as EBMs to propose the joint energy-based model (JEM).
A JEM allows the parameters of the model to be fit on unlabeled data with a likelihood-based loss, leading to improved accuracy, robustness, calibration, and out-of-distribution detection.
Similarly, \citenoun{liu:neurips20} developed an EBM for out-of-distribution detection to achieve state-of-the-art performance.
\citenoun{liu:neurips20} creates a purely discriminative training objective (in contrast with the probabilistic approach of JEM) and shows that unnormalized energy scores can be used directly for out-of-distribution detection.

A primary challenge of the EBM framework is learning with a potentially in-tractable partition function induced by the Gibbs distributions in \eqref{eq:gibbs_conditional}, \eqref{eq:gibbs_joint}, and \eqref{eq:gibbs_marginal}. 
Some of the earliest EBMs worked around the partition function using the contrastive divergence algorithm \citep{hinton:nc02} to estimate derivatives of the negative log-likelihood loss of an EBM with Markov chain Monte Carlo (MCMC) sampling from the Gibbs distribution. 
Later work on EBMs has improved the traditional biased MCMC sampling-based approximation methods with a sampler based on stochastic gradient Langevin dynamics (SGLD) \citep{welling:icml11}.
For instance, \citenoun{du:neurips19} use SGLD for training generative EBMs and \citenoun{grathwohl:iclr20} for discriminative models with a negative log-likelihood loss.

Score matching is an alternative probabilistic approach to training an EBM that fits the slope (or score) of the model density to the score of the data distribution, avoiding the need to estimate the Gibbs distribution partition function \citep{hyvarinen:jmlr05, kingma:neurips10, song:neurips19}.
\citenoun{hyvarinen:jmlr05} initially proposed score matching for estimating non-normalized statistical models.
Later, \citenoun{kingma:neurips10} used score matching to train an EBM for image denoising and super-resolution.
\citenoun{song:neurips19} suggested training an EBM to approximate the score of the data distribution that is then used with Langevin dynamics for generation.

EBMs may also be trained via non-probabilistic losses that do not require estimating the Gibbs distribution partition function \citep{lecun:ieee98, collins:emnlp02, scellier:fcn17}.
For instance, the perceptron loss, which is the difference between the energy of the observed training data and the minimum value of the energy function (see \secref{sec:learning-losses} for a formal definition),  has been used for recognizing handwritten digits \citep{lecun:ieee98} and part-of-speech tagging \citep{collins:emnlp02}.
More recently, \citenoun{scellier:fcn17} proposed \emph{equilibrium propagation}, a two-phase learning algorithm for training EBMs with a twice differentiable energy function.
The equilibrium propagation algorithm can be used to train EBMs with an arbitrary differentiable loss.
A step of the learning algorithm proceeds by minimizing the energy given some input (the free phase) and then minimizing the energy augmented with a cost function (the nudged phase).
The gradient of the learning objective is a function of the results of these two minimizations.

The EBM framework has proven effective for a wide range of tasks in both generative and discriminative modeling.
The versatility of EBMs supports modeling complex dependencies, the composition and fusion of models, and leveraging both labeled and unlabeled data.
Moreover, EBMs provide a common theoretical framework spanning probabilistic and non-probabilistic methods.

%% file: appendix/extended-modeling-paradigms.tex
\section{Extended Modeling Paradigms}
\label{sec:appendix-modeling-paradigms-formal}

This section provides formal mathematical definitions for the modeling paradigms introduced in \secref{sec:nesy-ebms-a-taxonomy-of-modeling-paradigms}.

To formalize the modeling paradigms, we introduce an additional layer of abstraction we refer to as \emph{symbolic potentials}, denoted by $\psi$.
Further, we collect symbolic potentials into \emph{symbolic potential sets}, denoted by $\mathbf{\Psi}$.
Symbolic potentials organize the arguments of the symbolic component by the role they play in formulating the prediction program in \eqref{eq:nesy-ebm-prediction}.
This organization is what formally distinguishes the three modeling paradigms.
\begin{definition}
   \label{def:symbolic-potentials}
A \textbf{symbolic potential} $\psi$ is a function of variables from a domain $V_{\psi}$ and parameters from a domain $Params_{\psi}$,  outputting a scalar value: 
\begin{align}
        \psi: V_{\psi} \times Params_{\psi} \to \mathbb{R}.
\end{align}
A \textbf{symbolic potential set}, denoted by $\mathbf{\Psi}$, is a set of potential functions indexed by $\mathbf{J}_{\mathbf{\Psi}}$. \qed
\end{definition}

With this definition, a modeling paradigm is defined by specifying a set of symbolic potentials along with their respective domains.

\subsection{Deep Symbolic Variables}

\begin{definition}
    \label{def:dsvar}
    In the \textbf{deep symbolic variables} (DSVar) modeling paradigm the symbolic potential set is a singleton $\mathbf{\Psi} = \{\psi\}$ with a trivial index set $\mathbf{J}_{\mathbf{\Psi}} = \{1\}$ such that $\mathbf{\Psi}_{1} = \psi$.
    Further, the neural prediction is treated as a variable by the symbolic potential; thus $V_{\psi} = \mathcal{Y} \times \mathcal{X}_{sy} \times \mathbb{R}^{d_{nn}}$.
    Then, the symbolic parameters are the symbolic weights, $Params_{\psi} = \mathcal{W}_{sy}$.
    The neural component controls the NeSy-EBM prediction via this function:
    \begin{align}
    \label{eq:deep-symbolic-variables-indicator}
        I_{\mathcal{Y}}(\mathbf{y}, \mathbf{g}_{nn}(\mathbf{x}_{nn}, \mathbf{w}_{nn})) := \begin{cases}
            0 & \mathbf{y}_{i} = \left [ \mathbf{g}_{nn}(\mathbf{x}_{nn}, \mathbf{w}_{nn}) \right]_{i}, \, \forall i \in \{1, \cdots, d_{nn}\} \\
            \infty & \textrm{o.w.}
        \end{cases}, 
    \end{align}
    where $\mathbf{y}_{i}$ and $\mathbf{g}_{nn}(\mathbf{x}_{nn}, \mathbf{w}_{nn})_{i}$ denote the $i$'th entry of the variable and neural output vectors, respectively.
    Then, the symbolic component expressed via the symbolic potential is: 
    \begin{align}
        & g_{sy} (\mathbf{y}, \mathbf{x}_{sy}, \mathbf{w}_{sy}, \mathbf{g}_{nn}(\mathbf{x}_{nn}, \mathbf{w}_{nn})) \\
        & := \psi(\left[\mathbf{y}, \mathbf{x}_{sy}, \mathbf{g}_{nn}(\mathbf{x}_{nn}, \mathbf{w}_{nn})\right], \mathbf{w}_{sy}) + I_{\mathcal{Y}}(\mathbf{y}, \mathbf{g}_{nn}(\mathbf{x}_{nn}, \mathbf{w}_{nn})), \nonumber
    \end{align}
    where $[\cdot]$ denotes concatenation. \qed
\end{definition}

\subsection{Deep Symbolic Parameters}

\begin{definition}
    In the \textbf{deep symbolic parameters} (DSPar) modeling paradigm, the symbolic potential set is a singleton $\mathbf{\Psi} = \{\psi\}$ with a trivial index set $\mathbf{J}_{\mathbf{\Psi}} = \{1\}$ such that $\mathbf{\Psi}_{1} = \psi$.
    Further, the neural prediction is treated as a parameter by the symbolic potential, thus $Params_{\psi} = \mathcal{W}_{sy} \times \mathbb{R}^{d_{nn}}$.
    Then the symbolic variables are the targets and the symbolic inputs: $V_{\psi} = \mathcal{Y} \times \mathcal{X}_{sy}$.
    The symbolic component expressed via the single symbolic potential is: 
    \begin{align}
        g_{sy} (\mathbf{y}, \mathbf{x}_{sy}, \mathbf{w}_{sy}, \mathbf{g}_{nn}(\mathbf{x}_{nn}, \mathbf{w}_{nn})) := \psi(\left[ \mathbf{y}, \mathbf{x}_{sy} \right], \left[\mathbf{w}_{sy}, \mathbf{g}_{nn}(\mathbf{x}_{nn}, \mathbf{w}_{nn}) \right]). \qed \nonumber
    \end{align}
\end{definition}

\subsection{Deep Symbolic Potentials}

\begin{definition}
    In the \textbf{deep symbolic potentials} modeling paradigm, the symbolic potential set $\mathbf{\Psi}$ is the set of all potential functions that can be created by a NeSy framework.
    $\mathbf{\Psi}$ is indexed by the output of the neural component, i.e., $\mathbf{J}_{\Psi} = Range(\mathbf{g}_{nn})$ and $\mathbf{\Psi}_{\mathbf{g}_{nn}(\mathbf{x}_{nn}, \mathbf{w}_{nn})}$ is the potential function indexed by the neural prediction.
    The variable and parameter domains of the sampled symbolic potential are $V_{\psi} = \mathcal{Y} \times \mathcal{X}_{sy}$, and $Params_{\psi} = \mathcal{W}_{sy}$, respectively.
    The symbolic component expressed via the symbolic potential is:  
    \begin{align}
        g_{sy} (\mathbf{y}, \mathbf{x}_{sy}, \mathbf{w}_{sy}, \mathbf{g}_{nn}(\mathbf{x}_{nn}, \mathbf{w}_{nn})) := \mathbf{\Psi}_{\mathbf{g}_{nn}(\mathbf{x}_{nn}, \mathbf{w}_{nn})}(\left[\mathbf{y}, \mathbf{x}_{sy} \right], \mathbf{w}_{sy}). \qed \nonumber
    \end{align}
\end{definition}

%% file: appendix/nesy-approaches-as-nesy-ebms.tex
\section{Expressing NeSy Approaches via NeSy-EBMs}
\label{sec:appendix-nesy-approaches-as-nesy-ebms}

\secref{sec:nesy-ebms} introduced NeSy-EBMs as a unifying framework for understanding neural-symbolic integration.
In this section, we demonstrate how to formulate three widely recognized NeSy approaches within this framework, highlighting their distinct strategies for combining neural learning with symbolic reasoning.
For clarity, each system is described from the perspective of a single modeling paradigm, though many approaches can be expressed in multiple ways within the NeSy-EBM framework.
These formulations reflect common uses of the systems while leaving room for broader or alternative representations as the framework continues to evolve.

\begin{itemize}
    \item \textbf{Semantic Loss (SL)} \citep{xu:icml18}:
    A probabilistic NeSy approach in which the neural network outputs define a probability distribution over variables or facts within probabilistic propositional logical constraints.
    These constraints are integrated as a regularization term within the loss function of the neural network.
    This approach is typically expressed as a deep symbolic parameter model.

    \item \textbf{DeepProbLog (DPL)} \citep{manhaeve:ai21}:
    A probabilistic NeSy approach in which the neural network outputs define a probability distribution over variables or facts within a probabilistic logic setting.
    Probabilistic constraints can be applied either as a loss regularization term or as an additional reasoning layer on top of the neural network outputs.
    This approach is typically expressed as a deep symbolic parameter model.
    
    \item \textbf{Logic Tensor Networks (LTNs)} \citep{badreddine:ai22}:
    A fuzzy NeSy approach in which the neural network outputs define variable values within a fuzzy logic system.
    Fuzzy logic constraints are incorporated through an additional computation graph layer on the neural network outputs.
    This approach is typically expressed as a deep symbolic variable model.
\end{itemize}

\subsection{Semantic Loss (SL)}

Semantic loss is a probabilistic NeSy approach that interprets neural network outputs as probability distributions over propositional variables constrained by logical rules \citep{xu:icml18}.
It measures the extent to which the neural network’s predictions satisfy the given constraints, serving as a regularization term to encourage logical consistency by penalizing violations.

\subsubsection{Definition and Background}

Formally, let $\mathbf{y} = \{y_1, y_2, \dots, y_n\}$ be a set of $n$ propositional variables and a \textit{world} $\omega$ is an instantiation of all variables $\mathbf{y}$, i.e., $\omega = \{y_1=v_1, \cdots, y_n = v_n\}$ for $v_i \in \{0,1\}$. A world $\omega$ satisfies a sentence $\alpha$, denoted $\omega \models \alpha$, if the sentence evaluates to true in that world. A sentence $\alpha$ \textit{entails} another sentence $\beta$, denoted $\alpha \models \beta$, if all worlds that satisfy $\alpha$ also satisfy $\beta$.

In semantic loss, each propositional variable $y_i \in \mathbf{y}$ is assigned a probability $p_i$ that it is true. These probabilities $\mathbf{p} = \{p_1, p_2, \dots, p_n\}$ are defined from a neural network. The \textit{semantic loss}, denoted as $L_{SL}(\alpha, \mathbf{p})$, measures how closely the neural network's predictions satisfy the sentence $\alpha$ with:

\[
L_{SL}(\alpha, \mathbf{p}) \propto - \log \left( \sum_{\omega \models \alpha} \prod_{i: \omega \models y_i} p_i \prod_{i: \omega \models \neg y_i} (1 - p_i) \right)
\]

The summation is taken over all worlds \(\omega\) that satisfy \(\alpha\).
The product \(\prod_{i: \omega \models y_i} p_i\) represents the probability of all variables assigned true in \(\omega\), while \(\prod_{i: \omega \models \neg y_i} (1 - p_i)\) accounts for the probability of all variables assigned false in \(\omega\).  
The right-hand side of the equation corresponds to the well-known reasoning task of weighted model counting (WMC) \citep{chavira:ai08}, which computes the total probability of satisfying \(\alpha\) under a given probability distribution.  

In practice, semantic loss—weighted by a hyperparameter \(\gamma\)—serves as a regularization term that encourages the neural model to produce predictions consistent with logical constraints.
Formally, semantic loss is incorporated as follows:  

\[
L_{\text{total}} = L_{\text{existing}} + \gamma \cdot L_{SL}(\alpha, \mathbf{p}),
\]

where \(L_{\text{existing}}\) represents the original loss function, and \(L_{SL}(\alpha, \mathbf{p})\) enforces logical consistency within the learned predictions.

With the semantic loss formalized, we now illustrate its application through an example: enforcing mutual exclusivity constraints in multi-class classification.

\begin{example}
    Consider a multi-class classification problem where exactly one of $n$ possible outcomes should be true. The sentence $\alpha_{\text{exactly\_one}}$ enforces that exactly one of the variables $y_1, \dots, y_n$ is true:
    \[
    \alpha_{\text{exactly\_one}} = (y_1 \land \neg y_2 \land \dots \land \neg y_n) \lor (\neg y_1 \land y_2 \land \dots \land \neg y_n) \lor \dots \lor (\neg y_1 \land \dots \land y_n).
    \]
    
    The semantic loss for this multi-class setting is computed as:
    \[
    L_{SL}(\alpha_{\text{exactly\_one}}, \mathbf{p}) \propto - \log \left( \sum_{i=1}^{n} p_i \prod_{\substack{j=1 \\ j \neq i}}^{n} (1 - p_j) \right),
    \]
    where $\mathbf{p} = \{p_1, \dots, p_n\}$ are the probabilities predicted by the neural network for each class.
\end{example}

While this demonstrates how more complex logical sentences can be incorporated into the problem, it is crucial to recognize that this requires computing the weighted model count \citep{chavira:ai08}, a task that is \#P-hard.
To mitigate this computational challenge, a common approach in the literature is to compile logical formulas into circuits, enabling efficient inference across multiple queries \citep{choi:unpub20, kisa:kr14, manhaeve:ai21}.
However, while circuit compilation can significantly accelerate inference, the compilation process itself can be exponential in both time and memory, with no guarantees that the resulting circuit size remains tractable \citep{krieken:neurips23}. 

\subsubsection{NeSy-EBM Formulation}

Semantic loss can be formulated as a DSPar NeSy-EBM where the propositional variable probabilities are neural network outputs. 
Since semantic loss does not take as input any observed random variables, then there are no external parameters, i.e., the domains $\mathcal{X}{sy}$ and $\mathcal{W}{sy}$ are empty.
Let $g_{nn}$ be a function with parameters $\mathbf{w}_{nn} \in W_{nn}$ and inputs $\mathbf{x}_{nn} \in X_{nn}$ such that:
\[
g_{nn}: W_{nn} \times X_{nn} \mapsto [0, 1]^n,
\]
where $n$ is the number of propositional variables involved in the constraints, this function outputs the predicted probabilities:
\[
g_{nn}(\mathbf{x}_{nn}, \mathbf{w}_{nn}) = [p_{i}]_{i=1}^{n},
\]
where $p_i$ represents the predicted probability for the $i$-th variable.

Define the logical sentence $\alpha$ as a hard constraint $C_{H}(\omega)$ using an indicator function that represents whether a world $\omega$ (an assignment of the propositional variables) satisfies the sentence $\alpha$:
\[
C_{H}(\omega) = \mathbb{I}(\omega \models \alpha),
\]
where \(\mathbb{I}(\omega \models \alpha)\) is 1 if the world satisfies the constraint \(\alpha\), and 0 otherwise.

Assuming the weighted model count formulation, the symbolic component of the NeSy-EBM is a DSPar potential function:
\begin{align*}
    \label{equ-semantic-loss-symbolic-component}
    g_{sy} (\mathbf{y}, \mathbf{x}_{sy}, \mathbf{w}_{sy}, \mathbf{g}_{nn}(\mathbf{x}_{nn}, \mathbf{w}_{nn})) &= \psi_{SL}(\left[ \mathbf{y}, \mathbf{x}_{sy} \right], \left[\mathbf{w}_{sy}, \mathbf{g}_{nn}(\mathbf{x}_{nn}, \mathbf{w}_{nn}) \right]) \\
    & =  - \log \left( \sum_{\omega \in \{0,1\}^{|y|}} C_{H}(\omega) \prod_{i=1}^{n} p_i^{\omega_i} (1 - p_i)^{1 - \omega_i} \right)
\end{align*}
where $\omega_i$ is the $i^{th}$ propositional variable value.

Given the symbolic potential and variables defined above, the semantic energy function is defined as:
\begin{align}
    E_{SL}(\mathbf{y}, \mathbf{x}_{sy}, \mathbf{x}_{nn}, \mathbf{w}_{sy}, \mathbf{w}_{nn}) & = g_{sy}(y, \mathbf{x}_{sy}, \mathbf{w}_{sy}, g_{nn}(\mathbf{x}_{nn}, \mathbf{w}_{nn})) \\
    & =  - \log \left( \sum_{\omega \in \{0,1\}^{|y|}} C_{H}(\omega) \prod_{i=1}^{n} p_i^{\omega_i} (1 - p_i)^{1 - \omega_i} \right) \nonumber
\end{align}

\subsection{DeepProbLog (DPL)}

DeepProbLog (DPL) \citep{manhaeve:ai21}, like semantic loss, is a probabilistic NeSy approach that integrates neural network predictions with symbolic reasoning.
However, while semantic loss is confined to propositional logic, DeepProbLog operates within a first-order logic framework, allowing for more expressive reasoning capabilities.
Specifically, DeepProbLog assigns neural network outputs as probabilities within the probabilistic programming language ProbLog \citep{deraedt:ijcai07} through the use of neural predicates.
This integration allows DeepProbLog to leverage both probabilistic logic programming and program induction, making it well-suited for reasoning over more complex symbolic structures and relational dependencies.

\subsubsection{Definition and Background}

To begin, we review the basics of probabilistic logic programming in ProbLog, following the presentation from \citep{manhaeve:ai21} (see \citep{deraedt:ijcai07} for further details).

\paragraph{ProbLog:} A ProbLog program consists of two main components:
\begin{itemize}
    \item A set of \textit{probabilistic facts} $\mathcal{F}$ of the form $p :: y$, where $p$ is the probability that the binary target random variable $y$ is true (i.e., $y \in \{0, 1\}$). Note: this can be extended to include a set of observed facts or evidence $p :: x_{sy}$.
    \item A set $\mathcal{R}$ of \textit{symbolic rules}, which describe how different facts relate to each other.
\end{itemize}

A subset of the probabilistic facts $F \subseteq \mathcal{F}$ defines a possible instantiation, or \textit{world} $\omega$. This world includes all facts in $F$ and all facts derivable from $F$ using the rules in $\mathcal{R}$:
\[
\omega = F \cup \{ y \mid \mathcal{R} \cup F \models y \},
\]
where $\mathcal{R} \cup F \models y$ means that the fact $y$ can be derived from the combination of rules $\mathcal{R}$ and facts $F$. The probability of a world $\omega$ is given by the product of the probabilities of the facts in that world:
\[
P(\omega) := \prod_{y_i \in F} p_i \prod_{y_i \in \mathcal{F} \setminus F} (1 - p_i),
\]
where $p_i$ is the probability assigned to fact $y_i$.

\begin{example}
    Consider the following ProbLog program, which models the likelihood of a burglary or earthquake causing an alarm:
    \begin{align*}
        & \textbf{Probabilistic Facts:} \\
        & 0.1 \, :: \, \text{burglary}, \quad 0.2 \, :: \, \text{earthquake}, \\
        & 0.5 \, :: \, \text{hearsAlarm(mary)}, \quad 0.4 \, :: \, \text{hearsAlarm(john)}. \\
        & \textbf{Rules:} \\
        & \text{alarm} \, :- \, \text{earthquake}. \\
        & \text{alarm} \, :- \, \text{burglary}. \\
        & \text{calls(X)} \, :- \, \text{alarm}, \text{hearsAlarm(X)}. \\
    \end{align*}
    
    Now, consider the subset $F = \{\text{burglary}, \text{hearsAlarm(mary)}\}$ of probabilistic facts. The corresponding world $\omega$ includes the derived facts:
    \[
    \omega = \{\text{burglary}, \text{hearsAlarm(mary)}, \text{alarm}, \text{calls(mary)}\}.
    \]
    The probability of this world is:
    \[
    P(\omega) = 0.1 \cdot 0.5 \cdot (1 - 0.2) \cdot (1 - 0.4) = 0.024.
    \]
\end{example}

\paragraph{DeepProbLog:}
A DeepProbLog program extends the syntax and semantics of ProbLog by introducing neural predicates, allowing the specification of probabilistic facts based on neural network outputs \citep{manhaeve:ai21}. Specifically, DeepProbLog introduces neural annotated disjunctions (nADs), which integrate neural network predictions directly into the logic. A neural annotated disjunction is specified as:
{\small
\begin{align}
nn(m_{g_{nn}}, \mathbf{x}_{nn}, u_{1}) \, :: \, h(\mathbf{x}_{nn}, u_{1}) \, ; \,
\cdots \, ; \,
nn(m_{g_{nn}}, \mathbf{x}_{nn}, u_{n}) \, :: \, h(\mathbf{x}_{nn}, u_{n}) \,
\vDash \, b_{1}, \cdots, b_{m},
\end{align}
}%
where $\mathbf{x}_{nn}$ is a vector of features accessible to the neural component identified by $m_{g_{nn}}$. The terms $u_1, \dots, u_n$ represent the possible outputs of the neural network, and the atoms $b_1, \dots, b_m$ are logical conditions. The output of the neural network, $nn(m_{g_{nn}}, \mathbf{x}_{nn}, u_{i})$, is interpreted as the probability that the atom $h(\mathbf{x}_{nn}, u_i)$ is true, and the sum of the neural model’s outputs must satisfy:

\[
\sum_{i=1}^{n} nn(m_{g_{nn}}, \mathbf{x}_{nn}, u_{i}) = 1.
\]

The meaning of the nAD is that whenever all the atoms $b_1, \dots, b_m$ hold true, each $h(\mathbf{x}_{nn}, u_i)$ becomes true with probability $nn(m_{g_{nn}}, \mathbf{x}_{nn}, u_i)$.

\subsubsection{NeSy-EBM Formulation}

DeepProbLog can be formulated as a DSPar NeSy-EBM where the fact probabilities are defined with both the symbolic parameters and the neural network outputs.
Let $\mathbf{x}_{sy}$ be the observed random variables (evidence), $\mathbf{y}$ be the target random variables (facts), and $\mathbf{w}_{sy}$ be symbolic parameters over facts not parameterized by a neural network (probabilities).
Let $g_{nn}$ be a function with parameters $\mathbf{w}_{nn} \in W_{nn}$ and inputs $\mathbf{x}_{nn} \in X_{nn}$ such that:
\[
g_{nn}: W_{nn} \times X_{nn} \mapsto [0, 1]^n,
\]
where $n$ is the number of propositional variables involved in the constraints. Without loss of generality the fact probabilities, $\mathbf{p}$, are paritioned into symbolic parameters and these neural network outputs:
\[
\mathbf{p} = \begin{bmatrix} \mathbf{w}_{sy} \\ \mathbf{g}(\mathbf{x}_{nn}, \mathbf{w}_{nn}) \end{bmatrix}
\]
where $p_i$ represents the predicted probability for the $i$-th variable.

The probability of a world $\omega$, defined by a subset of probabilistic facts $F \subseteq \mathcal{F}$, is a function of the fact probabilities $\mathbf{p}$, and therefore a function of $\mathbf{w}_{sy}$ and the neural network outputs $\mathbf{g}(\mathbf{x}_{nn}, \mathbf{w}_{nn})$:
{\tiny
\begin{align*}
    &P_{\omega}(\mathbf{w}_{sy}, \mathbf{g}(\mathbf{x}_{nn}, \mathbf{w}_{nn})) = \\
    &\prod_{\mathbf{w}_{sy}^j \in F} \mathbf{w}_{sy}^j \prod_{\mathbf{w}_{sy}^j \in \mathcal{F} \setminus F} (1 - \mathbf{w}_{sy}^j) \prod_{\mathbf{g}(\mathbf{x}_{nn}, \mathbf{w}_{nn})^j \in F} \mathbf{g}(\mathbf{x}_{nn}, \mathbf{w}_{nn})^j \prod_{\mathbf{g}(\mathbf{x}_{nn}, \mathbf{w}_{nn})^j \in \mathcal{F} \setminus F} (1 - \mathbf{g}(\mathbf{x}_{nn}, \mathbf{w}_{nn})^j)
\end{align*}
}%

Finally, unlike semantic loss, DeepProbLog typically evaluates probabilities with respect to \textit{queries}. A \textit{query} is a symbolic atom \( q \) whose probability we want to compute based on the probabilistic facts and the neural network outputs. For example, the marginal probability of a query atom \( q \) is computed by summing over the probabilities of all worlds \( \omega \) in which \( q \) is true.

Let \( C_H(\omega, q) \) be the constraint function, which acts as an indicator function returning 1 if the world \( \omega \) satisfies the condition that the query atom \( q \) is true:
\[
C_H(\omega, q) = 
\begin{cases} 
1 & \text{if } q \in \omega \\
0 & \text{otherwise}.
\end{cases}
\]

Assuming the weighted model count formulation, the symbolic component of the NeSy-EBM is a DSPar potential function for a single query q is defined as:
\begin{align*}
    \label{equ-semantic-loss-symbolic-component}
    g_{sy} (\mathbf{y}, \mathbf{x}_{sy}, \mathbf{w}_{sy}, \mathbf{g}_{nn}(\mathbf{x}_{nn}, \mathbf{w}_{nn})) &= \psi_{DPL}^{q}(\left[ \mathbf{y}, \mathbf{x}_{sy} \right], \left[\mathbf{w}_{sy}, \mathbf{g}_{nn}(\mathbf{x}_{nn}, \mathbf{w}_{nn}) \right]) \\
    & =  d\left(q,  \sum_{\omega \in \{0,1\}^{|\mathbf{y}}|} C_H(\omega, q) P_{\omega}(\mathbf{w}_{sy}, \mathbf{g}(\mathbf{x}_{nn}, \mathbf{w}_{nn})) \right)
\end{align*}

Where \(d(\cdot, \cdot)\) is the distance between the predicted probability of the query and its true probability. Now, the energy function is defined over a sum of all queries $\mathbf{q}$.
\[
E_{DPL}(\mathbf{y}, \mathbf{x}_{sy}, \mathbf{x}_{nn}, \mathbf{w}_{sy}, \mathbf{w}_{nn}) = \sum_{i=1}^{|\mathbf{q}|} \psi_{DPL}^{q_i}(\left[ \mathbf{y}, \mathbf{x}_{sy} \right], \left[\mathbf{w}_{sy}, \mathbf{g}_{nn}(\mathbf{x}_{nn}, \mathbf{w}_{nn})\right]).
\]

\subsection{Logic Tensor Networks (LTNs)}
Logic Tensor Networks (LTNs) \citep{badreddine:ai22} are a fuzzy neural-symbolic framework that integrates neural network predictions with logic-based reasoning.
In their most general form, LTNs support a wide range of fuzzy logics, aggregators, and semantics for evaluating logical formulae.
As a result, defining LTNs as NeSy-EBMs in complete generality is possible but often leads to an abstraction that is overly broad and less informative for understanding their practical structure.
Instead, we present a formulation corresponding to a commonly used instantiation in the literature, which clearly illustrates how LTNs can be expressed within the NeSy-EBM framework.
This representation \textbf{does} limit the full expressiveness of LTNs, as alternative fuzzy semantics or aggregation operators may require different energy constructions.
However, such variants can still be expressed as NeSy-EBMs by defining alternative symbolic energy functions corresponding to the chosen semantics.

For this example, let the LTNs use product real logic operators to define fuzzy truth values for logical connectives:
{\small
\[
\neg(a) := 1 - a, \quad \land(a, b) := a \cdot b, \quad \vee(a, b) := a + b - a \cdot b, \quad \implies(a, b) := 1 -a + a \cdot b.
\]
}%
Additionally, LTNs use formula aggregators, such as generalized mean semantics, to handle existential and universal quantifiers over collections of truth values, denoted by \(\mathbf{a} = [a_i]_{i=1}^{n}\):
\[
\exists(\mathbf{a}) := \left( \frac{1}{n} \sum_{i=1}^{n} a_{i}^p \right)^{\frac{1}{p}}, \quad \forall(\mathbf{a}) := 1 - \left( \frac{1}{n} \sum_{i=1}^{n} (1 - a_{i})^p \right)^{\frac{1}{p}},
\]
where \(p \in \mathbb{R}_{+}\) is a hyperparameter controlling the smoothness of the quantifiers.

In LTNs, neural networks instantiate predicates with values from \([0, 1]\), representing the degree to which a predicate is satisfied. For example, given two entities \(u\) and \(v\), the predicate \(P(u, v)\) can be defined as the output of a neural network \(g_{nn}(\mathbf{X}[u], \mathbf{X}[v]; \mathbf{w}_{nn})\), which maps the feature vectors \(\mathbf{X}[u]\) and \(\mathbf{X}[v]\) to a truth value in \([0, 1]\).

\begin{example}
    Consider the following logical formula, which expresses that for each entity \(u\), there exists some entity \(v\) such that both predicates \(P(u, v)\) and \(Q(v)\) hold true:
    \[
    \exists v \in \mathcal{V} \left( P(u, v) \land Q(v) \right).
    \]
    
    Let \(\mathbf{X}_{\mathcal{U}}\) and \(\mathbf{X}_{\mathcal{V}}\) represent the feature vectors for the sets of entities $\mathcal{U}$ and $\mathcal{V}$, respectively. The predicates \(P(u, v)\) and \(Q(v)\) can be instantiated with neural network outputs:
\begin{itemize}
    \item \(P(u, v)\) is given by the neural network \(g_{nn}(\mathbf{X}[u], \mathbf{X}[v]; \mathbf{w}_{nn})\).
    \item \(Q(v)\) could be a constant truth value or another neural network prediction.
\end{itemize}
    
    Using the generalized mean semantics for the existential quantifier, we define the real-valued logic function for the above formula as:
    \[
    h_{u}(\mathbf{X}_{\mathcal{U}}, \mathbf{X}_{\mathcal{V}}, \mathbf{x}_{Q}; \mathbf{w}_{nn}) = \left( \frac{1}{|\mathcal{V}|} \sum_{v \in \mathcal{V}} \left( g_{nn}(\mathbf{X}[u], \mathbf{X}[v]; \mathbf{w}_{nn}) \cdot \mathbf{x}_Q[v] \right)^p \right)^{\frac{1}{p}},
    \]
    where \(\mathbf{x}_Q[v]\) is the truth value for the predicate \(Q(v)\), and \(g_{nn}(\mathbf{X}[u], \mathbf{X}[v]; \mathbf{w}_{nn})\) is the neural network output for the predicate \(P(u, v)\).
    
    The satisfaction level of the formula for all entities \(u\) is then aggregated using the universal quantifier:
    \[
    G(\mathbf{w}_{nn}) = 1 - \left( \frac{1}{|\mathcal{U}|} \sum_{u \in \mathcal{U}} \left(1 - h_{u}(\mathbf{X}_{\mathcal{U}}, \mathbf{X}_{\mathcal{V}}, \mathbf{x}_{Q}; \mathbf{w}_{nn})\right)^p \right)^{\frac{1}{p}}.
    \]
\end{example}

This example illustrates how LTNs leverage neural networks to assign fuzzy truth values to predicates and apply these values in evaluating logical formulas. The next section will explain how LTNs can be represented as NeSy-EBMs using symbolic constraints.

\subsubsection{NeSy-EBM Formulation}

LTNs can be formulated as a DSVar NeSy-EBM, where the satisfaction of symbolic constraints is driven by neural network outputs. Let \(\mathbf{x}_{sy}\) be the observed random variables (constants), and since the symbolic component does not have trainable parameters, define \(\mathbf{w}_{sy} \in \emptyset\). Let \(g_{nn}\) be a function with parameters \(\mathbf{w}_{nn} \in W_{nn}\) and inputs \(\mathbf{x}_{nn} \in X_{nn}\) such that:
\[
g_{nn}: W_{nn} \times X_{nn} \mapsto [0, 1]^n,
\]
where \(n\) is the number of variables involved in the constraints.

Define \(C_{S}^{Agg}(\mathbf{g}_{nn}(\mathbf{x}_{nn}, \mathbf{w}_{nn}), \mathbf{x}_{sy}, \mathbf{w}_{sy})\) as the soft constraint function that takes as input the output of a set of neural networks \(\mathbf{g}_{nn}(\mathbf{x}_{nn}, \mathbf{w}_{nn})\) and applies the collection of aggregation functions \(Agg\) in some way that maintains differentability (e.g., the generalized mean or quantifiers).

The symbolic component of the NeSy-EBM is a DSVar potential function:
\begin{align*}
    g_{sy} (\mathbf{y}, \mathbf{x}_{sy}, \mathbf{w}_{sy}, \mathbf{g}_{nn}(\mathbf{x}_{nn}, \mathbf{w}_{nn})) &= \psi_{LTN}(\left[ \mathbf{y}, \mathbf{x}_{sy}, \mathbf{g}_{nn}(\mathbf{x}_{nn}, \mathbf{w}_{nn})\right], \left[\mathbf{w}_{sy} \right]) \\
    &= C_{S}^{Agg}(\mathbf{g}_{nn}(\mathbf{x}_{nn}, \mathbf{w}_{nn}), \mathbf{x}_{sy}, \mathbf{w}_{sy})
\end{align*}

Given the symbolic potential and variables defined above, the energy function for LTNs is defined as:
\begin{align*}
    E_{LTN}(\mathbf{y}, \mathbf{x}_{sy}, \mathbf{x}_{nn}, \mathbf{w}_{sy}, \mathbf{w}_{nn}) &= g_{sy} (\mathbf{y}, \mathbf{x}_{sy}, \mathbf{w}_{sy}, \mathbf{g}_{nn}(\mathbf{x}_{nn}, \mathbf{w}_{nn})) \\
    &=  C_{S}^{Agg}(\mathbf{g}_{nn}(\mathbf{x}_{nn}, \mathbf{w}_{nn}), \mathbf{x}_{sy}, \mathbf{w}_{sy}).
\end{align*}

In this formulation, the entity set $\mathcal{U}$ and $\mathcal{V}$ are treated as one training sample and the learning results of \secref{sec:nesy-ebm-learning} are still applicable.
However, if the entities in $\mathcal{U}$ are treated as separate samples, then for $p \neq 1$, $G(\mathbf{w}_{nn})$ cannot generally be decomposed into a sum of losses over the entities.
Such objectives are outside the empirical risk minimization framework studied in \secref{sec:nesy-ebm-learning}.

%% file: appendix/extended-neupsl.tex
\section{Extended Neural Probabilistic Soft Logic}
\label{appendix:extended_neupsl}

In this section, we expand on the smooth formulation of NeuPSL inference and provide proofs for the continuity results presented in \secref{sec:lcqp_inference}.

\subsection{Extended Smooth Formulation of Inference}
\label{appendix:lcqp_inference}
Recall the primal formulation of NeuPSL inference restated below:
\begin{align}
    \argmin_{\mathbf{y} \in \mathbb{R}^{n_{\mathbf{y}}}} \, \mathbf{w}_{sy}^T \mathbf{\Phi}(\mathbf{y} , \mathbf{x}_{sy}, \mathbf{g}_{nn}(\mathbf{x}_{nn}, \mathbf{w}_{nn})) \quad
        \textrm{s.t. }        
        \mathbf{y} \in \mathbf{\Omega}(\mathbf{x}_{sy}, \mathbf{g}_{nn}(\mathbf{x}_{nn}, \mathbf{w}_{nn}))
    \label{eq:inference_primal_app}.
\end{align}
Importantly, note the structure of the deep hinge-loss potentials defining $\mathbf{\Phi}$:
\begin{align}
        & \phi_k(\mathbf{y}, \mathbf{x}_{sy}, \mathbf{g}_{nn}(\mathbf{x}_{nn}, \mathbf{w}_{nn})) \\
        & \quad := (\max\{\mathbf{a}_{\phi_k, y}^T \mathbf{y} + \mathbf{a}_{\phi_k, \mathbf{x}_{sy}}^T \mathbf{x}_{sy} + \mathbf{a}_{\phi_k, \mathbf{g}_{nn}}^T \mathbf{g}_{nn}(\mathbf{x}_{nn}, \mathbf{w}_{nn}) + b_{\phi_k}, 0\})^{p_{k}}.
        \label{eq:deep_hlmrf_potential_app}
\end{align}
The LCQP NeuPSL inference formulation is defined using ordered index sets: $\mathbf{I}_{S}$  for the partitions of squared hinge potentials (indices $k$ which for all $j \in t_{k}$ the exponent term $p_j=2$) and $\mathbf{I}_{L}$ for the partitions of linear hinge potentials (indices $k$ which for all $j \in t_{k}$ the exponent term $p_j=1$).
With the index sets, we define 
{\small
\begin{align}
    \mathbf{W}_{S} :=
        \begin{bmatrix} 
            w_{\mathbf{I}_{S}[1]} \mathbf{I} & 0 & \cdots & 0 \\
            0 & w_{\mathbf{I}_{S}[2]} \mathbf{I} & &  \\
            \vdots & & \ddots & 
        \end{bmatrix} 
   \quad \quad \textrm{and} \quad \quad
    \mathbf{w}_{L} := 
        \begin{bmatrix} 
            w_{\mathbf{I}_{L}[1]} \mathbf{1} \\
           w_{\mathbf{I}_{L}[2]} \mathbf{1} \\
            \vdots 
        \end{bmatrix} 
    \label{eq:symbolic_weight_matrix_and_vector}
\end{align}
}%
Let $m_{S} := \vert \cup_{\mathbf{I}_{S}} t_{k} \vert$ and $m_{L} := \vert \cup_{\mathbf{I}_{L}} t_{k} \vert$, be the total number of squared and linear hinge potentials, respectively, and define slack variables $\mathbf{s}_{S} := [s_{j}]_{j = 1}^{m_{S}}$ and $\mathbf{s}_{L} := [s_{j}]_{j = 1}^{m_{L}}$ for each of the squared and linear hinge potentials, respectively.
NeuPSL inference is equivalent to the following LCQP:
{\small
\begin{subequations}  
\label{eq:lcqp}
\begin{align}
\label{eq:lcqp.1}
        &\min_{\mathbf{y} \in [0, 1]^{n_{y}}, \, \mathbf{s}_{S} \in \mathbb{R}^{m_{S}}, \, \mathbf{s_H} \in \mathbb{R}^{m_{L}}_{+}} \,  
        \mathbf{s}_{S}^T \mathbf{W}_{S} \mathbf{s}_{S} + \mathbf{w}_{L}^T \mathbf{s}_{L} \\
        \label{eq:lcqp.2}
        & \quad \quad\textrm{s.t.} \,         
        \quad  \mathbf{a}_{c_i, \mathbf{y}}^T \mathbf{y} + \mathbf{a}_{c_i, \mathbf{x}_{sy}}^T \mathbf{x}_{sy} + \mathbf{a}_{c_i, \mathbf{g}_{nn}}^T \mathbf{g}_{nn}(\mathbf{x}_{nn}, \mathbf{w}_{nn}) + b_{c_i} \leq 0 \quad \forall \, i =1,\dotsc,q, \\
        \label{eq:lcqp.3}
        & \quad \quad \quad \mathbf{a}_{\phi_j, \mathbf{y}}^T \mathbf{y} + \mathbf{a}_{\phi_j, \mathbf{x}_{sy}}^T \mathbf{x}_{sy} + \mathbf{a}_{\phi_j, \mathbf{g}_{nn}}^T \mathbf{g}_{nn}(\mathbf{x}_{nn}, \mathbf{w}_{nn}) + b_{\phi_j} - s_j \leq 0 \quad \forall j \in I_{S} \cup I_{L}.
\end{align}       
\end{subequations}
}%
We ensure strong convexity by adding a square regularization with parameter $\epsilon$ to the objective.
Let the bound constraints on $\mathbf{y}$ and $\mathbf{s}_{L}$ and linear inequalities in the LCQP be captured by the $(2 \cdot n_{y} + q + m_{S} + 2 \cdot m_{L}) \times (n_{y} + m_{S} + m_{L})$ matrix $\mathbf{A}$ and $(2 \cdot n_{y} + q + m_{S} + 2 \cdot m_{L})$ dimensional vector $\mathbf{b}(\mathbf{x}_{sy}, \mathbf{g}_{nn}(\mathbf{x}_{nn}, \mathbf{w}_{nn}))$.
More formally, $\mathbf{A} := [a_{ij}]$ where $a_{ij}$ is the coefficient of a decision variable in the implicit and explicit constraints in the formulation above:
{\small
\begin{align}
    a_{i, j} := 
    \begin{cases}
        0 & (i \leq q) \, \land \, (j \leq m_{S} + m_{L}) \\
        \mathbf{a}_{c_{i}, \mathbf{y}}[j - (m_{S} + m_{L})] & (i \leq q) \, \land \, (j > m_{S} + m_{L}) \\
        0 & (q < i \leq q + m_{S} + m_{L}) \, \land \, (j \leq m_{S} + m_{L}) \land (j \neq i - q) \\
        -1 & (q < i \leq q + m_{S} + m_{L}) \, \land \, (j \leq m_{S} + m_{L}) \land (j = i - q) \\
        \mathbf{a}_{\phi_{i - q}, \mathbf{y}}[j - (m_{S} + m_{L})] & (q < i \leq q + m_{S} + m_{L}) \, \land \, (j > m_{S} + m_{L}) \\
        0 & (q + m_{S} + m_{L} < i \leq q + m_{S} + 2 \cdot m_{L} + n_{y}) \, \\ & \land \, (j \neq i - (q + m_{L})) \\
        -1 & (q + m_{S} + m_{L} < i \leq q + m_{S} + 2 \cdot m_{L} + n_{y}) \, \\ & \land \, (j = i - (q + m_{L})) \\
        0 & (q + m_{S} + 2 \cdot m_{L} + n_{y} < i \leq q + m_{S} + 2 \cdot m_{L} + 2 \cdot n_{y}) \, \\ & \land \, (j \neq i - (q + m_{S} + m_{L})) \\
        1 & (q + m_{S} + 2 \cdot m_{L} + n_{y} < i \leq q + m_{S} + 2 \cdot m_{L} + 2 \cdot n_{y}) \, \\ & \land \, (j = i - (q + m_{S} + m_{L}))
    \end{cases}
\end{align}
}%
Furthermore, $\mathbf{b}(\mathbf{x}_{sy}, \mathbf{g}_{nn}(\mathbf{x}_{nn}, \mathbf{w}_{nn})) = [b_{i}(\mathbf{x}_{sy}, \mathbf{g}_{nn}(\mathbf{x}_{nn}, \mathbf{w}_{nn}))]$ is the vector of constants corresponding to each constraint in the formulation above:
{\small
\begin{align}
    & b_{i}(\mathbf{x}_{sy}, \mathbf{g}_{nn}(\mathbf{x}_{nn}, \mathbf{w}_{nn})) \\
    & \quad \quad := 
    \begin{cases}
        \mathbf{a}_{c_{i}, \mathbf{x}_{sy}}^T \mathbf{x}_{sy} + \mathbf{a}_{c_{i}, \mathbf{g}_{nn}}^T \mathbf{g}_{nn}(\mathbf{x}_{nn}, \mathbf{w}_{nn}) + b_{c_i} & i \leq q \\
        \mathbf{a}_{\phi_{i - q}, \mathbf{x}_{sy}}^T \mathbf{x}_{sy} + \mathbf{a}_{\phi_{i - q}, \mathbf{g}_{nn}}^T \mathbf{g}_{nn}(\mathbf{x}_{nn}, \mathbf{w}_{nn}) + b_{\phi_{i - q}} & q < i \leq q + m_{S} + m_{L} \\
        0 & q + m_{S} + m_{L} < i \\ & \quad \leq q + m_{S} + 2 \cdot m_{L} + n_{y} \\
        -1 & q + m_{S} + 2 \cdot m_{L} + n_{y} < i \\ & \quad \leq q + m_{S} + 2 \cdot m_{L} + 2 \cdot n_{y}
    \end{cases}
\end{align}
}%
Note that $\mathbf{b}(\mathbf{x}_{sy}, \mathbf{g}_{nn}(\mathbf{x}_{nn}, \mathbf{w}_{nn}))$ is a linear function of the neural network outputs, hence, if $\mathbf{g}_{nn}(\mathbf{x}_{nn}, \mathbf{w}_{nn})$ is a smooth function of the neural parameters, then $\mathbf{b}(\mathbf{x}_{sy}, \mathbf{g}_{nn}(\mathbf{x}_{nn}, \mathbf{w}_{nn}))$ is also smooth.

With this notation, the regularized inference problem is:
{\small
\begin{align}
      V(\mathbf{w}_{sy}, \mathbf{b}(\mathbf{x}_{sy}, \mathbf{g}_{nn}(\mathbf{x}_{nn}, \mathbf{w}_{nn}))) \nonumber \\
      :=  \min_{\mathbf{y}, \mathbf{s_{S}}, \mathbf{s_H}} \, &
        \begin{bmatrix}
            \mathbf{s}_{S} \\ \mathbf{s}_{L} \\ \mathbf{y}
        \end{bmatrix}^T
        \begin{bmatrix}
            \mathbf{W}_{S} + \epsilon I & 0 & 0\\
            0 & \epsilon I & 0 \\
            0 & 0 & \epsilon I \\
        \end{bmatrix}
        \begin{bmatrix}
            \mathbf{s}_{S} \\
            \mathbf{s}_{L} \\
            \mathbf{y}  
        \end{bmatrix} 
        + 
        \begin{bmatrix}
           0 \\ \mathbf{w}_{L} \\ 0
        \end{bmatrix}^T
        \begin{bmatrix}
            \mathbf{s}_{S} \\
            \mathbf{s}_{L} \\
            \mathbf{y}
        \end{bmatrix} \nonumber \\
        \textrm{s.t.} \quad         
        & \mathbf{A}         
        \begin{bmatrix}
            \mathbf{s}_{S} \\
            \mathbf{s}_{L} \\
            \mathbf{y}
        \end{bmatrix} + \mathbf{b}(\mathbf{x}_{sy}, \mathbf{g}_{nn}(\mathbf{x}_{nn}, \mathbf{w}_{nn})) \leq 0 \label{eq:lcqp_primal_appendix}. 
\end{align}
}%
For ease of notation, let 
{
\begin{align}
    D(\mathbf{w}_{sy}) :=         
    \begin{bmatrix}
        \mathbf{W}_{S} & 0 & 0\\
        0 & 0 & 0 \\
        0 & 0 & 0
    \end{bmatrix},
    \,
    \mathbf{c}(\mathbf{w}_{sy}) :=         
    \begin{bmatrix}
        0 \\
        \mathbf{w}_{L} \\
        0 
    \end{bmatrix}, 
    \,
    \mathbf{\nu} :=         
    \begin{bmatrix}
        \mathbf{s}_{S} \\
        \mathbf{s}_{L} \\
        \mathbf{y}
    \end{bmatrix}.
\end{align} 
}%
Then the regularized primal LCQP MAP inference problem is concisely expressed as
{
\begin{align}
    \min_{\mathbf{\nu} \in \mathbb{R}^{n_{\mathbf{y}} + m_{S} + m_{L}}} & \, 
        \mathbf{\nu}^T (\mathbf{D}(\mathbf{w}_{sy}) + \epsilon \mathbf{I}) \mathbf{\nu} + \mathbf{c}(\mathbf{w}_{sy})^T \mathbf{\nu} \label{eq:regularized_lcqp_primal_appendix}\\
        \textrm{s.t.} \quad         
        & \mathbf{A} \mathbf{\nu} + \mathbf{b}(\mathbf{x}_{sy}, \mathbf{g}_{nn}(\mathbf{x}_{nn}, \mathbf{w}_{nn})) \leq 0 \nonumber.
\end{align} 
}%

By Slater's constraint qualification, we have strong-duality when there is a feasible solution.
In this case, an optimal solution to the dual problem yields an optimal solution to the primal problem.
The Lagrange dual problem of \eqref{eq:regularized_lcqp_primal_appendix} is
{\small
\begin{align}
    & \argmax_{\mathbf{\mu}\ge 0} \min_{\mathbf{\nu} \in \mathbb{R}^{n_{\mathbf{y}} + m_{S} + m_{L}}} \, 
        \mathbf{\nu}^T
        (\mathbf{D}(\mathbf{w}_{sy}) + \epsilon \mathbf{I})
        \mathbf{\nu}
        + 
        \mathbf{c}(\mathbf{w}_{sy})^T \mathbf{\nu}
        +
        \mathbf{\mu}^T (\mathbf{A} \mathbf{\nu}
        + \mathbf{b}(\mathbf{x}_{sy}, \mathbf{g}_{nn}(\mathbf{x}_{nn}, \mathbf{w}_{nn}))) \nonumber \\
    & \quad = \argmax_{\mathbf{\mu}\ge 0}
        \,
        - \frac{1}{4} \mathbf{\mu}^T \mathbf{A} (\mathbf{D}(\mathbf{w}_{sy}) + \epsilon \mathbf{I})^{-1} \mathbf{A}^T \mathbf{\mu}  \label{eq:dual_lcqp_appendix} \\ 
        & \quad \quad \quad \quad \quad \quad \quad - \frac{1}{2} (\mathbf{A} (\mathbf{D}(\mathbf{w}_{sy}) + \epsilon \mathbf{I})^{-1} \mathbf{c}(\mathbf{w}_{sy}) - 2 \mathbf{b}(\mathbf{x}_{sy}, \mathbf{g}_{nn}(\mathbf{x}_{nn}, \mathbf{w}_{nn})))^T \mathbf{\mu} \nonumber
\end{align}
}%
where $\mathbf{\mu} = [\mu_{i}]_{i = 1}^{n_{\mu}}$ are the Lagrange dual variables.
For later reference, denote the negative of the Lagrange dual function of MAP inference as: 
\begin{align}
    & h(\mathbf{\mu}; \mathbf{w}_{sy}, \mathbf{b}(\mathbf{x}_{sy}, \mathbf{g}_{nn}(\mathbf{x}_{nn}, \mathbf{w}_{nn}))) \label{eq:lcqp_lagrangian_appendix} \\
        & := \frac{1}{4} \mathbf{\mu}^T \mathbf{A} (\mathbf{D}(\mathbf{w}_{sy}) + \epsilon \mathbf{I})^{-1} \mathbf{A}^T \mathbf{\mu} + \frac{1}{2} (\mathbf{A} (\mathbf{D}(\mathbf{w}_{sy}) + \epsilon \mathbf{I})^{-1} \mathbf{c}(\mathbf{w}_{sy}) \\
        & \quad - 2 \mathbf{b}(\mathbf{x}_{sy}, \mathbf{g}_{nn}(\mathbf{x}_{nn}, \mathbf{w}_{nn})))^T \mathbf{\mu}. \nonumber
\end{align}
The dual LCQP has more decision variables but is only over non-negativity constraints rather than the complex polyhedron feasible set.
The dual-to-primal variable translation is:
\begin{align}
    \mathbf{\nu} = - \frac{1}{2} (\mathbf{D}(\mathbf{w}_{sy}) + \epsilon \mathbf{I})^{-1} (\mathbf{A}^T \mathbf{\mu} + \mathbf{c}(\mathbf{w}_{sy}))
    \label{eq:dual_primal_translation_appendix}
\end{align}
As $(\mathbf{D}(\mathbf{w}_{sy}) + \epsilon \mathbf{I})$ is diagonal, it is easy to invert and hence it is practical to work in the dual space to obtain a solution to the primal problem.

\subsection{Extended Continuity of Inference}
\label{appendix:continuity_of_inference}

We now provide background on sensitivity analysis that we then apply in our proofs on the continuity properties of NeuPSL inference.

\subsubsection{Background}
\label{appendix:continuity_of_inference_preliminaries}

\begin{theorem}[\cite{boyd:book04} p. 81]
    \label{thm:pointwise_supremum_over_convex}
    If for each $\mathbf{y} \in \mathcal{A}$, $f(\mathbf{x}, \mathbf{y})$ is convex in $\mathbf{x}$ then the function
    \begin{align}
        g(\mathbf{x}) := \sup_{\mathbf{y} \in \mathcal{A}} f(\mathbf{x}, \mathbf{y})
    \end{align}
    is convex in $\mathbf{x}$.
\end{theorem}

\begin{theorem}[\cite{boyd:book04} p. 81]
    \label{thm:pointwise_infimum_over_concave}
    If for each $\mathbf{y} \in \mathcal{A}$, $f(\mathbf{x}, \mathbf{y})$ is concave in $\mathbf{x}$ then the function
    \begin{align}
        g(\mathbf{x}) := \inf_{\mathbf{y} \in \mathcal{A}} f(\mathbf{x}, \mathbf{y})
    \end{align}
    is concave in $\mathbf{x}$.
\end{theorem}

\begin{definition}[Convex Subgradient: \cite{boyd:book04} and \cite{shalevshwartz:ftml11}]
    Consider a convex function $f: \mathbb{R}^{n} \to [-\infty, \infty]$ and a point $\overline{\mathbf{x}}$ with $f(\overline{\mathbf{x}})$ finite.
    For a vector $\mathbf{v} \in \mathbf{R}^{n}$, one says that $\mathbf{v}$ is a (convex) subgradient of $f$ at $\overline{\mathbf{x}}$, written $\mathbf{v} \in \partial f(\overline{\mathbf{x}})$, iff
    \begin{align}
        f(\mathbf{x}) \geq f(\overline{\mathbf{x}}) + <\mathbf{v}, \mathbf{x} - \overline{\mathbf{x}}>, \quad \forall \mathbf{x} \in \mathbf{R}^{n}.
    \end{align}
\end{definition}

\begin{definition}[Closedness: \citenoun{bertsekas:book09}]
    If the epigraph of a function $f : \mathbb{R}^{n} \to [-\infty, \infty]$ is a closed set, we say that $f$ is a closed function.
\end{definition}

\begin{definition}[Lower Semicontinuity: \cite{bertsekas:book09}]
    The function $f : \mathbb{R}^{n} \to [-\infty, \infty]$ is \textit{lower semicontinuous} (lsc) at a point $\overline{\mathbf{x}} \in \mathbb{R}^{n}$ if
    \begin{align}
        f(\overline{\mathbf{x}}) \leq \liminf_{k \to \infty} f(\mathbf{x}_{k}),
    \end{align}
    for every sequence $\{\mathbf{x}_{k}\} \subset \mathbb{R}^{n}$ with $\mathbf{x}_{k} \to \overline{\mathbf{x}}$.
    We say $f$ is \textit{lsc} if it is lsc at each $\overline{\mathbf{x}}$ in its domain.
\end{definition}

\begin{theorem}[Closedness and Semicontinuity: \cite{bertsekas:book09} Proposition 1.1.2.]
    For a function $f: \mathbb{R}^{n} \to [-\infty, \infty]$, the following are equivalent:
    \begin{enumerate}
        \item The level set $V_{\gamma} = \{\mathbf{x} \, \vert \, f(\mathbf{x}) \leq \gamma \}$ is closed for every scalar $\gamma$.
        \item $f$ is lsc.
        \item $f$ is closed.
    \end{enumerate}
\end{theorem}

The following definition and theorem are from \cite{rockafellar:book97} and they generalize the notion of subgradients to non-convex functions and the chain rule of differentiation, respectively.
For complete statements see \cite{rockafellar:book97}.

\begin{definition}[Regular Subgradient: \cite{rockafellar:book97} Definition 8.3]
    Consider a function $ f: \mathbb{R}^{n} \to [-\infty, \infty] $ and a point $\overline{\mathbf{x}}$ with $f(\overline{\mathbf{x}})$ finite.
    For a vector $\mathbf{v} \in \mathbf{R}^{n}$, one says that $\mathbf{v}$ is a regular subgradient of $f$ at $\overline{\mathbf{x}}$, written $\mathbf{v} \in \hat{\partial} f(\overline{\mathbf{x}})$, iff
    \begin{align}
        f(\mathbf{x}) \geq f(\overline{\mathbf{x}}) + \langle \mathbf{v}, \mathbf{x} - \overline{\mathbf{x}} \rangle + \textrm{o}(\mathbf{x} - \overline{\mathbf{x}}), \quad \forall \mathbf{x} \in \mathbf{R}^{n},
    \end{align}
    where the $\textrm{o}(t)$ notation indicates a term with the property that 
    \begin{align}
        \lim_{t \to 0} \frac{\textrm{o}(t)}{t} = 0.
    \end{align}
\end{definition}

The relation of the regular subgradient defined above and the more familiar convex subgradient is the addition of the $o(\mathbf{x} - \mathbf{\overline{x}})$ term.
Evidently, a convex subgradient is a regular subgradient.

\begin{theorem}[Chain Rule for Regular Subgradients: \cite{rockafellar:book97} Theorem 10.6]
\label{theorem:subgradient_chain_rule}
Suppose $f(\mathbf{x}) = g(F(\mathbf{x}))$ for a proper, lsc function $g : \mathbb{R}^{m} \to [-\infty, \infty]$ and a smooth mapping $F : \mathbb{R}^{n} \to \mathbb{R}^{m}$.
Then at any point $\overline{\mathbf{x}} \in \textrm{dom} \, f = F^{-1} (\textrm{dom} \, g)$ one has
\begin{align}
    \hat{\partial} f(\overline{\mathbf{x}}) & \supset \nabla F(\overline{\mathbf{x}}) ^T \hat{\partial} g(F(\overline{\mathbf{x}})),
\end{align}
where $\nabla F(\overline{\mathbf{x}}) ^T$ is the Jacobian of $F$ at $\overline{\mathbf{x}}$.
\end{theorem}

\begin{theorem}[Danskin's Theorem: \cite{danskin:siam66} and \cite{bertsekas:phd71} Proposition A.22]
    \label{thm:danskin}
    Suppose $\mathcal{Z} \subseteq \mathbb{R}^{m}$ is a compact set and $g(\mathbf{x}, \mathbf{z}): \mathbb{R}^{n} \times \mathcal{Z} \to (-\infty, \infty]$ is a function.
    Suppose $g(\cdot, \mathbf{z}): \mathbb{R}^{n} \to \mathbb{R}$ is closed proper convex function for every $\mathbf{z} \in \mathcal{Z}$.
    Further, define the function $f: \mathbb{R}^{n} \to \mathbb{R}$ such that
    \begin{align}
        f(\mathbf{x}) := \max_{\mathbf{z} \in \mathcal{Z}} g(\mathbf{x}, \mathbf{z}). \nonumber
    \end{align}
    Suppose $f$ is finite somewhere.
    Moreover, let $\mathcal{X} := \textrm{int}(\textrm{dom}f)$, i.e., the interior of the set of points in $\mathbb{R}^{n}$ such that $f$ is finite.
    Suppose $g$ is continuous on $\mathcal{X} \times \mathcal{Z}$. 
    Further, define the set of maximizing points of $g(\mathbf{x}, \cdot)$ for each $\mathbf{x}$
    \begin{align}
        Z(\mathbf{x}) = \argmax_{\mathbf{z} \in \mathcal{Z}} g(\mathbf{x}, \mathbf{z}). \nonumber
    \end{align}
    Then the following properties of $f$ hold.
    \begin{enumerate}
        \item The function $f(\mathbf{x})$ is a closed proper convex function.
        \item For every $\mathbf{x} \in \mathcal{X}$,
        \begin{align}
            \partial f(\mathbf{x}) = \textrm{conv}\left \{ 
                \partial_{\mathbf{x}} g(\mathbf{x}, \mathbf{z}) \, \vert \, \mathbf{z} \in Z(\mathbf{x})
            \right \}.
        \end{align}
    \end{enumerate}
\end{theorem}

\begin{corollary}
    \label{corollary:danskins_unique_solution}
    Assume the conditions for Danskin's Theorem above hold.
    For every $\mathbf{x} \in \mathcal{X}$, if $Z(\mathbf{x})$ consists of a unique point, call it $\mathbf{z}^{*}$, and $g(\cdot, \mathbf{z}^{*})$ is differentiable at $\mathbf{x}$, then $f(\cdot)$ is differentiable at $\mathbf{x}$, and 
    \begin{align}
        \nabla f(\mathbf{x}) := \nabla_{\mathbf{x}} g(\mathbf{x}, \mathbf{z}^{*}).
    \end{align}
\end{corollary}

\begin{theorem}[\cite{bonnans:siam98} Theorem 4.2, \cite{rockafellar:rcsam74} p. 41]
    \label{thm:value_function_subdifferentiable_wrt_constraint_constants}
    Let $\mathbf{X}$ and $\mathbf{U}$ be Banach spaces.
    Let $\mathbf{K}$ be a closed convex cone in the Banach space $\mathbf{U}$.
    Let $G: \mathbf{X} \to \mathbf{U}$ be a convex mapping with respect to the cone $\mathbf{C} := -\mathbf{K}$ and $f: \mathbf{X} \to (-\infty, \infty]$ be a (possibly infinite-valued) convex function.
    Consider the following convex program and its optimal value function:
    \begin{align*}
        v_{P}(\mathbf{u}) := \min_{\mathbf{x} \in \mathbf{X}} & \, \, f(\mathbf{x}) & \textrm{(P)} \\
        \textit{s.t.} & \quad G(\mathbf{x}) + \mathbf{u} \in \mathbf{K}. \nonumber
    \end{align*}
    Moreover, consider the (Lagrangian) dual of the program:
    \begin{align*}
        v_{D}(\mathbf{u}) := \max_{\mathbf{\lambda} \in \mathbf{K}^{-}} & \, \, \min_{\mathbf{x} \in \mathbf{X}} f(\mathbf{x}) + \mathbf{\lambda}^T (G(\mathbf{x}) + \mathbf{u}) & \textrm{(D)} 
    \end{align*}
    Suppose $v_{P}(\mathbf{0})$ is finite.
    Further, suppose the feasible set of the program is nonempty for all $\mathbf{u}$ in a neighborhood of $\mathbf{0}$, i.e.,
    \begin{align}
        \label{eq:cone_constraint_qualification}
        \mathbf{0} \in \textrm{int} \{ G(\mathbf{X}) - \mathbf{K} \}.
    \end{align}
    Then, 
    \begin{enumerate}
        \item There is no primal dual gap at $u = 0$, i.e., $v_{P}(0) = v_{D}(0)$.
        \item The set, $\Lambda_{0}$, of optimal solutions to the dual problem with $\mathbf{u} = 0$ is non-empty and bounded.
        \item The optimal value function $v_{P}(\mathbf{u})$ is continuous at $\mathbf{u} = 0$ and $\partial v_{P}(\mathbf{0}) = \Lambda_{0}$.
    \end{enumerate}
\end{theorem}

\begin{theorem}[\cite{bonnans:book00} Proposition 4.3.2]
    \label{thm:bonnans-shapiro-00}
    Consider two optimization problems over a non-empty feasible set $\mathbf{\Omega}$:
    \begin{align}
        \min_{\mathbf{x} \in \mathbf{\Omega}} f_{1}(\mathbf{x})
        \quad \quad \textrm{and} \quad \quad
        \min_{\mathbf{x} \in \mathbf{\Omega}} f_{2}(\mathbf{x})
    \end{align}
    where $f_{1}, f_{2}: \mathcal{X} \to \mathbb{R}$.
    Suppose $f_{1}$ has a non-empty set $\mathbf{S}$ of optimal solutions over $\mathbf{\Omega}$.
    Suppose the second order growth condition holds for $\mathbf{S}$, i.e., there exists a neighborhood $\mathcal{N}$ of $\mathbf{S}$ and a constant $\alpha > 0$ such that
    \begin{align}
        f_{1}(\mathbf{x}) \geq f_{1}(\mathbf{S}) + \alpha (dist(\mathbf{x}, \mathbf{S}))^2,
        \quad \quad \forall \mathbf{x} \in \mathbf{\Omega} \cap \mathcal{N},
    \end{align}
    where $f_{1}(\mathbf{S}) := \textrm{inf}_{\mathbf{x} \in \mathbf{\Omega}} f_{1}(\mathbf{x})$.
    Define the difference function:
    \begin{align}
        \Delta(\mathbf{x}) := f_{2}(\mathbf{x}) - f_{1}(\mathbf{x}).
    \end{align}
    Suppose $\Delta(\mathbf{x})$ is $L$-Lipschitz continuous on $\mathbf{\Omega} \cap \mathcal{N}$.
    Let $\mathbf{x}^{*} \in \mathcal{N}$ be a $\delta$-solution to the problem of minimizing $f_{2}(\mathbf{x})$ over $\mathbf{\Omega}$. Then
    \begin{align}
        dist(\mathbf{x}^{*}, \mathbf{S}) \leq \frac{L}{\alpha} + \sqrt{\frac{\delta}{\alpha}}.
    \end{align}
\end{theorem}

\subsubsection{Proofs}
We provide proofs of theorems presented in the main paper and restate \thmref{thm:continuity_properties} here for completeness.

\newtheorem*{restatedcontinuityproperties}{\thmref{thm:continuity_properties} (restated)}
\begin{restatedcontinuityproperties}
    Suppose for any setting of $\mathbf{w}_{nn} \in \mathbb{R}^{n_g}$ there is a feasible solution to NeuPSL inference \eqref{eq:regularized_lcqp_primal}.
    Further, suppose $\epsilon > 0$, $\mathbf{w}_{sy} \in \mathbb{R}_{+}^{r}$, and $\mathbf{w}_{nn} \in \mathbb{R}^{n_g}$.
    Then:
    \begin{itemize}[leftmargin=*,noitemsep,topsep=0pt]
        \item The minimizer of \eqref{eq:regularized_lcqp_primal}, $\mathbf{y}^*(\mathbf{w}_{sy}, \mathbf{w}_{nn})$, is a $O(1 / \epsilon)$ Lipschitz continuous function of $\mathbf{w}_{sy}$.
        \item $V(\mathbf{w}_{sy}, \mathbf{b}(\mathbf{x}_{sy}, \mathbf{g}_{nn}(\mathbf{x}_{nn}, \mathbf{w}_{nn})))$, is concave over $\mathbf{w}_{sy}$ and convex over $\mathbf{b}(\mathbf{x}_{sy}, \mathbf{g}_{nn}(\mathbf{x}_{nn}, \mathbf{w}_{nn}))$.
        \item $V(\mathbf{w}_{sy}, \mathbf{b}(\mathbf{x}_{sy}, \mathbf{g}_{nn}(\mathbf{x}_{nn}, \mathbf{w}_{nn})))$ is differentiable with respect to $\mathbf{w}_{sy}$. Moreover,
        {\small
        \begin{align*}
            \nabla_{\mathbf{w}_{sy}} V(\mathbf{w}_{sy}, & \mathbf{b}(\mathbf{x}_{sy}, \mathbf{g}_{nn}(\mathbf{x}_{nn}, \mathbf{w}_{nn}))) = \mathbf{\Phi}(\mathbf{y}^{*}(\mathbf{w}_{sy}, \mathbf{w}_{nn}), \mathbf{x}_{sy}, \mathbf{g}_{nn}(\mathbf{x}_{nn}, \mathbf{w}_{nn})).
        \end{align*}
        }%
        Furthermore, $\nabla_{\mathbf{w}_{sy}} V(\mathbf{w}_{sy}, \mathbf{b}(\mathbf{x}_{sy}, \mathbf{g}_{nn}(\mathbf{x}_{nn}, \mathbf{w}_{nn})))$ is Lipschitz continuous over $\mathbf{w}_{sy}$.
        \item If there is a feasible point $\nu$ strictly satisfying the $i'th$ inequality constraint of \eqref{eq:regularized_lcqp_primal}, i.e., $\mathbf{A}[i] \mathbf{\nu} + \mathbf{b}(\mathbf{x}_{sy}, \mathbf{g}_{nn}(\mathbf{x}_{nn}, \mathbf{w}_{nn}))[i] < 0$, then $V(\mathbf{w}_{sy}, \mathbf{b}(\mathbf{x}_{sy}, \mathbf{g}_{nn}(\mathbf{x}_{nn}, \mathbf{w}_{nn})))$ is subdifferentiable with respect to the $i'th$ constraint constant $\mathbf{b}(\mathbf{x}_{sy}, \mathbf{g}_{nn}(\mathbf{x}_{nn}, \mathbf{w}_{nn}))[i]$. Moreover, 
        {\small
        \begin{align*}
            \partial_{\mathbf{b}[i]} V(\mathbf{w}_{sy}, & \mathbf{b}(\mathbf{x}_{sy}, \mathbf{g}_{nn}(\mathbf{x}_{nn}, \mathbf{w}_{nn}))) 
            \\ &
            = \{\mathbf{\mu}^{*}[i] \, \vert \, \mathbf{\mu}^{*} \in \argmin_{\mathbf{\mu} \in \mathbb{R}_{\geq 0}^{2 (n_{\mathbf{y}} + m) + q}}
            h(\mathbf{\mu}; \mathbf{w}_{sy}, \mathbf{b}(\mathbf{x}_{sy}, \mathbf{g}_{nn}(\mathbf{x}_{nn}, \mathbf{w}_{nn}))) \}.
        \end{align*}
        }%
        Furthermore, if $\mathbf{g}_{nn}(\mathbf{x}_{nn}, \mathbf{w}_{nn})$ is a smooth function of $\mathbf{w}_{nn}$, then so is $\mathbf{b}(\mathbf{x}_{sy}, \mathbf{g}_{nn}(\mathbf{x}_{nn}, \mathbf{w}_{nn}))$, and the set of regular subgradients of $V(\mathbf{w}_{sy}, \mathbf{b}(\mathbf{x}_{sy}, \mathbf{g}_{nn}(\mathbf{x}_{nn}, \mathbf{w}_{nn})))$ is:
        {\small
        \begin{align}
            \hat{\partial}_{\mathbf{w}_{nn}} & V(\mathbf{w}_{sy}, \mathbf{b}(\mathbf{x}_{sy}, \mathbf{g}_{nn}(\mathbf{x}_{nn}, \mathbf{w}_{nn}))) 
            \\ & 
            \supset \nabla_{\mathbf{w}_{nn}} \mathbf{b}(\mathbf{x}_{sy}, \mathbf{g}_{nn}(\mathbf{x}_{nn}, \mathbf{w}_{nn}))^T \partial_{\mathbf{b}} V(\mathbf{w}_{sy}, \mathbf{b}(\mathbf{x}_{sy}, \mathbf{g}_{nn}(\mathbf{x}_{nn}, \mathbf{w}_{nn}))) 
            \nonumber
            .
        \end{align}
        }%
    \end{itemize}
\end{restatedcontinuityproperties}

\begin{proof}[\textbf{Proof of \thmref{thm:continuity_properties}}]
    \label{proof:continuity_properties}
    We first show the minimizer of the LCQP formulation of NeuPSL inference, $\mathbf{\nu}^*$, with $\epsilon > 0$, $\mathbf{w}_{sy} \in \mathbb{R}_{+}^{r}$, and $\mathbf{w}_{nn} \in \mathbb{R}^{n_g}$ is a Lipschitz continuous function of $\mathbf{w}_{sy}$.
    Suppose $\epsilon > 0$ and $\mathbf{w}_{nn} \in \mathbb{R}^{n_g}$ is given.
    To show continuity over  $\mathbf{w}_{sy} \in \mathbb{R}^{r}_{+}$, first note the matrix $(\mathbf{D} + \epsilon \mathbf{I})$ is positive definite and the primal inference problem \eqref{eq:dual_lcqp_inference} is an $\epsilon$-strongly convex LCQP with a unique minimizer denoted by $\nu^{*}(\mathbf{w}_{sy}, \mathbf{w}_{nn})$.
    We leverage the Lipschitz stability result for optimal values of constrained problems from \cite{bonnans:book00} and presented here in \thmref{thm:bonnans-shapiro-00}.
    Define the primal objective as an explicit function of the weights:
    
    \begin{align}
         f(\mathbf{\nu}, \mathbf{w}_{sy}, \mathbf{w}_{nn}) := \mathbf{\nu}^T (\mathbf{D}(\mathbf{w}_{sy}) + \epsilon \mathbf{I}) \mathbf{\nu} + \mathbf{c}^T(\mathbf{w}_{sy}) \mathbf{\nu}
    \end{align}
    Note that the solution \(\mathbf{\nu}^* =
    \begin{bmatrix}
        \mathbf{s}_{S}^* \\
        \mathbf{s}_{L}^* \\
        \mathbf{y}^*
    \end{bmatrix}\) will always be bounded, since from \eqref{eq:lcqp.3} in LCQP we always have for all \(j \in I_{S} \cup I_{L}\),
    \begin{align}
        0 \le s_j^* &= \max ( \mathbf{a}_{\phi_k, y}^T \mathbf{y}^* + \mathbf{a}_{\phi_k, \mathbf{x}_{sy}}^T \mathbf{x}_{sy} + \mathbf{a}_{\phi_k, \mathbf{g}_{nn}}^T \mathbf{g}_{nn}(\mathbf{x}_{nn}, \mathbf{w}_{nn}) +  b_{\phi_k}, 0) \\
        &\le \|\mathbf{a}_{\phi_k, y}\| + |\mathbf{a}_{\phi_k, \mathbf{x}_{sy}}^T \mathbf{x}_{sy} + \mathbf{a}_{\phi_k, \mathbf{g}_{nn}}^T \mathbf{g}_{nn}(\mathbf{x}_{nn}, \mathbf{w}_{nn}) +  b_{\phi_k}|. 
    \end{align}
    Thus, setting these trivial upper bounds for \(s_j\) will not change the solution of the problem. 
    We can henceforth consider the problem in a bounded domain \(\| \mathbf{\nu} \| \le C\) where \(C\) does not depend on \(\mathbf{w}\)'s.
    
    Let $\mathbf{w}_{1, sy}, \mathbf{w}_{2, sy} \in \mathbb{R}_{+}^{r}$ and $\mathbf{w}_{nn} \in \mathcal{W}_{nn}$ be arbitrary.
    As $\epsilon > 0$, $f(\mathbf{\nu}, \mathbf{w}_{1, sy}, \mathbf{w}_{nn})$ is strongly convex in $\mathbf{\nu}$ and it therefore satisfies the second-order growth condition in $\mathbf{\nu}$. 
    Define the difference function:
    \begin{align}
            \Delta_{\mathbf{w}_{sy}}(\mathbf{\nu}) &:= f(\mathbf{\nu}, \mathbf{w}_{2, sy}, \mathbf{w}_{nn}) - f(\mathbf{\nu}, \mathbf{w}_{1, sy}, \mathbf{w}_{nn})\\
            &= \mathbf{\nu}^T (\mathbf{D}(\mathbf{w}_{2, sy}) + \epsilon \mathbf{I}) \mathbf{\nu} + \mathbf{c}^T(\mathbf{w}_{2, sy}) \mathbf{\nu} - \big( \mathbf{\nu}^T (\mathbf{D}(\mathbf{w}_{1, sy}) + \epsilon \mathbf{I}) \mathbf{\nu} + \mathbf{c}^T(\mathbf{w}_{1, sy}) \mathbf{\nu} \big) \\
            &= \mathbf{\nu}^T (\mathbf{D}(\mathbf{w}_{2, sy}) - \mathbf{D}(\mathbf{w}_{1, sy})) \mathbf{\nu} + (\mathbf{c}(\mathbf{w}_{2, sy}) - \mathbf{c}(\mathbf{w}_{1, sy}))^T \mathbf{\nu}. 
    \end{align}
    The difference function $\Delta_{\mathbf{w}_{sy}}(\mathbf{\nu})$ over $\mathcal{N}$ has a finitely bounded gradient:
        \begin{align} 
        \Vert \nabla \Delta_{\mathbf{w}_{sy}}(\mathbf{\nu}) \Vert_{2}
        &= \Big \Vert 
        2 (\mathbf{D}(\mathbf{w}_{2, sy}) - \mathbf{D}(\mathbf{w}_{1, sy})) \mathbf{\nu} + \mathbf{c}(\mathbf{w}_{2, sy}) - \mathbf{c}(\mathbf{w}_{1, sy})
        \Big \Vert_{2} \\
        &\le \Vert \mathbf{c}(\mathbf{w}_{2, sy}) - \mathbf{c}(\mathbf{w}_{1, sy}) \Vert_{2} + 2 \Vert (\mathbf{D}(\mathbf{w}_{2, sy}) - \mathbf{D}(\mathbf{w}_{1, sy})) \mathbf{\nu} \Vert_{2} \\
        &\le \Vert \mathbf{w}_{2, sy} - \mathbf{w}_{1, sy} \Vert_{2} +
        2\Vert \mathbf{w}_{2, sy} - \mathbf{w}_{1, sy} \Vert_{2} \; \| \mathbf{\nu} \|_{2} \\
        &\le \Vert \mathbf{w}_{2, sy} - \mathbf{w}_{1, sy} \Vert_{2} (1 + 2C) =: 
        L_{\mathcal{N}}(\mathbf{w}_{1, sy}, \mathbf{w}_{2, sy})
        .
    \end{align}
    Thus, the distance function, $\Delta_{\mathbf{w}_{sy}}(\mathbf{\nu})$ is \(L_{\mathcal{N}}(\mathbf{w}_{1, sy}, \mathbf{w}_{2, sy})\)-Lipschitz continuous over $\mathcal{N}$.
    Therefore, by \cite{bonnans:book00} (\thmref{thm:bonnans-shapiro-00}), the distance between $\mathbf{\nu}^{*}(\mathbf{w}_{1, sy}, \mathbf{w}_{nn})$ and $\mathbf{\nu}^{*}(\mathbf{w}_{2, sy}, \mathbf{w}_{nn})$ is bounded above:
    {\small
    \begin{align}
        \Vert \mathbf{\nu}^{*}(\mathbf{w}_{2, sy}, \mathbf{w}_{nn}) - \mathbf{\nu}^{*}(\mathbf{w}_{1, sy}, \mathbf{w}_{nn}) \Vert_{2} \leq \frac{L_{\mathcal{N}}(\mathbf{w}_{1, sy}, \mathbf{w}_{2, sy})}{\epsilon} = \frac{(1 + 2C)}{\epsilon} \Vert \mathbf{w}_{2, sy} - \mathbf{w}_{1, sy} \Vert_{2}.
    \end{align}
    }%
    Therefore, the function $\mathbf{\nu}^{*}(\mathbf{w}_{sy}, \mathbf{w}_{nn})$ is 
\(O(1/\epsilon)\)-Lipschitz continuous in $\mathbf{w}_{sy}$ for any $\mathbf{w}_{nn}$.

    Next, we prove curvature properties of the value-function with respect to the weights.
    Observe NeuPSL inference is an infimum over a set of functions that are concave (affine) in $\mathbf{w}_{sy}$.
    Therefore, by \thmref{thm:pointwise_infimum_over_concave}, we have that $V(\mathbf{w}_{sy}, \mathbf{b}(\mathbf{x}_{sy}, \mathbf{g}_{nn}(\mathbf{x}_{nn}, \mathbf{w}_{nn})))$ is concave in $\mathbf{w}_{sy}$.

    We use a similar argument to show $V(\mathbf{w}_{sy}, \mathbf{b}(\mathbf{x}_{sy}, \mathbf{g}_{nn}(\mathbf{x}_{nn}, \mathbf{w}_{nn})))$ is convex in the constraint constants, $\mathbf{b}(\mathbf{x}_{sy}, \mathbf{g}_{nn}(\mathbf{x}_{nn}, \mathbf{w}_{nn}))$.
    Assuming for any setting of the neural weights, $\mathbf{w}_{nn} \in \mathbb{R}^{n_g}$, there is a feasible solution to the NeuPSL inference problem, then \eqref{eq:regularized_lcqp_primal} satisfies the conditions for Slater's constraint qualification.
    Therefore, strong duality holds, i.e., $V(\mathbf{w}_{sy}, \mathbf{b}(\mathbf{x}_{sy}, \mathbf{g}_{nn}(\mathbf{x}_{nn}, \mathbf{w}_{nn})))$ is equal to the optimal value of the dual inference problem \eqref{eq:dual_lcqp_appendix}.
    Observe that the dual NeuPSL inference problem is a supremum over a set of functions convex (affine) in $\mathbf{b}(\mathbf{x}_{sy}, \mathbf{g}_{nn}(\mathbf{x}_{nn}, \mathbf{w}_{nn}))$.
    Therefore, by \thmref{thm:pointwise_supremum_over_convex}, we have that $V(\mathbf{w}_{sy}, \mathbf{b}(\mathbf{x}_{sy}, \mathbf{g}_{nn}(\mathbf{x}_{nn}, \mathbf{w}_{nn})))$ is convex in $\mathbf{b}(\mathbf{x}_{sy}, \mathbf{g}_{nn}(\mathbf{x}_{nn}, \mathbf{w}_{nn}))$.

    We can additionally prove convexity in \(\mathbf{b}\) from first principles. 
    For simplicity, fix other parameters, and write the objective and the value function as \(Q(\nu)\) and \(V(\mathbf{b})\), respectively.
    Let us first consider the domain where the optimization is bounded and the optimal solution exists.
    Given $\mathbf{b}_1$ and $\mathbf{b}_2$, let the corresponding optimal solutions of \eqref{eq:regularized_lcqp_primal_appendix} parameterized by $\mathbf{b}_1$ and $\mathbf{b}_2$ be $\nu_1$ and $\nu_2$. Take any $\alpha \in [0,1]$, note that $\alpha \nu_1  + (1-\alpha) \nu_2$ is feasible for the optimization problem parameterized by $\mathbf{b} = \alpha \mathbf{b}_1 + (1-\alpha) \mathbf{b}_2$. Because we take the inf over all $\nu$s, the optimal $\nu$ for this $\mathbf{b}$ might be even smaller. Thus, we have  (for convex quadratic objective $Q$) that
    \begin{equation}
    \begin{aligned}
    V(\alpha \mathbf{b}_1 + (1-\alpha) \mathbf{b}_2) & \le  Q(\alpha \nu_1 + (1-\alpha) \nu_2)  \\
    & \le \alpha Q(\nu_1) + (1-\alpha) Q(\nu_2) \\
    & = \alpha V(\mathbf{b}_1) + (1-\alpha) V(\mathbf{b}_2),
    \end{aligned}
    \end{equation}
    which shows that \(V\) is convex in \(\mathbf{b}\).
    To establish the convexity when \(V(\mathbf{b})\) takes extended real-values (\(\mathbb{R} \cup \{-\infty\}\)) to allow for unbounded optimization problems, it suffices to consider sequences \(\{\nu_{i}^k\}_{k=1}^{\infty}\) for \(\mathbf{b}_i\) (\(i=1,2\), \(\mathbf{b}_1 \ne \mathbf{b}_2\)) as follows:
    
    (1) If \(V(\mathbf{b}_i)\) is finite, let \(\nu_{i}^k = \nu_i\) for all \(k\), where \(\nu_i\) is the optimal solution.
    
    (2) If \(V(\mathbf{b}_i) = -\infty\), there exists sequence \(\{\nu_{i}^k\}_{k=1}^{\infty}\) such that \(Q(\nu_{i}^k) \to -\infty\) as \(k \to \infty\).

    Now, for any \(0 < \alpha < 1\), observe:

    Case 1: Both \(V(\mathbf{b}_1)\) and \(V(\mathbf{b}_2)\) are finite. We can reuse the argument above.

Case 2: At least one of \(V(\mathbf{b}_1)\) and \(V(\mathbf{b}_2)\) is \(-\infty\).
    By convexity of \(Q\), \(Q(\alpha \nu_1^k  + (1-\alpha) \nu_2^k) \le \alpha Q(\nu_1^k) + (1-\alpha) Q(\nu_2^k)\). Therefore, we have \(Q(\alpha \nu_1^k  + (1-\alpha) \nu_2^k) \to -\infty\) as \(k \to \infty\) when \(0 < \alpha < 1\). Note that for all \(k\), \(\alpha \nu_1^k  + (1-\alpha) \nu_2^k\) is feasible for the optimization problem parameterized by $\mathbf{b} = \alpha \mathbf{b}_1 + (1-\alpha) \mathbf{b}_2$. It follows that \(V(\alpha \mathbf{b}_1 + (1-\alpha) \mathbf{b}_2) = -\infty\).

    Therefore, convexity holds when \(V(\mathbf{b})\) takes extended real-values (\(\mathbb{R} \cup \{-\infty\}\)). 

    Next, we prove (sub)differentiability properties of the value-function.
    Suppose $\epsilon > 0$.
    First, we show the optimal value function, $V(\mathbf{w}_{sy}, \mathbf{b}(\mathbf{x}_{sy}, \mathbf{g}_{nn}(\mathbf{x}_{nn}, \mathbf{w}_{nn})))$, is differentiable with respect to the symbolic weights.
    Then we show subdifferentiability properties of the optimal value function with respect to the constraint constants.
    Finally, we apply the Lipschitz continuity of the minimzer result to show the gradient of the optimal value function is Lipschitz continuous with respect to $\mathbf{w}_{sy}$.

    Starting with differentiability with respect to the symbolic weights, $\mathbf{w}_{sy}$, note, the optimal value function of the regularized LCQP formulation of NeuPSL inference, \eqref{eq:regularized_lcqp_primal}, is equivalently expressed as the following maximization over a continuous function in the primal target variables, $\mathbf{y}$, the slack variables, $\mathbf{s}_{S}$ and $\mathbf{s}_{L}$, and the symbolic weights, $\mathbf{w}_{sy}$:
    \begin{align}
        & V(\mathbf{w}_{sy}, \mathbf{b}(\mathbf{x}_{sy}, \mathbf{g}_{nn}(\mathbf{x}_{nn}, \mathbf{w}_{nn}))) \\
        & \quad = - \Bigg ( \max_{\mathbf{y}, \mathbf{s_{H}}, \mathbf{s_L}} \,
        - \Big ( \begin{bmatrix}
            \mathbf{s}_{S} \\ \mathbf{s}_{L} \\ \mathbf{y}
        \end{bmatrix}^T
        \begin{bmatrix}
            \mathbf{W}_{S} + \epsilon I & 0 & 0\\
            0 & \epsilon I & 0 \\
            0 & 0 & \epsilon I \\
        \end{bmatrix}
        \begin{bmatrix}
            \mathbf{s}_{S} \\
            \mathbf{s}_{L} \\
            \mathbf{y}  
        \end{bmatrix}
        + 
        \begin{bmatrix}
           0 \\ \mathbf{w}_{L} \\ 0
        \end{bmatrix}^T
        \begin{bmatrix}
            \mathbf{s}_{S} \\
            \mathbf{s}_{L} \\
            \mathbf{y}
        \end{bmatrix}
        \Big ) \Bigg ) \nonumber \\
        & \quad \quad \quad \quad \textrm{s.t.} \quad         
        \mathbf{A}         
        \begin{bmatrix}
            \mathbf{s}_{S} \\
            \mathbf{s}_{L} \\
            \mathbf{y}
        \end{bmatrix} + \mathbf{b}(\mathbf{x}_{sy}, \mathbf{g}_{nn}(\mathbf{x}_{nn}, \mathbf{w}_{nn})) \leq 0 \nonumber,
    \end{align}
    where the matrix $\mathbf{W}_{s}$ and vector $\mathbf{w}_{L}$ are functions of the symbolic parameters $\mathbf{w}_{sy}$ as defined in \eqref{eq:symbolic_weight_matrix_and_vector}.
    Moreover, the objective above is convex (affine) in $\mathbf{w}_{sy}$.
    Additionally, note that the decision variables can be constrained to a compact domain without breaking the equivalence of the formulation.
    Specifically, the target variables are constrained to the box $[0, 1]^{\mathbf{n}_{y}}$, while the slack variables are nonnegative and have a trivial upper bound derived from \eqref{eq:lcqp.3}:,
    \begin{align}
        0 \le s_j^* &= \max ( \mathbf{a}_{\phi_k, y}^T \mathbf{y}^* + \mathbf{a}_{\phi_k, \mathbf{x}_{sy}}^T \mathbf{x}_{sy} + \mathbf{a}_{\phi_k, \mathbf{g}_{nn}}^T \mathbf{g}_{nn}(\mathbf{x}_{nn}, \mathbf{w}_{nn}) +  b_{\phi_k}, 0) \nonumber \\
        &\le \|\mathbf{a}_{\phi_k, y}\| + |\mathbf{a}_{\phi_k, \mathbf{x}_{sy}}^T \mathbf{x}_{sy} + \mathbf{a}_{\phi_k, \mathbf{g}_{nn}}^T \mathbf{g}_{nn}(\mathbf{x}_{nn}, \mathbf{w}_{nn}) +  b_{\phi_k}|, 
    \end{align}
    for all \(j \in I_{S} \cup I_{L}\).
    Therefore, the negative optimal value function satisfies the conditions for Danskin's theorem \cite{danskin:siam66} (stated in \appref{appendix:continuity_of_inference_preliminaries}).
    Moreover, as there is a single unique solution to the inference problem when $\epsilon > 0$, and the quadratic objective in \eqref{eq:regularized_lcqp_primal} is differentiable for all $\mathbf{w}_{sy} \in \mathbb{R}^{r}_{+}$, we can apply \corollaryref{corollary:danskins_unique_solution}.
    The optimal value function is therefore concave and differentiable with respect to the symbolic weights with
    \begin{align}
        \nabla_{\mathbf{w}_{sy}} V(\mathbf{w}_{sy}, \mathbf{b}(\mathbf{x}_{sy}, \mathbf{g}_{nn}(\mathbf{x}_{nn}, \mathbf{w}_{nn})) = \mathbf{\Phi(\mathbf{y}^{*}, \mathbf{x}_{sy}, \mathbf{g}_{nn}(\mathbf{x}_{nn}, \mathbf{w}_{nn}))}.
    \end{align}

    Next, we show subdifferentiability of the optimal value-function with respect to the constraint constants, $\mathbf{b}(\mathbf{x}_{sy}, \mathbf{g}_{nn}(\mathbf{x}_{nn}, \mathbf{w}_{nn}))$.
    Suppose at a setting of the neural weights $\mathbf{w}_{nn} \in \mathbb{R}^{n_g}$ there is a feasible point $\nu$ for the NeuPSL inference problem. 
    Moreover, suppose $\nu$ strictly satisfies the $i'th$ inequality constraint of \eqref{eq:regularized_lcqp_primal}, i.e., $\mathbf{A}[i] \mathbf{\nu} + \mathbf{b}(\mathbf{x}_{sy}, \mathbf{g}_{nn}(\mathbf{x}_{nn}, \mathbf{w}_{nn}))[i] < 0$.
    Observe that the following strongly convex conic program is equivalent to the LCQP formulation of NeuPSL inference, \eqref{eq:regularized_lcqp_primal}:
    \begin{align}
        \min_{\mathbf{\nu} \in \mathbb{R}^{n_{\mathbf{y}} + m_{S} + m_{L}}} & \, 
            \mathbf{\nu}^T (\mathbf{D}(\mathbf{w}_{sy}) + \epsilon \mathbf{I}) \mathbf{\nu} + \mathbf{c}(\mathbf{w}_{sy})^T \mathbf{\nu} + P_{\Omega \setminus i}(\mathbf{\nu}) \label{eq:regularized_lcqp_cone_primal}\\
            \textrm{s.t.} \quad         
            & \mathbf{A}[i] \mathbf{\nu} + \mathbf{b}(\mathbf{x}_{sy}, \mathbf{g}_{nn}(\mathbf{x}_{nn}, \mathbf{w}_{nn}))[i] \in \mathbb{R}_{\leq 0} \nonumber,
    \end{align}
    where $P_{\Omega \setminus i}(\mathbf{\nu}): \mathbb{R}^{n_{\mathbf{y}} + m_{S} + m_{L}} \to \{0, \infty\}$ is the indicator function identifying feasibility w.r.t. all the constraints of the LCQP formulation of NeuPSL inference in \eqref{eq:regularized_lcqp_primal} except the $i'th$ constraint: $\mathbf{A}[i] \mathbf{\nu} + \mathbf{b}(\mathbf{x}_{sy}, \mathbf{g}_{nn}(\mathbf{x}_{nn}, \mathbf{w}_{nn}))[i] \leq 0$.
    In other words, in the conic formulation above only the $i'th$ constraint is explicit.
    Note that $\mathbb{R}_{\leq 0}$ is a closed convex cone in $\mathbb{R}$.
    Moreover, both the objective in the program and the mapping $G(\mathbf{\nu}) := \mathbf{A}[i] \mathbf{\nu} + \mathbf{b}(\mathbf{x}_{sy}, \mathbf{g}_{nn}(\mathbf{x}_{nn}, \mathbf{w}_{nn}))[i]$ are convex.
    Lastly, note the constraint qualification \eqref{eq:cone_constraint_qualification} is similar to Slater's condition in the case of \eqref{eq:regularized_lcqp_cone_primal} which is satisfied by the supposition there exists a feasible $\nu$ that strictly satisfies the $i'th$ inequality constraint of \eqref{eq:regularized_lcqp_primal}.
    Therefore, \eqref{eq:regularized_lcqp_cone_primal} satisfies the conditions of \thmref{thm:value_function_subdifferentiable_wrt_constraint_constants}.
    Thus, the value function is continuous in the constraint constant $\mathbf{b}(\mathbf{x}_{sy}, \mathbf{g}_{nn}(\mathbf{x}_{nn}, \mathbf{w}_{nn}))[i]$ at $\mathbf{w}_{nn}$ and 
    \begin{align}
        & \partial_{\mathbf{b}[i]} V(\mathbf{w}_{sy}, \mathbf{b}(\mathbf{x}_{sy}, \mathbf{g}_{nn}(\mathbf{x}_{nn}, \mathbf{w}_{nn}))) \\ 
        & \quad = \{\mathbf{\mu}^{*}[i] \, \vert \, \mathbf{\mu}^{*} \in \argmin_{\mathbf{\mu} \in \mathbb{R}_{\geq 0}^{2 \cdot n_{\mathbf{y}} + m + q}}
        h(\mathbf{\mu}; \mathbf{w}_{sy}, \mathbf{b}(\mathbf{x}_{sy}, \mathbf{g}_{nn}(\mathbf{x}_{nn}, \mathbf{w}_{nn}))) \}. \nonumber
    \end{align}
    Moreover, when $\mathbf{b}(\mathbf{x}_{sy}, \mathbf{g}_{nn}(\mathbf{x}_{nn}, \mathbf{w}_{nn}))$ is a smooth function of the neural weights $\mathbf{w}_{nn}$, then we can apply the chain rule for regular subgradients, \thmref{theorem:subgradient_chain_rule}, to get 
    \begin{align}
        & \hat{\partial}_{\mathbf{w}_{nn}} V(\mathbf{w}_{sy}, \mathbf{b}(\mathbf{x}_{sy}, \mathbf{g}_{nn}(\mathbf{x}_{nn}, \mathbf{w}_{nn})) \\ 
        & \quad  \supset \nabla \mathbf{b}(\mathbf{x}_{sy}, \mathbf{g}_{nn}(\mathbf{x}_{nn}, \mathbf{w}_{nn}) ^T \partial_{\mathbf{b}} V(\mathbf{w}_{sy}, \mathbf{b}(\mathbf{x}_{sy}, \mathbf{g}_{nn}(\mathbf{x}_{nn}, \mathbf{w}_{nn})). \nonumber
    \end{align}

    To prove the optimal value function is Lipschitz smooth over $\mathbf{w}_{sy}$, it is equivalent to show it is continuously differentiable and that all gradients have bounded magnitude.
    To show the value function is continuously differentiable, we first apply the result asserting the minimizer is unique and a continuous function of the symbolic parameters $\mathbf{w}_{sy}$.
    Therefore, the optimal value function gradient is a composition of continuous functions, hence continuous in $\mathbf{w}_{sy}$.
    The fact that the value function has a bounded gradient magnitude follows from the fact that the decision variables $\mathbf{y}$ have a compact domain over which the gradient is finite; hence a trivial and finite upper bound exists on the gradient magnitude.
\end{proof}

\subsection{Dual Block Coordinate Descent}
\label{appendix:dual_bcd}

The regular subgradients in \thmref{thm:continuity_properties} are functions of the optimal dual variables of the LCQP inference problem in \eqref{eq:dual_lcqp_inference}.
Thus, one could compute value-function gradients for learning by solving the primal inference problem \eqref{eq:inference_primal} using an existing algorithm and then map the optimal primal variables to dual variables.
However, the primal-to-dual mapping requires computing a pseudo-inverse of the constraint matrix.
For this reason, we develop a block coordinate descent (BCD)~\citep{wright:mp15} algorithm for working directly with the dual LCQP formulation of inference.
The dual BCD algorithm (D-BCD) is the first method specialized for the NeuPSL dual LCQP inference problem.
It is, therefore, also the first to produce optimal dual variables that directly yield both optimal primal variables and principled gradients for learning, all without the need to compute a pseudo-inverse of the constraint matrix.

The D-BCD inference algorithm proceeds by successively minimizing the objective along the subgradient of a block of dual variables.
For this reason, dual BCD guarantees descent at every iteration.
The algorithm is stopped when the primal-dual gap drops below a threshold $\delta>0$.

In this section, we omit the symbolic and neural weights from the function arguments to simplify notation.
We define $U_i$, $i = 1,2,\dotsc, p$ to be a cover of the dual variable components $\{1,2,\dotsc,n_{\mathbf{y}} + m + q\}$.
In practice, blocks are defined as a single dual variable corresponding to a constraint from the feasible set or a deep hinge-loss function, along with the dual variables corresponding to the bounds of the primal variables in the constraint or hinge-loss.

We will deal with a slightly more general objective,
\begin{align}
    h(\mathbf{\mu}) := \frac{1}{2} \mathbf{\mu}^T \mathbf{A} \tilde{\mathbf{D}} \mathbf{A}^T \mathbf{\mu} + \tilde{\mathbf{c}}^T \mathbf{\mu},
\end{align}
from which we can recover \eqref{eq:lcqp_lagrangian_appendix} by replacing \(\tilde{\mathbf{D}} \gets (\mathbf{D} + \epsilon \mathbf{I})^{-1}\) and \(\tilde{\mathbf{c}} \gets \mathbf{A} (\mathbf{D} + \epsilon \mathbf{I})^{-1} \mathbf{c} - 2 \mathbf{b}\).

We will use the superscript \(\cdot^{(l)}\) to denote values in the \(l\)-th iteration and subscript \(\cdot_{[i]}\) for the values corresponding to the block \(U_i\). 
The row submatrix of $\mathbf{A}$ that corresponds to block $i$ is denoted by $\mathbf{A}_{[i]}$.

At each iteration $l$, we choose one block $i \in \{1,2,\dotsc,p\}$ at random and compute the subvector of $\nabla h(\mu^{[l]})$ that corresponds to this block,
\begin{align}
    \mathbf{d}^{(l)}_{[i]} := \nabla_{[i]} h(\mathbf{\mu}^{(l)}) = (\mathbf{A} \tilde{\mathbf{D}} \mathbf{A}^T \mathbf{\mu}^{(l)} + \tilde{\mathbf{c}})_{[i]}.
\end{align}
Defining $\mathbf{d}^{(l)}$ to be the vector in $\mathbb{R}^N$ whose $i$th block is $\mathbf{d}^{(l)}_{[i]}$ with zeros elsewhere, we perform a line search along the negative of this direction. 
Note that 
\begin{align}
        h(\mathbf{\mu}^{(l)} - \alpha \mathbf{d}^{(l)}) &= \frac{1}{2} \alpha^2 \mathbf{d}^{(l)T} \mathbf{A} \tilde{\mathbf{D}} \mathbf{A}^T \mathbf{d}^{(l)} - \alpha \mathbf{d}^{(l)T}(\mathbf{A} \tilde{\mathbf{D}} \mathbf{A}^T \mathbf{\mu}^{(l)} + \tilde{\mathbf{c}}) + \textbf{constant} \\
        &= \frac{1}{2} \alpha^2 \mathbf{d}_{[i]}^{(l)T} \mathbf{A}_{[i]} \tilde{\mathbf{D}} \mathbf{A}_{[i]}^T \mathbf{d}_{[i]} ^{(l)} - \alpha \mathbf{d}_{[i]}^{(l)T} \mathbf{d}_{[i]}^{(l)} + \textbf{constant}.
\end{align}
The unconstrained minimizer of this expression is
\begin{equation}
    \alpha_l^* = \frac{\mathbf{d}_{[i]}^{(l)T}\mathbf{d}_{[i]}^{(l)}}{\mathbf{d}_{[i]}^{(l)T} \mathbf{A}_{[i]} \tilde{\mathbf{D}} \mathbf{A}_{[i]}^T \mathbf{d}_{[i]}^{(l)}}.
\end{equation}
Given the nonnegativity constraints, we also need to ensure that \(\mathbf{\mu}^{(l)}_{[i]} - \alpha \mathbf{d}^{(l)}_{[i]} \ge 0\). 
Therefore, our choice of steplength is 
\begin{equation}
    \label{eq:step_length}
    \alpha_l = \min \left \{ \alpha_l^*, \min_{j \in U_i \, : \, \mathbf{d}^{(l)}_j>0} \frac{\mathbf{\mu}^{(l)}_j}{\mathbf{d}^{(l)}_j} \right \}.
\end{equation}
To save some computation, we introduce intermediate variables \(\mathbf{f}^{(l)} := \mathbf{A}^T \mathbf{d}^{(l)} = \mathbf{A}^T_{[i]} \mathbf{d}^{(l)}_{[i]}\), and \(\mathbf{m}^{(l)} := \mathbf{A}^T \mu^{(l)}\). 
With the intermediate variables, the updates of the BCD algorithm are:
\begin{align}
    \mathbf{d}_{[i]}^{(l)} &\gets \mathbf{A}_{[i]} \tilde{\mathbf{D}} \mathbf{m}^{(l)} + \tilde{\mathbf{c}}_{[i]}, \, \mathbf{f}^{(l)} \gets \mathbf{A}^T_{[i]} \mathbf{d}^{(l)}_{[i]} \\
    \mathbf{m}^{(l+1)} &\gets \mathbf{A}^T (\mu^{(l)} - \alpha_l \mathbf{d}^{(l)}) = \mathbf{m}^{(l)} - \alpha_l \mathbf{f}^{(l)}.
\end{align}
With the steplength suggested by \eqref{eq:step_length}, descent is guaranteed at each iteration.
This property is partially why the dual BCD algorithm is effective at leveraging warm-starts, which is valuable for improving the runtime of learning algorithms.

\begin{algorithm}[ht]
\caption{Dual LCQP Block Coordinate Descent}
\label{algo:dual_lcqp_bcd}
\begin{algorithmic}[1]
    \STATE Set $l = 0$ and compute an initial feasible point $\mathbf{\mu}^{(0)}$;
    \STATE Compute \(\mathbf{m}^{(0)} = \mathbf{A}^T \mu^{(0)}\);
    \WHILE{Stopping Criterion Not Satisfied}
    \STATE \(S_k \gets\) \textrm{Permutation}(\([1, 2, \dotsc,p]\));
        \FORALL{$i \in S_k$ (in order)}
            \STATE Compute \(\mathbf{d}^{(l)}_{[i]} \gets \mathbf{A}_{[i]} \tilde{\mathbf{D}} \mathbf{m}^{(l)} + \tilde{\mathbf{c}}_{[i]}\); \quad \(\mathbf{f}^{(l)} \gets \mathbf{A}^T_{[i]} \mathbf{d}^{(l)}_{[i]}\);
            \STATE Compute \(\alpha_l \gets \min \left \{\frac{\mathbf{d}^{(l)T}_{[i]} \mathbf{d}^{(l)}_{[i]}}{\mathbf{f}^{(l)T} \tilde{\mathbf{D}} \mathbf{f}^{(l)}},
            \min_{j \in U_i :\,  \mathbf{d}^{(l)}_j>0} \frac{\mathbf{\mu}^{(l)}_j}{\mathbf{d}^{(l)}_j} \right \}\) 
            \STATE \(
                \mathbf{\mu}^{(l + 1)}_{[i]} \gets \mathbf{\mu}^{(l)}_{[i]} - \alpha_l\mathbf{d}^{(l)}_{[i]}\); \quad $\mu^{(l+1)}_{[j]} \gets \mu^{(l)}_{[j]}$ for all $j \ne i$;
            \STATE \(\mathbf{m}^{(l+1)} \gets \mathbf{m}^{(l)} - \alpha_l \mathbf{f}^{(l)}\);
            \STATE{$l \leftarrow l + 1$};
        \ENDFOR
    \STATE \(k \gets k + 1\);
    \ENDWHILE
    \end{algorithmic}
\end{algorithm}

As strong duality holds for the LCQP formulation of deep HL-MRF inference, stopping when the primal-dual gap is below a given threshold $\delta>0$ is a principled stopping criterion.
Formally, at any iteration of \algoref{algo:dual_lcqp_bcd} applied to \eqref{eq:dual_lcqp_appendix}, we recover an estimate of the primal variable $\mathbf{\nu}$ from \eqref{eq:dual_primal_translation_appendix} and terminate when the gap between the primal and the dual objective falls below $\delta$.
The stopping criterion is checked after every permutation block has been completely iterated over. 

%% file: appendix/experiments.tex
\section{Extended Empirical Analysis}
\label{appendix:experiments}

In this section, we provide additional information on the empirical analysis.
The subsequent subsections will examine the modular datasets and the hyperparameters employed for each experiment.
Additional model details, including neural model architectures and symbolic model constraints, can be found at \url{https://github.com/linqs/dickens-arxiv24}.

\subsection{Modular Datasets}
\label{appendix:modular-datasets}
\begin{itemize}
    \item \textbf{4Forums and CreateDebate}:
    Stance-4Forums and Stance-CreateDebate are two datasets containing dialogues from online debate sites: \url{4forums.com} and \url{createdebate.com}, respectively.
    In this paper, we study stance classification, i.e., the task of identifying the stance of a speaker in a debate as being for or against.

    \item \textbf{Epinions}:
    Epinions is a trust network with $2,000$ individuals connected by $8,675$ directed edges representing whether they know each other and whether they trust each other \cite{richardson:iswc03}. 
    We study link prediction, i.e., we predict if two individuals trust each other.

    In each of the $5$ data splits, the entire network is available, and the prediction performance is measured on $\frac{1}{8}$ of the trust labels. 
    The remaining set of labels are available for training.
    We use the NeuPSL model from \citenoun{bach:jmlr17}.
    The data and NeuPSL model are available at \url{https://github.com/linqs/psl-examples/tree/main/epinions}.
    \item \textbf{Citeseer and Cora}:
    Citeseer and Cora are citation networks introduced by \citenoun{sen:aim08}.
    For Citeseer, $3,312$ documents are connected by $4,732$ edges representing citation links. 
    For Cora, $2,708$ documents are connected by $5,429$ edges representing citation links.
    We study node classification, i.e., we classify the documents into one of $6$ topics for Citeseer and $7$ topics for Cora.

    For each of the $10$ folds, we randomly sample $5\%$ of the node labels for training, $5\%$ of the node labels for validation, and $1,000$ for testing.
    The models for modular learning performance experiments are extended versions from \citenoun{bach:jmlr17} \cite{bach:jmlr17}.
    Specifically, a copy of each rule is made that is specialized for the topic.
    Moreover, topic propagation across citation links is considered for papers with differing topics. 
    For instance, the possibility of a citation from a paper with topic $'\pslarg{A}'$ could imply a paper is more or less likely to be topic $'\pslarg{B}'$.
    The extended models are available at \url{https://github.com/linqs/dickens-arxiv24/tree/main/modular_learning/psl-extended-examples}.
    The models for learning prediction performance experiments are from \citenoun{pryor:ijcai23}.
    The data and models are available at: \url{https://github.com/linqs/dickens-arxiv24/tree/main/citation/models/symbolic}.
    \item \textbf{DDI}:
    Drug-drug interaction (DDI) is a network of $315$ drugs and $4,293$ interactions derived from the DrugBank database \citep{wishart:nar06}.
    The edges in the drug network represent interactions and seven different similarity metrics.
    In this paper, we perform link prediction, i.e., we infer unknown drug-drug interactions.
    
    The $5$ data splits and the NeuPSL model we evaluate in this paper originated from \citenoun{sridhar:bio16}. 
    The data and NeuPSL models are available at: \url{https://github.com/linqs/psl-examples/tree/main/drug-drug-interaction}.
    \item \textbf{Yelp}:
    Yelp is a network of $34,454$ users and $3,605$ items connected by $99,049$ edges representing ratings.
    The task is to predict missing ratings, i.e., regression, which could be used in a recommendation system.
    
    In each of the $5$ folds, $80\%$ of the ratings are randomly sampled and available for training, and the remaining $20\%$ is held out for testing.
    We use the NeuPSL model from \citenoun{kouki:recsys15}.
    The data and NeuPSL model are available at: \url{https://github.com/linqs/psl-examples/tree/main/yelp}.
\end{itemize}

\subsection{Hyperparameters}
\label{appendix:experiments-hyperparameters}

The hyperparameter ranges were decided upon based on the results presented in \cite{pryor:ijcai23}, \cite{dickens:icml24}, and \cite{dickens:make24}.
For the complete set of hyperparameter settings, please refer to the original papers or visit \textit{https://github.com/linqs/dickens-arxiv24}.

\begin{table}[ht]
    \footnotesize
    \centering
    \caption{Hyperparameter ranges and final values for the NeSy-EBM learning experiments.}
    \label{tab:appendix-hyperparameters-learning}
    \begin{tabular}{c|c|c||c|c}
        \toprule
            & \textbf{Algorithm} & \textbf{Parameter} & \textbf{Range} & \textbf{Final Value} \\
        \midrule
        \midrule
            \multirow{5}{*}{\emph{MNIST-Add1}} & \multirow{1}{*}{\textbf{Energy}} 
            & \textbf{Neural Learning Rate} & $\{10^{-3}, 10^{-4}, 10^{-5} \}$ & $10^{-4}$ \\
        \cline{2-5}
            & \multirow{2}{*}{\textbf{Bilevel}}
            & \textbf{Energy Loss Coefficient} & $\{10^{-1}, 1, 10 \}$ & $10$ \\
            & & \textbf{Neural Learning Rate} & $\{10^{-3}, 10^{-4}, 10^{-5} \}$ & $10^{-4}$ \\
        \cline{2-5}
            & \multirow{2}{*}{\textbf{Policy}}
            & \textbf{Energy Loss Coefficient} & $\{10^{-1}, 1, 10 \}$ & $10$ \\
            & & \textbf{Neural Learning Rate} & $\{10^{-3}, 10^{-4}, 10^{-5} \}$ & $10^{-4}$ \\
        \midrule
            \multirow{5}{*}{\emph{MNIST-Add2}} & \multirow{1}{*}{\textbf{Energy}} 
            & \textbf{Neural Learning Rate} & $\{10^{-3}, 10^{-4}, 10^{-5} \}$ & $10^{-4}$ \\
        \cline{2-5}
            & \multirow{2}{*}{\textbf{Bilevel}}
            & \textbf{Energy Loss Coefficient} & $\{10^{-1}, 1, 10 \}$ & $10$ \\
            & & \textbf{Neural Learning Rate} & $\{10^{-3}, 10^{-4}, 10^{-5} \}$ & $10^{-4}$ \\
        \cline{2-5}
            & \multirow{2}{*}{\textbf{Policy}}
            & \textbf{Energy Loss Coefficient} & $\{10^{-1}, 1, 10 \}$ & $10$ \\
            & & \textbf{Neural Learning Rate} & $\{10^{-3}, 10^{-4}, 10^{-5} \}$ & $10^{-4}$ \\
        \midrule
            \multirow{8}{*}{\emph{Visual-Sudoku}} & \multirow{1}{*}{\textbf{Energy}} 
            & \textbf{Neural Learning Rate} & $\{10^{-3}, 10^{-4}, 10^{-5} \}$ & $10^{-4}$ \\
            & & \textbf{Alpha} & $\{0.1, 0.5, 0.9\}$ & $0.1$ \\
        \cline{2-5}
            & \multirow{2}{*}{\textbf{Bilevel}}
            & \textbf{Energy Loss Coefficient} & $\{10^{-1}, 1, 10 \}$ & $10$ \\
            & & \textbf{Neural Learning Rate} & $\{10^{-3}, 10^{-4}, 10^{-5} \}$ & $10^{-3}$ \\
            & & \textbf{Alpha} & $\{0.1, 0.5, 0.9\}$ & $0.1$ \\
        \cline{2-5}
            & \multirow{2}{*}{\textbf{Policy}}
            & \textbf{Energy Loss Coefficient} & $\{10^{-1}, 1, 10 \}$ & $10$ \\
            & & \textbf{Neural Learning Rate} & $\{10^{-3}, 10^{-4}, 10^{-5} \}$ & $10^{-3}$ \\
            & & \textbf{Alpha} & $\{0.1, 0.5, 0.9\}$ & $0.1$ \\
        \midrule
            \multirow{6}{*}{\emph{Path-Finding}} & \multirow{1}{*}{\textbf{Energy}} 
            & \textbf{Neural Learning Rate} & $\{10^{-3}, 10^{-4}, 10^{-5} \}$ & $10^{-3}$ \\
        \cline{2-5}
            & \multirow{2}{*}{\textbf{Bilevel}} 
            & \textbf{Energy Loss Coefficient} & $\{10^{-1}, 1\}$ & $1$ \\
            & & \textbf{Neural Learning Rate} & $\{5^{-4}, 10^{-4}, 10^{-5} \}$ & $5^{-4}$ \\
        \cline{2-5}
            & \multirow{3}{*}{\textbf{Policy}}
            & \textbf{Energy Loss Coefficient} & $\{10^{-1}, 1\}$ & $1$ \\
            & & \textbf{Neural Learning Rate} & $\{5^{-4}, 10^{-4}, 10^{-5}\}$ & $5^{-4}$ \\
            & & \textbf{Alpha} & $\{0.1, 0.5, 0.9\}$ & $0.1$ \\
        \midrule
            \multirow{8}{*}{\emph{Citeseer}} & \multirow{2}{*}{\textbf{Energy}}
            & \textbf{Neural Learning Rate} & $\{10^{-1}, 10^{-2}, 10^{-3} \}$ & $10^{-3}$ \\
            & & \textbf{Step Size} & $\{10^{-1}, 10^{-2}, 10^{-3} \}$ & $10^{-3}$ \\
        \cline{2-5}
            & \multirow{3}{*}{\textbf{Bilevel}} 
            & \textbf{Energy Loss Coefficient} & $\{0, 10^{-1}, 1, 10\}$ & $1$ \\
            & & \textbf{Neural Learning Rate} & $\{10^{-1}, 10^{-2}, 10^{-3} \}$ & $10^{-3}$ \\
            & & \textbf{Step Size} & $\{10^{-1}, 10^{-2}, 10^{-3} \}$ & $10^{-3}$ \\
        \cline{2-5}
            & \multirow{3}{*}{\textbf{Policy}}
            & \textbf{Energy Loss Coefficient} & $\{0, 10^{-1}, 1, 10\}$ & $1$ \\
            & & \textbf{Neural Learning Rate} & $\{10^{-1}, 10^{-2}, 10^{-3} \}$ & $10^{-3}$ \\
            & & \textbf{Alpha} & $\{0.1, 0.5, 0.9\}$ & $0.1$ \\
        \midrule
            \multirow{8}{*}{\emph{Cora}} & \multirow{2}{*}{\textbf{Energy}}
            & \textbf{Neural Learning Rate} & $\{10^{-1}, 10^{-2}, 10^{-3} \}$ & $10^{-3}$ \\
            & & \textbf{Step Size} & $\{10^{-1}, 10^{-2}, 10^{-3} \}$ & $10^{-3}$ \\
        \cline{2-5}
            & \multirow{3}{*}{\textbf{Bilevel}} 
            & \textbf{Energy Loss Coefficient} & $\{0, 10^{-1}, 1, 10\}$ & $1$ \\
            & & \textbf{Neural Learning Rate} & $\{10^{-1}, 10^{-2}, 10^{-3} \}$ & $10^{-3}$ \\
            & & \textbf{Step Size} & $\{10^{-1}, 10^{-2}, 10^{-3} \}$ & $10^{-3}$ \\
        \cline{2-5}
            & \multirow{3}{*}{\textbf{Policy}}
            & \textbf{Energy Loss Coefficient} & $\{0, 10^{-1}, 1, 10\}$ & $1$ \\
            & & \textbf{Neural Learning Rate} & $\{10^{-1}, 10^{-2}, 10^{-3} \}$ & $10^{-3}$ \\
            & & \textbf{Alpha} & $\{0.1, 0.5, 0.9\}$ & $0.1$ \\
        \bottomrule
    \end{tabular}
\end{table}

\subsection{LCQP Regularization Sensitivity}
\label{appendix:empirical_lcqp_regularization}

The regularization parameter added to the LCQP formulation of NeuPSL inference in \eqref{eq:regularized_lcqp_primal} ensures strong convexity of the optimal value of the energy function.
However, adding regularization makes the new formulation an approximation.
In this section, the runtime and prediction performance of the D-BCD inference algorithm (see \appref{appendix:dual_bcd}) are evaluated at varying levels of regularization to understand its effect on NeuPSL inference.
The regularization parameter varies in the range $\epsilon \in \{100, 10, 1, 0.1, 0.01\}$.
The D-BCD algorithm is stopped when the primal-dual gap drops below $\delta = 0.1$.
Inference time is provided in seconds.
The datasets are described in \appref{appendix:modular-datasets}.
Results are provided in \tabref{tab:dual-bcd-regularization}.

\begin{table}[ht]
    \centering
    \caption{D-BCD inference time in seconds and prediction performance with varying values for the LCQP regularization parameter $\epsilon$ in \eqref{eq:regularized_lcqp_primal}.}
    \label{tab:dual-bcd-regularization}
    \scalebox{0.7}{
    \begin{tabular}{cc||c|c}
    \toprule 
    \textbf{Dataset} & $\epsilon$ & \textbf{Time (sec)} & \textbf{Perf.} \\
    \midrule
    \midrule
    \multirow{5}{*}{\textbf{CreateDebate (AUROC)}} 
    & $100$ & $0.02 \pm 0.01$ & $ 64.77 \pm 10.61 $ \\
    & $10$ & $0.02 \pm 0.01$ & $ 64.83 \pm 10.53 $ \\
    & $1$ & $0.02 \pm 0.01$ & $ 64.74 \pm 10.67 $ \\
    & $0.1$ & $0.05 \pm 0.02$ & $ 65.39 \pm 9.07 $ \\
    & $0.01$ & $0.42 \pm 0.51$ & $ 66.01 \pm 9.35 $ \\
    \hline    
    \multirow{5}{*}{\textbf{4Forums (AUROC)}} 
    & $100$ & $0.11 \pm 0.02$ & $ 61.31 \pm 6.17 $ \\
    & $10$ & $0.10 \pm 0.03$ & $ 61.26 \pm 6.16 $ \\
    & $1$ & $0.09 \pm 0.01$ & $ 61.12 \pm 6.18 $ \\
    & $0.1$ & $0.43 \pm 0.11$ & $ 62.73 \pm 5.46 $ \\
    & $0.01$ & $7.11 \pm 3.05$ & $ 62.31 \pm 5.47 $ \\
    \hline    
    \multirow{5}{*}{\textbf{Epinions (AUROC)}} 
    & $100$ & $0.33 \pm 0.05$ & $ 72.59 \pm 2.27 $ \\
    & $10$ & $0.28 \pm 0.04$ & $ 72.69 \pm 2.21 $ \\
    & $1$ & $0.33 \pm 0.05$ & $ 74.24 \pm 1.95 $ \\
    & $0.1$ & $1.08 \pm 0.16$ & $ 77.05 \pm 1.06 $ \\
    & $0.01$ & $5.21 \pm 0.37$ & $ 77.45 \pm 0.70 $ \\
    \hline    
    \multirow{5}{*}{\textbf{Citeseer (Accuracy)}} 
    & $100$ & $0.95 \pm 0.14$ & $71.28 \pm 1.31$ \\
    & $10$ & $1.00 \pm 0.12$ & $71.28 \pm 1.30$ \\
    & $1$ & $1.48 \pm 0.29$ & $71.59 \pm 1.01$ \\
    & $0.1$ & $7.01 \pm 1.57$ & $71.75 \pm 1.10$ \\
    & $0.01$ & $62.41 \pm 14.67$ & $71.92 \pm 1.09$ \\
    \hline    
    \multirow{5}{*}{\textbf{Cora (Accuracy)}} 
    & $100$ & $4.53 \pm 2.20$ & $81.31 \pm 1.73$ \\
    & $10$ & $4.56 \pm 2.39$ & $81.57 \pm 1.83$ \\
    & $1$ & $7.36 \pm 4.19$ & $81.48 \pm 1.70$ \\
    & $0.1$ & $42.24 \pm 25.06$ & $81.88 \pm 1.82$ \\
    & $0.01$ & $269.45 \pm 49.50$ & $81.79 \pm 1.72$ \\
    \hline    
    \multirow{5}{*}{\textbf{DDI (AUROC)}} 
    & $100$ & $24.56 \pm 0.25$ & $94.85 \pm 0.00$ \\
    & $10$ & $29.23 \pm 0.59$ & $94.85 \pm 0.00$ \\
    & $1$ & $47.15 \pm 0.95$ & $94.82 \pm 0.00$ \\
    & $0.1$ & $280.62 \pm 5.19$ & $94.80 \pm 0.00$ \\
    & $0.01$ & $266.07 \pm 42.68$ & $94.81 \pm 0.00$ \\
    \hline    
    \multirow{5}{*}{\textbf{Yelp (MAE)}} 
    & $100$ & $105.60 \pm 5.03$ & $ 0.23 \pm 0.01 $ \\
    & $10$ & $3,239 \pm 81$ & $ 0.22 \pm 0.01 $ \\
    & $1$ & $3,227 \pm 54$ & $ 0.19 \pm 0.01 $ \\
    & $0.1$ & $421 \pm 202$ & $ 0.18 \pm 0.00 $ \\
    & $0.01$ & $2,472 \pm 297$ & $ 0.18 \pm 0.00 $ \\
    \bottomrule
    \end{tabular}
    }
\end{table}

\tabref{tab:dual-bcd-regularization} shows there is a consistent correlation between the LCQP regularization parameter and the runtime and performance of inference.
As $\epsilon$ increases, there is a significant decrease in the runtime as the D-BCD algorithm can find a solution with a gradient meeting the stopping criterion in fewer iterations.
Notably, for the Citeseer inference problem, the D-BCD algorithm realizes a roughly $45 \times$ speedup.
On the other hand, while the runtime performance improves with increasing $\epsilon$, the prediction performance can sometimes decay.
There is a tradeoff between runtime and prediction performance when setting the $\epsilon$ regularization parameter.
Importantly, the prediction performance is stable across a wide range of $\epsilon$ values: for example, Citeseer accuracy varies less than $1\%$ across four orders of magnitude ($\epsilon \in \{0.01, \ldots, 100\}$), indicating that the semantic perturbation introduced by the regularization is small in practice.

%% file: nesy25.bib
@article{abraham:arxiv24,
    title = {CLEVR-POC: Reasoning-Intensive Visual Question Answering in Partially Observable Environments},
    author = {Savitha Sam Abraham and Marjan Alirezaie and Luc De Raedt},
    year = {2024},
    url={https://arxiv.org/abs/2403.03203}, 
    journal = {arXiv}
}

@article{ackley:cs85,
    title = {A Learning Algorithm for Boltzmann Machines},
    author = {David Ackley and Geoffrey Hinton and Terrence Sejnowski},
    journal = {Cognitive Science},
    year = {1985},
    pages = {147--169},
    volume = {9},
    number = {1},
}

@article{agrawal:jano19,
    title = {Differentiating Through a Cone Program},
    author = {Ashkay Agrawal and Shane Barratt and Stephen Boyd and Enzo Busseti and Walaa {M. Moursi}},
    journal = {Journal of Applied and Numerical Optimization},
    volume = {1},
    number = {2},
    pages = {107--115},
    year = {2019},
}

@inproceedings{agrawal:neurips19,
    title = {Differentiable Convex Optimization Layers},
    author = {Ashkay Agrawal and Brandon Amos and Shane Barratt and Stephen Boyd and Steven Diamond and {J. Zico} Kolter},
    booktitle = {Neural Information Processing Systems (NeurIPS)},
    year = {2019},
    pages = {},
}

@inproceedings{ahmed:neurips22,
    title = {Semantic Probabilistic Layers for Neuro-Symbolic Learning},
    author = {Kareem Ahmed and Stefano Teso and Kai-Wei Chang and Guy {Van den Broeck} and Antonio Vergari},
    booktitle = {Neural Information Processing Systems (NeurIPS)},
    year = {2022},
}

@inproceedings{ahmed:neurips23,
    title = {A Pseudo-Semantic Loss for Autoregressive Models with Logical Constraints},
    author = {Kareem Ahmed and Kai-Wei Chang and Guy {Van den Broeck}},
    booktitle = {Neural Information Processing Systems (NeurIPS)},
    year = {2023},
}

@inproceedings{ahmed:aistats23,
    title = {Semantic Strengthening of Neuro-Symbolic Learning},
    author = {Kareem Ahmed and Kai-Wei Chang and Guy Van den Broeck},
    booktitle = {Artificial Intelligence and Statistics (AISTATS)},
    year = {2023}
}

@inproceedings{ahmed:uai22,
    title = {Neuro-Symbolic Entropy Regularization},
    author = {Kareem Ahmed and Eric Wang and Kai-Wei Chang and Guy Van den Broeck},
    booktitle = {Conference on Uncertainty in Artificial Intelligence (UAI)},
    year = {2022},
}

@inproceedings{amos:icml17,
    title = {OptNet: Differentiable Optimization as a Layer in Neural Networks},
    author = {Brandom Amos and {J. Zico} Kolter},
    booktitle = {International Conference on Machine Learning (ICML)},
    year = {2017},
}

@article{arrotta:acmimwut24,
    title = {Semantic Loss: A New Neuro-Symbolic Approach for Context-Aware Human Activity Recognition},
    author = {Luca Arrotta and Gabriele Civitarese and Claudio Bettini},
    journal = {Proceeding of the ACM on Interactive, Mobile, Wearable and Ubiquitous Technologies},
    year = {2024},
    volume = {7},
    issue = {4},
    number = {144},
    pages = {1--29},
}

@inproceedings{augustine:nesy22,
    title = {Visual Sudoku Puzzle Classification: A Suite of Collective Neuro-Symbolic Tasks},
    author = {Eriq Augustine and Connor Pryor and Charles Dickens and Jay Pujara and William Yang Wang and Lise Getoor},
    booktitle = {International Workshop on Neural-Symbolic Learning and Reasoning (NeSy)},
    year = {2022},
    _publisher = {CEUR}
}

@article{bach:jmlr17,
    title = {Hinge-Loss {M}arkov Random Fields and Probabilistic Soft Logic},
    author = {Stephen Bach and Matthias Broecheler and Bert Huang and Lise Getoor},
    journal = {Journal of Machine Learning Research (JMLR)},
    number = {1},
    pages = {1--67},
    publisher = {JMLR},
    volume = {18},
    year = {2017},
}

@article{bader:wwst05,
    author = {Sebastian Bader and Pascal Hitzler},
    title = {Dimensions of Neural-symbolic Integration - {A} Structured Survey},
    journal = {ArXiv},
    url = {http://arxiv.org/abs/cs/0511042},
    year = {2005},
}

@article{badreddine:arxiv23,
    title = {logLTN: Differentiable Fuzzy Logic in the Logarithm Space},
    author = {Samy Badreddine and Luciano Serafini and Michael Spranger},
    journal = {arXiv},
    url={https://arxiv.org/abs/2306.14546},
    year = {2023}
}

@article{badreddine:ai22,
    title = {Logic Tensor Networks},
    author = {Samy Badreddine and Artur {d'Avila Garcez} and Luciano Serafini and Michael Spranger},
    journal = {Artificial Intelligence},
    number = {4},
    pages = {103649},
    publisher = {Elsevier},
    volume = {303},
    year = {2022},
}

@book{bard:book13,
    title = {Practical Bilevel Optimization: Algorithms and Applications},
    author = {Jonathan {F. Bard}},
    year = {2013},
    publisher = {Springer Science \& Business Media},
}

@inproceedings{belanger:icml17,
    title = {End-to-End Learning for Structure Prediction Energy Networks},
    author = {David Belanger and Bishan Yang and Andrew McCallum},
    booktitle = {International Conference on Machine Learning (ICML)},
    year = {2017},
}

@book{bertsekas:book09,
    title = {Convex Optimization Theory},
    author = {Dimitri Bertsekas},
    year = {2009},
    publisher = {Athena Scientific},
    _address = {Nashua, NH, USA},
}

@phdthesis{bertsekas:phd71,
    title = {Control of Uncertain Systems with a Set-Membership Description of Uncertainty},
    author = {Dimitri Bertsekas},
    year = {1971},
    school = {MIT},
    _address = {Cambridge, MA, USA},
}

@article{besold:nesyai22,
    title = {Neural-Symbolic Learning and Reasoning: {A} Survey and Interpretation},
    author = {Tarek R. Besold and Artur S. {d'Avila Garcez} and Sebastian Bader
        and Howard Bowman and Pedro M. Domingos and Pascal Hitzler and Kai{-}Uwe K{\"{u}}hnberger
        and Lu{\'{\i}}s C. Lamb and Daniel Lowd and Priscila Machado Vieira Lima and Leo de Penning
        and Gadi Pinkas and Hoifung Poon and Gerson Zaverucha},
    journal = {Neuro-Symbolic Artificial Intelligence: The State of the Art},
    year = {2022},
}

@article{bracken:or71,
    title = {Mathematical Programs with Optimization Problems in the Constraints},
    author = {Jerome Bracken and James T. McGill},
    journal = {Operations Research},
    volume = {21},
    number = {1},
    pages = {37--44},
    year = {1973},
}

@article{bommasani:arxiv22,
    title = {On the Opportunities and Risks of Foundation Models},
    author = {Rishi Bommasani and {Drew A.} Hudson and Ehsan Adeli and Russ Altman and Simran Arora and Sydney {von Arx} and Michael {S. Bernstein} and Jeannette Bohg and Antoine Bosselut and Emma Brunskill and et al.},
    year = {2022},
    journal = {Arxiv},
    url = {https://arxiv.org/abs/2108.07258}
}

@book{bonnans:book00,
    title={Perturbation Analysis of Optimization Problems},
    author={Joseph Bonnans and Alexander Shapiro},
    year={2000},
    publisher={Springer}
}

@article{bonnans:siam98,
    title = {Optimization Problems with Perturbations: A Guided Tour},
    author={Joseph Bonnans and Alexander Shapiro},
    journal = {SIAM Review},
    volume = {40},
    number = {2},
    pages = {228--264},
    year = {1998},
    publisher = {SIAM},
}

@inproceedings{bovsnjak:icml17,
    title = {Programming with a Differentiable Forth Interpreter},
    author = {Matko Bo{\v{s}}njak and Tim Rockt{\"a}schel and Jason Naradowsky and Sebastian Riedel},
    booktitle = {International Conference on Machine Learning (ICML)},
    year = {2017},
    _pages = {547--556},
}

@book{boyd:book04,
    title = {Convex Optimization},
    author = {Stephen Boyd and Lieven Vandenberghe},
    publisher = {Cambridge University Press},
    year = {2004},
}

@article{brewka:acm11,
    title = {Answer Set Programming at a Glance},
    author = {Gerhard Brewka and Thomas Eiter and Miroshlaw Truszczynski},
    journal = {Communication of the ACM},
    year = {2011},
    pages = {92--103},
    volume = {54},
    number = {12},
}

@inproceedings{carion:eccv20,
  title={End-to-end Object Detection with Transformers},
  author={Nicolas Carion and Francisco Massa and Gabriel Synnaeve and Nicolas Usunier and Alexander Kirillov and Sergey Zagoruyko},
  booktitle={European conference on computer vision},
  _pages={213--229},
  year={2020},
  _organization={Springer}
}

@inproceedings{carraro:aixia22,
    title = {Logic Tensor Networks for Top-N Recommendation},
    author = {Tommaso Carraro and Alessandro Daniele and Fabio Aiolli and Luciano Serafini},
    booktitle = {International Conference of the Italian Association for Artificial Intelligence (AIxIA)},
    year = {2022},
}

@inproceedings{chang:acl07,
    title = {Guiding Semi-Supervision with Constraint-Driven Learning},
    author = {Ming-Wei Chang and Lev Ratinov and Dan Roth},
    booktitle = {Annual Meeting of the Association for Computational Linguistics (ACL)},
    year = {2007},
    _publisher = {ACL},
    _pages = {1210--1219},
}

@article{chavira:ai08,
    title = {On Probabilistic Inference by Weighted Model Counting},
    author = {Mark Chavira and Adnan Darwiche},
    journal = {Artificial Intelligence},
    volume = {172},
    number = {6-7},
    pages = {772--799},
    year = {2008},
    publisher = {Elsevier}
}

@inproceedings{chen:icml20,
    title = {A Simple Framework for Contrastive Learning of Visual Representations},
    author = {Ting Chen and Simon Kornblith and Mohammad Norouzi and Geoffrey Hinton},
    year = {2020},
    booktitle = {International Conference on Machine Learning (ICML)}
}

@unpublished{choi:unpub20,
    title = {Probabilistic Circuits: A Unifying Framework for Tractable Probabilistic Modeling},
    author = {Yoojung Choi and Antonio Vergari and Guy {Van den Broeck}},
    year = {2020},
    note = {UCLA},
}

@article{cohen:jair20,
    title = {TensorLog: {A} Probabilistic Database Implemented Using Deep-Learning Infrastructure},
    author = {William W. Cohen and Fan Yang and Kathryn Mazaitis},
    journal = {Journal of Artificial Intelligence Research (JAIR)},
    volume = {67},
    pages = {285--325},
    year = {2020},
    publisher = {AAAI},
}

@inproceedings{collins:emnlp02,
    title = {Discriminative Training Methods for Hidden {M}arkov Models: Theory and Experiments with Perceptron Algorithms},
    author = {Collins, Michael},
    booktitle = {Empirical Methods in Natural Language Processing ({EMNLP})},
    year = {2002},
    _publisher = {ACM},
}

@article{colson:aor07,
    title = {An Overview of Bilevel Optimization},
    author = {Benoît Colson and Patrice Marcotte and Gilles Savard},
    journal = {Annals of Operations Research},
    volume = {153},
    number = {1},
    pages = {235--256},
    year = {2007},
    publisher = {Springer},
}

@inproceedings{cornelio:iclr23,
    title = {Learning Where and When to Reason In Neuro-Symbolic Inference},
    author = {Cristina Cornelio and Jan Stuehmer and Shell {Xu Hu} and Timothy Hospedales},
    booktitle = {International Conference on Learning Representations (ICLR)},
    year = {2023},
}

@article{cunnington:arXiv24,
    title = {The Role of Foundation Models in Neuro-Symbolic Learning and Reasoning},
    author = {Daniel Cunnington and Mark Law and Jorge Lobo and Alessandra Russo},
    journal = {arXiv},
    year = {2024},
    url={https://arxiv.org/abs/2402.01889}
}

@article{danskin:siam66,
    title = {The Theory of Max-Min, with Applications},
    author = {John Danskin},
    journal = {SIAM Journal on Applied Mathematics},
    volume = {14},
    number = {4},
    pages = {641--664},
    year = {1966},
    publisher = {SIAM},
}

@article{dash:sr22,
  title = {A Review of Some Techniques for Inclusion of Domain-Knowledge into Deep Neural Networks},
  author = {Tirtharaj Dash and Sharad Chitlangia and Aditya Ahuja and Ashwin Srinivasan},
  journal = {Scientific Reports},
  year={2022},
  volume = {12},
  number = {1},
  pages = {1040},
}

@article{dayan:nc95,
    title = {The Helmholtz Machine},
    author = {Peter Dayan and Geoffrey Hinton and Radford Neal and Richard Zemel},
    journal = {Neural Computation},
    year = {1995},
    pages = {889--904},
    volume = {7},
    number = {5},
}

@inproceedings{demeester:emnlp16,
    title = {Lifted Rule Injection for Relation Embeddings},
    author = {Thomas Demeester and Tim Rockt{\"a}schel and Sebastian Riedel},
    booktitle = {Empirical Methods in Natural Language Processing ({EMNLP})},
    year = {2016},
    _publisher = {ACM},
}

@book{dempe:book20,
    author = {Stephan Dempe and Alain Zemkoho},
    title = {Bilevel Optimization},
    year = {2020},
    publisher = {Springer},
}

@inproceedings{deraedt:ijcai07,
    title = {ProbLog: A Probabilistic Prolog and Its Application in Link Discovery},
    author = {Luc {De Raedt} and Angelika Kimmig and Hannu Toivonen},
    booktitle = {International Joint Conference on Artificial Intelligence (IJCAI)},
    year = {2007},
    _pages = {2462--2467},
}

@article{derkinderen:ijar24,
    title = {Semirings for Probabilistic and Neuro-Symbolic Logic Programming},
    author = {Vincent Derkinderen and Robin Manhaeve and Pedro Zuidberg Dos Martires and Luc De Raedt},
    journal = {International Journal of Approximate Reasoning},
    pages = {109130},
    year = {2024},
    publisher = {Elsevier}
}

@inproceedings{desmet:neurips23,
    title = {Differetiable Sample of Categorical Distributions Using the CatLog-Derivative Trick},
    author = {Lennert {De Smet} and Emanuele Sansone and Pedro {Zuidberg Dos Martires}},
    booktitle = {Neural Information Processing Systems (NeurIPS)},
    year = {2023}
}

@inproceedings{devlin:arxiv19,
  title={BERT: Pre-training of Deep Bidirectional Transformers for Language Understanding},
author = {Jacob Devlin and {Ming-Wei} Chang and Kenton Lee and Kristina Toutanova},
  booktitle={North American Chapter of the Association for Computational Linguistics},
  year={2019},
}

@inproceedings{dickens:icml24,
    title = {Convex and Bilevel Optimization for Neuro-Symbolic Inference and Learning},
    author = {Charles Dickens and Changyu Gao and Connor Pryor and Stephen Wright and Lise Getoor},
    year = {2024},    
    booktitle = {International Conference on Machine Learning (ICML)}
}

@inproceedings{dickens:make24,
    title = {Modeling Patterns for Neural-Symbolic Reasoning using Energy-Based Models},
    author = {Charles Dickens and Connor Pryor and Lise Getoor},
    booktitle = {AAAI Spring Symposium on Empowering Machine Learning and Large Language Models with Domain and Commonsense Knowledge},
    year = {2024},
    _publisher = {AAAI}
}

@phdthesis{dickens:phd24,
    title = {A Unifying Mathematical Framework for Neural-Symbolic Systems},
    author = {Charles Dickens},
    year = {2024},
    publisher = {UCSC},
    address = {Santa Cruz, CA, USA},
    school={University of California, Santa Cruz}
}

@inproceedings{diligenti:icmla17,
    title = {Integrating Prior Knowledge into Deep Learning},
    author = {Michelangelo Diligenti and Soumali Roychowdhury and Marco Gori},
    booktitle = {International Conference on Machine Learning and Applications (ICMLA)},
    year = {2017},
    _pages = {920--923},
}

@article{diligenti:jmlr17,
    title = {Semantic-Based Regularization for Learning and Inference},
    author = {Michelangelo Diligenti and Marco Gori and Claudio Sacc{\`a}},
    journal = {Journal of Machine Learning Research (JMLR)},
    volume = {18},
    pages = {1--45},
    year = {2017},
}

@inproceedings{do:neurips07,
    title = {Efficient Multiple Hyperparameter Learning for Log-Linear Models},
    author = {Chuong Do and Chuan-Sheng Foo and Andrew Ng},
    booktitle = {Neural Information Processing Systems (NeurIPS)},    year = {2007},
}

@inproceedings{domke:aistats12,
    title = {Generic Methods for Optimization-Based Modeling},
    author = {Justin Domke},
    booktitle = {Artifical Intelligence and Statistics (AISTATS)},
    year = {2012},
}

@inproceedings{donadello:ijcai17,
    title = {Logic Tensor Networks for Semantic Image Interpretation},
    author = {Ivan Donadello and Luciano Serafini and Artur S. {d'Avila Garcez}},
    booktitle = {International Joint Conference on Artificial Intelligence (IJCAI)},
    year = {2017},
    _pages = {1596--1602},
    _publisher = {IJCAI},
}

@inproceedings{du:icml23,
    title = {Reduce, Reuse, Recycle: Compositional Generation with Energy-based Diffusion Models and MCMC},
    author = {Yilun Du and Conor Durkan and Robin Strudel and {Joshua B.} Tenenbaum and Sander Dieleman and Rob Fergus and Jascha {Sohl-Dickstein} and Arnaud Doucet and Will Grathwohl},
    booktitle = {International Conference on Machine Learning (ICML)},
    year = {2023},
}

@inproceedings{du:neurips19,
    title = {Implicit Generation and Modeling with Energy-based Models},
    author = {Yilun Du and Igor Mordatch},
    booktitle = {Neural Information Processing Systems (NeurIPS)},
    year = {2019},
}

@inproceedings{feng:naacl24,
    title = {Language Models can be Deductive Solvers},
    author = {Jiazhan Feng and Ruochen Xu and Junheng Hao and Hiteshi Sharma and Yelong Shen and Dongyan Zhao and Weizhu Chen},
    booktitle = {Findings of the Association for Computational Linguistics: NAACL},
    year = {2024}
}

@book{fiacco:book68,
    title = {Nonlinear Programming: Sequential Unconstrained Minimization Techniques},
    author = {Anthony Fiacco and Garth McCormick},
    publisher = {John Wiley and Sons},
    year = {1968},
}

@inproceedings{franceschi:icml18,
    title = {Bilevel Programming for Hyperparameter Optimization and Meta-Learning},
    author = {Luca Franceschi and Paolo Frasconi and Saverio Salzo and Riccardo Grazzi and Massimiliano Pontil},
    booktitle = {International Conference on Machine Learning (ICML)},
    year = {2018},
}

@book{garcez:book02,
    title = {Neural-Symbolic Learning Systems: Foundations and Applications},
    author = {Artur S. {d'Avila Garcez} and Krysia Broda and Dov M. Gabbay},
    publisher = {Springer},
    year = {2002},
}

@book{garcez:book09,
    title = {Neural-Symbolic Cognitive Reasoning},
    author = {Artur S. {d'Avila Garcez} and Lu{\'{\i}}s C. Lamb and Dov M. Gabbay},
    publisher = {Springer},
    year = {2009},
}

@article{garcez:jal19,
    title = {Neural-Symbolic Computing: An Effective Methodology for Principled Integration of Machine Learning and Reasoning},
    author = {Artur {d'Avila Garcez} and Marco Gori and Lu{\'{\i}}s C. Lamb and Luciano Serafini and Michael Spranger and Son N. Tran},
    journal = {Journal of Applied Logics},
    volume = {6},
    number = {4},
    pages = {611--632},
    year = {2019},
    publisher = {College Publications},
}

@article{ghadimi:arxiv18,
    title = {Approximation Methods for Bilevel Programming},
    author = {Saeed Ghadimi and Mengdi Wang},
    year = {2018},
    journal = {Arxiv},
}

@article{giovannelli:arxiv22,
    title = {Inexact Bilevel Stochastic Gradient Methods for Constrained and Unconstrained Lower-Level Problems},
    author = {Tommaso Giovannelli and Griffin Kent and Luis {Nune Vicente}},
    year = {2022},
    journal = {Arxiv},
    url={https://arxiv.org/abs/2110.00604}
}

@inproceedings{giunchiglia:ijcai22,
    title = {Deep Learning with Logical Constraints},
    author = {Eleonora Giunchiglia and Mihaela Catalina Stoian and Thomas Lukasiewicz},
    booktitle = {International Joint Conference on Artificial Intelligence (IJCAI)},
    year = {2022},
}

@article{giunchiglia:ml23,
    title = {ROAD-R: The Autonomous Driving Dataset with Logical Requirements},
    author = {Eleonora Giunchiglia and Mihaela {Catalina Stoian} and Salman Khan and Thomas Lukasiewicz},
    journal = {Machine Learning},
    year = {2023},
    pages = {3261--3291},
    volume = {112},
    number = {1},
}

@inproceedings{goodfellow:neurips14,
    title = {Generative Adversarial Nets},
    author = {Ian Goodfellow and Jean {Pouget-Abadie} and Mehdi Mirza and Bing Xu and David {Warde-Farley} and Sherjil Ozair and Aaron Courville and Yoshua Bengio},
    booktitle = {Neural Information Processing Systems (NeurIPS)},
    year = {2014},
}

@inproceedings{grathwohl:iclr20,
    title = {Your Classifier is Secretly an Energy-based Model and You Should Treat it Like One},
    author = {Will Grathwohl and {KuanChieh} Wang and {Joern-Henrik} Jacobsen and David Duvenaud and Mohammad Norouzi and Kevin Swersky},
    booktitle = {International Conference on Learning Representations (ICLR)},
    year = {2020},
}

@book{griewank:book08,
     author = {Andreas Griewank and Andrea Walther},
     title = {Evaluating Derivatives: Principles and Techniques of Algorithmic Differentiation},
     year = {2008},
     publisher = {SIAM}
}

@misc{gurobi:misc24,
    author = {Gurobi Optimization, LLC},
    title = {Gurobi Optimizer Reference Manual},
    year = {2024},
    url = {https://www.gurobi.com}
}

@inproceedings{hasan:ijcnlp13,
    title = {Stance classification of ideological debates: Data, models, features, and constraints},
    author = {Kazi Saidul Hasan and Vincent Ng},
    booktitle = {International Joint Conference on Natural Language Processing (IJCNLP)},
    _pages = {1348--1356},
    year = {2013},
}

@inproceedings{he:cvpr16,
    title = {Deep Residual Learning for Image Recognition},
    author = {Kaiming He and Xiangyu Zhang and Shaoqing Ren and Jian Sun},
    booktitle = {International Conference on Computer Vision and Patter Recognition (CVPR)},
    year = {2016},
}

@article{hinton:nc02,
    title = {Training Products of Experts by Minimizing Contrastive Divergence},
    author = {Geoffrey Hinton},
    journal = {Neural Computation},
    year = {2002},
    pages = {1771--1800},
    volume = {14},
    number = {8},
}

@inproceedings{hu:iclr22,
    title = {LoRA: Low-Rank Adaptation of Large Language Models},
    author = {Edward {J. Hu} and Yelong Shen and Phillip Wallis and Zeyuan {Allen-Zhu} and Yuanzhi Li and Shean Wang and Lu Wang and Weizhu Chen},
    booktitle = {International Conference on Learning Representations (ICLR)},
    year = {2022},
}

@article{hyvarinen:jmlr05,
    title = {Estimation of Non-normalized Statistical Models by Score Matching},
    author = {Aapo Hyvarinen},
    journal = {Journal of Machine Learning Research (JMLR)},
    pages = {695--709},
    publisher = {JMLR},
    volume = {6},
    year = {2005},
}

@inproceedings{ji:icml21,
    title = {Bilevel Optimization: Convergence Analysis and Enhanced Design},
    author = {Kaiyi Ji and Junjie Yang and Yingbin Liang},
    booktitle = {International Conference on Machine Learning (ICML)},
    year = {2021},
}

@inproceedings{khanduri:icml23,
    title = {Linearly Constrained Bilevel Optimization: A Smoothed Implicit Gradient Approach},
    author = {Prashant Khanduri and Ioannis Tsaknakis and Yihua Zhang and Jia Liu and Sijia Liu and Jiawei Zhang and Mingyi Hong},
    booktitle = {International Conference on Machine Learning (ICML)},
    year = {2023},
}

@inproceedings{kingma:neurips10,
    title = {Regularized Estimation of Image Statistics by Score Matching},
    author = {Diedrik {P. Kingma} and Yann LeCun},
    booktitle = {Neural Information Processing Systems (NeurIPS)},
    year = {2010},
}

@inproceedings{kisa:kr14,
  title = {Probabilistic Sentential Decision Diagrams},
  author = {Doga Kisa and Guy Van den Broeck and Arthur Choi and Adnan Darwiche},
  booktitle = {Proceedings of the Fourteenth International Conference on Principles of Knowledge Representation and Reasoning},
  year = {2014},
}

@book{klir:book95,
    author = {George J. Klir and Bo Yuan},
    title = {Fuzzy Sets and Fuzzy Logic - Theory and Applications},
    publisher = {Prentice Hall},
    year = {1995},
}

@inproceedings{kouki:recsys15,
    title = {HyPER: A Flexible and Extensible Probabilistic Framework for Hybrid Recommender Systems},
    author = {Pigi Kouki and Shobeir Fakhraei and James Foulds and Magdalini Eirinaki and Lise Getoor},
    booktitle = {ACM Conference on Recommender Systems (RecSys)},
    year = {2015},
    address = {Vienna, Austria},
}

@article{krieken:arxiv24,
    title = {ULLER: A Unified Language for Learning and Reasoning},
    author = {Emile van Krieken and Samy Badreddine and Robin Manhaeve and Eleonora Giunchiglia},
    journal = {arXiv},
    url={https://arxiv.org/abs/2405.00532},
    year = {2024}
}

@article{krieken:ai22,
    title = {Analyzing Differentiable Fuzzy Logic Operators},
    author = {Emile {van Krieken} and Erman Acar and Frank {van Harmelen}},
    journal = {Artificial Intelligence (AI)},
    volume = {302},
    pages = {103602},
    year = {2022},
    publisher = {Elsevier},
}

@inproceedings{krieken:neurips23,
    title = {A-NeSI: A Scalable Approximate Method for Probabilistic Neurosymbolic Inference },
    author = {Emile {van Krieken} and Thiviyan Thanapalasingam and {Jakub M.} Tomczak and Frank {van Harmelen} and Annette {ten Teije}},
    booktitle = {Neural Information Processing Systems (NeurIPS)},
    year = {2023},
}

@inproceedings{kwon:icml23,
    title = {A Fully First-Order Method for Stochastic Bilevel Optimization},
    author = {Jeongyeol Kwon and Dohyun Kwon and Stephen Wright and Robert Nowak},
    year = {2023},
    booktitle = {International Conference on Machine Learning (ICML)},
}

@inproceedings{lamb:ijcai20,
    title = {Graph Neural Networks Meet Neural-Symbolic Computing: {A} Survey and Perspective},
    author = {Lu{\'{\i}}s C. Lamb and Artur {d'Avila Garcez} and Marco Gori and Marcelo O. R. Prates and Pedro H. C. Avelar and Moshe Y. Vardi},
    booktitle = {International Joint Conference on Artificial Intelligence (IJCAI)},
    year = {2020},
}

@article{lecun:book06,
    title = {A Tutorial on Energy-Based Learning},
    author = {Yann LeCun and Sumit Chopra and Raia Hadsell and Marc’Aurelio Ranzato and Fu Jie Huang},
    journal = {Predicting Structured Data},
    volume = {1},
    number = {0},
    year = {2006}
}

@article{lecun:ieee98,
    title = {Gradient-based learning applied to document recognition},
    author = {Yann LeCun and L{\'e}on Bottou and Yoshua Bengio and Patrick Haffner},
    journal = {Proceedings of the IEEE},
    volume = {86},
    number = {11},
    pages = {2278--2324},
    year = {1998},
    publisher = {IEEE},
}

@inproceedings{liu:icml21,
    title = {A Value-Function-based Interior-point Method for Non-convex Bi-level Optimization},
    author = {Risheng Liu and Xuan Liu and Xiaoming Yuan and Shangzhi Zeng and Jin Zhang},
    booktitle = {International Conference on Machine Learning (ICML)},
    year = {2021},
}

@inproceedings{liu:neurips20,
    title = {Energy-based Out-of-distribution Detection},
    author = {Weitang Liu and Xiaoyun Wang and {John D.} Owens and Yixuan Li},
    booktitle = {Neural Information Processing Systems (NeurIPS)},
    year = {2020},
}

@inproceedings{liu:neurips22,
    title = {BOME! Bilevel Optimization Made Easy: A Simple First-Order Approach},
    author = {Bo Liu and Mao Ye and Stephen Wright and Peter Stone and Qiang Liu},
    booktitle = {Neural Information Processing Systems (NeurIPS)},
    year = {2022},
}

@article{liu:arxiv23,
    title={Value-function-based sequential minimization for bi-level optimization},
    author={Liu, Risheng and Liu, Xuan and Zeng, Shangzhi and Zhang, Jin and Zhang, Yixuan},
    journal={IEEE Transactions on Pattern Analysis and Machine Intelligence},
    volume={45},
    number={12},
    pages={15930--15948},
    year={2023},
    publisher={IEEE}
}

@inproceedings{loshchilov:iclr19,
    title = {Decoupled Weight Decay Regularization},
    author = {Ilya Loshchilov and Frank Hutter},
    booktitle = {International Conference on Learning Representations (ICLR)},
    year = {2019},
}

@proceedings{nesy05,
    key = {NeSy2005},
    title = {Neural-Symbolic Learning and Reasoning Workshop at IJCAI},
    year = {2005}
}

@proceedings{nesy24,
    key = {NeSy2024},
    title = {Neural-Symbolic Learning and Reasoning (NeSy)},
    year = {2024}
}

@inproceedings{maene:neurips24,
    title = {Soft-Unification in Deep Probabilistic Logic},
    author = {Jaron Maene and Luc De Raedt},
    booktitle = {Neural Information Processing Systems (NeurIPS)},
    year = {2024}
}

@article{maene:arxiv24,
    title = {On the Hardness of Probabilistic Neurosymbolic Learning},
    author = {Jaron Maene and Vincent Derkinderen and Luc De Raedt},
    year = {2024},
    url={https://arxiv.org/abs/2406.04472},
    journal = {arXiv}
}

@article{manhaeve:ai21,
    title = {Neural Probabilistic Logic Programming in {DeepProbLog}},
    author = {Robin Manhaeve and Sebastijan Duman{\v{c}}i{\'c} and Angelika Kimmig and Thomas Demeester and Luc {De Raedt}},
    journal = {Artificial Intelligence (AI)},
    volume = {298},
    pages = {103504},
    year = {2021},
    publisher = {Elsevier},
}

@inproceedings{manhaeve:icpkrr21,
    title = {Approximate Inference for Neural Probabilistic Logic Programming},
    author = {Robin Manhaeve and Giuseppe Marra and Luc {De Raedt}},
    booktitle = {{International Conference on Principles of Knowledge Representation and Reasoning (ICPKRR)}},
    year = {2021},
}

@article{marconato:arxiv24,
    title = {BEARS Make Neuro-Symbolic Models Aware of their Reasoning Shortcuts},
    author = {Emanuele Marconato and Samuele Bortolotti and Emile van Krieken and Antonio Vergari and Andrea Passerini and Stefano Teso},
    year = {2024},
    journal = {arXiv},
    url={https://arxiv.org/abs/2402.12240}
}

@inproceedings{marconato:neurips23,
  title = {Not All Neuro-Symbolic Concepts Are Created Equal: Analysis and Mitigation of Reasoning Shortcuts},
  author = {Emanuele Marconato and Stefano Teso and Antonio Vergari and Andrea Passerini},
  booktitle = {Neural Information Processing Systems (NeurIPS)},
  year = {2023}
}

@article{marra:ai24,
  title = {From Statistical Relational to Neurosymbolic Artificial Intelligence: A Survey},
  author = {Giuseppe Marra and Sebastijan Duman{\v{c}}i{\'c} and Robin Manhaeve and Luc {De Raedt}},
  journal = {Artificial Intelligence},
  volume = {328},
  pages = {104062},
  year = {2024},
}

@inproceedings{marra:ecmlkdd19,
    title = {Integrating Learning and Reasoning with Deep Logic Models},
    author = {Giuseppe Marra and Francesco Giannini and Michelangelo Diligenti and Marco Gori},
    booktitle = {Machine Learning and Knowledge Discovery in Databases: European Conference (ECML PKDD)},
    year = {2019}
}

@article{milgrom:econ02,
    title = {Envelope Theorems for Arbitrary Choice Sets},
    author={Paul Milgrom and Ilya Segal},
    journal = {Econometrica},
    volume = {70},
    number = {2},
    pages = {583--601},
    year = {2002},
    publisher = {Econometric Society},
}

@inproceedings{morra:nesy23,
    title = {Designing Logic Tensor Networks for Visual Sudoku Puzzle Classification},
    author = {Lia Morra and Alberto Azzari and Letizia Bergamasco and Marco Braga and Luigi Capogrosso and Federico Delrio and Giuseppe {Di Giacomo} and Simone Eiraudo and Giorgia Ghione and Rocco Giudice and Alkis Koudounas and Luca Piano and Daniele {Rege Cambrin} and Matteo Risso and Marco Rondina and Alessandro {Sebastien Russo} and Marco Russo and Francesco Taioli and Lorenzo Vaiani and Chiara Vercellino},
    booktitle = {International Workshop on Neural-Symbolic Learning and Reasoning (NeSy)},
    year = {2023},
    _publisher = {CEUR}
}

@book{nocedal:wright:book06,
    title = {Numerical Optimization},
    author = {Jorge Nocedal and {Stephen J.} Wright},
    edition = {second},
    publisher = {Springer},
    year = {2006},
}

@techreport{openai:techreport24,
    title = {GPT-4 Technical Report}, 
    author = {OpenAI},
    _author = {OpenAI and Josh Achiam and Steven Adler and Sandhini Agarwal and Lama Ahmad and Ilge Akkaya and Florencia Leoni Aleman and Diogo Almeida and Janko Altenschmidt and Sam Altman and Shyamal Anadkat and Red Avila and Igor Babuschkin and Suchir Balaji and Valerie Balcom and Paul Baltescu and Haiming Bao and Mohammad Bavarian and Jeff Belgum and Irwan Bello and Jake Berdine and Gabriel Bernadett-Shapiro and Christopher Berner and Lenny Bogdonoff and Oleg Boiko and Madelaine Boyd and Anna-Luisa Brakman and Greg Brockman and Tim Brooks and Miles Brundage and Kevin Button and Trevor Cai and Rosie Campbell and Andrew Cann and Brittany Carey and Chelsea Carlson and Rory Carmichael and Brooke Chan and Che Chang and Fotis Chantzis and Derek Chen and Sully Chen and Ruby Chen and Jason Chen and Mark Chen and Ben Chess and Chester Cho and Casey Chu and Hyung Won Chung and Dave Cummings and Jeremiah Currier and Yunxing Dai and Cory Decareaux and Thomas Degry and Noah Deutsch and Damien Deville and Arka Dhar and David Dohan and Steve Dowling and Sheila Dunning and Adrien Ecoffet and Atty Eleti and Tyna Eloundou and David Farhi and Liam Fedus and Niko Felix and Simón Posada Fishman and Juston Forte and Isabella Fulford and Leo Gao and Elie Georges and Christian Gibson and Vik Goel and Tarun Gogineni and Gabriel Goh and Rapha Gontijo-Lopes and Jonathan Gordon and Morgan Grafstein and Scott Gray and Ryan Greene and Joshua Gross and Shixiang Shane Gu and Yufei Guo and Chris Hallacy and Jesse Han and Jeff Harris and Yuchen He and Mike Heaton and Johannes Heidecke and Chris Hesse and Alan Hickey and Wade Hickey and Peter Hoeschele and Brandon Houghton and Kenny Hsu and Shengli Hu and Xin Hu and Joost Huizinga and Shantanu Jain and Shawn Jain and Joanne Jang and Angela Jiang and Roger Jiang and Haozhun Jin and Denny Jin and Shino Jomoto and Billie Jonn and Heewoo Jun and Tomer Kaftan and Łukasz Kaiser and Ali Kamali and Ingmar Kanitscheider and Nitish Shirish Keskar and Tabarak Khan and Logan Kilpatrick and Jong Wook Kim and Christina Kim and Yongjik Kim and Jan Hendrik Kirchner and Jamie Kiros and Matt Knight and Daniel Kokotajlo and Łukasz Kondraciuk and Andrew Kondrich and Aris Konstantinidis and Kyle Kosic and Gretchen Krueger and Vishal Kuo and Michael Lampe and Ikai Lan and Teddy Lee and Jan Leike and Jade Leung and Daniel Levy and Chak Ming Li and Rachel Lim and Molly Lin and Stephanie Lin and Mateusz Litwin and Theresa Lopez and Ryan Lowe and Patricia Lue and Anna Makanju and Kim Malfacini and Sam Manning and Todor Markov and Yaniv Markovski and Bianca Martin and Katie Mayer and Andrew Mayne and Bob McGrew and Scott Mayer McKinney and Christine McLeavey and Paul McMillan and Jake McNeil and David Medina and Aalok Mehta and Jacob Menick and Luke Metz and Andrey Mishchenko and Pamela Mishkin and Vinnie Monaco and Evan Morikawa and Daniel Mossing and Tong Mu and Mira Murati and Oleg Murk and David Mély and Ashvin Nair and Reiichiro Nakano and Rajeev Nayak and Arvind Neelakantan and Richard Ngo and Hyeonwoo Noh and Long Ouyang and Cullen O'Keefe and Jakub Pachocki and Alex Paino and Joe Palermo and Ashley Pantuliano and Giambattista Parascandolo and Joel Parish and Emy Parparita and Alex Passos and Mikhail Pavlov and Andrew Peng and Adam Perelman and Filipe de Avila Belbute Peres and Michael Petrov and Henrique Ponde de Oliveira Pinto and Michael and Pokorny and Michelle Pokrass and Vitchyr H. Pong and Tolly Powell and Alethea Power and Boris Power and Elizabeth Proehl and Raul Puri and Alec Radford and Jack Rae and Aditya Ramesh and Cameron Raymond and Francis Real and Kendra Rimbach and Carl Ross and Bob Rotsted and Henri Roussez and Nick Ryder and Mario Saltarelli and Ted Sanders and Shibani Santurkar and Girish Sastry and Heather Schmidt and David Schnurr and John Schulman and Daniel Selsam and Kyla Sheppard and Toki Sherbakov and Jessica Shieh and Sarah Shoker and Pranav Shyam and Szymon Sidor and Eric Sigler and Maddie Simens and Jordan Sitkin and Katarina Slama and Ian Sohl and Benjamin Sokolowsky and Yang Song and Natalie Staudacher and Felipe Petroski Such and Natalie Summers and Ilya Sutskever and Jie Tang and Nikolas Tezak and Madeleine B. Thompson and Phil Tillet and Amin Tootoonchian and Elizabeth Tseng and Preston Tuggle and Nick Turley and Jerry Tworek and Juan Felipe Cerón Uribe and Andrea Vallone and Arun Vijayvergiya and Chelsea Voss and Carroll Wainwright and Justin Jay Wang and Alvin Wang and Ben Wang and Jonathan Ward and Jason Wei and CJ Weinmann and Akila Welihinda and Peter Welinder and Jiayi Weng and Lilian Weng and Matt Wiethoff and Dave Willner and Clemens Winter and Samuel Wolrich and Hannah Wong and Lauren Workman and Sherwin Wu and Jeff Wu and Michael Wu and Kai Xiao and Tao Xu and Sarah Yoo and Kevin Yu and Qiming Yuan and Wojciech Zaremba and Rowan Zellers and Chong Zhang and Marvin Zhang and Shengjia Zhao and Tianhao Zheng and Juntang Zhuang and William Zhuk and Barret Zoph},
    year = {2024},
    institution = {OpenAI},
    journal = {ArXiv},
    url={https://arxiv.org/abs/2303.08774}
}

@article{outrata:zor90,
    title = {On the Numerical Solution of a Class of Stackelberg Problems},
    author = {Jivrí {V. Outrata}},
    journal = {Methods and Models of Operations Research},
    year = {1990},
    volume = {34},
    number = {4},
    pages = {255-277}
}

@inproceedings{pan:emnlp23,
  title = {Logic-LM: Empowering Large Language Models with Symbolic Solvers for Faithful Logical Reasoning},
  author = {Liangming Pan and Alon Albalak and Xinyi Wang and William Yang Wang},
  booktitle = {Empirical Methods in Natural Language Processing (EMNLP)},
  year = {2023},
}

@book{papadimtriou:book98,
    author = {Christos {H. Papadimitriou} and Kenneth Steiglitz},
    title = {Combinatorial Optimization: Algorithms and Complexity},
    publisher = {Courier Corporation},
    year = {1998},
}

@article{parikh:ftml13,
    title = {Proximal Algorithms},
    author = {Neal Parikh and Stephen Boyd},
    journal = {Foundations and Trends in Machine Learning (FTML)},
    number = {1},
    pages = {123-231},
    volume = {3},
    year = {2013},
    publisher = {Now Publishers Inc.},
}

@inproceedings{pedregosa:icml16,
    title = {Hyperparameter Optimization with Approximate Gradient},
    author = {Fabian Pedregosa},
    booktitle = {International Conference on Machine Learning (ICML)},
    year = {2016},
}

@inproceedings{pryor:acl23,
    title = {Using Domain Knowledge to Guide Dialog Structure Induction via Neural Probabilistic Soft Logic},
    author = {Connor Pryor and Quan Yuan and Jeremiah Zhe Liu and Seyed Mehran Kazemi and Deepak Ramachandran and Tania Bedrax-Weiss and Lise Getoor},
    booktitle = {Annual Meeting of the Association for Computational Linguistics (ACL)},
    year = {2023},
    _publisher = {ACL},
    address = {Toronto, Canada},
}

@inproceedings{pryor:ijcai23,
    title = {Neu{PSL}: Neural {P}robabilistic {S}oft {L}ogic},
    author = {Connor Pryor and Charles Dickens and Eriq Augustine and Alon Albalak and William Yang Wang and Lise Getoor},
    booktitle = {International Joint Conference on Artificial Intelligence (IJCAI)},
    year = {2023},
    _publisher = {IJCAI},
}

@phdthesis{pryor:thesis24,
  title={Foundations of Neural-Symbolic AI: Architecture and Design},
  author={Connor Pryor},
  year={2024},
  school={University of California, Santa Cruz}
}

@inproceedings{rajeswaran:neurips19,
    title = {Meta-Learning with Implicit Gradients},
    author = {Aravind Rajeswaran and Chelsea Finn and Sham {M. Kakade} and Sergey Levine},
    booktitle = {Conference on Neural Information Processing Systems (NeurIPS)},
    year = {2019},
}

@inproceedings{richardson:iswc03,
    title = {Trust Management for the Semantic Web},
    author = {Matthew Richardson and Rakesh Agrawal and Pedro Domingos},
    booktitle = {International Semantic Web Conference (ISWC)},
    _pages = {351--368},
    year = {2003},
}

@article{robinson:mor80,
    title = {Strongly Regular Generalized Equations},
    author = {Stephen Robinson},
    journal = {Mathematics of Operations Research},
    volume = {5},
    number = {1},
    pages = {43--62},
    year = {1980}
}

@book{rockafellar:book70,
    title = {Convex Analysis},
    author = {{R. Tyrrell} Rockafellar},
    publisher = {Princeton University Press},
    year = {1970},
}

@inproceedings{rockafellar:rcsam74,
    title = {Conjugate Duality and Optimization},
    author = {{R. Tyrrell} Rockafellar},
    booktitle = {Regional Conference Series in Applied Mathematics},
    year = {1974},
}

@book{rockafellar:book97,
    title = {Variational Analysis},
    author = {{R. Tyrrell} Rockafellar and {Roger J-B} Wets},
    publisher = {Springer},
    year = {1997},
}

@inproceedings{rocktaschel:neurips17,
    title = {End-to-end Differentiable Proving},
    author = {Tim Rockt{\"a}schel and Sebastian Riedel},
    booktitle = {Neural Information Processing Systems (NeurIPS)},
    year = {2017},
    _pages = {3791--3803},
}

@inproceedings{sachan:neurips18,
    title = {Learning Pipelines with Limited Data and Domain Knowledge: A Study in Parsing Physics Problems},
    author = {Mrinmaya Sachan and Kumar Avinava Dubey and Tom M Mitchell and Dan Roth and Eric P Xing},
    booktitle = {Neural Information Processing Systems (NeurIPS)},
    year = {2018},
}

@inproceedings{salakhutdinov:aistats10,
    title = {Efficient Learning of Deep Boltzmann Machines},
    author = {Ruslan Salakhutdinov and Hugo Larochelle},
    booktitle = {Artifical Intelligence and Statistics (AISTATS)},
    year = {2010},
}

@article{scellier:fcn17,
    title = {Equilibrium Propagation: Bridging the Gap Between Energy-Based Models and Backpropagation},
    author = {Benjamin Scellier and Yoshua Bengio},
    journal = {Frontiers in Computational Neuroscience},
    volume = {11},
    year = {2017},
}

@article{sen:aim08,
    title = {Collective Classification in Network Data},
    author = {Prithviraj Sen and Galileo Mark Namata and Mustafa Bilgic and Lise Getoor and Brian Gallagher and Tina {Eliassi-Rad}},
    journal = {AI Magazine},
    year = {2008},
    publisher = {AAAI},
    pages = {93--106},
    volume = {29},
    number = {3},
}

@article{shalevshwartz:ftml11,
    title = {Online learning and online convex optimization},
    author = {Shai Shalev-Shwartz},
    journal = {Foundations and Trends in Machine Learning (FTML)},
    year = {2012},
    volume = {4},
    number = {2},
    pages = {107--194},
}

@techreport{sikka:techreport20,
    title = {Deep Adaptive Semantic Logic (DASL): Compiling Declarative Knowledge into Deep Neural Networks},
    author = {Karan Sikka and Andrew Silberfarb and John Byrnes and Indranil Sur and Ed Chow and Ajay Divakaran and Richard Rohwer},
    year = {2020},
    institution = {SRI International},
}

@article{singh:tpa2021,
    title = {ROAD: The Road Event Awareness Dataset for Autonomous Driving},
    author = {Gurkirt Singh and Stephen Akrigg and Manuele {Di Maio} and  Valentina Fontana and Reza {Javanmard Alitappeh} and Suman Saha and Kossar {Jeddi Saravi} and Farzad Yousefia and Jacob Culley and Tom Nicholson and Jordan Omokeowa and Salman Khan and Stanislao Grazioso and Andrew Bradley and Giuseppe {Di Gironimo} and Fabio Cuzzolin},
    journal = {IEEE TPAMI},
    year = {2021},
    volume = {45},
    pages = {1036-1054}
}

@inproceedings{song:neurips19,
    title = {Generative Modeling by Estimating Gradient of the Data Distribution},
    author = {Yang Song and Stefano Ermon},
    booktitle = {Neural Information Processing Systems (NeurIPS)},
    year = {2019},
}

@article{sow:arxiv22,
    title = {A Primal-Dual Approach to Bilevel Optimization with Multiple Inner Minima},
    author = {Daouda Sow and Kaiyi Ji and Ziwei Guan and Yingbin Liang},
    year = {2022},
    journal = {Arxiv},
}

@article{sridhar:bio16,
    title = {A Probabilistic Approach for Collective Similarity-Based Drug-Drug Interaction Prediction},
    author = {Dhanya Sridhar and Shobeir Fakhraei and Lise Getoor},
    journal = {Bioinformatics},
    year = {2016},
    publisher = {Oxford},
    pages = {3175--3182},
    volume = {32},
    number = {20},
}

@article{srinivasan:mlj21,
    title = {A Taxonomy of Weight Learning Methods for Statistical Relational Learning},
    author = {Sriram Srinivasan and Charles Dickens and Eriq Augustine and Golnoosh Farnadi and Lise Getoor},
    journal = {Machine Learning},
    year = {2021},
    publisher = {Springer},
}

@article{srivastava:arxiv22,
  title = {Beyond the Imitation Game: Quantifying and extrapolating the capabilities of language models},
  author={Aarohi Srivastava and Abhinav Rastogi and Abhishek Rao and Abu Awal Md Shoeb and Abubakar Abid and Adam Fisch and Adam R. Brown and Adam Santoro and Aditya Gupta and Adri{\`a} Garriga-Alonso and Agnieszka Kluska and Aitor Lewkowycz and Akshat Agarwal and Alethea Power and Alex Ray and Alex Warstadt and Alexander W. Kocurek and Ali Safaya and Ali Tazarv and Alice Xiang and Alicia Parrish and Allen Nie and Aman Hussain and Amanda Askell and Amanda Dsouza and Ambrose Slone and Ameet Annasaheb Rahane and Anantharaman S. Iyer and Anders Andreassen and Andrea Madotto and Andrea Santilli and Andreas Stuhlmuller and Andrew M. Dai and Andrew La and Andrew Kyle Lampinen and Andy Zou and Angela Jiang and Angelica Chen and Anh Vuong and Animesh Gupta and Anna Gottardi and Antonio Norelli and Anu Venkatesh and Arash Gholamidavoodi and Arfa Tabassum and Arul Menezes and Arun Kirubarajan and Asher Mullokandov and Ashish Sabharwal and Austin Herrick and Avia Efrat and Aykut Erdem and Ayla Karakacs and et al.
  },
  journal = {ArXiv},
  year = {2022},
  url={https://arxiv.org/abs/2206.04615}
}

@inproceedings{stoian:nesy23,
    title = {Exploiting T-norms for Deep Learning in Autonomous Driving},
    author = {{Mihaela C{\u{a}}t{\u{a}}lina} Stoian and Eleonora Giunchiglia and Thomas Lukasiewicz},
    booktitle = {Neural-Symbolic Learning and Reasoning (NeSy)},
    year = {2023}
}

@inproceedings{stoyanov:aistats11,
    title = {Empirical Risk Minimization of Graphical Model Parameters Given Approximate Inference, Decoding, and Model Structure},
    author = {Veselin Stoyanov and Alexander Ropson and Jason Eisner},
    booktitle = {Artifical Intelligence and Statistics (AISTATS)},
    year = {2011},
}

@book{sutton:book18,
    title={Reinforcement Learning: An Introduction},
    author={Richard Sutton and Andrew Barto},
    year={2018},
    publisher={MIT Press}
}

@inproceedings{sutton:neurips99,
    title = {Policy Gradient Methods for Reinforcement Learning with Function Approximation},
    author = {Richard Sutton and David McAllester and Satinder Singh and Yishay Mansour},
    booktitle = {Neural Information Processing Systems (NeurIPS)},
    year = {1999},
}

@article{tran:ieee18,
    author={{Son N.} Tran and {Artur S.} d’Avila Garcez},
    journal={IEEE Transactions on Neural Networks and Learning Systems}, 
    title={Deep Logic Networks: Inserting and Extracting Knowledge From Deep Belief Networks}, 
    year={2018},
    volume={29},
    number={2},
    pages={246-258}
}

@article{vanengelen:ml20,
    title = {A Survey on Semi-Supervised Learning},
    author = {Jesper {E. van Engelen} and Holger {H. Hoos}},
    journal = {Machine Learning (ML)},
    year = {2020},
    volume = {109},
    pages = {373--440},
}

@inproceedings{vlastelica:iclr20,
    title = {Differentiation of Blackbox Combinatorial Solvers},
    author = {Marin Vlastelica and Anselm Paulus and Vít Musil and Georg Martius and Michal Rolínek},
    booktitle = {International Conference on Learning Representations (ICLR)},
    year = {2020},
}

@inproceedings{walker:lrec12,
    title = {A Corpus for Research on Deliberation and Debate},
    author = {Marilyn A Walker and Jean E Fox Tree and Pranav Anand and Rob Abbott and Joseph King},
    booktitle = {LREC},
    _pages = {812--817},
    year = {2012},
}

@article{wan:arxiv24,
    title = {Towards Cognitive AI Systems: A Survey and Prospective on Neuro-Symbolic AI},
    author = {Zishen Wan and Che-Kai Liu and Hanchen Yang and Chaojian Li and Haoran You and Yonggan Fu and Cheng Wan and Tushar Krishna and Yingyan Lin and Arijit Raychowdhury},
    journal = {arXiv},
    url={https://arxiv.org/abs/2401.01040},
    year = {2024}
}

@inproceedings{wang:icml19,
    title = {Satnet: Bridging Deep Learning and Logical Reasoning Using a Differentiable Satisfiability Solver},
    author = {{Po-Wei} Wang and Priya Donti and Bryan Wilder and Zico Kolter},
    booktitle = {International Conference on Machine Learning (ICML)},
    year = {2019},
}

@inproceedings{wei:neurips22,
    title = {Chain-of-thought prompting elicits reasoning in large language models},
    author = {Jason Wei and Xuezhi Wang and Dale Schuurmans and Maarten Bosma and Fei Xia and Ed Chi and Quoc V Le and Denny Zhou},
    booktitle = {Neural Information Processing Systems (NeurIPS)},
    year = {2022},
}

@inproceedings{welling:icml11,
    title = {Bayesian Learning via Stochastic Gradient Langevin Dynamics},
    author = {Max Welling and {Yee W} Teh},
    year = {2011},
    booktitle = {International Conference on Machine Learning (ICML)},
}

@article{williams:ml92,
    author = {Ronald Williams},
    title = {Simple Statistical Gradient-Following Algorithms for Connectionist Reinforcement Learning},
    journal = {Machine Learning},
    volume = {8},
    pages = {229--256},
    year = {1992},
}

@inproceedings{winters:aaai22,
    title = {DeepStochLog: Neural Stochastic Logic Programming},
    author = {Thomas Winters and Giuseppe Marra and Robin Manhaeve and Luc De Raedt},
    booktitle = {Association for the Advancement of Artificial Intelligence (AAAI)},
    year = {2022}
}

@article{wishart:nar06,
    title = {DrugBank: a comprehensive resource for in silico drug discovery and exploration},
    author = {David S Wishart and Craig Knox and An Chi Guo and Savita Shrivastava and Murtaza Hassanali and Paul Stothard and Zhan Chang and Jennifer Woolsey},
    journal = {Nucleic Acids Research (NAR)},
    volume = {34},
    pages = {D668--D672},
    year = {2006},
    publisher = {Oxford University Press}
}

@article{wright:mp15,
    title = {Coordinate descent algorithms},
    author = {{Stephen J.} Wright},
    journal = {Mathematical Programming},
    year = {2015},
    publisher = {Springer},
    volume = {151},
    pages = {3--34},
}

@inproceedings{wu:icml19,
    title = {Simplifying Graph Convolutional Networks},
    author = {Felix Wu and Tianyi Zhang and Amauiri {Holanda de Souza Jr.} and Christopher Fifty and Tao Yu and Kilian {Q. Weinberger}},
    year = {2019},
    booktitle = {International Conference on Machine Learning (ICML)}
}

@inproceedings{xu:icml18,
    title = {A Semantic Loss Function for Deep Learning with Symbolic Knowledge},
    author = {Jingyi Xu and Zilu Zhang and Tal Friedman and Yitao Liang and Guy {Van den Broeck}},
    booktitle = {International Conference on Machine Learning (ICML)},
    year = {2018},
}

@inproceedings{yang:ijcai20,
    title = {NeurASP: Embracing Neural Networks into Answer Set Programming},
    author = {Zhun Yang and Adam Ishay and Joohyung Lee},
    booktitle = {International Joint Conference on Artificial Intelligence (IJCAI)},
    year = {2020}
}

@article{ye:opt95,
    title = {Optimality Conditions for Bilevel Programming Problems},
    author = {J. {J. Ye} and D. {L. Zhu}},
    journal = {Optimization},
    year = {1995},
    volume = {33},
    number = {1},
    pages = {9-27},
}

@inproceedings{yi:neurips19,
    title = {Neural-Symbolic VQA: Disentanlging Reasoning from Vision and Language Understanding},
    author = {Kexin Yi and Jiajun Wu and Chuang Gan and Antonio Torralba and Pushmeet Kohli and Joshua {B. Tenenbaum}},
    booktitle = {Neural Information Processing Systems (NeurIPS)},
    year = {2019},
}

@inproceedings{zhao:iclr17,
    title = {Energy-based Generative Adversarial Networks},
    author = {Junbo Zhao and Michael Mathieu and Yann LeCun},
    booktitle = {International Conference on Learning Representations (ICLR)},
    year = {2017},
}

@inproceedings{zhang:icml23,
    title = {Tractable Control for Autoregressive Language Generation},
    author = {Honghua Zhang and Meihua Dang and Nanyun Peng and Guy {Van den Broeck}},
    booktitle = {International Conference on Machine Learning (ICML)},
    year = {2023},
    _publisher = {ICML},
    _address = {Honolulu, Hawaii},
}

@inproceedings{zhou:icml23,
    title = {ESC: Exploration with Soft Commonsense Constraints for Zero-shot Object Navigation},
    author = {Kaiwen Zhou and Kaizhi Zheng and Connor Pryor and Yilin Shen and Hongxia Jin and Lise Getoor and Xin Eric Wang},
    booktitle = {International Conference on Machine Learning (ICML)},
    year = {2023},
    _publisher = {ICML},
    _address = {Honolulu, Hawaii},
}
